\documentclass[twoside,openany,12pt]{book}
\usepackage{multicol}
\usepackage{graphicx}
\usepackage{rotating}
\usepackage{dblfloatfix}
\usepackage[ruled]{algorithm2e}

\usepackage[zerostyle=c,scaled=.94]{newtxtt}
\SetCommentSty{mycommfont}
\SetFuncSty{myprocname}

\usepackage{color}
\usepackage{amsthm}
\usepackage{amsmath}
\usepackage{bbm}
\usepackage{dsfont}
\usepackage{float}
\usepackage{bm}
\usepackage{amssymb}
\usepackage{upgreek}
\usepackage{multirow}
\usepackage{amsfonts}
\usepackage{here}
\usepackage[normalem]{ulem}
\usepackage{newclude}
\usepackage[percent]{overpic}

\usepackage{enumitem}
\setlist[itemize, 1]{label =\raisebox{-0.3\height}{\scalebox{1.75}{\textbullet}}} 

\usepackage{afterpage}
\usepackage{pdflscape}
\usepackage{hyperref}  
\usepackage{subcaption}
\usepackage{url}
\usepackage[font=small, skip=5pt]{caption}
\usepackage{listings}
\usepackage{setspace}
\usepackage{booktabs}
\usepackage{changepage}
\usepackage{tabularx}
\usepackage{graphicx}
\usepackage{xcolor}
\usepackage{longtable}

\newcommand\blfootnote[1]{%
  \begingroup
  \renewcommand\thefootnote{}\footnote{#1}%
  \addtocounter{footnote}{-1}%
  \endgroup
}

\DeclareMathOperator*{\argmax}{argmax}
\DeclareMathOperator*{\argmin}{argmin}

\newcommand{\E}{\mathbb{E}}

\graphicspath{{figs/}}

\newtheorem{definition}{Definition}

\newtheorem{lemma}{Lemma}
\newcommand{\norm}[1]{\left\lVert#1\right\rVert}
\newcommand{\eg}{\emph{e.g.}}
\newcommand{\ie}{\emph{i.e.}}
\newcommand{\aka}{\emph{a.k.a.}}

\usepackage{times} 
\usepackage{amssymb}  
\usepackage{gensymb}
\usepackage{tikz}
\DeclareMathAlphabet{\mathcal}{OMS}{cmsy}{m}{n}
\DeclareMathAlphabet\mathbfcal{OMS}{cmsy}{b}{n}

\newcommand{\mc}[1]{\mathcal{#1}}

\newcommand{\Sspace}{\mathcal{S}}
\newcommand{\Aspace}{\mathcal{A}}
\newcommand{\Ospace}{\mathcal{O}}
\newcommand{\Zspace}{\mathcal{Z}}
\newcommand{\MOVE}{\textsc{Move}}
\newcommand{\MOVEOP}{\textsc{MoveOp}}
\newcommand{\LOOK}{\textsc{Look}}
\newcommand{\DETECT}{\textsc{Find}}
\newcommand{\indicator}{\bm{1}}
\newcommand{\Val}{\textsc{Val}}
\newcommand{\PriorVal}{\textsc{PriorVal}}
\newcommand{\Norm}{\textsc{Norm}}
\newcommand{\Children}{\textsc{Ch}}
\newcommand{\abst}[1]{\hat{#1}}

\newcommand{\rmax}{R_{\text{max}}}
\newcommand{\rmin}{R_{\text{min}}}
\newcommand{\SO}{\text{SO}}

\newcommand{\mindeg}{\zeta_{\text{min}}}
\newcommand{\maxdeg}{\zeta_{\text{max}}}
\newcommand{\cosstate}{\xtarget,\srobot}
\newcommand{\fullstate}{\substack{x_1,\cdots,x_n,\\\xtarget,\srobot}}
\newcommand{\SFPOMDP}{\mc{S}_{\text{F}}}
\newcommand{\PrFPOMDP}{\Pr_{\text{F-POMDP}}}
\newcommand{\PrCOS}{\Pr_{\text{COS-POMDP}}}

\newcommand{\target}{\text{target}}

\newcommand{\zrobot}{z_{\text{robot}}}
\newcommand{\zobjects}{z_{\text{objects}}}
\newcommand{\srobot}{s_{\text{robot}}}
\newcommand{\srobotinit}{s_{\text{robot}}^{\text{init}}}
\newcommand{\starget}{s_{\text{target}}}

\newcommand{\xtarget}{x_{\text{target}}}
\newcommand{\btarget}{b_{\text{target}}}
\newcommand{\brobot}{b_{\text{robot}}}
\newcommand{\Srobot}{\mc{S}_{\text{robot}}}

\newcommand{\Dtarget}{D_{\text{target}}}
\newcommand{\ztarget}{z_{\text{target}}}

\newcommand{\Orobot}{O_{\text{robot}}}

\newcommand{\DECLARE}{\texttt{Done}}

\newcommand{\dsep}{d_{\text{sep}}}

\newcommand{\sprl}{\mathcal{R}}

\newcommand{\abs}[1]{\lvert#1\rvert}
\makeatletter
\newcommand*\bigcdot{\mathpalette\bigcdot@{.5}}
\newcommand*\bigcdot@[2]{\mathbin{\vcenter{\hbox{\scalebox{#2}{$\m@th#1\bullet$}}}}}
\makeatother

\definecolor{RoyalRed}{RGB}{157,16, 45}
\definecolor{deepblue}{RGB}{0, 0, 139}
\newcommand{\smcatg}[1]{\textcolor{deepblue}{\textbf{#1}}}
\newcommand{\sq}[1]{\textsuperscript{#1}}

\usepackage{jair_phd, rawfonts}

\usepackage[apaciteclassic]{apacite}
\usepackage{natbib}

\usepackage[]{hyperref}  
\hypersetup{
    citecolor={blue!80!black}, 
    colorlinks=true,
    linkcolor=blue,
    urlcolor=blue
  }

\renewcommand*{\cite}[1]{\citep{#1}}
\usepackage{inconsolata}

\usepackage[mmddyyyy]{datetime}

\newcommand{\thesistitle}{Generalized Object Search}

\jairheading{\today}{}{}{}{}
\ShortHeadings{\thesistitle}{Kaiyu Zheng}
\firstpageno{1}

\usepackage{titlesec}
\titleformat{\chapter}[display]
   {\centering \normalsize \huge  \color{black}}
   {\vspace{1in}\LARGE \textsc{\chaptertitlename}\hspace{0.5ex} \thechapter}{15pt}{\bfseries\huge}
\makeatletter
\newcommand\semiLarge{\@setfontsize\semiLarge{14pt}{14pt}}
\makeatother
\titleformat{\section}{\normalfont\Large\bfseries}{\thesection}{15pt}{\Large\bfseries}
\titleformat{\subsection}{\normalfont\large\bfseries}{\thesubsection}{10pt}{\large\bfseries}
\titleformat{\subsubsection}{\normalfont\normalsize\bfseries}{\thesubsubsection}{10pt}{\normalsize\bfseries}
\titleformat{\paragraph}{\normalfont\normalsize\bfseries}{\thesubsubsection}{10pt}{\normalsize\bfseries}
\titleformat{\subparagraph}{\normalfont\normalsize\bfseries}{\thesubsubsection}{10pt}{\normalsize\bfseries}

\usepackage{fancyhdr}
\usepackage{tocloft}

\usepackage[nottoc]{tocbibind} 

\usepackage{lettrine}  

\defcitealias{dls2021}{Roy*, Zheng*, et al., 2021}

\usepackage{color}
\definecolor{deepblue}{rgb}{0,0,0.5}
\definecolor{deepred}{rgb}{0.6,0,0}
\definecolor{deepgreen}{rgb}{0,0.5,0}
\newcommand\pythonstyle{\lstset{
    language=Python,
    basicstyle=\ttm,
    otherkeywords={self},             
    keywordstyle=\ttb\color{deepblue},
    emph={MyClass,__init__},          
    emphstyle=\ttb\color{deepred},    
    stringstyle=\color{deepgreen},
    commentstyle=\ttm\color{olive},
    showstringspaces=false            %
}}
\lstnewenvironment{python}[1][]
{
\pythonstyle
\lstset{#1}
}
{}

\newcommand\pythoninline[1]{{\pythonstyle\lstinline!#1!}}

\definecolor{purple}{rgb}{0.858, 0.08, 0.85}
\definecolor{darkgreen}{rgb}{0.08, 0.55, 0.08}
\definecolor{blue1}{rgb}{0.2, 0.2, 0.6}
\definecolor{green1}{rgb}{0.2, 0.6, 0.2}
\definecolor{green2}{rgb}{0.1, 0.4, 0.1}

\begin{document}

\thispagestyle{empty}
\newgeometry{hmargin=1.25in, vmargin=1.5in}
\noindent Abstract of ''Generalized Object Search''\\ by Kaiyu Zheng, Ph.D., Brown
University, February, 2023\\

\noindent Future collaborative robots must be capable of finding objects. As such a
fundamental skill, we expect object search to eventually become an off-the-shelf
capability for any robot, similar to \eg, object detection, SLAM, and motion
planning.  However, existing approaches either make unrealistic compromises
(\eg,~reduce the problem from 3D to 2D), resort to ad-hoc, greedy search
strategies, or attempt to learn end-to-end policies in simulation that are yet
to generalize across real robots and environments. This thesis argues that
through using Partially Observable Markov Decision Processes (POMDPs) to model
object search while exploiting structures in the human world (\eg,~octrees,
correlations) and in human-robot interaction (\eg, spatial language), a
practical and effective system for generalized object search can be achieved. In
support of this argument, I develop methods and systems for (multi-)object
search in 3D environments under uncertainty due to limited field of view,
occlusion, noisy, unreliable detectors, spatial correlations between objects,
and possibly ambiguous spatial language (\eg,~"The red car is \emph{behind}
Chase Bank"). Besides evaluation in simulators such as PyGame, AirSim, and
AI2-THOR, I design and implement a robot-independent, environment-agnostic
system for generalized object search in 3D and deploy it on the Boston Dynamics
Spot robot, the Kinova MOVO robot, and the Universal Robots UR5e robotic arm, to
perform object search in different environments.  The system enables, for
example, a Spot robot to find a toy cat hidden underneath a couch in a kitchen
area in under one minute. This thesis also broadly surveys the object search literature,
proposing taxonomies in object search problem settings, methods and systems.

\restoregeometry
\newpage

\clearpage
\pagenumbering{roman}
\newgeometry{hmargin=1.5in, vmargin=1.5in}
{\linespread{1.0}
  \begin{center}
    \thispagestyle{empty}
    $\qquad$ \\[3cm]
    {\LARGE \thesistitle }\\[4.85cm] 
    {
      {\large Kaiyu Zheng}
    }\\
    [3cm]
    {
      {\large A dissertation submitted in partial fulfillment of the\\[0.1cm]
      requirements of the Degree of Doctor of Philosophy\\[0.2cm]
      in the Department of Computer Science at Brown University}
    }
    \\[3.3cm]
    {\large {Providence, Rhode Island}}\\
    [.1cm]
    {\large February 2023}\\
    [1cm]

    \newpage
    \thispagestyle{empty}
    $\qquad$ \\[5cm]
    {\large \textcopyright$\ $Copyright 2023 by Kaiyu Zheng}

    \newpage
    {This dissertation by Kaiyu Zheng is accepted in its present form\\[0.2cm]
      by the Department of Computer Science as satisfying the\\[0.2cm]
    dissertation requirement of the degree of Doctor of Philosophy.}\\
    [2.0cm]

    {\normalsize
        \begin{tabularx}{
            \textwidth}{p{3.5cm} p{1.0cm} X}
          Date    \hrulefill &      &\hrulefill \\
                             &      &Stefanie Tellex, Advisor
        \end{tabularx}
      }\\
    [1.75cm]
    {Recommended to the Graduate Council}\\
    [1.0cm]
    {\normalsize
        \begin{tabularx}{
            \textwidth}{p{3.5cm} p{1.0cm} X}
          Date    \hrulefill &      &\hrulefill \\
                             &      &George D. Konidaris, Reader\\[1.0cm]
          Date    \hrulefill &      &\hrulefill \\
                             &      &Michael L. Littman, Reader\\[1.0cm]
          Date    \hrulefill &      &\hrulefill \\
                             &      &Ellie Pavlick, Reader\\[1.0cm]
          Date    \hrulefill &      &\hrulefill \\
                             &      &Leslie P. Kaelbling, Reader
        \end{tabularx}
      }\\
    [1.75cm]
    {Approved by the Graduate Council}\\
    [1.0cm]
    {\normalsize
        \begin{tabularx}{
            \textwidth}{p{3.5cm} p{1.0cm} X}
          Date   \hrulefill &      &\hrulefill \\
                            &      &Thomas A. Lewis, Dean of the Graduate School
        \end{tabularx}
      }\\


\end{center}
}

\restoregeometry

\newpage
\clearpage
\newgeometry{hmargin=1.25in, top=1.0in, bottom=1.25in}
\phantomsection
$\qquad$\\[1cm]
\begin{center}
  {\bfseries\large Abstract of ``Generalized Object Search''\\[0.1cm]
    by Kaiyu Zheng, Ph.D., Brown University, February 2023.\\[0.3cm]}
\end{center}
\addcontentsline{toc}{chapter}{Abstract}

\noindent Future collaborative robots must be capable of finding objects. As such a
fundamental skill, we expect object search to eventually become an off-the-shelf
capability for any robot, similar to \eg, object detection, SLAM, and motion
planning.  However, existing approaches either make unrealistic compromises
(\eg,~reduce the problem from 3D to 2D), resort to ad-hoc, greedy search
strategies, or attempt to learn end-to-end policies in simulation that are yet
to generalize across real robots and environments. This thesis argues that
through using Partially Observable Markov Decision Processes (POMDPs) to model
object search while exploiting structures in the human world (\eg,~octrees,
correlations) and in human-robot interaction (\eg, spatial language), a
practical and effective system for generalized object search can be achieved. In
support of this argument, I develop methods and systems for (multi-)object
search in 3D environments under uncertainty due to limited field of view,
occlusion, noisy, unreliable detectors, spatial correlations between objects,
and possibly ambiguous spatial language (\eg,~"The red car is \emph{behind}
Chase Bank"). Besides evaluation in simulators such as PyGame, AirSim, and
AI2-THOR, I design and implement a robot-independent, environment-agnostic
system for generalized object search in 3D and deploy it on the Boston Dynamics
Spot robot, the Kinova MOVO robot, and the Universal Robots UR5e robotic arm, to
perform object search in different environments.  The system enables, for
example, a Spot robot to find a toy cat hidden underneath a couch in a kitchen
area in under one minute. This thesis also broadly surveys the object search literature,
proposing taxonomies in object search problem settings, methods and systems.

\restoregeometry

\newpage
\clearpage
\newgeometry{hmargin=1.25in, top=1.0in, bottom=1.25in}
\phantomsection
$\qquad$\\[1cm]
\begin{center}
  {\large \bfseries Curriculum Vitae}
\end{center}
\addcontentsline{toc}{chapter}{Curriculum Vitae}

Kaiyu Zheng was raised in Guangzhou, China. He graduated from the Guangzhou No.~2 High School. Afterwards, he received a B.S. and M.S. with honors in computer science and a minor in mathematics from the University of Washington, Seattle. He then pursued a Ph.D. in computer science at Brown University and defended his thesis in December, 2022. His research interests lie in making robots reliably act in the human world and interact with humans under uncertainty. He received the IROS RoboCup Best Paper Award in 2021 and was selected as a Robotics: Science and Systems (RSS) Pioneer in 2022.\\

Service activities during Ph.D.:
\begin{itemize}[noitemsep,topsep=0pt,label={}]
\item Organizer, Brown Robotics Group Seminar Series, 2021
\item Head Teaching Assistant, Learning $\&$ Sequential Decision Making, 2021
\item Presenter, Staff Development Day, Brown University, 2021
\item Peer mentor, PhD Mentorship Program, Brown CS, 2020 - 2022
\item Moderator, PhD Alumni Panel, Brown CS, 2020
\item Peer mentor, International Graduate Student Orientation, 2019, 2021
\item Organizer, Hiking Group, 2019
\item Representative, Graduate Student Council, Brown CS, 2019
\end{itemize}

\newpage
\clearpage
\newgeometry{hmargin=1.2in, top=1.0in, bottom=1.25in}
\setlength{\parskip}{0.05in}
\phantomsection
$\qquad$\\[1cm]
\begin{center}
  {{\large \bfseries Acknowlegements}}
\end{center}
\addcontentsline{toc}{chapter}{Acknowlegements}
\hyphenation{Mohit}
\hyphenation{Rajesh}

My journey in research could not have been more fortunate because of all the people who I have met and who have given me a hand along the way. First and foremost, I want to thank my amazing advisor, Stefanie Tellex, for her advice, mentorship, understanding, and unwavering support. I would always appreciate her direct and sharp advice, “do good research,” “don’t spread yourself too thin,” which have been invaluable to me.

I would like to thank my committee, George Konidaris, Michael Littman, Ellie Pavlick and Leslie Kaelbling for your feedback and advice, and coming to my defense in person. In particular, I want to thank George for his insightful perspectives, Michael for his guidance on theoretical research, Ellie for how to lead discussions, and Leslie for her critique.

I want to thank a few key people who shaped my research journey. First, I thank Andrzej Pronobis, my undergraduate advisor, for taking me on, teaching me details, and guiding me in research. Then, I want to thank Yoonchang Sung, my mentor, from whom I learned the critical attitude towards research. Finally, I want to thank Bo Wang, one of my best friends, for introducing me to research in our freshman year. Otherwise, I wouldn’t even know.

I want to thank all my collaborators, especially Rohan Chitnis, Yoonchang Sung, Deniz Bayazit, and Rebecca Matthew for the project times we shared.  I enjoyed working as a mentor with many students, including Semanti Basu, Sreshtaa Rajesh, Monica Roy, Anirudha Paul, Shangqun Yu, Eliza Sun, and Vedant Gupta. Moreover, my peers in and out of the H2R and IRL labs and the department: Eric Rosen, Sam Lobel, Thao Nguyen, Ifrah Idrees, Max Merlin, Ben Abbatematteo, Matthew Corsaro, Seungchan Kim, Johnathan Chang, Nick DeMarinis, Akhil Bagaria, Albert Webson, Josh Roy, Nishanth Kumar, Tuluhan Akbulut, Nihal Nayak, Saket Tiwari, Shane Parr, Cam Allen, Skye Thompson, Haotian Fu, Anita de Mello Koch, Lilika Markatou, Denizalp Goktas, Lucy Qin, Aaron Traylor, Franco Solleza, Ghulam Murtaza, Alessio Mazzetto, Zheng-Xin Yong, Kenny Jones, Mikhail Okunev, Rao Fu, Ji Won Chung, Aditya Ganeshan, Benjamin Spiegel, Benedict Quartey, Calvin Luo, Marilyn George, Daniel Engel, Roma Patel, Charles Lovering, Jack Merullo, Tian Yun, Ziyi Yang, Jessica Forde, Talie Massachi, Brandon Woodard, Peilin Yu, David Tao, Lishuo Pan, Rahul Sajnani, Ruochen Zhang, Catherine Chen, Vivienne Chi, Wasiwasi Mgonzo,~Amina Abdullahi, George Zerveas, Yingjie Xue, Yanqi Liu, Selena Ling, Ewina Pun, Kai Wang, Qian Zhang, Yanyan Ren, Fumeng Yang, Anna Wei  -- the list is inenumerable -- have been my community and friends that enriched this journey with moments and discussions. Thank you. In addition, I was fortunate to interact with and learn from my respectful peers across different research fields: Andreea Bobu, Weiyu Liu, Shaoxiong Wang, Yuchen Xiao, Mohit Sridhar, Anirudha Vemula, Serena Booth, Yilun Zhou, Rachel Holladay, Xiaolin Fang, Tom Silver, Jiayuan Mao, Aidan Curtis, Caris Moses, Dhruv Shah, Ted Xiao, Shivam Vats, Angela Chen, Linfeng Zhao, Lena Downes, Kevin Doherty, and the list goes on and on.

I also want to thank many senior folks who have given me memorable advice and feedback, whom I also respect:  Tomás Lozano-Pérez, Julie Shah, Dieter Fox, Vikash Kumar, R.~Iris Bahar, Vikram Saraph,  David Whitney, David Abel, Lawson Wong, David Paulius, Ankit Shah, Lucas Lehnert, Kavosh Asadi, Yu Xiang, Arunkumar Byravan, James Tompkin, Amy Greenwald, Paul Beame, Rajesh~P.~N.~Rao, Ugur Cetintemel, just to name a few. I especially want to thank Matthew Gombolay for his insightful comments following my talk at Georgia Tech, which influenced my later presentations. The same goes to Nakul Gopalan, and others in the audience.

I want to thank the wonderful people I worked with at Viam, who helped me get the UR5e arm demo done at the closure of my PhD: Fahmina Ahmed, Matt Dannenberg, Bijan Haney, Nicholas Franczak, Gautham Varadarajan, Khari Jarrett, and Raymond Bjorkman.

I want to thank the amazing staff at Brown CS, including Suzanne Alden, Lauren Clarke, Kathy Kirman, John Tracey-Ursprung, Eugenia DeGouveia, and others, as well as Ronni Edmonds who manages Brown Auxiliary Housing. You made my life easier.

I want to thank my \emph{friends} at Brown, in and out of the lab; from UW and around Seattle; from high school, middle school; from internships, elsewhere. You make up a big part of my life and will continue to do so. My writing pattern in these acknowledgements might have been dull, but the bond we share will remain lively and true. I wish you all the best.

I also want to thank all my art teachers. I enjoyed art classes the most growing up.

Finally, I want to sincerely thank all my parents. Especially, my father. All beyond paper.

Yes, I want to thank you, too.\hfill --- Kaiyu Zheng

\setlength{\parskip}{0.075in}

\newpage
\clearpage
\newgeometry{hmargin=1.25in, top=1.0in, bottom=1.25in}
\setcounter{tocdepth}{2}
{
  \hypersetup{linkcolor=black}
  \tableofcontents
}

\newpage
\thispagestyle{empty}
\listoftables
\newpage
\listoffigures

\newpage
\clearpage
\thispagestyle{empty}
$\qquad$
\newpage


\addtocontents{toc}{\vspace{1\baselineskip}}  

\clearpage
\pagenumbering{arabic}
\pagestyle{fancy}

\chapter{Introduction}
\label{ch:intro}
\chaptermark{Introduction}
\section{Significance and Challenges of Object Search}
\label{sec:sigch}

\lettrine{S}{earching} for objects has been a fundamental skill for people since ages ago. From fruits to glasses, water sources to ore mines, the ability to search has enabled people's daily lives and provided much value to society.
However, searching is unpleasant, to say the least, and can be difficult or prohibitive depending on the situation. Imagine an elderly person searching for their missing pair of glasses. Imagine a factory worker going back and searching for their tools.  Imagine a wildfire that devastates a neighborhood, and a search and rescue team has to risk their lives in search of survivors. Now, imagine a robot that can help.  Autonomy in object search is therefore intrinsically valuable.

Besides its intrinsic value, object search is also a prerequisite for basically any task involving the objects being searched for (\aka~the ``target objects''). For intelligent, collaborative robots of the future, this means that a robot that can cook should better be able to search for the ingredients and the pan before cooking, and a robot that can clean should be able to find the cleaning tools. This point is well reflected in Figure~\ref{fig:jacob_andreas_slide} that shows six examples of a large language model (LLM) taking in a high-level task description (\eg \emph{``pick two apples then heat them''}) and outputting a sequence of sentences each corresponding to a lower-level action \citep{sharma2022skill}.  In all examples, the output begins with \emph{``Find ...''} (\eg \emph{``Find the apple''}, \emph{``Find the lettuce''}, etc.)  as the first action. The SayCan system developed by Google \citep{ahn2022can} with a similar capability exhibits the same behavior (Figure~\ref{fig:saycan_intro}). Object search is literally the prerequisite skill.

\begin{figure}
  \centering
  \begin{subfigure}[t]{0.565\textwidth}
    \includegraphics[width=\textwidth]{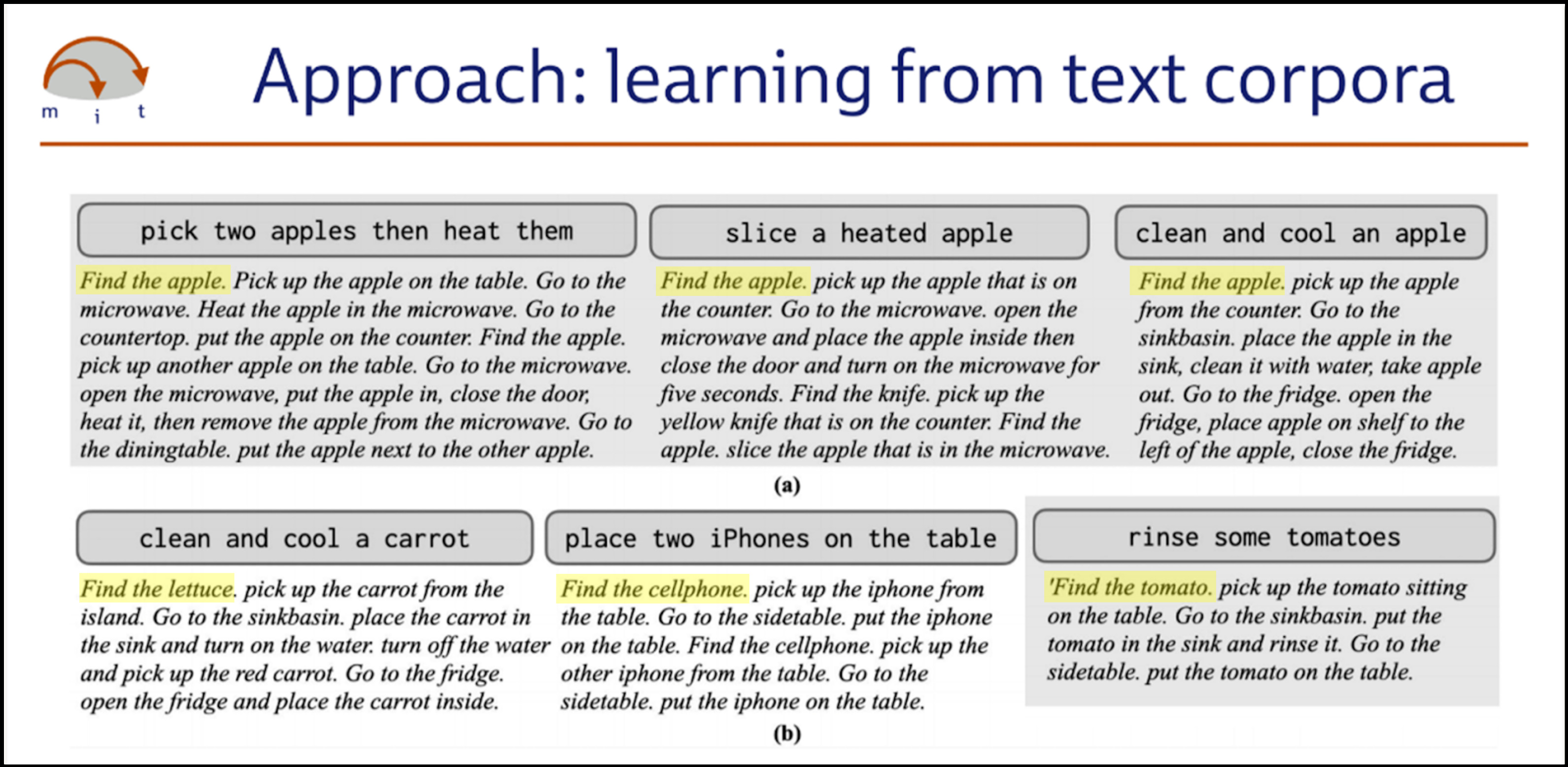}
    \caption{Examples from \citep{sharma2022skill} showing that an LLM breaks down given
      high-level descriptions into sequences of low-level actions. In all
      examples, the output begins with \emph{``Find ...''}. Source:
      slide from Jacob Andreas's presentation at RLDM 2022. Courtesy of Jacob Andreas.}
    \label{fig:jacob_andreas_slide}
  \end{subfigure}\hfill
  \begin{subfigure}[t]{0.415\textwidth}
    \raisebox{0.25in}{
      \includegraphics[width=\textwidth]{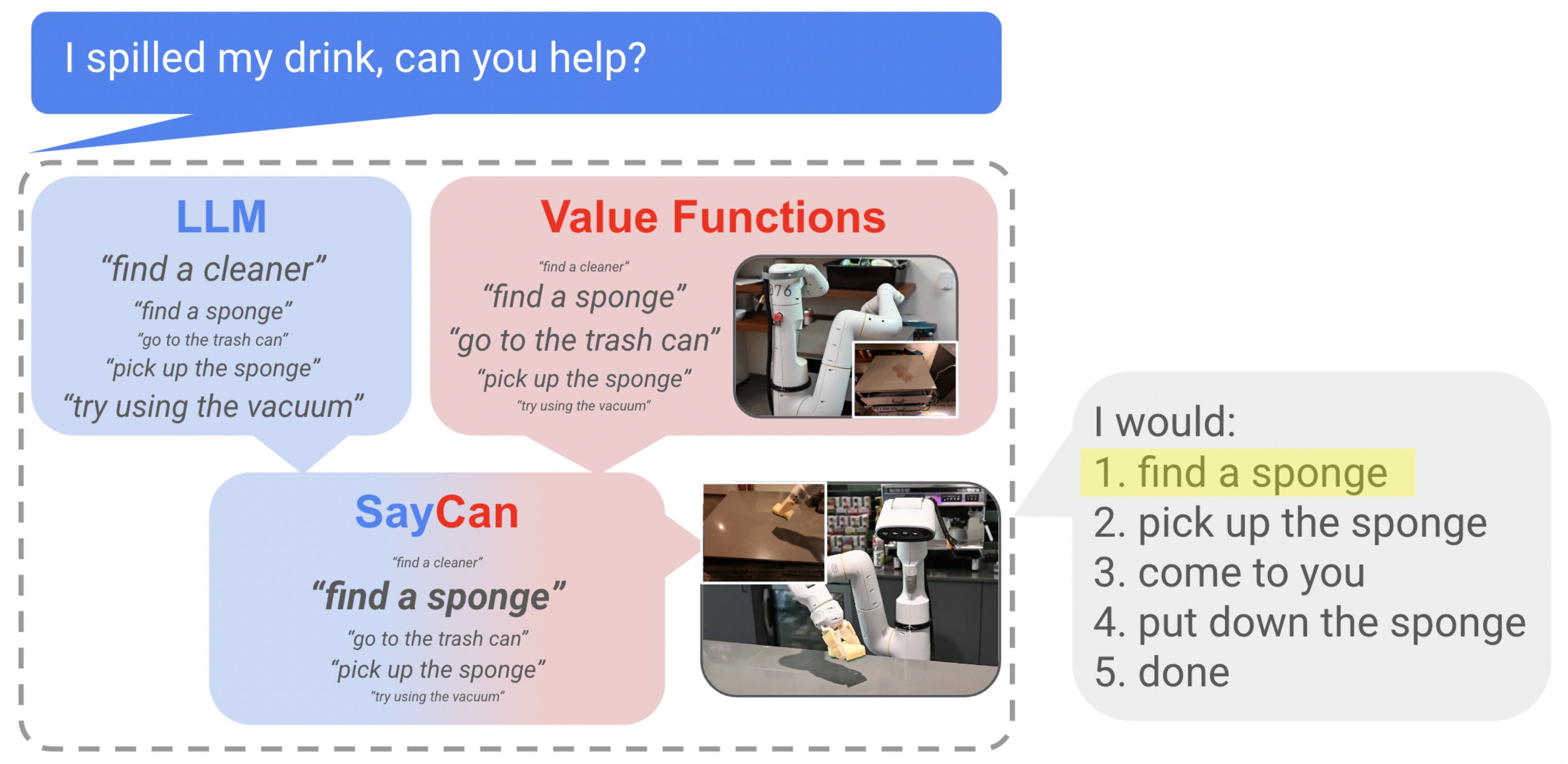}
    }
    \caption{Example trial of the SayCan system that outputs a sequence of
      low-level skills feasible for a robot given a high-level task (cleaning
      spilled drink). Again, the output begins with \emph{``Find ...''}.
      Source: \citep{ahn2022can}.}
    \label{fig:saycan_intro}
  \end{subfigure}
  \caption{The significance of object search is evident in part because it is a
    prerequisite for basically any task involving the target objects. Two
    supporting examples are provided above. In all cases, object search
    (\ie. \emph{``Find ...''}) comes first.}
  \label{fig:main}
\end{figure}

While significant, object search is also challenging.  The challenges of object search are best illustrated through a practical scenario.  Before diving into that, let us for a moment examine the phenomenon where \emph{``Find ...''} is an atomic, low-level action.  What happens in Figure~\ref{fig:jacob_andreas_slide} is in part due to the use of the ALFRED dataset \citep{shridhar2020alfred} to train the LLM, where Amazon Mechanical Turkers provide step-by-step instructions (low-level actions) for a video demonstration of a high-level task. However, rather than breaking down the steps for object search into concrete commands such as \emph{``Go to ...''} or \emph{``Turn left''}, the Turkers summarize those steps as \emph{``Find ...''}, which is vague for execution. Interestingly, SayCan (Figure~\ref{fig:saycan_intro}) exhibits the same phenomenon, using PaLM \citep{chowdhery2022palm}, a 540-billion parameter LLM trained with a huge dataset of natural language in diverse contexts (\eg~books, webpages, etc.).  Albeit speculative, one explanation is that carrying out \emph{``Find ...''}  would likely entail other sub-tasks such as navigation (\eg~maneuvering the robot's viewpoint), manipulation (\eg~opening a container), or another search task, recursively (\eg~to find an apple, one sub-task could be to first find the fridge), which is hard to elaborate when providing a single-step instruction.  This duality of object search, being both an atomic skill in people's mind (and in LLM) as well as an abstract-level skill in reality, makes it a unique problem to study in robotics.
\label{ch:intro:unique_problem}

Returning to the discussion of challenges, Figure~\ref{fig:spot_challenges} shows a typical real-world scenario of object search. A Boston Dynamics Spot robot in our lab is tasked to find the book which is located in front of a monitor as quickly as possible.  In this version of Spot, the most agile viewport, and the only colored one, is the camera on the mounted arm's gripper. The robot's objective is to perform search by executing a sequence of view pose changes to look around in the lab and eventually automatically declare the book to be found at some location. The robot's performance is evaluated based on success and efficiency.  Several challenges emerge as we consider building a real-world system effective for this task:
\label{sec:sigch:challenges:begin}

\begin{figure}
  \centering
  \includegraphics[width=\textwidth]{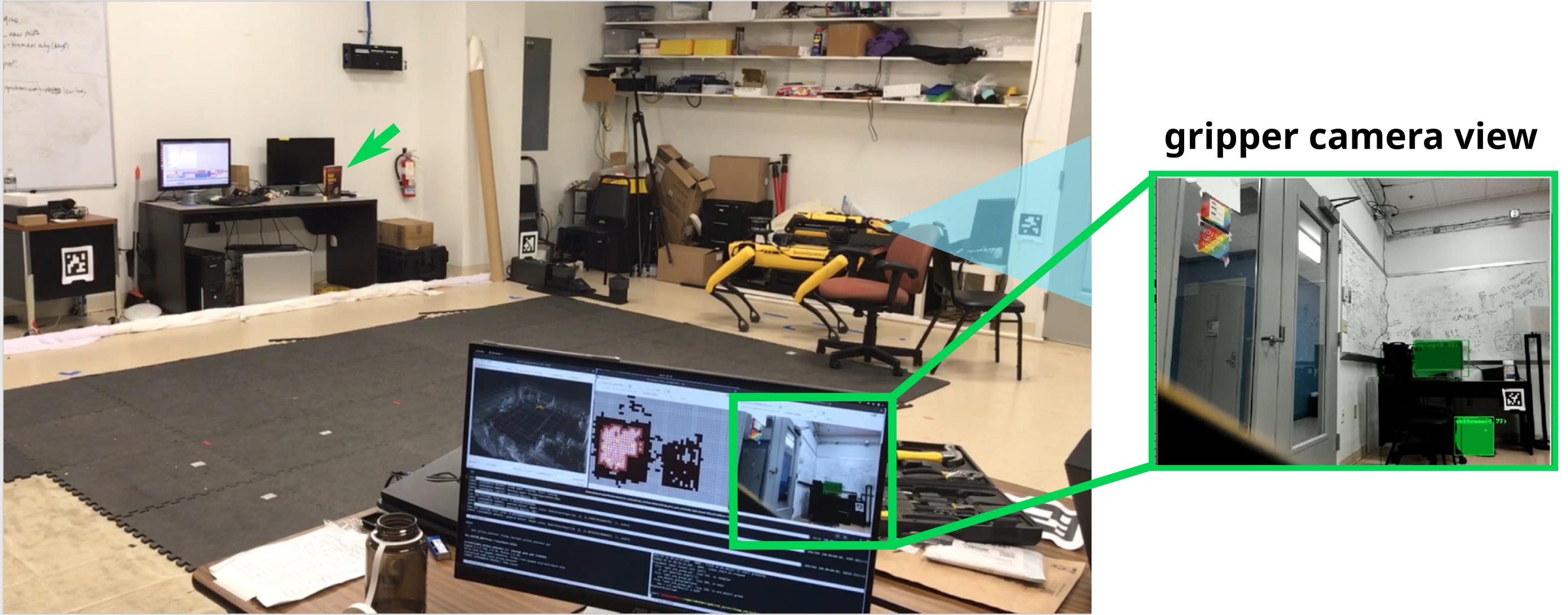}
  \caption{An example real-world object search scenario. The Spot robot (yellow legged robot) is tasked to find the book in front of a monitor (green arrow). The robot searches by navigating and controlling its 6-DOF gripper camera to look around within the room; the camera has a limited field of view subject to occlusion and the robot's perception pipeline is prone to uncertainty and errors.}
  \label{fig:spot_challenges}
\end{figure}

\begin{itemize}[itemsep=0.5pt,topsep=0pt]
\item \textbf{Partial observability.} Although, as a third-person audience, we immediately see where the book is, the robot has no knowledge of that and it has a limited field of view through its gripper camera subject to occlusion. The robot must reason about where to look under such partial observability.

\item \textbf{Perceptual uncertainty.} Unlike the human eye, which can accurately distinguish most recognizable objects within the field of view when close enough, the robot's perception system is subject to noise, uncertainty, and errors in object detection and segmentation, even using state-of-the-art computer vision models.  False negatives and false positives are inevitable for most real-world robotic systems since they often receive images not in line with the training distribution of those models. Such perceptual uncertainty is not limited to camera-based object detectors, and it is necessary to be considered by the robot for effective search using its on-board sensors.

\item \textbf{Complex, unstructured environment.}  The environment that the robot operates in is complex and unstructured, meaning that it does not just contain the robot itself and the target object. It also contains many other objects arranged in an unstructured fashion for the robot.

\item \textbf{Human input (\eg~language hints).} Since the robot performs the search in a human environment or alongside a human teammate, it would be beneficial for searching more efficiently if the robot can take advantage of human input, for example, through natural language hints (\eg,~\emph{''the book is in front of the monitor''}).

\item \textbf{Generalization (across environments, across robots).} We do not just want object search to work for this robot in this room for this book. We want the ability to search to generalize across different environments and robots. Ultimately, the ability to search should become an off-the-shelf capability that a robot can \emph{just have}.
\end{itemize}

\noindent Addressing the above challenges with an emphasis on building principled, practical robotic systems is at the core of this thesis (Chapters~\ref{ch:overarching} to \ref{ch:sloop}). Our effort aims to enable most robots today with a movable viewport to perform object search in an off-the-shelf manner; we call this problem \emph{generalized object search}. This is a feasible objective today, thanks to recent progress in mapping, navigation, motion planning, and object detection. It is also widely applicable and a first step towards realizing the value of autonomy in object search.

Note that, however, object search can become even more difficult when additional challenges are factored in, such as when a human teammate can engage in a dialogue with the robot, when target objects are dynamic or adversarial, when physical environment interactions such as container opening, decluttering, or tool use are necessary and within the robot's arsenal of abilities, or when the robot performs search while exploring an unknown region. In Chapter~\ref{ch:background}, I provide a broad review on object search, including methods designated to address these challenges. There is much left to do here for future work.


Most prior work in object search focuses on the kind of basic but essential setting in Figure~\ref{fig:spot_challenges}. As discussed next, despite the large body of literature, previous work fails to address the above challenges in a realistic, practical and generalizable way. A more thorough review and taxonomy of previous work in object search is provided in Chapter~\ref{ch:background}.
\label{sec:sigch:challenges:end}

\section{Limitations of Previous Work}
Object search can be broken down into two subproblems: (1) which
region(s) to search in, and (2) how to search within a region.

For the first subproblem, there has been a long line of work
\cite{wixson1994using,
  kollar2009utilizing,aydemir2013active,liu2021learning}, where
the focus has been on scalability and common sense: Can the robot reduce the uncertainty down to small search region(s), such as a room, within a large search environment, such as an entire floor?  Can the robot make use of semantic knowledge (\eg~sponges are usually in kitchens) between objects and places to select the search region(s)? The pioneering work along this line by \citet{wixson1994using} made the following interesting remark:
\begin{quote}
  Once such a subregion has been selected, however, a harder problem arises, namely the problem of how to select views that search the subregion. This problem is not really addressed here.
\end{quote}
This remark nicely leads to the second subproblem, which is considered harder. Indeed, object search in a 3D region by planning views under a time budget is shown to be NP-Complete \citep{Ye1997SensorPF}. Consequently, previous work uses ad-hoc, greedy search strategies despite the problem being inherently sequential \citep{aydemir2011search, zeng2020semantic}. Other work makes unrealistic compromises, for example, by simplifying the problem from 3D to 2D \citep{kollar2009utilizing,wandzel2019oopomdp,bejjani2021occlusion}, or by manually specifying the set of viewpoints \citep{xiao_icra_2019}.

Additionally, work such as \citep{saidi2007online} and \citep{tsuru2021online} focuses on building real-world systems for object search but for specific humanoid robots. On the other hand, data-driven approaches for semantic visual navigation that map raw RGB images directly to navigation actions are popular today yet most works have only conducted evaluation in simulation \cite{yang2018visual,chaplot2020object,mayo2021visual,deitke2022procthor,schmalstieg2022learning}. Generalization of such models across different environments and robots in the real-world poses a serious challenge.\footnote{The issue of semantic visual navigation approaches being predominantly evaluated in simulation is recognized and progress is being made to evaluate those methods in the real world through sim-to-real transfer \citep{deitke2020robothor, gervet2022navigating}. However, current efforts are limited to a discrete, ego-centric action space (\eg, \emph{move forward}, \emph{turn left}) and the real-robot evaluations involve only one robot platform compatible with the agent in simulation. Changing the action space or the type of robot platform breaks applicability and generalizability of the trained models, which already suffer in sim-to-real transform on a compatible robot platform. Our approach places no such constraints.}

The focus of this thesis is on the problem of how to search within a region. Ultimately, that determines the success of object search and is arguably harder, involving the low-level challenges fundamental to an effective object search system. Compared to the data-driven approaches which aim for \emph{model generalization} (\ie, training a model that generalizes to all test scenarios), our work differs fundamentally in that we pursue \emph{methodological generalization} (\ie, proposing a method that is robot- and environment-independent).\footnote{I first learned about this dichotomy of generalization articulated by Leslie P.~Kaelbling.} We demonstrate generalization of our method by developing a system and package for generalized object search and integrate it with different robots in different environments.\footnote{In the long run (when tackling object search problems involving tool use, for example), the two ends will likely meet somewhere in the middle, I believe.}

\section{Contributions}
\label{sec:intro:contribs}
The central argument of this thesis is:
\begin{quote}
  Through using Partially Observable Markov Decision Processes (POMDPs) to model object search while exploiting structures in the human world and in human-robot interaction, a practical and effective system for generalized object search can be achieved.
\end{quote}
In support of this argument, we briefly summarize the primary contributions of this thesis as follows; Refer to Figure~\ref{fig:overview} for an overview of research projects covered in this thesis.

  \begin{figure}
  \centering
  \begin{overpic}[width=\textwidth]{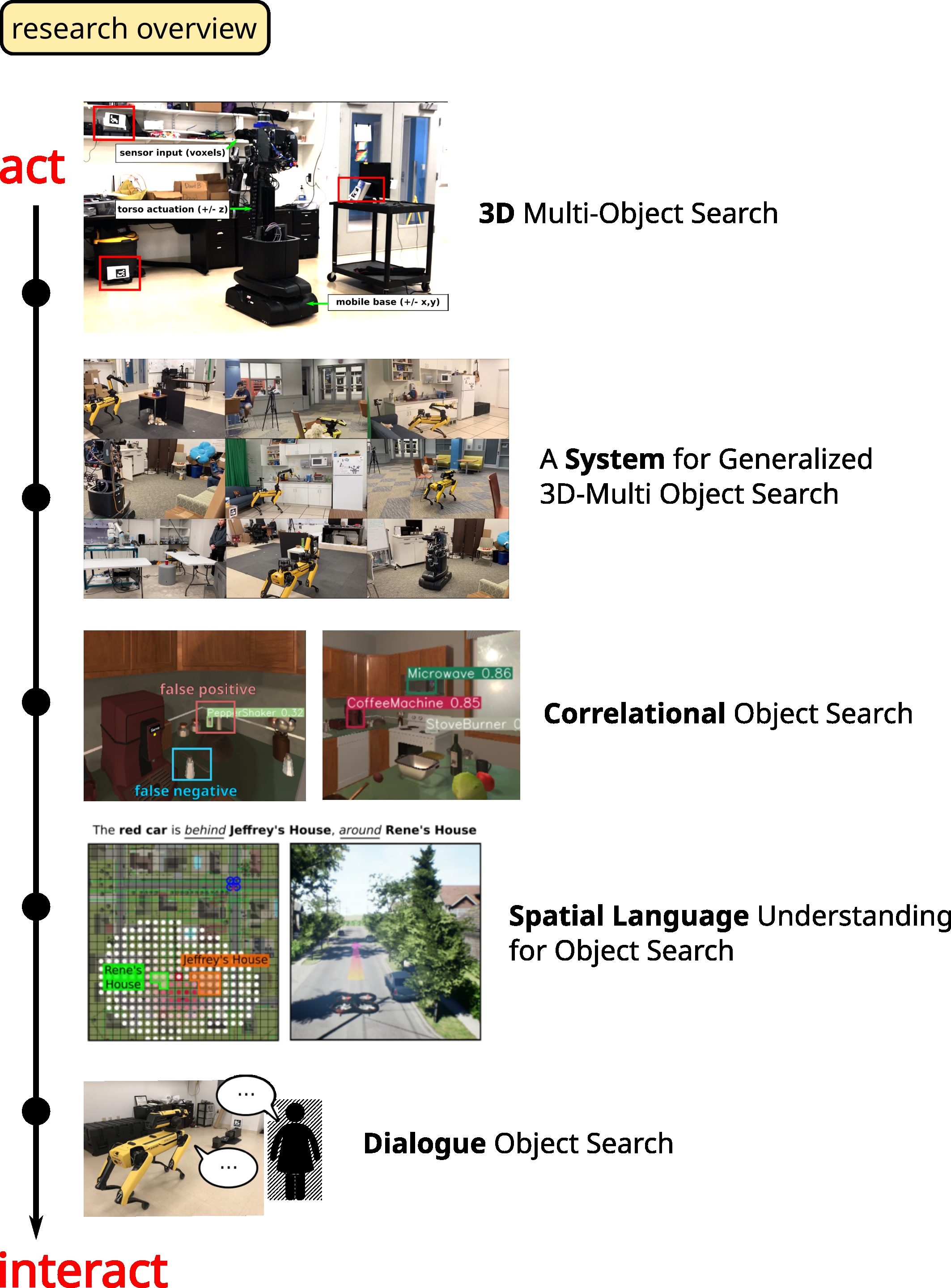}
        \put (37.25,80) {\small Chapter \ref{ch:3dmos} \cite{zheng2021multi}}
        \put (42,58) {\small Chapter \ref{ch:genmos} \cite{zheng2023system}}
        \put (42.5,41) {\small Chapter \ref{ch:cospomdp} \cite{zheng2022towards}}
        \put (39.75,22.5) {\small Chapter \ref{ch:sloop} \cite{sloop-roman-2021}}
        \put (28.25,7.5) {\small Chapter \ref{ch:dialog} \citepalias{dls2021}}
      \end{overpic}
  \caption{Overview of research projects covered in this thesis. The projects are arranged on an axis from robot \emph{acting} in human environments to robot \emph{interacting} with humans. See Section~\ref{sec:intro:contribs} for a high-level overview of each project.}
  \label{fig:overview}
\end{figure}
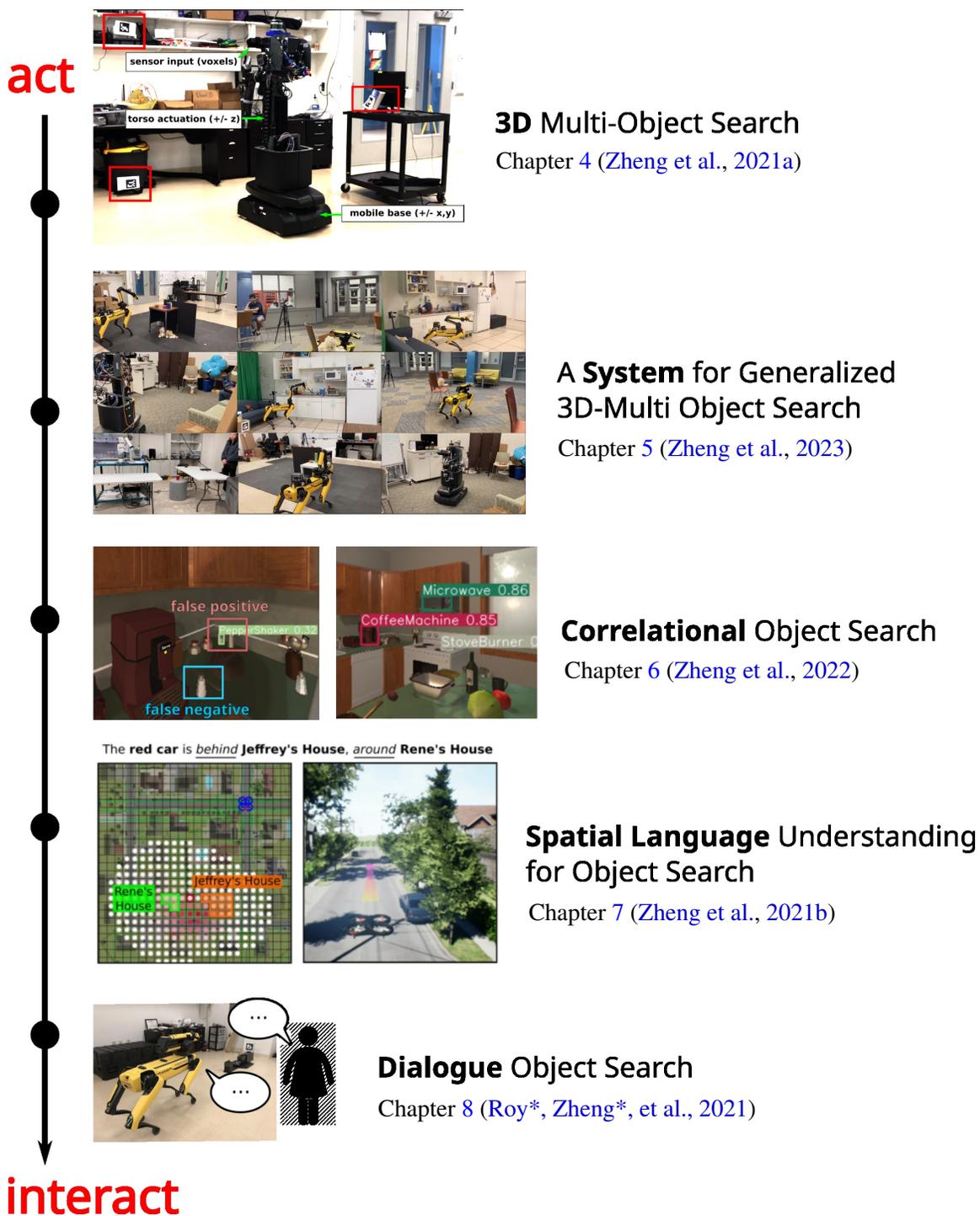

\begin{itemize}[itemsep=0.5pt,topsep=0pt]
\item We exploit the structure of octrees for 3D multi-object search and develop a scalable and efficient multi-resolution planning algorithm with practicality for a real robot.
\item We then design and implement the first robot-independent and environment-agnostic system for generalized 3D multi-object search and deploy it on various robot platforms in different environments.
\item We exploit the spatial correlation between easier-to-detect (\eg, microwave) and hard-to-detect objects (\eg~pepper shaker) to find hard-to-detect objects faster.
\item We integrate spatial language (\eg,~``the book is in front of the monitor'') as an additional modality of stochastic observation for more efficient search.
\item Moving along, we motivate and propose the dialogue object search task and draw insights from a pilot study between pairs of human participants.

\item We formulate a general, overarching POMDP model for generalized object search as a  ``parent model'' that ties together the above works on specific object search settings.

\item We discuss how the above POMDP model be extended to address tasks involving additional challenges such as object dynamics and environment interaction.

\item Last but not least, we provide a literature review on object search from a robotics perspective based on a survey of more than 125 related papers.\\
\end{itemize}

\noindent Here, we provide a high-level overview to elaborate the above contributions. For 3D multi-object search, we introduce \textbf{3D Multi-Object Search (3D-MOS)}, a general POMDP formulation of the problem with volumetric observation space, and we solve it with a novel multi-resolution planning algorithm that uses a new belief representation called \textbf{octree belief}. Our work demonstrates that such challenging POMDPs can be solved online efficiently and scalably with practicality for a real robot by extending existing general POMDP solvers with domain-specific structure and belief representation.

As a step further, we design and implement \textbf{GenMOS (Generalized Multi-Ob-ject Search)}, a robot-independent and environment-agnostic system of multi-object search in 3D regions. We evaluated the system by deploying it on several robotic platforms, including the Boston Dynamics Spot robot, the Kinova MOVO robot, and the Universal Robotics UR5e robotic arm. Our work makes 3D object search an off-the-shelf capability for different robots in different environments.

To find small, hard-to-detect objects, we propose \textbf{Correlational Object Search POMDP (COS-POMDP)}, which can be solved to produce search strategies that exploit spatial correlations between target objects and other objects in the environment that are easier-to-detect.
Our experiments were conducted in AI2-THOR \cite{kolve2017ai2}, a realistic simulator featuring diverse-looking environments of four types: bedroom, bathroom, kitchen and living room, and we use YOLOv5 \cite{glenn_jocher_2020_4154370} as the object detector. Results show that our method finds objects more successfully and efficiently compared to baselines, particularly for hard-to-detect objects such as srub brush and remote control.

To consider human input, we integrate spatial language into the object search POMDP model as a form of stochastic observation and derive a spatial language observation model. We call the resulting model \textbf{SLOOP (Spatial Language Object-Oriented POMDP)}. Evaluation based on crowdsourced language data, collected over areas of five cities in OpenStreetMap, shows that our approach achieves faster search and higher success rate compared to baselines, with a wider margin as the spatial language becomes more complex. We demonstrate the proposed method in AirSim, a realistic simulator where a drone is tasked to find cars in a neighborhood environment. We further deploy our system on a Boston Dynamics Spot to accomplish tasks such as that in Figure~\ref{fig:spot_challenges}.

To move farther along interaction with humans, we motivate and propose the task of \textbf{Dialogue Object Search} where a robot must simultaneously decide what to say and how to act while searching for an object in collaboration with a remote human assistant. We conduct a pilot study between human participants and analyze how humans approach the task in the place of the robot.

We tie together the problems and POMDP models proposed in the above works by defining a single, overarching POMDP model (the ``parent'' model). We briefly discuss in Chapter~\ref{ch:closing}, as future work, how this parent model may be extended to additional challenges and its potential limitations.


Finally, due to the fundamental value and broad applicability of object search yet a lack of survey in this field, we conduct a literature review with the goal of organizing this field from a robotics perspective, clarify the relationships between different problem variants and approaches, and provide a taxonomy of object search.

\section{Thesis Outline}
The remainder of this thesis is organized as follows:
\begin{itemize}[itemsep=0.5pt,topsep=0pt]
\item In Chapter~\ref{ch:background}, we provide the background material of our research, including a literature review on object search and an introduction of POMDPs from a robotics perspective.
\item In Chapter~\ref{ch:overarching}, we give the definition of the overarching POMDP model for object search, as hinted in the section above.
\item In Chapter~\ref{ch:3dmos}, we describe our work on 3D multi-object search and focus on addressing the algorithmic challenges of this problem that concludes with a real robot demonstration. This chapter covers the paper \emph{Multi-Resolution {POMDP} Planning for Multi-Object Search in 3D} \cite{zheng2021multi}.
\item In Chapter~\ref{ch:genmos}, we describe our work on further addressing the system-level challenges to make 3D multi-object search an off-the-shelf system and package for robots. This chapter covers the paper \emph{A System for Generalized 3D Multi-Object Search} \cite{zheng2023system}.

\item In Chapter~\ref{ch:cospomdp}, we describe our work on correlational object search for finding small, hard-to-detect objects such as a pepper shaker or a credit card. This chapter covers the paper \emph{Towards Optimal Correlational Object Search} \cite{zheng2022towards}.

\item In Chapter~\ref{ch:sloop}, we describe our work on interpreting spatial language object search in city-scale environments. This chapter covers the paper \emph{Spatial Language Understanding for Object Search in Partially Observed Cityscale Environments} \cite{sloop-roman-2021}. We include a follow-up integration of this work with the Boston Dynamic Spot robot to perform object search with spatial language understanding.

\item In Chapter~\ref{ch:dialog}, we describe our work on introducing the dialogue object search task and a pilot study. This chapter covers \emph{Dialogue Object Search} \citepalias{dls2021}. Here, * denotes equal contribution.


\item In Chapter~\ref{ch:closing}, we conclude this thesis and remark on future work directions.

\item Appendix~\ref{ch:pomdp-py} describes the \texttt{pomdp\_py} library which supported the implementation of all POMDP models and algorithms proposed in this thesis. It covers the paper \emph{pomdp\_py: A Framework to Build and Solve POMDP Problems} \cite{pomdp-py-2020}. The library is available at: \url{https://github.com/h2r/pomdp-py}.

\end{itemize}

\section{Remark on Terminology}
\label{sec:intro:term}
For consistency throughout this theis, we use \emph{``viewpoint''} to mean the same thing as \emph{``view
pose''}, which is a 6D camera pose.  We use \emph{"view position"} to mean the 3D position
of the camera, and \emph{"view orientation"} or \emph{"view direction"} to mean
the 3D rotation of the camera.

Additionally, we use \emph{``generalized'' } instead of \emph{``generalizeable''} in the title of this thesis to emphasize the fact that our POMDP models (\eg, 3D Multi-Object Search) are by definition not specific to any particular robot or environment, and our approach is based on general-purpose online POMDP planning algorithms (with problem-specific structure).


Finally, we refer the reader to Section~\ref{sec:bg:related_probs} for a discussion of \emph{object search} in relation to other problems, some with similar names, to clarify possible confusion.



\chapter{Background}
\label{ch:background}
\section{Object Search in Robotics: A Review}
\label{sec:bg:objectsearchreview}
\lettrine{D}{irecting} searchers to find targets is a fundamental problem with broad applicability. Researchers with diverse backgrounds have studied variants of this problem with distinct emphases since the 1970s. If one zooms out, it does not take long before one finds the literature on this topic to be a seemingly hodgepodge. Given such diversity and long history, surveying past work is therefore valuable.

Although this thesis mainly concentrates on how to search within a region, we look broadly here in the background chapter to take on the challenge of surveying this field. The review below is the result of surveying more than 125 papers related to object search, including taxonomies of different aspects of object search.

This chapter is organized as follows:
\begin{itemize}[itemsep=0.5pt,topsep=0pt]
\item We begin by clarifying what the term ``object search'' entails in robotics, disambiguating it from related and similar problems.

\item We provide taxonomies of basic aspects in object search research, including problem settings, methods and systems, and then apply them to organize papers from the field.

\item Finally, we briefly summarize the history of research in object search so far.
\end{itemize}

\subsubsection{Difference from Related Reviews}

\noindent The closest literature reviews compared to ours are \citet{chung2011search} and
\citet{robin2016multi}. \citet{chung2011search} focuses on pursuit-evasion in mobile robotics. They provide a taxonomy of pursuit-evasion and autonomous search problems that highlights the differences resulting from varying assumptions on the searchers, targets, and the environment, with a discussion of robotic systems. This survey is excellent, yet it has been more than a decade since it is written, and we pay direct attention on object search and not pursuit-evasion.

\citet{robin2016multi} reviews major algorithms for different problem settings related to target management, a term that summarizes both target detection and tracking. Search is considered a variant of target detection where the target(s) are actively pursued. However, this survey concerns providing a bird's-eye view of the spectrum of problems from target detection to tracking. In contrast, we focus on object search with the end goal of achieving practical and effective robotic systems for this task.\\

\vspace{-1.65em}
\subsection{The Elements of Object Search in Robotics}
\label{sec:bg:elem}
What does ``object search'' entail in robotics? Imagine a person searching for their missing pair of glasses at home. It is most likely that the person would move around to search in different places rather than standing at a fixed location, and then be able to identify the glasses somehow, probably through vision. Analogously, when we think of ``object search'' in robotics, we envision a robot being able to do the same thing. This, in the most basic sense,  should involve the following elements:
\begin{enumerate}[itemsep=0.5pt,topsep=0.5pt]
  \item an object to be searched for,
  \item an environment where the search happens,
  \item a robot that can move itself in order to perceive different aspects of the physical environment that would not be perceivable if the robot does not move, and
  \item the robot begins without knowing where the object is, and
  \item the robot decides what to do and executes actions sequentially in search for the object.
\end{enumerate}
The above basic set of conditions makes a problem in robotics an object search problem. Variants of object search, such as moving object search, multi-object search,  mechanical search (object search by clearing clutter), etc., append additional conditions to this set.

Broadly, object search could refer to the parent problem of all variants. In a narrow sense, it refers to the problem setting in Figure~\ref{fig:spot_challenges} (Chapter~\ref{ch:intro}), where no additional conditions are added. Indeed, the variants usually do not come across plainly as object search, but with a more characterizing name. For example, if the object actively moves, then the problem is called \emph{moving object search} \cite{ding2018moving}. If the object's placement varies over time, the problem is called \emph{dynamic object search} \cite{zhang2019efficient}.


\subsection{Differentiating Object Search from Related Problems}
\label{sec:bg:related_probs}
Below, we describe how object search differs from several related problems that might cause confusion. This is not an exhaustive list, but hopefully sufficient for clarification.

\subsubsection{Object Detection, Object Tracking, and Object Pose Estimation}

Object search differs from problems such as \textbf{object detection}, \textbf{object tracking}, and \textbf{object pose estimation} in that object search is fundamentally a \emph{decision-making problem}, while the rest are \emph{perception problems}. The goal of object detection \cite{zou2019object} is to identify objects of certain classes given fixed raw sensor input, as an RGB image or a point cloud. Object tracking refers to estimating the trajectory of an object in the image plane or in the 3D space \cite{yilmaz2006object, weng20203d}. Object pose estimation \cite{rad2017bb8,xiang2017posecnn} is about estimating the pose (3D or 6D) of objects given sensor input. Both object detection and object pose estimation can serve as important building blocks of an object search system.

Variants of object detection or object pose estimation, such as \textbf{active object detection} where the sensor is controlled to maximize detection performance \cite{xu2021dynamic}, do involve decision-making and are often approached with similar techniques as object search, such as next-best view \cite{doumanoglou2016recovering}, reinforcement learning \cite{luo2019end}, and POMDP planning \cite{atanasov2014nonmyopic}. However, in contrast to object search, the environment typically revolves around the object to be detected, and the object typically starts out inside the field of view of the sensor. A similar comparison can be made between object search and the problem of \textbf{active object reconstruction} \cite{delmerico2018comparison}.

\subsubsection{Object Retrieval and Object Localization}

\textbf{Object retrieval} can be both a perception problem, choosing an object that matches a description from a set of objects \cite{nguyen2022affordance}, or a decision-making problem involving manipulation, where the robot is tasked to grasp and take out an object in clutter \cite{ahn2022coordination}.  Object retrieval is a more specific problem than object search since in general, retrieval is not always necessary after the object is found. Often, object search is part of an object retrieval system, considering occlusion from clutter \cite{bejjani2021occlusion}.

The term \textbf{object localization} commonly refers to the computer vision problem where the task is to produce a heatmap for an object's location given image \cite{xue2019danet, zhang2021weakly}; \textbf{active object localization} \cite{caicedo2015active} has been used to describe object detection models that can actively zoom in on an image for better detection. In robotics, object localization is often a subproblem for object retrieval in clutter \cite{liu2012fast,du2021vision}. It has been used as a synonym to object search \cite{andreopoulos2010active}, though such usage is rare.

\subsubsection{Target Capture, Target Pursuit and Target Detection}

In \textbf{target capture} or \textbf{target pursuit} \cite{eaton1962optimal}, one or a group of agents coordinate to surround one or multiple moving targets on a graph, where capture typically means an agent occupies the same node as the target \cite{isaza2008cover}, or the target is enclosed by the agents \cite{sharma2010cooperative}. This problem has important practical implications and has been studied by researchers in the game theory \cite{sticht1975sufficiency}, graph theory \cite{svensson2011target}, multi-agent control \cite{bono2022swarm} communities. If the target is adversarial, this becomes a classical problem called \textbf{pursuit-evasion} \cite{ho1965differential}. Historically, this line of work has used \textbf{moving target search (MTS)} \cite{ishida1991moving} as the name of the target capture problem, causing some confusion with the object search problem in robotics.\footnote{Searching for ``target search'' on the web leads to papers from this community, although the literal meaning of the phrase is nearly identical to object search.} Distinctively, however, object search in robotics differs in that the locations of target objects are not known and uncertainty in robot perception is inevitable, both aspects often not considered in that line of work.\footnote{Despite not considering uncertainty in perception, research in target capture strategy is valuable especially when the agents are, for example, human teammates.} Additionally, in robotics, the target and robot states are fundamentally metric and continuous instead of on top of a discrete graph.


In addition, \textbf{target detection} is yet another problem with different meanings in different communities. When the target is being actively detected by an agent, the problem is closer to target search or target capture (discussed above) \cite{robin2016multi,dadgar2016pso}. When the detection is passive, the problem is more closely related to sensor placement \cite{o1987art,sadeghi2020optimal}, detection mechanism \cite{koopman1956theory}, and signal processing \cite{tian2002target}. The sensors can be sonars \cite{mukherjee2011symbolic}, radars \cite{bekkerman2006target},  infrared \cite{chen2013local}, cameras \cite{sun2016camera}, etc.

\subsubsection{Semantic Visual Navigation}
\label{sec:bg:semanticvisnav}
Several problems share a similar practical motivation as object search but were originated from different fields.
Notably, in computer vision, the \textbf{ObjectGoal} \cite{anderson2018evaluation} and \textbf{ObjectNav} \cite{batra2020objectnav} tasks, which are instances of \textbf{semantic visual navigation}
\cite{chang2020semantic}, or sometimes just \textbf{visual navigation} \cite{gupta2017cognitive,wortsman2019learning}, can be viewed as object search problems as well.\footnote{Semantic visual navigation is the problem where an agent is placed in an unknown environment and tasked to navigate towards a given semantic target (such as ``kitchen'' or ``chair''). The agent typically has access to behavioral datasets for training on the order of millions of frames and the challenge is typically in generalization. Practically, semantic visual navigation is motivated as a home robot application \cite{batra2020objectnav} and serves the same purpose as object search using cameras in robotics.} Broadly speaking,  object search is a more general problem than semantic visual navigation as it does not require vision input or searching by navigation. Setting aside this technicality, researchers in the two communities tend to make different assumptions in problem settings, leading to the use of different methods.  Many works in object search are planning-based, so a model of the environment is given,
and the evaluation culminates at real robot demonstrations.  On the other hand, end-to-end deep learning has been the predominant method for semantic visual navigation as the agent searches in an unknown environment and leverages simulators for large-scale training and benchmark testing (\eg, AI2THOR by \citet{kolve2017ai2}, Habitat by \citet{ramakrishnan2021habitat}, and Gibson by \citet{xia2018gibson}, ThreeDWorld by \cite{gan2020threedworld}). Both communities can learn from each other, although interactions have currently been limited.\footnote{For example, one improvement of ObjectNav over ObjectGoal is the introduction of \emph{intentionality} as a success criterion, which requires the agent to signal that the object is found. The same issue has been previously considered in the object search literature in robotics \cite{li2016act,wandzel2019multi}.} Nevertheless, it is hopeful that as both approach the common goal of a general robotic system for object search, more interaction will occur. This thesis takes a step in that direction.





\subsubsection{Active Visual Search}

The terms \textbf{visual search} \cite{vogel2007non} \textbf{visual object search} \cite{druon2020visual}, \textbf{active visual search} \cite{sjoo2012topological}, \textbf{active visual object search} \cite{aydemir2011search,zeng2020semantic}, and \textbf{active object search} \cite{elfring2013active}, are all synonymous to object search using a camera in robotics. Since cameras are predominantly used by robots for perceptual tasks such as face recognition and object detection in human environments, ``visual'' is often omitted (\eg, in \citet{Ye1997SensorPF}). Since object search in robotics is inherently an active process, ``active'' is also often omitted. Note that in psychological testing \cite{erickson1964visual,williams1967effects,duncan1989visual} and computer vision \cite{tsotsos1992relative,eckstein2011visual}, ``visual search'' has historically referred to what we call object detection today. So, it makes sense that researchers coming from those backgrounds prefix it with ``active'' to mean object search by a robot. For the purpose of this thesis, we simply use object search (see Section~\ref{sec:bg:elem} for elaboration).


\subsection{A Taxonomy of Object Search Problem Settings}
\label{sec:bg:taxo_prob}

There are several dimensions along which we can classify object search problem settings:
\begin{enumerate}[itemsep=0.5pt,topsep=0pt]
\item number of objects: single or multiple;
\item number of robots: single or multiple;
\item property of the object(s): \eg~static, dynamic, or adversarial;
\item property of the robot(s): \eg~camera-only mobile robot (including drones), eye-in-hand robotic arm, or mobile manipulator;
\item property of the environment: \eg~unknown, cluttered.
\end{enumerate}
Due to the large number of dimensions, it quickly gets out of hand if we try to enumerate all combinations. Instead, for clarity, we consider the following naming pattern:

\begin{center}
  \begin{tabular}{cccc}
[count] & [object property] & object search  & [other property]\\
    secondary & primary &   & tertiary
  \end{tabular}
\end{center}

This naming rule consists of four parts. The three parts surrounding ``object search'' are used to characterize the specific problem setting, effectively appending conditions to the basic set that define object search laid out in Section~\ref{sec:bg:elem}. These parts are organized by the following ranking:\footnote{The ordering is analogous to a radiation field centered at ``object search'', such that the radiation goes from attributes of objects to attributes of the robot, finally to the attributes of the environment.}
\begin{itemize}[itemsep=0.5pt,topsep=0pt]
\item \textbf{Primary: object property.} An adjective that describes the object's property. If not provided, this adjective is assumed to be ``static.''
\item \textbf{Secondary: count.} An adjective that describes the number of robots and the number of objects, in that order. If not provided, then both are assumed to be ``single.''  For example, we would say, ``multi-robot multi-object search'' rather than ``multi-object multi-robot search.''\footnote{In general, clarity should come first especially in paper narration. That means instead of saying ``multi-robot multiple dynamic object search,'' it is better and more natural to say ``multi-robot search for dynamic objects.'' The fact that consistent problem naming is difficult shows the multitude of object search variants.}
\item \textbf{Tertiary: other property.} A noun phrase that characterizes the robot, the environment, and (or) the task. For example, ``open-world object search'' means a robot explores in an open-world while performing search;``object search in clutter'' implies that the environment is cluttered and the robot likely can manipulate clutter.
\end{itemize}
Figure~\ref{fig:bg:complexity} is a sketch of the relationships between different variants in terms of complexity, system-level difficulty and applicability. In Table~\ref{tab:bg:survey}, we apply this taxonomy to organize object search problem settings.

\begin{figure*}[t]
  \centering
  \includegraphics[width=0.85\textwidth,draft=false]{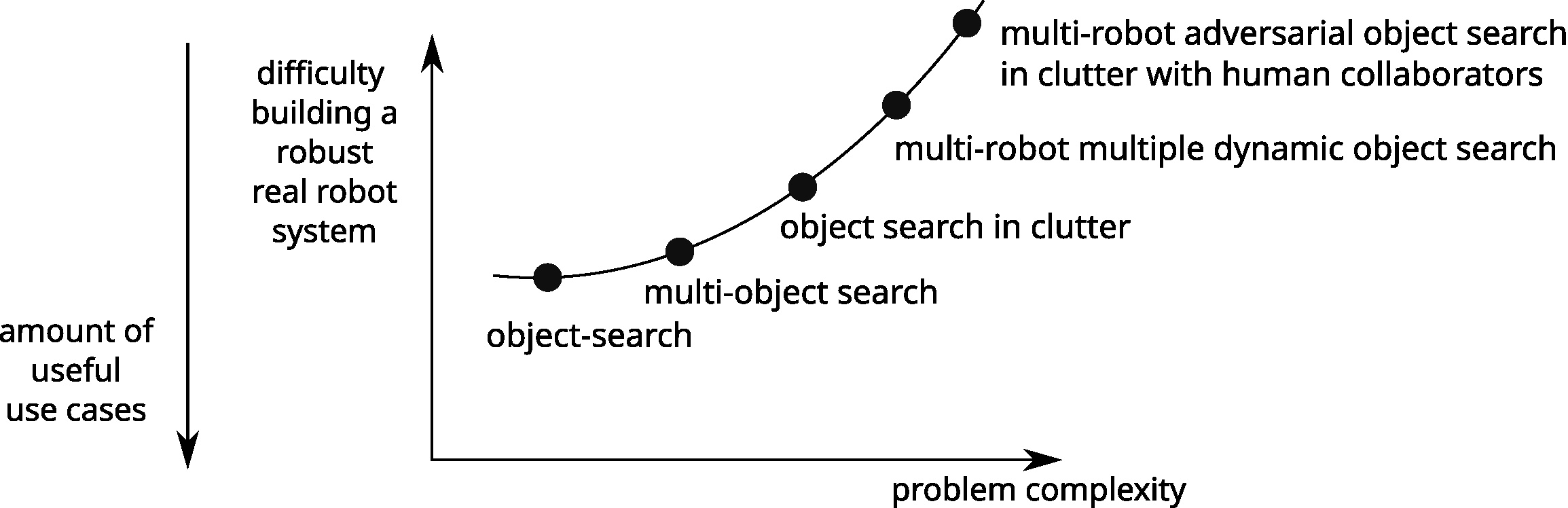}
  \caption{A sketch of the relationship between the complexity of object search variants and difficulty in building an actual system. A system that performs basic object search is arguably more broadly applicable than an intricate variant.}
  \label{fig:bg:complexity}
\end{figure*}

\subsection{A Taxonomy of Object Search Methods}
\label{sec:bg:taxo_methods}
As object search is a sequential decision-making problem, common approaches to this class of problems can be applied. We can therefore categorize the object search methods based on the type of decision-making algorithm used to generate the search strategy:
\begin{enumerate}[itemsep=0.5pt,topsep=0pt]
  \item heuristics-based (\eg, bio-inspired);
  \item mathematical analysis;
  \item greedy, next-best view;
  \item graph search;
  \item POMDP planning;
  \item reinforcement learning.
  \end{enumerate}

  We elaborate the above methods through examples.

  \textbf{Heuristics-based.} \citet{goldsmith1998collective} studies multi-robot search for a static object (a threat such as an explosive) and proposes an ad-hoc strategy that divides the robots into two social groups: alpha and beta. Robots in the alpha group are risk takers, while those in the beta group are conservative. Each group has a distinct rule-based policy for when and when not to move. This kind of method is typically viewed as less favorable compared to more principled approaches, but may work well in practice.

  \textbf{Mathematical analysis.}  \citet{pollock1970simple} formulated arguably the first model for moving object search where a target moves between two cells according to a Markovian model.  The problem here is framed as finding an allocation of the amount of effort (\eg, time) spent searching at each cell, which is solved it through dynamic programming. \citet{stone1976theory} generalized the problem and provides a thorough introduction on this topic. In this formulation, the searcher is given a probability distribution over the target's initial location, and a detection function that characterizes the probability density of detecting the target given a target location and an amount of effort.
  Later, \citet{mangel1981search} considered the case where the target moves as a diffusion process. Instead of finding a search plan (\ie, searcher's trajectory), the problem is to describe the probability density function  $f(x,S,t)$ that search has been unsuccessful up to time $t$ following a trajectory $S$, given that the target is at $x$ at $t$ (the so-called ``descriptive problem''). Analytical methods were dominant in the early days when object search was studied in operations research, but results from this body of work are often general yet abstract. Direct transfer to real robotic systems has been rare.


  \label{sec:bg:taxo_methods:greedy}
  \textbf{Greedy, next-best view.} \citet{Ye1997SensorPF} formulated 3D object search as choosing sensing actions over time under a cost budget; each action controls camera parameters such as position and zoom in some way and has an associated cost and probability of detecting the target. This is shown to be NP-Complete via a straightforward reduction to the Knapsack problem. Follow-up works hence often favor a greedy search strategy to build robotic systems for object search \cite{tsotsos1998playbot,shubina2010visual,andreopoulos2010active,rasouli2015attention}. \citet{andreopoulos2010active} builds a vision-based system that enables a 26-DOF humanoid robot (the Honda ASIMO) to find an object (\eg cup) in a 3D region. The system maintains a probability map over a 4$\times$4$\times$4 grid, and uses a greedy, one-step look-ahead algorithm to search. This algorithm first hypothesizes a set of candidate sensor states, each made up of a body pose and a gaze pose, and then selects a sensor state that maximizes the probability of localizing the target position. Greedy, next-best view algorithms are commonly used, often simple to implement and perform well in practice. Nevertheless, when the search actions carry different costs, or when multiple objects need to be searched for, sequential planning can consider important factors regarding action ordering that a greedy method does not.

  \textbf{Graph search.} It is uncommon, but graph search algorithms such as A$^*$ have been used for object search, primarily when searching for occluded object in clutter \cite{dogar2014object,lin2015planning}. Here, the search process is thought of as revealing occluded space by removing a sequence of visible obstacles until the target object becomes visible, where each removal action may carry a different cost.\footnote{The costs differ since removing different obstacles may affect visibility (how much volume will be revealed) and accessibility (how accessible other obstacles will become) in different ways.}
The greedy approach is shown to be suboptimal \cite{dogar2014object}. Instead, imagine a graph where a node is a subset of all visible objects, and an edge indicates the removal of a single object. Then, the object search problem becomes a graph search problem of finding the shortest path from the current node\footnote{The current node corresponds to the current state (\ie,~current subset of visible objects).} to the node corresponding to $\varnothing$ (all space revealed, target is therefore considered found). This is the approach taken by both \citet{dogar2014object} and \citet{lin2015planning}. It is elegant, yet requires the obstacles to be taken off from the environment at every step and does not consider uncertainty in perception.

\textbf{POMDP planning.} \citet{roy2005finding} uses belief compression to trade off belief space dimensionality and information loss, allowing value iteration to be feasible in enabling a mobile robot to efficiently search for a person in a hallway-and-room environment. More recently, advances in online POMDP planning algorithms have sparked research interest in applying them for object search (\eg, DESPOT \cite{somani2013despot} in \citet{li2016act} and POMCP \cite{silver2010monte} in \citet{wandzel2019multi}); refer to Section~\ref{sec:bg:pomdp_policies} for a closer look at these algorithms.) For multi-robot search, POMDP planning was unfavored computationally by \citet{hollinger2009efficient} but continue to receive attention as both fields progress \cite{rizk2019cooperative}. The appeal of POMDP planning-based approaches is in their generalizability, versatility, and reasoning over action sequences. Unpleasant theoretical results (\eg, \cite{madani1999undecidability}) and practical computational challenges have been the main reasons of criticism to this kind of approach. Previous object search works have sacrificed realism for efficiency (\eg, reducing the problem from 3D to 2D). This thesis uses POMDP's strengths in robot task modeling and mitigates its weaknesses through exploiting structures to achieve practical, effective object search systems.





\textbf{Reinforcement learning.}
\citet{zhu2017target} is a seminal work in semantic visual navigation\footnote{See Section~\ref{sec:bg:semanticvisnav} (page~\pageref{sec:bg:semanticvisnav}) for a discussion on semantic visual navigation in relation to object search.} that trained a goal-conditioned actor-critic model in the AI2-THOR simulator and demonstrated sim-to-real generalization to a mobile robot. A large body of work followed and developed various end-to-end, model-free techniques \cite{wortsman2019learning,mousavian2019visual,Savarese-RSS-19,chaplot2020object,liang2021sscnav}; Sample complexity and real-world generalization are hurdles to overcome for this approach.
Recently, \citet{shah2022robotic}
achieved impressive outdoor visual navigation results on a mobile robot through integrating pre-trained models like CLIP \cite{radford2021learning} with a navigation planning pipeline (no sim-to-real). Separately, in target pursuit, \citet{shkurti2018model} proposed a model-based reinforcement learning approach using POMDP planners.\footnote{Note that the technique of model-based reinforcement learning with POMDP planning is equivalent to that of POMDP planning with learned models (instead of analytical). However, typically in reinforcement learning, the agent starts with no knowledge of the world and learns from scratch.} Although our work uses learned object detectors, we formulate POMDP models analytically. In the long run, inspired by AlphaGo  \cite{silver2017mastering} and AlphaZero \cite{silver2018general}, we believe that combining online tree search-based POMDP planning with learned models is a promising vein for research in object search, especially when physical interaction in unstructured environments is involved.

\subsection{A Taxonomy of Object Search Systems}
Despite often motivated from in a robotics context, papers in object search without any real robotic system demonstration outnumbers those that do. Nevertheless, the state-of-the-art in object search is best reflected by what researchers can make real robots do, which do vary a lot. It is therefore useful to come up with a taxonomy that captures a spectrum of object search systems developed to evaluate object search algorithms.

It is straightforward that on one end of the spectrum, the system is idealistic; for example, the robot and the target are points on a plane. It is more difficult to determine what is the other end. Note that we are not classifying the object search system's performance,\footnote{We do not attempt a taxonomy of the system's performance because one could always hide a target object so that a robot, or even a person, would fail to find it. Therefore, a system that ``always finds objects'' is not a realistic goal. The highest expectation in performance is that the robot can find objects as efficiently or more efficiently than a person, which is something we are so far away from today to be worth classifying.} but where it is deployed in. This includes considerations about what robotic capabilities are involved, and what environments the system is expected to operate in.

Therefore, we consider the following taxonomy to categorize object search systems:
\begin{longtable}{>{\raggedright}p{0.25in}p{5.0in}}
  \toprule
  S0 &  idealistic simulation\\
  S1 &  realistic simulation\\
 \midrule
 N0 &  single, navigation-only robot in a single environment\\
    N1 &  single, navigation-only robot in different environments\\
    N2 &  different navigation-only robots in different environments\\
 \midrule
 M0 &  single, manipulation-only robot in a single environment\\
    M1 &  single, manipulation-only robot in different environments\\
    M2 &  different manipulation-only robot in different environments\\
 \midrule
 R0 &  single, mobile manipulator robot in a single environment\\
 R1 &  single, mobile manipulator robot in different environments\\
 R2 &  different real mobile manipulator robot in different environments\\
 \bottomrule
 \caption{A taxonomy of object search systems}
\end{longtable}

\noindent This taxonomy is inspired by how levels of autonomous driving are concisely represented (from L0 to L5) \cite{briney2004state}. The name of each category (\eg, S0, N1, etc.) consists of a letter that represents the type of system and a number that represents the level of generalization of that system. To elaborate, ``S'' = simulation, ``N'' = navigation-only, ``M'' = manipulation-only, and ``R'' = mobile-manipulator (both navigation and manipulation). The number starts from 0, goes up to 1 for ``S'', and up to 2 for the rest.

A \textbf{realistic simulation} must be 3D, and the objects, motion, and events in the simulator can easily be relatable to their counterparts in the real world. \textbf{``Navigation-only''} means no manipulation is involved \emph{during} the search process; This includes common mobile robots but also, for example, an eye-in-hand robotic arm that only searches by looking through its gripper camera but does not manipulate its surroundings. \textbf{``Manipulation-only''} means the robot does \emph{not} have a mobile base and it manipulates the environment \emph{during} search, such as removing clutter or opening containers. \textbf{``Mobile manipulator''}  means the robot has a mobile base \emph{and} it manipulates the environment during search. Note that this does not require the robot to have a robotic arm -- a mobile robot that knows how to use its body to clear up movable obstacles in search for an object also counts.

Note that we do not have a category for ``different types of [...] robots in a single environment.'' This is because different robots typically are designed to be more suitable for operation in some environments than others.  It is more natural and sufficient to have a robot be able to perform search in environments suitable for it. If a system can achieve this for different robots, it achieves level 2. We provide an imaginary example for each below:

\begin{longtable}{>{\raggedright}p{0.25in}p{5.0in}}
  \toprule
  S0 & \emph{the robot and the target are points on a plane or grids in a grid world.}\\
    S1 & \emph{AI2-THOR, Habitat, Gibson, ThreeDWorld}\\
 \midrule
 N0 & \emph{a turtlebot searches in an office.}\\
 N1 & \emph{a turtlebot searches in an office and in a kitchen,\footnote{This emphasizes the capability to search in different environments, not necessarily simultaneously.} or two different offices}\\
 N2 & \emph{a turtlebot searches in an office and a drone searches outdoors.}\\
 \midrule
 M0 & \emph{a Panda arm searches over a cluttered tabletop.}\\
 M1 & \emph{a Panda arm searches over a cluttered tabletop and over a cluttered shelf.}\\
 M2 & \emph{a Panda arm searches over a cluttered tabletop, and a UR5e robot searches over a cluttered shelf.}\\
 \midrule
 R0 &  \emph{a MOVO searches in a room with a cluttered tabletop.}\\
 R1 &  \emph{a MOVO searches in a room with a cluttered tabletop and in a room with two cabinets.}\\
 R2 &  \emph{a MOVO searches in a room with a cluttered tabletop, and a Spot searches in a room with two cabinets.}\\
 \bottomrule
 \caption{Examples for the categories in the taxonomy; Application in Table~\ref{tab:bg:survey}}
\end{longtable}

\subsection{Application of the Taxonomies for Literature Survey}
Below, we follow the taxonomies introduced in the previous sections to categorize papers from the literature, which concludes this review. We only consider papers that assume the target location to be initially unknown (excluding some papers on target capture).
The result of our survey is in the table below. We divide up problem names applying the three-level taxonomy and we grouping references by methods used. The system classification is indicated by the superscript.

Notably, to the best of our knowledge, our work \cite{zheng2023system} is the first system to achieve N2 among the class of navigation-only robots.  For manipulation-only robots, the current state of the art is limited to at most M0. \citet{xiao_icra_2019} is the first and only (so far) that achieves R0, in which a mobile manipulator robot, Fetch \cite{wise2016fetch}, searches for a target object in a tabletop scene by removing clutter and looking from different viewpoints. However, the setup is restricted to only two manually specified viewpoints, one on each side of the table, from where the entire scene is captured under the field of view. Works on multi-robot search and moving object search are yet to reach N1, or any manipulation-involved search task. No work has reached R1 or R2, which can be regarded as the holy grail in terms of object search systems.

\vspace{2.0in}
\begin{center}
  THE TABLE STARTS FROM THE NEXT PAGE.
\end{center}

\newpage
\begin{longtable}{>{\raggedright}p{1.75in}>{\small}p{3.75in}}
  \toprule
  name & {\normalsize references} \\
  \midrule
  \textbf{(Static) object search} &\textsc{\textbf{Single-object}}\\
  &\smcatg{heuristics-based:} \cite{nomatsu2015development}\sq{N0}
  \cite{izquierdo2016searching}\sq{S0},
  \cite{wixson1994using}\sq{N0}\\
  &\smcatg{analytical:} \cite{Ye1997SensorPF}\sq{N0}
  \cite{hernandez2021searching}\sq{S1}\\
  &\smcatg{greedy:} \cite{Ye1997SensorPF}\sq{N0}
  \cite{gelenbe1998autonomous}\sq{S0}
  \cite{saidi2007online}\sq{N0}
  \cite{kollar2009utilizing}\sq{N0}
  \cite{andreopoulos2010active}\sq{N0}
  \cite{shubina2010visual}\sq{N0}
  \cite{aydemir2011search}\sq{N1}
  \cite{chen2013visual}\sq{N0}
  \cite{zeng2020semantic}\sq{S1,N0}\\
  &\smcatg{graph search:} \cite{gelenbe1998autonomous}\sq{S0}
  \cite{song2020two}\sq{S0}
  \cite{zhang2021building}\sq{S1,N1}
  \cite{shah2022robotic}\sq{N1}\\
  &\smcatg{POMDP planning:} \cite{vogel2007non}\sq{S1}
  \cite{aydemir2013active}\sq{N1}
  \cite{lu2018target}\sq{N0}
  \cite{wang2018efficient}\sq{S1,N0}
  \cite{holzherr2021efficient}\sq{S0}
  \cite{zheng2022towards}\sq{S1}\\
  &\\[-0.1cm]
  &\textsc{\textbf{Single-object (unknown environment)}}\\
  &\smcatg{heuristics-based:} \cite{li2022remote}\sq{S1,N0}\\
  &\smcatg{graph search:} \cite{joho2011learning}\sq{N1}
  \cite{tsuru2021online}\sq{N0}\\
  &\smcatg{POMDP planning:} \cite{wang2020pomp}\sq{S1}
  \cite{giuliari2021pomp++}\sq{S1}\\
  &\smcatg{reinforcement learning:} \cite{zhu2017target}\sq{S1,N0}
  \cite{ye2018active}\sq{S1,N0}
  \cite{mousavian2019visual}\sq{S1}
  \cite{Savarese-RSS-19}\sq{S1}
  \cite{chaplot2020object}\sq{S1}
  \cite{schmid2019explore}\sq{S1}
  \cite{qiu2020learning}\sq{S1}
  \cite{liang2021sscnav}\sq{S1}
  \cite{schmalstieg2022learning}\sq{S1,N0}
  \cite{zhu2022navigating}\sq{S1,N0}\\
  &\\[-0.1cm]
  &\textsc{\textbf{Single-object (in clutter)}}\\
  &\smcatg{heuristics-based:} \cite{huang2022mechanical}\sq{M0}\\
  &\smcatg{analytical:} \cite{wong2013manipulation}\sq{S1}\\
  &\smcatg{greedy:} \cite{moldovan2014occluded}\sq{S1}\\
  &\smcatg{graph search:} \cite{dogar2014object}\sq{S0,M0}
  \cite{lin2015planning}\sq{M0}
  \cite{nam2019planning}\sq{S0,M0}
  \cite{huang2021mechanical}\sq{S1,M0}\\
  &\smcatg{POMDP planning:} \cite{li2016act}\sq{S1}
  \cite{nie2016searching}\sq{S0}
  \cite{xiao_icra_2019}\sq{S1,R0}
  \cite{zhao2021efficient}\sq{S1}\\
  &\smcatg{reinforcement learning:} \cite{novkovic2019object}\sq{S1,M0}
  \cite{danielczuk2019mechanical}\sq{M0}
  \cite{kurenkov2020visuomotor}\sq{S1}
  \cite{bejjani2021occlusion}\sq{S0,M0}\\
  &\\[-0.1cm]
  &\textsc{\textbf{Multi-object}}\\
  &\smcatg{POMDP planning:} \cite{wandzel2019multi}\sq{S0,N0}
  \cite{zheng2021multi}\sq{S0,N0}
  \cite{sloop-roman-2021}\sq{S1,N0}
  \cite{zheng2023system}\sq{S1,N2}\\
  \midrule
  &\\[0.5cm]

  \midrule
  \textbf{Multi-robot (static)} &
  \textsc{\textbf{Single-object}}\\
  \textbf{object search} & \smcatg{heuristics-based:}
  \cite{goldsmith1998collective}\sq{S0}
  \cite{zhang2008mobile}\sq{S0}\\
  &\\[-0.1cm]
  &\textsc{\textbf{Single-object (unknown environment)}}\\
  &\smcatg{particle swarm optimization:}
  \cite{shirsat2020multi}\sq{S1}
  \cite{tang2021gwo}\sq{S0}
  \cite{ebert2022hybrid}\sq{S0}\\
  &\\[-0.1cm]
  &\textsc{\textbf{Multi-object}}\\
  &\smcatg{heuristics-based:}   \cite{rybski2002mindart}\sq{N0}\\
  &\smcatg{analytical:} \cite{czyzowicz2016search}\\

  \midrule
  \textbf{Moving object search}
 &\textsc{\textbf{Single-object}}\\
  &\smcatg{analytical:}
  \cite{pollock1970simple} \cite{dobbie1974two}
  \cite{stone1979necessary} \cite{washburn1980search}
  \cite{dogan2006unmanned}\sq{S0}
  \cite{fidan2013adaptive}
  \cite{radmard2017active}\sq{N0}\\
    &\smcatg{POMDP planning:}  \cite{roy2005finding}\\


  \midrule
  \textbf{Multi-robot moving} &
  \textsc{\textbf{Single-object}}\\
  \textbf{object search}
  & \smcatg{heuristics-based}: \cite{hereford2010bio}\sq{S0,N0}
  \cite{kulich2015comparison}\sq{S0}\\
  &\smcatg{analytical:} \cite{zengin2011cooperative}\sq{S0}\\
  &\smcatg{greedy:} \cite{sarmiento2004multi}\sq{S0}
  \cite{BOURQUE201945}\sq{S0}
  \cite{kulich2015comparison}\sq{S0}\\
  &\smcatg{graph search:} \cite{hollinger2009efficient}\sq{S0,N0}\\
  &\\[-0.1cm]
  &\textsc{\textbf{Single-object (unknown environment)}}\\
  &\smcatg{graph search}: \cite{marjovi2009multi}\sq{S0,N0}
  \cite{kulich2015comparison}\sq{S0}\\
  &\smcatg{particle swarm optimization:}
  \cite{tang2021gwo}\sq{S0}\\
  &\\[-0.1cm]
  &\textsc{\textbf{Multi-object}}\\
  &\smcatg{analytical:} \cite{dames2020distributed}\sq{S0}\\
  &\\[-0.1cm]
  &\textsc{\textbf{Multi-object (unknown environment)}}\\
  &\smcatg{heuristics-based:} \cite{baxter2007multi}\sq{S0}\\
  \midrule

  \textbf{Ethical issues in object search\footnote{The topic ``ethical issues in object search'' is not technically an object search problem, but it is an important dimension to it. Researchers referenced have started to discuss ethical challenges related to robot's help in search and rescue; helps by robots are not equally good.}} &
  \cite{sharkey201117} \cite{harbers2017exploring} \cite{battistuzzi2021ethical}\\
  \bottomrule
  \caption{Organization of papers from the object search literature using the proposed taxonomies}
    \label{tab:bg:survey}
\end{longtable}

\section{Partially Observable Markov Decision Process}
\label{sec:bg:pomdp}

POMDP was originally introduced to model control systems with incomplete state information in applied mathematics and operations research \cite{aastrom1965optimal,sondik1971optimal,smallwood1973optimal} and later introduced to robotics by \citet{kaelbling1998planning}.\footnote{See \citet{littman2009tutorial} for an excellent summary of the early literature on POMDPs.}  A POMDP models a sequential decision making problem where the environment state is not fully observable by the agent, which is almost always the case if the agent is a robot in a human environment. As a result, it has gained popularity in robotics.

In this section, I first motivate the use of POMDP in robotics from first principles. Then, I provide a precise definition of a POMDP and discuss different methods to obtain policies to a POMDP with an emphasis on practicality.
A more comprehensive literature review of POMDP in robotics can be found in \citet{thrun2005probabilistic, kurniawati2022partially, lauri2022partially}.

\subsection{Motivation from a  Robotics Perspective}
\label{sec:bg:pomdp_motiv}

A robot is a system of sensors and actuators\footnote{``Actuator'' here means more than mechanical ones, including \eg, speakers, virtual messaging, etc.} that operates in an environment. The robot can perform actions through its actuators and receive observations from its on-board sensors. The relationship between the robot and the environment can be illustrated by the perception-action loop (Figure~\ref{fig:paloop}), common in the sequential decision making literature \cite{littman1996algorithms, kochenderfer2015decision}.

\begin{figure*}[t]
  \centering
  \includegraphics[width=0.65\textwidth]{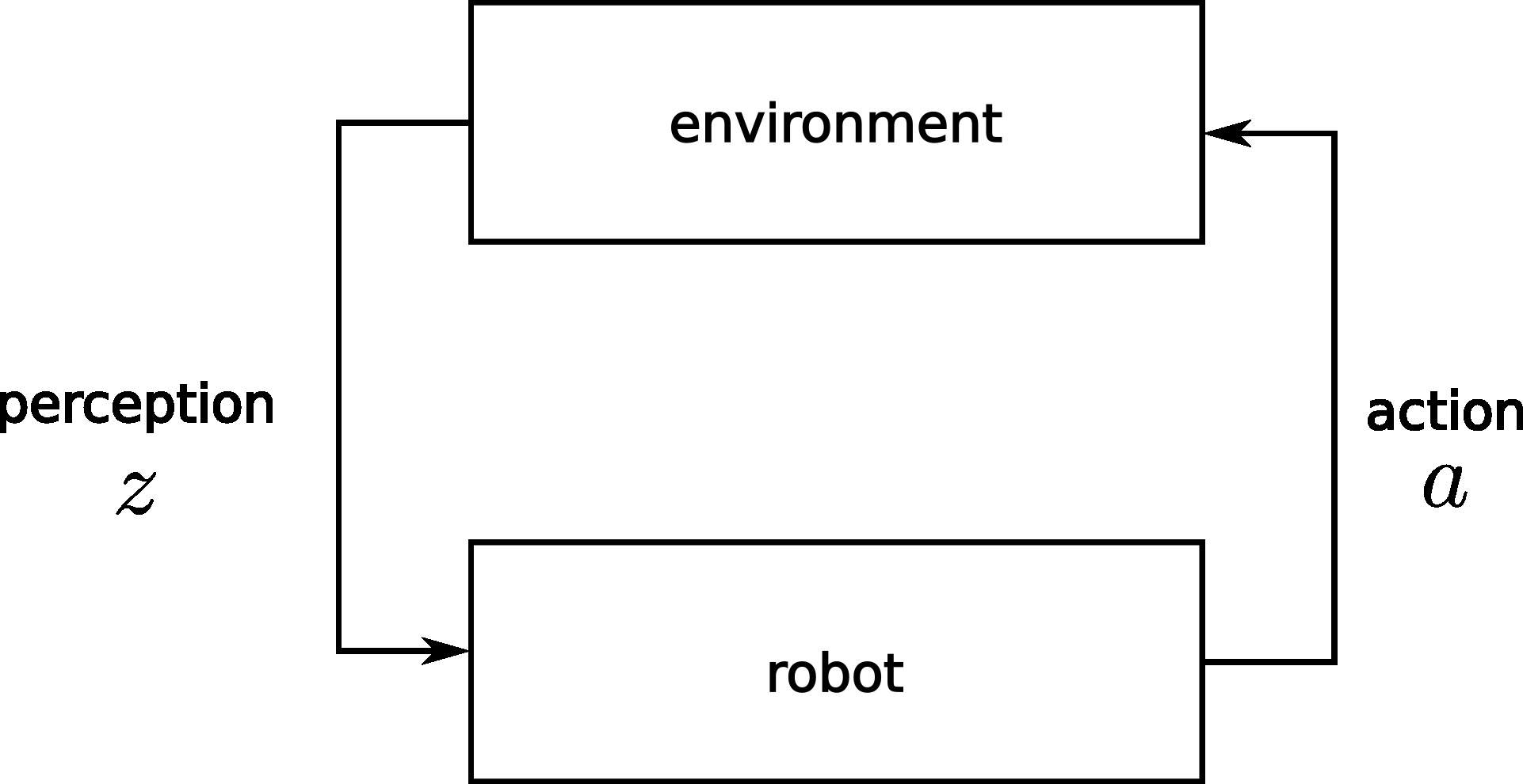}
  \caption{The perception-action loop. A robot is a system of sensors and actuators. The sensors can receive observations (perception), and the actuators can change the configuration of the robot, and perhaps, the external environment (action), which affects subsequent observations.}
  \label{fig:paloop}
\end{figure*}

Suppose the robot can choose an action from a set of possible actions, denoted as $\Aspace$, and suppose it receives observations from a set of possible observations is denoted as $\Zspace$. The perspective in Figure~\ref{fig:paloop} tells us that in the lifetime of a robot, it experiences a sequence of actions and observations $a_1,z_1,a_2,z_2,\cdots$ (and nothing else). Define the history at time $t$ to be the sequence of \emph{past} actions and observations, denoted as $h_t=(az)_{1:t-1}$.

The robot's purpose is to produce a sequence of actions that best achieves some task goal (over its lifetime). Suppose that the robot starts performing the task at time $t$.\footnote{This is in fact the setting of online POMDP planning.}  There are two questions here: how to represent the goal and how to choose an action at each step.
Assuming that the robot behaves rationally, a common idea to address both questions is to imagine that the robot observes a numeric \emph{reward} each time it takes an action.\footnote{Whether rewards are sufficient to express all goals we would ever want a robot to achieve is a topic of ongoing research \citep{silver2021reward,abel2021expressivity,bowling2022settling}. This discussion is entirely out of scope for this thesis. } Note that this reward is conditioned on the history of the robot's experience rather than just the action taken.  Then, the utility of an action sequence $a_t,a_{t+1},\cdots$ can be defined as the sum of all rewards $r_t+r_{t+1}+\cdots$ (cumulative rewards). To limit this sum to be finite, we introduce a discount factor $\gamma\in[0,1)$. So, instead of observing $r_{t+k}$ ($k\geq 0$), the robot observes $\gamma^{k}r_{t+k}$. The utility is then $\gamma^0r_t+\gamma^1r_{t+1}+\cdots$, the discounted cumulative reward.\footnote{Or, synonymously, the discounted return.}

This utility establishes the ground for rational choice over action sequences by the robot. As a consequence, following the utility theorem by \citet{von1947theory}, the robot shall behave as if it is maximizing the expectation of the utility, that is, maximizing the expectation of the discounted cumulative reward conditioned on history $h_t$:
\begin{align}
  \E\left[ \sum_{k=0}^{\infty} \gamma^{k}r_{t+k}\Big| h_t  \right]\label{eq:pomdp:objective}
\end{align}
As mentioned, each reward should be conditioned on the \emph{history} of the robot's experience. Intuitively, however, it is inconvenient to specify the goal based directly on the sequence of actions and observations the robot has experienced. This is because the robot's performance may depend on factors external to the robot, and accounting for all enumerations of history quickly gets out of hand.

Instead, define the notion of a \emph{state}, denoted by $s\in \Sspace$, which shall naturally capture the necessary and potentially external information to specify a goal through a \emph{reward function} $R$ such that $r_t=R(s_t,a_t)$.
For example, if the robot's goal is to clean a room, then whether there still exists a dirty spot in the entire room could be a useful piece of information to include in the state, even though the robot does not observe that information directly. Clearly, the robot's action may cause the state to change.

Since the robot does not observe the state (partial observability), the robot can at most maintain a distribution over states given what it knows, that is, the history. This distribution is called the \emph{belief state}, denoted $b_t(s)=\Pr(s|h_t)\equiv\Pr(s|b_t)$, where $b_t$ acts as a sufficient statistic for history $h_t$. The robot begins with an initial belief $b_1(s)=\Pr(s)$ \label{math:init_belief} based on prior knowledge, as the corresponding history $h_1$ is empty.

Now, define $V(b_t)$ as the expected utility of the optimal behavior at time $t$:
\begin{align}
  V(b_t)&=\max_{a_t,a_{t+1},\cdots\in \Aspace}\E\left[\sum_{k=0}^{\infty} \gamma^{k} r_{t+k} \Big| b_t\right]\label{eq:pomdp:value_def}
\end{align}
Note that the expectation in Equation~\ref{eq:pomdp:value_def} is a variant of that in Equation~\ref{eq:pomdp:objective} using the fact that $b_t$ is a sufficient statistic for $h_t$. We call $V$ the \emph{value function} and the value at $b_t$ means the maximum expected utility at $b_t$. From the definition of $V$ in Equation~\ref{eq:pomdp:value_def}, we can derive the following expression, which turns out to be the Bellman equation for POMDPs:\footnote{In general, $\Sspace$, $\Aspace$ and $\Omega$ can be continuous.}
\begin{align}
    V(b_t)=\max_{a_t\in \Aspace}\left\{ \sum_{s\in \Sspace}b_t(s)R(s,a_t) + \gamma\sum_{z\in\Zspace}\Pr(z|b_t,a_t)V(b_{t+1}^{a_t,z}) \right\}\label{eq:pomdp:bellman}
\end{align}
where $b_{t+1}^{a_t,z}(s)=\Pr(s_{t+1}|b_t,a_t,z)$ is the result of recursive Bayesian state estimation (\ie, belief update) based on $b_t$ after taking action $a_t$ and receiving observation $z$ at time $t$. I provide a derivation of the Bellman equation in the appendix to this chapter (Section~\ref{appdx:pomdp:derivation}).
Equation~\ref{eq:pomdp:bellman} incorporates into the robot decision making objective two major types of uncertainty, \textbf{partial observability}, through the belief state $b_t$, and \textbf{perceptual uncertainty}, through the probability distribution $\Pr(z|b_t,a_t)$.

 It is, however, difficult to define $\Pr(z|b_t,a_t)$ directly. Instead, we can do the following expansion to derive the components that depend on states which are easier to specify:
\begin{align}
    \Pr(z|b_t,a)=\sum_{s\in \Sspace}b_t(s)\sum_{s'\in \Sspace}\Pr(s'|s,a)\Pr(z|s',a_t)
\end{align}
Define $O(s', a_t,z)=\Pr(z|s',a_t)$ and call that the \textit{observation model}. Then define $T(s_t,a,s')=\Pr(s'|s,a)$ and call that the\textit{ transition model}, which exhibits the Markov property. It so happens that a POMDP is formally defined as a tuple $\langle S, A, \Zspace, T, O, R, \gamma \rangle$.  This framework therefore encapsulates through first principles what a robot must necessarily consider to accomplish tasks during its lifetime.

Next, we provide a typical introduction of a POMDP for a quick review.

\subsection{Formal Definition of POMDP}
\label{sec:pomdp:formal}

A POMDP models a sequential decision making problem where the environment state
is not fully observable by the agent. It is formally defined as a tuple
$\langle\Sspace,\Aspace,\Zspace,T,O,R,\gamma\rangle$, where
$\Sspace,\Aspace,\Zspace$ denote the state, action and observation spaces, and the functions $T(s,a,s')=\Pr(s'|s,a)$, $O(s',a,z)=\Pr(z|s',a)$, and
$R(s,a)\in\mathbb{R}$ denote the transition, observation, and reward models. The
agent takes an action $a\in\mathcal{A}$ that causes the environment state to transition from
$s\in\mathcal{S}$ to $s'\in\mathcal{S}$. The environment in turn returns the agent an observation $z\in\Zspace$ and
reward $r\in\mathbb{R}$. A \emph{history} $h_t=(az)_{1:t-1}$ captures all past actions and
observations. The agent maintains a distribution over states given current
history $b_t(s)=\Pr(s|h_t)$.
The agent updates its belief after taking an action and receiving an observation by recursive Bayesian state estimation:
\begin{align}
  b_{t+1}(s')=\eta\Pr(z|s',a)\sum_{s\in\Sspace}\Pr(s'|s,a)b_t(s)
  \label{eq:pomdp:belief_update}
\end{align}
where
$\eta={\sum_s\sum_{s'}\Pr(z|s',a)\Pr(s'|s,a)b_t(s)}$
is the normalizing constant. The solution to a POMDP is a \emph{policy}
$\pi$ that maps a belief state or the corresponding history to an action.  The \emph{value} of a POMDP at a belief under policy
$\pi$ is the expected discounted cumulative reward following that policy:
\begin{align}
  V_{\pi}(b_t)=\E\left[\sum_{k=0}^\infty \gamma^kR(s_{t+k},\pi(b_{t+k}))\Big| b_t\right]
\label{eq:pomdp:bellman_policy}
\end{align}
where $\gamma\in [0,1)$ is the discount factor. The optimal value at belief $b_t$ is
$V(b_t)=\max_{\pi}V_\pi(b_t)$.
Equation~\ref{eq:pomdp:bellman_policy} can also be written equivalently as
$V_{\pi}(h_t)=\E[\sum_{k=0}^\infty \gamma^kR(s_{t+k},\pi(h_{t+k})) | h_t]$

\subsection{Object-Oriented POMDP}

An Object-Oriented POMDP (OO-POMDP) \cite{wandzel2019oopomdp} (generalization of OO-MDP
  \cite{diuk2008object}) is a POMDP that considers the state
  and observation spaces to be factored by a set of $n$ objects, $\Sspace = \Sspace_1\times \cdots \times \Sspace_n$, $\Zspace = \Zspace_1\times\cdots\times \Zspace_n$, where each object belongs to a class with a set of attributes. A simplifying assumption is made for the 2D multi-object search domain \cite{wandzel2019oopomdp} that objects are independent so that the belief space scales linearly rather than exponentially in the number of objects: $b_t(s)=\prod_ib_t^i(s_i)$. We make this assumption for the same computational reason when we formulate multi-object search in 3D (Chapter~\ref{ch:3dmos}).

\subsection{Obtaining Policies to POMDPs}
\label{sec:bg:pomdp_policies}

General-purpose algorithms for POMDP planning can be grouped into offline algorithms and online algorithms. Data-driven approaches (\eg, end-to-end deep reinforcement learning) can be applied to learn POMDP policies for domains where POMDP models are hard to define. Here, we provide an overview for the ideas behind major POMDP planning algorithms.\footnote{We remarked on learning-based approaches for POMDP in our taxonomy of object search methods (Section~\ref{sec:bg:taxo_methods}, page~\pageref{sec:bg:taxo_methods}); Refer to~\citet{arulkumaran2017deep} and \citet{mousavi2016deep} for more discussions.}
We spend more effort on tree search-based online algorithms as we use them to produce object search policies, due to  scalability to large domains, asymptotic optimality and ease of implementation. For a more detailed review, please refer to \citet{lauri2022partially}.

\subsubsection{Offline Planning with Exact and Point-based Methods}
Exact methods for solving POMDP include linear programming \cite{smallwood1973optimal} and value iteration \cite{cassandra1994acting}. The basis of exact methods is that the value function of a POMDP is piecewise linear and convex. To elaborate, the value function under policy $\pi$ as in Equation~\ref{eq:pomdp:bellman_policy} can be written as a dot product:
\begin{align}
  V_{\pi}(b_t) = b_t\cdot\alpha_{\pi}
\end{align}
where $b_t\cdot\alpha_{\pi}=\sum_{s\in\Sspace}b_t(s)\alpha(s)$ and $\alpha_{\pi}$ is called an $\alpha$-vector \cite{smallwood1973optimal}. Each element $\alpha_{\pi}(s)$ equals to the expected discounted cumulative reward for state trajectories $sa_1o_1a_2o_2\cdots$ starting at $s$, with actions from $\pi$ and observations according to $O$:\footnote{The idea is that the policy $\pi$ can be thought of as encoding a set of trees, each rooting at some belief $b$ and a path from the root $ba_1o_1a_2o_2\cdots$ represents a belief trajectory with actions from $\pi$ and observations from $\Zspace$. Given a state $s$ sampled from $b$, a trajectory $sa_1o_1a_2o_2\cdots$ can be defined the same way.}
\begin{align}
  \alpha_{\pi}(s)=R(s,a_i)+\gamma \sum_{s'\in\Sspace}T(s,a_i,s')\sum_{z_i\in\Zspace}O(s',a_i,z_i)\alpha_{\pi}(s')
\end{align}
Geometrically, $V_{\pi}$ is a hyperplane over the belief space, and the optimal value function can be viewed as the upper envelope of the hyperplanes corresponding to all $\alpha$-vectors, which is piecewise linear and convex (refer to \citet{kaelbling1998planning} for an illustration).

Exact methods compute all $\alpha$-vectors that form the optimal value function. They often produce policies that exhibit interesting behavior but are often too slow to be practical for large domains
\cite{ross2008online}. In fact, solving POMDPs exactly is likely unwise for robotics problems since it is PSPACE-complete \cite{papadimitriou1987complexity}, and undecidable whether a desirable solution exists \cite{madani1999undecidability}.

Point-based methods \cite{pineau2003point,spaan2005perseus,kurniawati2008sarsop} take a different approach. Instead of computing the optimal value function over the entire belief space, a set of belief states (\ie, points) are selected, and one $\alpha$-vector is maintained per belief state. The set of $\alpha$-vectors then approximates the optimal value function. Point-based methods allow computing policies offline with tunable approximation and are of sustained interest in robotics \cite{lauri2022partially}, yet their effectiveness is limited to small domains due to the intractability in belief update (Equation~\ref{eq:pomdp:belief_update}) for larger domains.

\subsubsection{Online Planning with Sparse Tree Search-based Methods}
In contrast to offline planning which aims to compute or approximate the optimal policy given any belief state, online planning interleaves planning and execution, and cares about outputting the action to be executed by the robot given its current belief state $b_t$.

The idea behind sparse tree search-based algorithms for online POMDP planning is simple and appealing (Figure~\ref{fig:mcts_pomdp}). Online POMDP planning can be thought of as computing the Q-value $Q(b_t,a)$ for any $a\in\Aspace$, \ie, the expected return after taking action $a$ at belief state $b_t$ over all possible future belief trajectories. These trajectories form a \emph{belief tree}, illustrated in Figure~\ref{fig:mcts_pomdp}. This tree can get wide and deep for large domains with long horizons. Computing Q-values exactly becomes intractable as a result. One idea to work around this is to approximate the full tree with a subtree, which represents an estimation of the Q-values. This is a general and intuitive idea, and it is exactly
what sparse tree search algorithms do: A set of belief-based samples $\{s\sim b_t\}$ individually travels down the subtree to ``experience the domain'' and expands the subtree in the process.
The difference between specific algorithms boils down to two key issues regarding building the subtree: (1) how a sample obtains its ``experience'' (\ie, an action-observation sequence) and (2) how the rewards observed by a sample are incorporated into the estimation of value. Currently, the two state-of-the-art sparse tree search-based algorithms are POUCT \cite{silver2010monte} and DESPOT \cite{somani2013despot,ye2017despot}. We briefly summarize how each algorithm deals with the two issues above, which reflects their differences.

\begin{figure*}[t]
  \centering
  \includegraphics[width=\textwidth, draft=false]{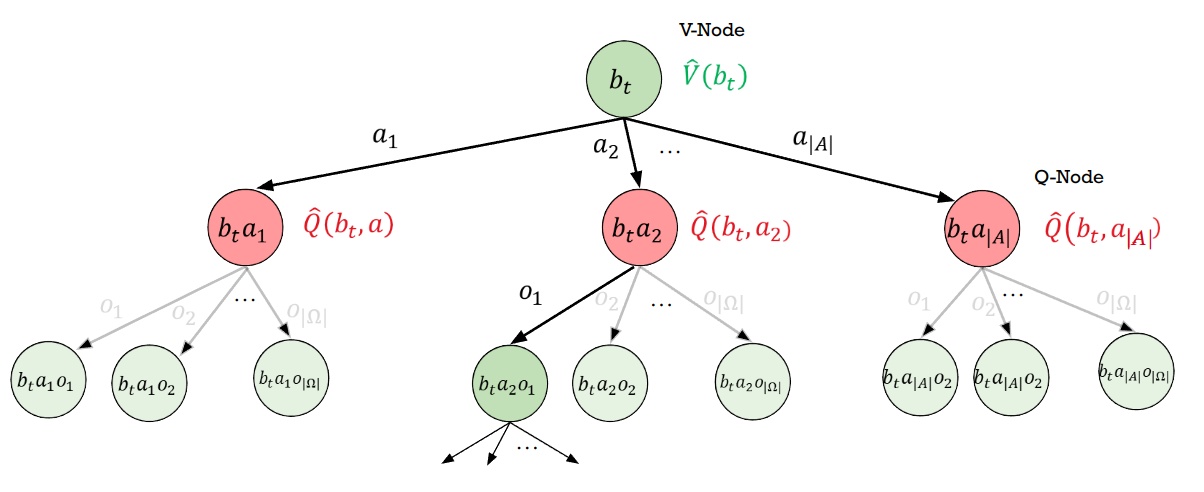}
  \caption{Computing the exact value over the full belief tree (black $+$ gray) is infeasible for practical domains. Instead, approximating the value based on a subtree of the full belief tree (black) is an intuitive and promising approach. This is the idea behind sparse tree search-based online POMDP planning algorithms.}
  \label{fig:mcts_pomdp}
\end{figure*}

\label{sec:bg:pomdp:pouct}
POUCT (Partially-Observable UCT) is based on Monte Carlo Tree Search (MCTS), which is a general sample-based sequential decision making algorithm with a track record of breaking state-of-the-art results in game playing, most notably the difficult game of Go \cite{browne2012survey,silver2017mastering}. POUCT is an extension of UCT (Upper Confidence bounds for Trees) \cite{kocsis2006bandit} to partially observable domains, which is a version of MCTS that uses the UCB1 algorithm \cite{auer2002finite} for action selection. Effectively, POUCT addresses (1) with UCB1 and generator sampling, that is, $a\gets\argmax_{a}Q(b,a)+c\sqrt{\frac{N(b)}{N(ba)}}$, and $(s',z,r)\sim\mathcal{G}(s,a)$; The latter ensures sparse observation branching. It addresses (2) with the update rule $\hat{Q}(b,a)\gets\hat{Q}(b,a) + \frac{R-\hat{Q}(b,a)}{N(ba)}$ that so that $\hat{Q}$ is the average of the observed returns so far.\footnote{Here, $R$ is the observed return and $N(b)$ denotes the visitation count to the node for belief $b$.} Steps for POUCT are shown in Algorithm~\ref{alg:pouct} (Section~\ref{sec:apdx_pouct}). Note that POMCP is a specific version of POUCT with particle-based belief representation. We use POUCT for its generality and our work does not rely on particle beliefs.

DESPOT (Determinized Sparse Partially Observable Trees) is both known as a planning algorithm as well as the name for the belief subtree built by this algorithm. DESPOT addresses issue (1) with action selection $a\gets\argmax_{a}U(b,a)$ based on an upper bound estimate $U(b,a)$ of the Q-value $Q(b,a)$, and observation selection based on a pre-determined subset of observations; This subset is deterministic under a set of $K$ pre-specified \emph{scenarios} (\ie, random seeds), which leads to sparse observation branching ahead of time, instead of during the tree search as done by POUCT. DESPOT addresses (2) by maintaining an upper and lower bound on $Q(b,a)$; Refer to \citet{somani2013despot} for details on the lower bound. 

POUCT is asymptotically optimal as the number of samples approaches infinity, while DESPOT outputs a near-optimal policy if $K$ is large enough. Subsequent works have extended both algorithms to handle continuous domains \cite{sunberg2018online,garg2019despot}. Empirically, one is not shown to be strictly better than the other \cite{sunberg2018online}, and both are affected by the rollout policy (\aka, default policy) used.  We build upon POUCT for planning due to its optimality and simplicity of implementation, although our POMDP formulations are general.

\newpage
\subsection{Appendix: Derivation of the POMDP Bellman Equation}
\label{appdx:pomdp:derivation}

Below, I provide a derivation of the POMDP Bellman Equation in Equation~\ref{eq:pomdp:bellman}.
Start by defining $U(b_t)$ to be the utility of some arbitrary action sequence given $b_t$,
\begin{align}
  U(b_t)&=\E\left[\sum_{k=0}^{\infty} \gamma^{k} r_{t+k} \Big| b_t\right]\label{eq:pomdp:util_def}
    \intertext{Take out the first term of the summation,}
    &=\E\left[r_t + \sum_{k=1}^\infty \gamma^{k} r_{t+k} \Big| b_t\right]\\
    \intertext{By linearity of the expectation,}
    &= \E\left[r_t | b_t\right] + \E\left[\sum_{k=1}^\infty \gamma^{k}
    r_{t+k} \Big| b_t\right]\\
    \intertext{Pull out $\gamma$ and rewriting the summation index,}
    &= \E\left[r_t | b_t\right] + \gamma\E\left[\sum_{k=0}^\infty \gamma^{k} r_{t+1+k} \Big| b_t\right]\label{eq:pomdp:pull_gamma}\\
  \intertext{Recall the Tower property of conditional expectation: $\E[X|Y] = \E[\E[X|Y,Z] | Y]$. In this case, ``$X$'' is the summation, ``$Y$'' is $b_t$, and ``$Z$'' is $b_{t+1}$.
  The utility starting at time $t+1$ (summation) is conditionally independent of $b_t$ given $b_{t+1}$,
  since the history corresponding to $b_{t+1}$ subsumes $b_t$. Therefore, we drop $b_t$ in the inner conditional expectation: }
    &=\E\left[r_t | b_t\right] + \gamma\E\left[ \E\left[\sum_{k=0}^\infty \gamma^{k} r_{t+1+k} \Big| b_{t+1}\right] \Bigg| b_t\right]\\
    \intertext{Using the definition of $U$,}
        &=\E\left[r_t | b_t\right] + \gamma\E\left[ U(b_{t+1}) \big| b_t\right]
\end{align}
Recall the definition of $V(b_t)$ in Equation~\ref{eq:pomdp:value_def}. We have:
\begin{align}
  V(b_t)&=\max_{a_t,a_{t+1},\cdots\in A}  U(b_t)\\
        &=\max_{a_t,a_{t+1},\cdots\in A}\left\{\E\left[r_t | b_t\right] + \gamma\E\left[ U(b_{t+1}) \big| b_t\right] \right\}\\
        &=\max_{a_t}\left\{\E\left[r_t | b_t\right] +\max_{a_{t+1},a_{t+2}\cdots\in A} \gamma\E\left[ U(b_{t+1}) \big| b_t\right]\right\}\\
        &=\max_{a_t}\left\{\sum_{s\in S}b_t(s)R(s,a_t)  + \max_{a_{t+1},a_{t+2}\cdots\in A}  \gamma \sum_{z\in \Zspace}\Pr(z|b_t,a_t)U{(b^{a_t,z}_{t+1})}\right\}\\
        &=\max_{a_t}\left\{\sum_{s\in S}b_t(s)R(s,a_t)  + \gamma \sum_{z\in \Zspace}\Pr(z|b_t,a_t)\max_{a_{t+1},a_{t+2}\cdots\in A}  U{(b^{a_t,z}_{t+1})}\right\}\\
  &=\max_{a_t}\left\{\sum_{s\in S}b_t(s)R(s,a_t)  + \gamma \sum_{z\in \Zspace}\Pr(z|b_t,a_t)V{(b^{a_t,z}_{t+1})}\right\}
\end{align}

\subsection{Appendix: The POUCT Algorithm}
\label{sec:apdx_pouct}
In Algorithm~\ref{alg:pouct}, we provide the pseudocode for the POUCT (Partially Observable UCT) \cite{silver2010monte} algorithm. This is the same algorithm as POMCP as presented in \citet{silver2010monte} without particle belief representation. We slightly modified the notation to match the that of Algorithm~\ref{alg:planning} (Chapter~\ref{ch:3dmos}, page~\pageref{alg:planning}).

Additional notations:

\begin{tabular}{l@{\hspace{1cm}}l}
      $\Pr(s|h)$          & belief state corresponding to history $h$ \\
      $\mathcal{G}$          & black-box generator \\
      $d$  & maximum tree depth (planning depth)\\
      $T$  & belief tree\\
      $V$  & value estimate\\
      $N$  & visitation count\\
      $R$  & observed discounted return\\
      $\gamma$  & discount factor\\
      $\pi_{rollout}$  & rollout policy\\
      $hao$  & a history from $h$ following $a$ and $o$ \\
\end{tabular}

\begin{algorithm}[h]
\caption{Partially Observable UCT $(\mathcal{G},h,d)\rightarrow a$}
\label{alg:pouct}
\SetKwFunction{search}{Search}
\SetKwFunction{simulate}{Simulate}
\SetKwFunction{rollout}{Rollout}
\SetKwProg{myproc}{procedure}{}{}
\myproc{\search{$h$}}{
\tcp{Entry function of POUCT}
\Repeat{\textsc{Timeout()}}{
    $s\sim \Pr(s|h)$\tcp*{$\Pr(s|h)$ is the belief state}
    Simulate($s,h,0$)\;
}
\KwRet{$\argmax_a V(ha)$}\tcp*{$V(ha)$ is the Q-value of action $a$}
}
\BlankLine
\BlankLine

\myproc{\simulate($s,h,depth$)}{
    \If{$depth > d$}{
        \KwRet{0}
    }
    \If{$h\not\in T$}{
        \ForEach{$a\in \mathcal{A}$}{
            $T(ha)\leftarrow (N_{init}(ha), V_{init}(ha))$\;
        }
        \KwRet{\rollout($s,h,depth$)}
    }
    $a\gets \argmax_{a} V(ha) + c\sqrt{\frac{\log N(h)}{N(ha)}}$\;
    $(s',o,r)\sim \mathcal{G}(s,a)$\;
    $R\gets r+\gamma\cdot \text{Simulate}(s',hao,depth+1)$\tcp*{$R$ is the discounted return}
    $N(h)\gets N(h)+1$\;
    $N(ha)\gets N(ha)+1$\;
    $V(ha)\gets V(ha) + \frac{R-V(ha)}{N(ha)}$\;
    \KwRet{R}
}
\BlankLine
\BlankLine
\myproc{\rollout($s,h,depth$)}{
    \If{$depth > H$}{
        \KwRet{0}
    }
    $a\gets \pi_{rollout}(h,\cdot)$\;
    $(s',o,r)\sim \mathcal{G}(s,a)$\;
    \KwRet{$r+\gamma\cdot \rollout(s',hao,depth+1)$}\;
}
\end{algorithm}

\chapter{Overarching Methodology}
\label{ch:overarching}
\lettrine{A}{t} its core, this thesis argues for modeling object search as a POMDP while exploiting structures for practicality. The idea is for the model to take advantage of structures in human environments and human-robot interaction, while being independent of any specific robot or environment. Then, a system or package that implements a solution strategy to this POMDP can be general across, and thus integrable with, different robots and environments.
What does this POMDP look like, and how do we approach ``solving'' this POMDP? The goal of this chapter is to address these questions.

\section{A Generic POMDP Model for Object Search}
First, what does this POMDP looks like? Here, let us consider the basic problem setting as motivated in Figure~\ref{fig:spot_challenges} (Chapter~\ref{ch:intro}, pp.~\pageref{sec:sigch:challenges:begin}-\pageref{sec:sigch:challenges:end}),
where a robot with a movable camera searches for a single static target object. I describe a sufficient and generic formulation of this POMDP as a starting point.\footnote{Refer to Section~\ref{sec:bg:pomdp} for an introduction of POMDPs.} This POMDP model can be considered the ``parent'' of the variants in later chapters, which tackle specific problem settings (\eg, multi-object search, searching in 3D, correlations, spatial language, etc.). This draws parallels with the ``parent problem'' of object search variants in Background (Section~\ref{sec:bg:elem}, page \pageref{sec:bg:elem}). I briefly discuss how this model can be extended
to address the additional challenges in Chapter~\ref{ch:closing}.

The formulation of this POMDP is as follows:

\begin{itemize}[itemsep=0.5pt,topsep=0pt]
\item \textbf{State space $\Sspace$.} A state $s\in\Sspace$, $s=(\srobot, \starget)$ is factored into the robot state $\srobot$ and the target state $\starget$. The robot state contains the robot pose (position and orientation of the camera). The target state contains the location of the target object.

\item \textbf{Action space $\Aspace$.} Generically, we consider just two action types \textsc{Look} and \textsc{Find}; The purpose of \textsc{Look} is to perceive some part of the search region, and the purpose of \textsc{Find} is to declare object(s) to be found at some location. When camera is used, \textsc{Look} corresponds to changing the position and orientation of the camera, either relative to the current pose or to some goal pose, and then project a field of view to receive observations. Note that these actions are abstract and can be broken down into finer-grained, parameterized action types as needed, as done, for example, in the 3D Multi-Object Search (3D-MOS) model (Section~\ref{sec:3dmos:action_space}, page~\pageref{sec:3dmos:action_space}).

\item \textbf{Observation space $\Ospace$.} An observation $z=(\zrobot,\ztarget)$ is factored into the observation of the robot itself $\zrobot$ and the observation of the target object $\ztarget$. For object search, $\ztarget$ is typically the detected location of the object (though in 3D-MOS we consider it to be a set of voxels in the field of view; see Section~\ref{sec:3dmos:observation_space}, page~\pageref{sec:3dmos:observation_space}) and $\zrobot$ should be an estimation of the robot state, for example, an estimated robot pose.

\item \textbf{Transition model $T(s,a,s')$.}  When the robot takes a \textsc{Look} action, the robot should change its pose to the desired destination with some domain-specific noise. When the robot takes \textsc{Find}, the target should be marked as found if the condition for success is satisfied, such as when the target object is within the field of view of the robot, which should be determined based on the state $s$. The target is assumed to be static.

\item \textbf{Observation model $O(s',a,z)$.} The robot observation is an estimate of the robot state. So $\zrobot$ is independent of the target observation given the robot state. Therefore, the observation model $O(s',a,z)$ can be factored as $O(s',a,z)=\Pr(z|s',a)=\Pr(\zrobot|\srobot',a)\Pr(\ztarget|s',a)$. Here, $\Pr(\zrobot|\srobot',a)$ can be regarded as a model of the robot's localization module, and $\Pr(\ztarget|s',a)$ models the object detection mechanism (\eg~through a field of view) and the uncertainty of object detection.

\item \textbf{Reward function $R(s,a)$.} The reward of a \textsc{Look} action should depend on the robot state and the viewpoint changing action, which should reflect the cost needed to complete the \textsc{Look} action, such as time or travel distance.  Taking \textsc{Find} signals a commitment by the robot its belief of the object location. If correct, then the robot successfully completes the task, receiving a high reward $\rmax$. However, the correctness of declarations can be expensive to verify, for example, by a human teammate. Therefore, taking \textsc{Find} is also significantly more costly than \textsc{Look} actions, receiving a high penalty $\rmin$.
\end{itemize}

\subsection{Remark}

This model might seem simple, but it is not obvious.  Since POMDP is such a general model for robot behavior, it may be tempting to pack too much or too little information into the model.  The key is to determine the right level of abstraction this POMDP model should live at in a practical robot system.

Our insight is that the perception and action of an object search system should happen at a level higher than the basic building blocks, such as localization, object detection, navigation or low-level control. It should also be at a lower level of abstraction than a high-level task such as \emph{``pick two apples and then heat them''} (example from Figure~\ref{fig:jacob_andreas_slide}). The model above echos that vision.  This somewhat ``middle-level'' of abstraction of this object search model also coincides with the duality of object search's role in people's mind, as discussed in Chapter~\ref{ch:intro} (page \pageref{ch:intro:unique_problem}).

\subsection{Solution Method}
This thesis takes the explicit, online planning approach to obtaining policies for POMDPs. Concretely, this means to explicitly maintain the belief state and model the POMDP's components, such as programmatically define observation models based on analytical functions, and then apply general-purpose online POMDP planning algorithms based on Monte Carlo Tree Search to obtain an approximately optimal policy.  To improve performance and practicality, I develop and employ techniques such as multi-resolution planning, hierarchical planning, view position graph sampling, and using heuristic rollout policies depending on the specific problem setting, discussed in the following Chapters (Chapters \ref{ch:3dmos}-\ref{ch:sloop}).

\vspace{2in}
\begin{center}
  THIS IS THE END OF THIS CHAPTER.
\end{center}


\chapter{3D Multi-Object Search}
\label{ch:3dmos}
\begin{figure}[t]
\centering
\includegraphics[width=\linewidth]{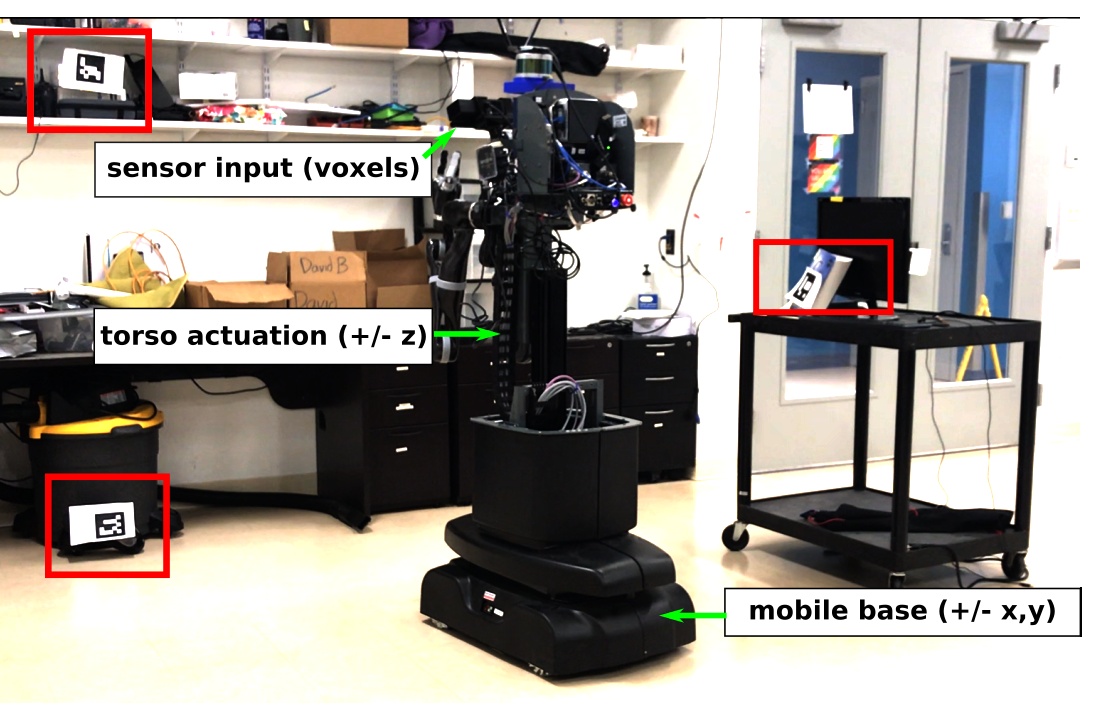}
\caption{An example of the 3D-MOS problem where a torso-actuated mobile robot is
  tasked to search for three objects placed at different heights in a lab
  environment. The objects are represented by paper AR tags marked by red
  boxes. Note that the robot must actively move itself due to limited field of
  view, and the objects can be occluded by the attached obstacles if viewed from
  the side.}
\label{fig:mos3d}
\end{figure}

\section{Motivation - Why 3D Object Search?}
\lettrine{R}{obots} operating in human spaces must find objects such as glasses, books, or cleaning supplies that could be on the floor, shelves, or tables. This search space is naturally 3D. When multiple objects must be searched for, such as a cup and a mobile phone, an intuitive strategy is to first hypothesize likely search regions for each target object based on semantic knowledge or past experience~\cite{kollar2009utilizing,aydemir2013active}, then search carefully within those regions. Since the latter directly determines the success of the search, it is \emph{essential} for the robot to produce an efficient search policy within a designated 3D search region under limited field of view (FOV), where target objects could be partially or completely occluded. In this chapter and next, we consider the problem setting where a robot must search for multiple objects in a 3D search region by actively moving its camera, with as few steps as possible (Figure~\ref{fig:mos3d}).

This chapter begins by articulating both algorithmic and system-level challenges for 3D multi-object search, followed by remark on previous work and a summary of our contributions. From then on, the chapter focuses on the theoretical side of this problem and tackling algorithm-level challenges. The next chapter (Chapter~\ref{ch:genmos}) tackles system-level challenges by presenting a system for generalized 3D multi-object search deployed on different robots.



\subsection{Algorithm-Level Challenges}
Searching for objects in a large search region requires acting over long horizons under various sources of uncertainty in a partially observable environment. For this reason, previous works have used Partially Observable Markov Decision Process (POMDP) as a principled decision-theoretic framework for object search \cite{xiao_icra_2019,atanasov2014nonmyopic,danielczuk2019mechanical}.  However, to ensure the POMDP is manageable to solve, previous works reduce the search space or robot mobility to 2D \cite{aydemir2013active,wandzel2019multi,li2016act}, although objects exist in rich 3D environments. The key challenges lie in the intractability of maintaining exact belief due to large state space \cite{silver2010monte}, and the high branching factor for planning due to large observation space \cite{sunberg2018pomcpow,garg2019despot}.

\subsection{System-Level Challenges}
As such a valuable and fundamental skill for robots, we expect that eventually object search becomes an \emph{off-the-shelf} ability any robot can acquire to search for objects in the environment that it operates in, similar to other capabilities such as object detection, SLAM, and motion planning. However, unlike the other aforementioned robotic capabilities, to the best of our knowledge, there is no general-purpose object search package available for robotics researchers and practitioners. Sophisticated mobile robot platforms, such as the Kinova MOVO \cite{movo} and the Boston Dynamics Spot \cite{bdspot}, do not come equipped with an object search system, despite their otherwise impressive capabilities. This thesis takes the first step towards a robot-independent, environment-agnostic system and package for generalized object search.

\subsection{Remark on Previous Work}
Searching for a single, static object in 3D by planning sensing parameters (\eg position, orientation, and zoom of a camera) under a time budget is NP-complete \citep{Ye1997SensorPF}.\footnote{See Section~\ref{sec:bg:taxo_methods} page~\pageref{sec:bg:taxo_methods:greedy} for elaboration.}
Previous work primarily address the computational complexity of object search by
hypothesizing likely regions based on object co-occurrence \cite{kollar2009utilizing,wixson1994using}, semantic knowledge \cite{aydemir2013active} or language \cite{wandzel2019multi}, reducing the state space from 3D to 2D \cite{wandzel2019multi,wang2018efficient,sarmiento2003efficient,nie2016searching}, or constrain the sensor to be stationary \cite{danielczuk2019mechanical,dogar2014object}. The work in this chapter focuses on multi-object search within a 3D region where the robot actively moves the mounted camera, for example, through pan or tilt, or by moving itself.

Several works explicitly reason over the arrangement of occluded objects based on given geometry models of clutter \cite{xiao_icra_2019,nie2016searching,wong2013manipulation}. Our approach considers occlusion as part of the observation that contains no information about target locations and we do not require geometry models.

Many works formulate object search as a POMDP. Notably,
\citet{aydemir2013active} first infer a room to search in then perform search by
calculating candidate viewpoints in a 2D plane.  \citet{li2016act} plan sensor
movements online, yet assume objects are placed at the same surface level in a
container with partial occlusion. \citet{xiao_icra_2019} address object fetching
on a cluttered tabletop where the robot's FOV fully covers the scene, and that
occluding obstacles are removed permanently during
search. \citet{wandzel2019multi} formulate the multi-object search (MOS) task
on a 2D map using the proposed Object-Oriented POMDP (OO-POMDP). We extend that
work to 3D and tackle additional challenges by proposing a new observation model
and belief representation, and a multi-resolution planning algorithm. In
addition, our POMDP formulation allows fully occluded objects and can be in
principle applied on different robots such as mobile robots or drones.

Beginning with CARMEN~\citep{montemerlo2003perspectives}, open source libraries for
SLAM have greatly lowered the barrier to entry into
robotics~\citep{grisetti2007improved, hess2016real}.  Similarly, for
motion planning, libraries such as
OMPL~\citep{sucan2012the-open-motion-planning-library} and
MoveIt!~\citep{chitta2016moveit} have broadened access to motion
planning to a variety of different robotic platforms.  Our work aims
to do the same thing for object search.

\citet{wixson1994using} remarked that selecting views for object search in a local region is a harder problem than the selection of which region to search in. Most works that demonstrate real-world robotic search are constrained within a 2D search region or reduce some aspect of the problem (\eg, the observation or action space) \cite{aydemir2013active,wandzel2019multi,li2016act,zeng2020semantic,bejjani2021occlusion,holzherr2021efficient,giuliari2021pomp++,schmalstieg2022learning}.

Deep learning methods that typically map raw observations to actions \cite{yang2018visual,chaplot2020object,mayo2021visual,deitke2022procthor,schmalstieg2022learning} can enable 3D object search, yet it is hard to train such a model on a new robot and ensure generalization to a new real-world environment; ongoing work (\eg, by \citet{deitke2020robothor, schmalstieg2022learning, gervet2022navigating}) is addressing this challenge through sim-to-real transfer.
In contrast, our approach only requires basic perception capabilities such as object detection and localization to enable object search; point cloud observations can be optionally considered by our system to be occlusion-aware.


\section{Contributions}
\label{sec:3dmos:contribs}
The contributions of our work are as follows:

\begin{itemize}[itemsep=0.5pt,topsep=0pt]
  \item We introduce \textbf{3D Multi-Object Search (3D-MOS)}, a general POMDP formulation for the multi-object search task with 3D state and action spaces, and a realistic observation space in the form of labeled voxels within the viewing frustum from a mounted camera. Following the Object-Oriented POMDP (OO-POMDP) framework proposed by \citet{wandzel2019multi}, the state, observation spaces are factored by independent objects, allowing the belief space to scale linearly instead of exponentially in the number of objects.

  \item We address the algorithmic challenges of computational complexity in solving 3D-MOS by developing several techniques that converge to an online multi-resolution planning algorithm:
    \begin{itemize}[itemsep=0.5pt,topsep=0pt, label=$\circ$]
    \item \emph{First,} we propose a per-voxel observation model which drastically
      reduces the size of the observation space necessary for planning.

    \item \emph{Next,} we present a novel belief representation, called
      \textbf{octree belief}, that captures beliefs at different resolutions and
      allows efficient and exact belief updates.

    \item \emph{Then,} we exploit the octree structure and derive abstractions of the
      ground problem at different resolution levels leveraging abstraction
      theory for MDPs \cite{li2006towards,bai2016markovian}.

    \item \emph{Finally,} a Monte-Carlo Tree Search (MCTS) based online planning
      algorithm, called Partially-Observable Upper Confidence bounds for Trees
      (POUCT) \cite{silver2010monte}, is employed to solve these abstract
      instances in parallel, and the action with highest value in its MCTS tree
      is selected for execution.

    \end{itemize}

    We evaluate the proposed algorithm in a simulated, discretized 3D
    domain where a robot with a 6 DOF camera searches for objects of different
    shapes and sizes randomly generated and placed in a grid environment. In
    addition, we demonstrate our approach as a proof-of-concept system on a
    torso-actuated mobile robot in our lab. 

  \item To address system-level challenges, we present \textbf{GenMOS
      (Generalized Multi-Ob-ject Search)}, a general-purpose object search system that is robot-independent and environment-agnostic. GenMOS takes as input point cloud observations of the local region (when available), 3D object detection bounding boxes (if detection occurs), and localization of robot camera pose, and outputs a viewpoint to move to as the result of sequential online planning. The point cloud observations are
    used in three ways: (1) simulate occlusion; (2) inform occupancy and initialize octree belief; and (3) sample a belief-based graph of view positions
    that avoids obstacles.
    \begin{itemize}[itemsep=0.5pt,topsep=0pt, label=$\circ$]
    \item I implemented this system as a software package based on
      gRPC~\cite{grpc}; Besides evaluating it in simulation, I deployed it
      on the Boston Dynamics Spot robot, the Kinova MOVO robot, and the
      Universal Robotics UR5e robotic arm, performing object search in different
      environments.
    \end{itemize}

  \end{itemize}

\section{Formulation of the 3D-MOS POMDP}
\label{sec:mos3d}


The robot is tasked to search for $n$ static target objects (e.g. cup and book) of known type but unknown location in a search space that also contains static non-target obstacles. We assume the robot has access to detectors for the objects that it is searching for. The search region is a 3D grid map environment denoted by $G$. Let $g\in G\subseteq{R}^3$ be a 3D grid cell in the environment. We use $G^l$ to denote a grid at \emph{resolution level} $l\in\mathbb{N}$, and $g^l\in G^l$ to denote a grid cell at this level. When $l$ is omitted, it is assumed that $g$ is at the ground resolution level. We introduce the 3D-MOS domain as an OO-POMDP as follows. See Figure~\ref{fig:3dmospomdp} for illustrations. This model is a multi-object extension and a 3D specialization of the overarching object search POMDP described in Chapter~\ref{ch:overarching}.

\begin{figure}[t]
\centering
\includegraphics[width=\linewidth]{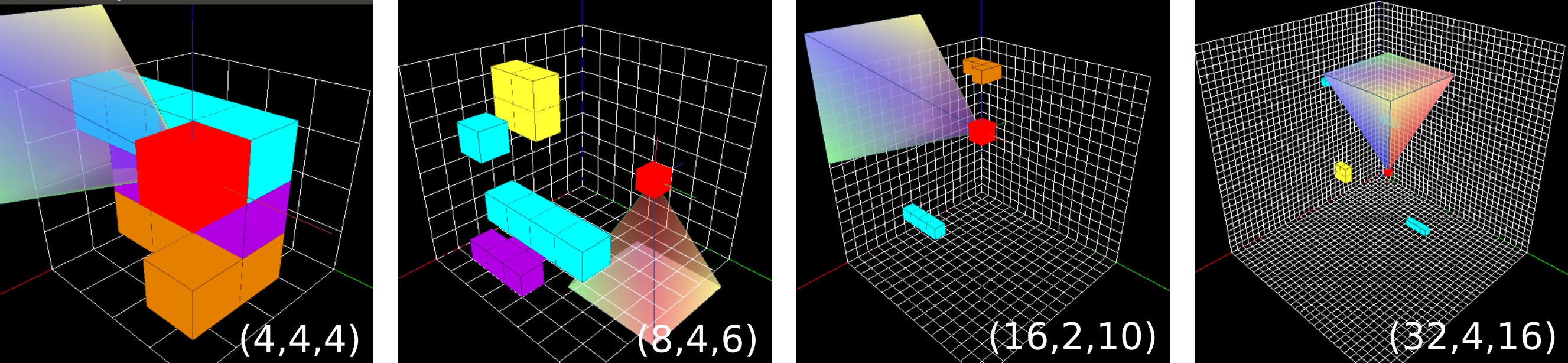}
\caption{Example illustrations of the 3D-MOS POMDP model.  The robot (represented as a red cube) can project a viewing frustum to observe the search space, in which objects are represented by sets of cubes.  In these examples, the tuple $(m,n,d)$ at lower-right of each image means that the search space in total has $m\times m\times m$ grid cells, with $n$ randomly placed objects, and the robot can project a 45-degree frustum with a far plane at distance $d$ grid cells to the robot. The percentage of search space covered by each viewing frustum, parameterized by field-of-view depth $d$, decreases as the world size increases.}
\label{fig:3dmospomdp}
\end{figure}

\subsection{State space $\Sspace$}
An environment state $s=\{s_1,\cdots,s_n,s_r\}$ is factored in an object-oriented way, where $s_r\in\mathcal{S}_r$ is the state of the robot, and $s_i\in\mathcal{S}_i$ is the state of target object $i$. A robot state is defined as $s_r=(p,\mathcal{F})\in\mathcal{S}_r$ where $p$ is the 6D camera pose and $\mathcal{F}$ is the set of found objects. The robot state is assumed to be observable to the robot. In this work, we consider the object state to be specified by one attribute,  the 3D object pose at its center of mass, corresponding to a cell in grid $G$. We denote a state $s_i^l\in\mathcal{S}_i^l$ to be an object state at resolution level $l$, where $\mathcal{S}_i^l=G^l$.

\subsection{Observation space $\Ospace$}
\label{sec:3dmos:observation_space}
The robot receives an observation through a viewing frustum projected from a mounted camera. The viewing frustum forms the FOV of the robot, denoted by $V$, which consists of $|V|$ voxels. Note that the resolution of a voxel should be no lower than that of a 3D grid cell $g$. We assume both resolutions to be the same in this chapter for notational convenience, hence $V\subseteq G$, but in general a voxel with higher
resolution can be easily mapped to a corresponding grid cell.

For each voxel $v\in V$, a \emph{detection function} $d(v)$ labels the voxel to be either an object $i\in\{1,\cdots,n\}$, \textsc{Free}, or \textsc{Unknown} (Figure~\ref{fig:fov}). \textsc{Free} denotes that the voxel is a free space or an obstacle. We include the label $\textsc{Unknown}$ to take into account occlusion incurred by target objects or static obstacles. In this case, the corresponding voxel in $V$ does not give any information about the environment. An observation $o=\{(v,d(v)) | v\in V\}$ is defined as a set of voxel-label tuples. This can be thought of as the result of voxelization and object segmentation given a raw point cloud.

\begin{figure}[t]
\centering
\includegraphics[width=0.95\linewidth]{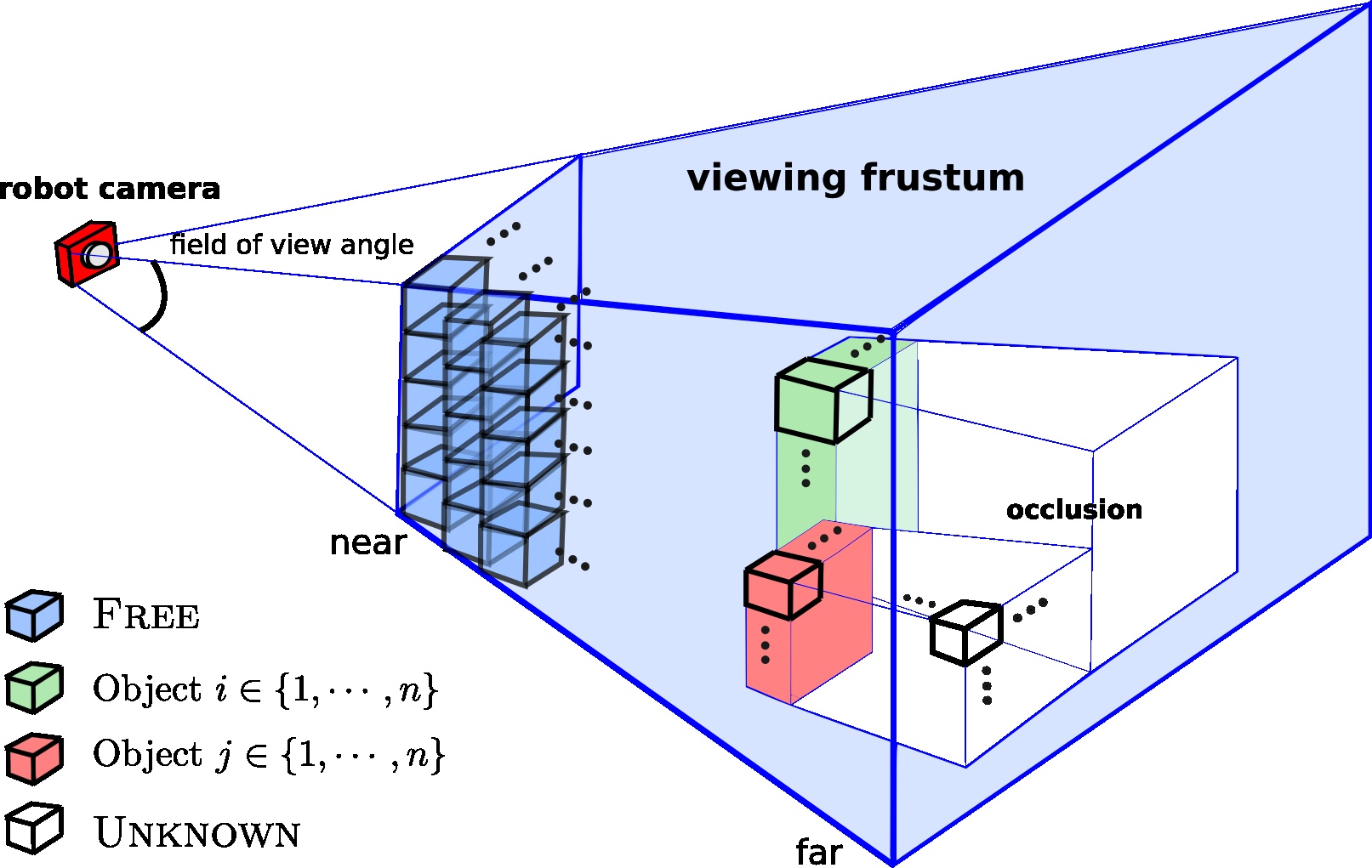}
\caption{Illustration of the viewing frustum and volumetric observation. The
  viewing frustum $V$ consists of $|V|$ voxels, where each $v\in V$ can be
  labeled as $i\in\{1,\cdots,n\}$, \textsc{Free} or \textsc{Unknown}.}
\label{fig:fov}
\vspace{-1em}
\end{figure}



We can factor $o$ by objects in the following way. First, given the robot state $s_r$ at which $o$ is received, the voxels in $V$ have known locations. Under this condition, $V$ can be reduced to exclude voxels labeled \textsc{Unknown} while still maintaining the same information. Then, $V$ can be decomposed by objects into $V_1,\cdots,V_n$ where for any $v\in V_i$, $d(v)\in\{i,\textsc{Free}\}$ which retain the same information as $V$ for a given robot state.\footnote{The FOV $V$ is fixed for a given camera pose in the robot state, therefore excluding \textsc{Unknown} voxels does not lose information.} Hence, the observation $o=\bigcup_{i=1}^n o_i$ where $o_i=\{(v,d(v))|v\in V_i\}$.


\subsection{Action space $\Aspace$.}
\label{sec:3dmos:action_space}
Searching for objects generally requires three basic
capabilities: \emph{moving}, \emph{looking}, and \emph{declaring} an object to be found at some location.
Formally, the action space consists of these three types of primitive actions: $\MOVE(s_r,g)$ action moves the robot from pose in $s_r$ to destination $g\in G$ stochastically. $\LOOK(\theta)$ changes the
camera pose to look in the direction specified by $\theta\in\mathbb{R}^3$, and
projects a viewing frustum $V$. $\DETECT(i,g)$ declares object $i$ to be found
at location $g$. The implementation of these actions may vary depending on the
type of search space or robot. Note that this formulation allows macro actions, such as ``look after move'' to be composed for planning.


\subsection{Transition function $T$.} Target objects and obstacles are static objects, thus $\Pr(s_i'|s,a)=\bm{1}(s_i' = s_i)$. For the robot, the actions \textsc{Move}($s_r$, $g$) and $\LOOK(\theta)$ change the camera location and direction to $g$ and $\theta$ following a domain-specific stochastic dynamics function. The action $\DETECT(i,g)$ adds $i$ to the set of \emph{found objects} in the robot state only if $g$ is within the FOV determined by $s_r$.

\subsection{Reward function $R$.} The correctness of declarations can only be determined by, for example, a human who has knowledge about the target object or additional interactions with the object; therefore, we consider declarations to be expensive. The robot receives $\rmax\gg 0$ if an object is correctly identified by a $\DETECT$
action, otherwise the $\DETECT$ action incurs a $\rmin\ll 0$ penalty. $\MOVE$ and
$\LOOK$ receive a negative step cost $R_{\text{step}}<0$ dependent on the robot state and the action itself. This is a sparse reward function.

\subsection{Observation Model $O$}
\label{sec:observation_model}
We have previously described how a volumetric observation $o$ can be factored by objects into $o_1,\cdots,o_n$. Here, we describe a method to model $\Pr(o_i|s',a)$, the probabilistic distribution over an observation $o_i$ for object $i$.

Modeling a distribution over a 3D volume is a challenging problem \cite{park2019deepsdf}. To develop an efficient model, we make the simplifying assumption that object $i$ is contained within a single voxel located at the grid cell $g=s_i'$. We address the case of searching for objects of unknown sizes with our planning algorithm (Section~\ref{sec:planning_algorithm}) that plans at multiple resolutions in parallel.

Under this assumption, $d(v)=\textsc{Free}$ deterministically for $v\neq s_i'$, and the uncertainty of $o_i$ is reduced to the uncertainty of $d(s_i')$. As a result, $\Pr(o_i|s',a)$ can be simplified to $\Pr(d(s_i)|s',a)$. When $s_i'\not\in V_i$, either $d(s_i')=\textsc{Unknown}$ (occlusion) or $s_i'\not\in V$ (not in FOV). In this case, there is no information regarding the value of $d(s_i')$ in the observation $o_i$, therefore $\Pr(d(s_i')|s',a)$ is a uniform distribution. When $s_i'\in V_i$, that is, the non-occluded region within the FOV covers $s_i'$, the case of $d(s_i')=i$ indicates correct detection while $d(s_i')=\textsc{Free}$ indicates sensing error. We let $\Pr(d(s_i')=i|s',a)=\alpha$ and $\Pr(d(s_i')=\textsc{Free}|s',a)=\beta$. It should be noted that $\alpha$ and $\beta$ do not necessarily sum to one because the belief update equation (Equation~\ref{eq:pomdp:belief_update}) does not require the observation model to be normalized, as explained in~Section \ref{sec:pomdp:formal}. Thus, hyperparameters $\alpha$ and $\beta$ independently control the reliability of the observation model.

\section{Octree Belief Representation}
\label{sec:octree_belief}

\begin{figure}[t]
\centering
\includegraphics[width=0.9\linewidth]{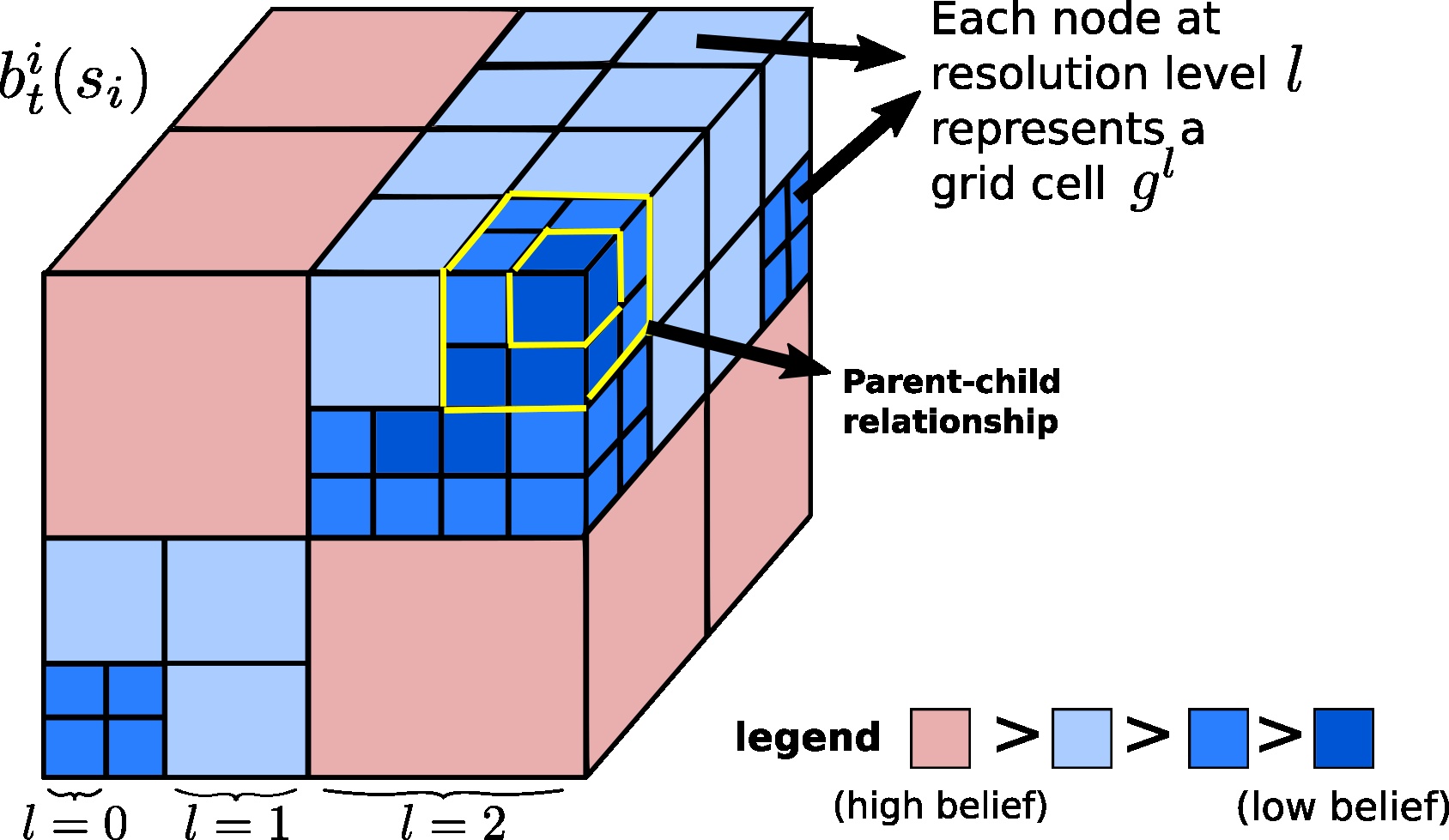}
\caption{Illustration the octree belief representation $b_t^i(s_i)$. The color
  on a node $g^l$ indicates the belief $\Val_t^i(g^l)$ that the object is
  located within $g^l$. The highlighted grid cells indicate parent-child
  relationship between a grid cell at resolution level $l=1$ (parent) and one at
  level $l=0$.}
\label{fig:octree}
\end{figure}

Particle belief representation \cite{silver2010monte,somani2013despot} suffers
from particle depletion under large observation spaces.  Moreover, if the resolution
of $G$ is dense, it may be possible that most of 3D grid cells do not contribute
to the behavior of the robot.

We represent a belief state $b_t^i(s_i)$ for object $i$ as an octree, referred to as an \emph{octree belief}. It can be constructed incrementally as observations are received and it tracks the belief of object state at different resolution levels. Furthermore, it allows efficient belief sampling and belief update using the per-voxel observation model (Sec.~\ref{sec:observation_model}).

An octree belief consists of an octree and a normalizer. An octree is a tree where every node has $8$ children. In our context, a node represents a grid cell $g^l\in G^l$, where $l$ is the resolution level, such that $g^l$ covers a cubic volume of $(2^l)^3$ ground-level grid cells; the ground resolution level is given by $l=0$. The 8 children of the node equally subdivide the volume at $g^l$ into smaller volumes at resolution level $l-1$ (Figure~\ref{fig:octree}). Each node stores a value $\Val_t^i(g^l)\in\mathbb{R}$, which represents the unnormalized belief that $s_i^l=g^l$, that is, object $i$ is located at grid cell $g^l$.  We denote the set of nodes at resolution level $k<l$ that reside in a subtree rooted at $g^l$ by $\Children^{k}(g^l)$. By definition, $b_t^i(g^l)=\Pr(g^l|h_t)=\sum_{c\in\Children^{k}(g^l)}\Pr(c|h_t)$. Thus, with a normalizer $\Norm_t=\sum_{g\in G}\Val_t^i(g)$, we can rewrite the normalized belief as:
\begin{align}
  \label{eq:octbdef}
b_t^i(g^l)=\frac{\Val_t^i(g^l)}{\Norm_t}=\sum_{c\in\Children^{k}(g^l)}\left(\frac{\Val_t^i(c)}{\Norm_t}\right),
\end{align}
which means $\Val_t^i(g^l)=\sum_{c\in\Children^{k}(g^l)}\Val_t^i(c)$. In words, the value stored in a node is the sum of values stored in its children. The normalizer equals to the sum of values stored in the nodes at the ground resolution level.

The octree does not need to be constructed fully in order to query the
probability at any grid cell. This can be achieved by setting a default value
$\Val_0^i(g)=1$ for all ground grid cells $g\in G$ not yet present in
the octree. Then, any node corresponding to $g^l$ has a default value of
$\Val_0^i(g^l)=\sum_{c\in\Children^0(g^l)}\Val_0^i(c)=|\Children^0(g^l)|$.

For clarity (\eg, in Section~\ref{ref:prioroctree}), we provide the definition of \emph{default value} and \emph{initial value} in an octree belief node as follows:
\begin{definition}[\textbf{Default value}]\label{def:3dmos:default_val}
  The \emph{default value} of an octree belief node $\Val_0^i(g^l)$ is the value \emph{before} the node is present in the octree.
\end{definition}

\begin{definition}[\textbf{Initial value}]\label{def:3dmos:initial_val}
  The \emph{initial value} of an octree belief node $\Val_1^i(g^l)$ is the value \emph{when} the node is inserted into the octree.
\end{definition}
\noindent Note that the notation of initial value is consistent with that of initial belief as introduced in Section~\ref{sec:bg:pomdp_motiv}, page \pageref{math:init_belief}.


\subsection{Belief Update}
\label{sec:octree:belief_update}
We have defined a per-voxel observation model for $\Pr(o_i|s',a)$ that reduces to $\Pr(d(s_i')|s',a)$ if $s_i'\in V_i$, or a uniform distribution if $s_i'\not\in V_i$. This suggests that the belief update need only
happen for voxels that are inside the FOV to reflect the information in the
observation.

Upon receiving observation $o_i$ within the FOV
$V_i$, belief is updated according to Algorithm \ref{alg:belief_update}.
This algorithm updates the value of the ground-level node $g$
corresponding to each voxel $v\in V_i$ as $\Val_{t+1}^i(g)=\Pr(d(v)|s',a)\Val_{t}^i(g)$. The normalizer is updated to make sure $b_{t+1}^i$ is normalized
\begin{lemma}
  The normalizer $\Norm_t$ at time $t$ can be correctly updated by adding the
  incremental update of values as in Algorithm \ref{alg:belief_update}.
\end{lemma}
\begin{proof}
  \label{pf:norm_proof}
The normalizer must be equal to the sum of node values at the ground level for the next belief $b_{t+1}^i$ to be valid (Equation~\ref{eq:octbdef}). That is, $\Norm_{t+1}=\sum_{s_i\in G}\Val^i_{t+1}(s_i)$. This sum can be decomposed into two cases where the object $i$ is inside of $V_i$ and outside of $V_i$; For object locations $s_i\not\in V_i$, the \emph{unnormalized} observation model is uniform, thus $\Val^i_{t+1}(s_i)=\Pr(d(s_i)|s',a)\Val^i_{t}(s_i)=\Val^i_{t}(s_i)$. Therefore, $\Norm_{t+1}=\sum_{s_i\in V_i}\Val^i_{t+1}(s_i) + \sum_{s_i\not\in V_i}\Val^i_{t}(s_i)$. Note the set $\{s_i | s_i\not\in V_i\}$ is equivalent as $\{s_i | s_i \in G\setminus V_i\}$. Using this fact and the definition of $\Norm_{t}$, we obtain
$\Norm_{t+1}=\Norm_t+\sum_{s_i\in V_i}\left(\Val^i_{t+1}(s_i) - \Val^i_{t}(s_i)\right)$
which proves the lemma.
\end{proof}

This belief update is therefore exact since the objects are static.
The complexity of this algorithm is $O(|V|\log(|G|)$;
 Inserting nodes and updating values of nodes can be done by traversing the tree depth-wise.

\begin{algorithm}[t]
\caption{OctreeBeliefUpdate $(b_t^i,a,o_i)\rightarrow b_{t+1}^i$}
\label{alg:belief_update}
\SetKwInOut{Input}{input}
\SetKwInOut{Output}{output}
\Input{$b_t^i$: octree belief for object $i$; $a$: action taken by robot; $o_i=\{(v,d(v)|v\in V_i\}$: factored observation for object $i$}
\Output{$b_{t+1}^i$: updated octree belief}
\tcp{Let $\Psi(b_i^t)$ denote the octree underlying $b_t^i$. }
\For{$v\in V_i$}{
  $s_i\gets v$\tcp*{State at grid cell corresponding to voxel $v$}
  \If{$s_i\not\in\Psi(b_i^t)$}{
    Insert node at $s_i$ to $\Psi(b_i^t)$\;
  }
  $\Val_{t+1}^i(s_i)\gets \Pr(d(v)|s',a)\Val_{t}^i(s_i)$\;
  $\Norm_{t+1}\gets\Norm_{t} + \Val_{t+1}^i(s_i) - \Val_{t}^i(s_i)$\;
}
\end{algorithm}

\subsection{Sampling}
Octree belief affords exact belief sampling at any resolution level in
logarithmic time complexity with respect to the size of the search space
$|G|$, despite not being completely built.
Given resolution level $l$, we sample from $\mathcal{S}_i^l$ by traversing the
octree in a depth-first manner.  Let $l_{max}$ denote the maximum resolution
level for the search space. Let $l_{des}$ be the \emph{desired} resolution level
at which an object state is sampled.\footnote{Recall that sampling a object state $s_i^l$ here
means that the object is considered to be located at $s_i^l$.} If $s_i^{l_{des}}$ is sampled, then all nodes in
the octree that cover $s_i^{l_{des}}$ , i.e, $s_i^{l_{max}},\cdots,s_i^{l_{des}+2},s_i^{l_{des}+1}$, must also be implicitly sampled,
Also, the event that $s_i^{l+k}$ is sampled is independent from other samples given that $s_i^{l+k+1}$ is sampled. Hence,
\begin{align}
\begin{split}
  &\Pr(s_i^{l_{des}}|h_t)\\
  &=\Pr(s_i^{l_{max}},\cdots,s_i^{l_{des}+2},s_i^{l_{des}+1},s_i^{l_{des}}|h_t)
\end{split}\\
&=\Pr(s_i^{l_{des}}|s_i^{l_{des}+1},h_t)\times\Pr(s_i^{l_{des}+1}|s_i^{l_{des}+2},h_t)\times\cdots\times\Pr(s_i^{l_{max}-1}|s_i^{l_{max}},h_t)
\end{align}
Therefore, the task of sampling $s^{l_{des}}$ is translated into sampling a
sequence of samples
$s_i^{l_{max}},\cdots,s_i^{l_{des}+2},s_i^{l_{des}+1},s_i^{l_{des}}$, each
according to the distribution $\Pr(s_i^{l}|s_i^{l+1},h_t)$. We can show that
this distribution can be efficiently obtained using octree belief itself as follows:
\begin{align}
  \Pr(s_i^{l}|s_i^{l+1},h_t)&=\frac{\Pr(s_i^{l},s_i^{l+1}|h_t)}{\Pr(s_i^{l+1}|h_t)}\\
                        &=\frac{\Pr(s_i^{l}|h_t)}{\Pr(s_i^{l+1}|h_t)}\label{eq:octree_sampling:reduction}\\
                        &=\frac{\Val_t^i(s_i^{l}) / \Norm_t} {\Val_t^i(s_i^{l+1})/\Norm_t}\\
                            &=\frac{\Val_t^i(s_i^{l})}{\Val_t^i(s_i^{l+1})}
\end{align}
Step~\ref{eq:octree_sampling:reduction} holds because sampling both $s_i^{l}$
and $s_i^{l+1}$ is equivalent to sampling just $s_i^{l}$ since
the latter (the event that $s_i^{l+1}$ is sampled) is deterministic when the former (the event that $s_i^{l}$ is sampled) happens.  Sampling from this probability
distribution is efficient, as the sample space, i.e. the children of node
$s_i^{l+1}$ is only of size 8. Therefore, this sampling scheme yields a sample
$s^{l_{des}}$ exactly according to $b_t^{i}(s^{l_{des}})$ with time complexity
$O(\log(|G|))$.


\section{Using Octree Belief for Multi-Resolution Planning} 
  \label{sec:planning_algorithm}


POUCT expands an MCTS tree using a \emph{generative function} $(s',o,r)\sim\mathcal{G}(s,a)$,
which is straightforward to acquire since we explicitly define the 3D-MOS models. However,
directly applying POUCT is subject to high branching factor due to the large observation space in our domain.

Our intuition is that octree belief imposes a spatial state abstraction, which
can be used to derive an abstraction over observations, reducing the branching
factor for planning.  Below, we formulate an \emph{abstract 3D-MOS} with smaller
spaces, and propose our multi-resolution planning algorithm.

\subsection{Abstract 3D-MOS}
\label{sec:abstract_3dmos}
We adopt the abstraction scheme in \citet{li2006towards} where in general, the abstract transition and reward functions are weighted sums of the original problem's transition and reward functions, respectively with weights that sum up to~1. We define an abstract 3D-MOS
$\langle \hat{\Sspace}, \hat{\Aspace}, \hat{\Ospace},
\hat{T},\hat{O},R,\gamma,l\rangle$ at resolution level $l$ as follows.

\textbf{State space $\hat{\mathcal{S}}$.} For each object $i$, an abstraction function $\phi_i:\Sspace_i\rightarrow\Sspace_i^l$ transforms the ground-level object state $s_i$ to an abstract object state $s_i^l$ at resolution level $l$. The abstraction of the full state is $\hat{s}=\phi(s)=\{s_r\}\cup\bigcup_i \phi_i(s_i)$ where the robot state $s_r$ is kept as is. The \emph{inverse image} $\phi^{-1}_i(s_i^l)$ is the set of ground states that correspond to $s_i^l$ under $\phi_i$ \cite{li2006towards}.

\textbf{Action space $\hat{\Aspace}$.} Since state abstraction lowers the resolution of the search space, we consider macro move actions that move the robot over longer distance at each planning step. Each macro move action \textsc{MoveOp}$(s_r,g)$ is an \emph{option} \cite{sutton1999between} that moves $s_r$ to goal location $g$ using multiple \textsc{Move} actions. The primitive \textsc{Look} and \textsc{Find} actions are kept.

\textbf{Transition function $\hat{T}.$} Targets and obstacles are still static, and the robot state still transitions according to the ground-level transition function. However, the transition of the found set from $\mathcal{F}$ to $\mathcal{F}'$ is special since the action $\textsc{Find}(i,g)$ operates at the ground level while $s_i^l$ has a lower resolution ($l>0$).
Let $f_i$ be the binary state variable that is true if and only if object $i\in\mathcal{F}$. Because the action \textsc{Find}$(i,g)$ affects $f_i$ based only on whether object $i$ is located at $g$, and that the problem is no longer Markovian due to state abstraction \cite{bai2016markovian}, $f_i$ transitions to $f_i'$ following
\begin{align}
\begin{split}
&\Pr(f_i'|f_i,s_i^l,h_t,\textsc{Find}(i,g))
\end{split}\\
&\ \ =\sum_{s_i\in\phi^{-1}_i(s_i^l)}\Pr(f'_i | s_i, f_i,\textsc{Find}(i,g))\Pr(s_i|s_i^l,h_t).
\end{align}
The above is consistent with the abstract transition function in the works \cite{li2006towards,bai2016markovian} where the first term corresponds to the ground-level deterministic transition function and the second term $\Pr(s_i|s_i^l,h_t)$, stored in the octree belief, is the \emph{weight} that sums up to 1 for all $s_i\in\Sspace_i$.

\textbf{Observation space $\hat{\Ospace}$ and function $\hat{O}$.}
For the purpose of planning, we again use the assumption that an object is contained within a single voxel (yet at resolution level $l$). Then, given state $\hat{s}'$, the abstract observation $o_i^l$ is regarded as a voxel-label pair $(s_i^l, d(s_i^l))$. Since it is computationally expensive to sum out all object states, we approximate the observation model by ignoring objects other than $i$:
\begin{align}
&\Pr(o_i^l|\abst{s}',a,h_t)=\Pr(d(s_i^l)|\hat{s}',a,h_t)\\
&\qquad\approx\Pr(d(s_i^l)|s^l_i,s_r,a,h_t)\\\
&\qquad=\sum_{s_i\in\phi_i^{-1}(s_i^l)}\Pr(d(s_i^l)|s_i,s_r,a)\Pr(s_i|s_i^l,h_t).
\end{align}
This resembles the abstract transition function, where $\Pr(d(s_i^l)|s_i,s_r,a)$
is the ground observation function, and $\Pr(s_i|s_i^l,h_t)$ is again the
weight.

For practical POMDP planning, it can be inefficient to sample from this abstract
observation model if $l$ is large. In our implementation, we approximate this distribution by Monte Carlo sampling \cite{shapiro2003monte}: We sample $k$ ground states from $\phi_i^{-1}(s_i^l)$ according to their weights.\footnote{We tested $k=10$ and $k=40$ and observed similar search performance. We used $k=10$ in our experiments.} Then we set $d(s_i^l)=i$ if the majority of these samples have $d(s_i)=i$, and $d(s_i^l)=\textsc{Free}$ otherwise. A similar approach is used for sampling from the abstract transition model.

\textbf{Reward function $R$.} The reward function is the same as the one in ground 3D-MOS, since computing the reward only depends on the robot state which is not abstracted and the abstract action space consists of the same primitive actions as 3D-MOS. Therefore, solving an abstract 3D-MOS is solving the same task as the original 3D-MOS.

\begin{algorithm}[t]
\caption{MR-POUCT $(\mathcal{P},b_t,d)\rightarrow \abst{a}$}
\label{alg:planning}
\SetKwInOut{Input}{input}
\SetKwInOut{Output}{output}
\SetKwFunction{plan}{Plan}
\SetKwFunction{solve}{PO-UCT}
\SetKwFunction{simulate}{Simulate}
\SetKwFunction{buildgenerator}{GenerativeFunction}
\SetKwFunction{rollout}{Rollout}
\SetKwProg{myproc}{procedure}{}{}
\Input{$\mathcal{P}$: a set of abstract 3D-MOS instances at different resolution levels; $b_t$: belief at time $t$; $d$: planning depth}
\Output{$\hat{a}$: an action in the action space of some $P_l\in\mathcal{P}$}
\myproc{\plan{$b_t$}}{
\ForEach{$P_l\in\mathcal{P}$ in parallel}{
  \tcp{Recall that $P_l=\langle \abst{\Sspace},\abst{\Aspace},\abst{\Ospace},\abst{T},\abst{O},R,\gamma,l \rangle$}
  $\mathcal{G}\gets$ GenerativeFunction($P_l$)\;
  $Q_P(b_t, \abst{a})\gets$ POUCT($\mathcal{G},h_t,d$)\;
}
$\abst{a}\gets\argmax_{\abst{a}}\{Q_P(b_t,\abst{a}) | P\in\mathcal{P}\}$\;
\KwRet{$\abst{a}$}
}
\end{algorithm}

\subsection{Multi-Resolution Planning Algorithm}
Abstract 3D-MOS is smaller than the original 3D-MOS which may provide benefit
in online planning. However, it may be difficult to define a single resolution
level, due to the uncertainty of the size or shape of objects, and the unknown
distance between the robot and these objects.

Therefore, we propose to solve a number of abstract 3D-MOS problems in parallel,
and select an action from $\abst{\Aspace}$ with the highest value for execution.
The algorithm is formally presented in Algorithm \ref{alg:planning}.  The set of
abstract 3D-MOS problems, $\mathcal{P}$, can be defined based on the
dimensionality of the search space and the particular object search
setting. Then, it is straightforward to define a \emph{generative function}
$\mathcal{G}(\hat{s},\hat{a})\rightarrow(\hat{s}',\hat{o},r)$ from an abstract
3D-MOS instance $P$ using its transition, observation and reward functions.
POUCT uses $\mathcal{G}$ to build a search tree and plan the next action. Thus,
all problems in $\mathcal{P}$ are solved online in parallel, each by a separate
POUCT. The final action with the highest value $Q_P(b_t,\hat{a})$ in its
respective POUCT search tree is chosen as the output (see
\cite{silver2010monte} for details on POUCT).
We call this
algorithm Multi-Resolution POUCT (MR-POUCT).

Next, we describe evaluation for this proposed multi-resolution planning algorithm.
We assess the hypothesis that our approach, MR-POUCT, improves the robot's
ability to efficiently and successfully find objects especially in large search
spaces. We conduct a simulation evaluation (Section~\ref{sec:3dmos:sim})
and a study on a real robot (Section~\ref{sec:3dmos:robot}).

\subsection{Evaluation in Simulation}
\label{sec:3dmos:sim}

\subsubsection{Setup} We implement our approach in a simulated environment designed to reflect the essence of the 3D-MOS domain (Figure~\ref{fig:3dmospomdp}). Each simulated problem instance is defined by a tuple $(m,n,d)$, where the search region $G$ has size $|G|=m^3$ with $n$ randomly generated, randomly placed objects. The on-board camera projects a viewing frustum with 45 degree FOV angle, an 1.0 aspect ratio, a minimum range of 1 grid cell, and a maximum range of $d$ grid cells. Hence, we can increase the difficulty of the problem by increasing $m$ and $n$, or by reducing the percentage of voxels covered by a viewing frustum through reducing the FOV range $d$. Occlusion is simulated using perspective projection and treating each grid cell as a point.

There are two primitive $\MOVE$ actions per axis (e.g. $+z$, $-z$) that each moves the robot along that axis by one grid cell. There are two $\LOOK$ actions per axis, one for each direction. Finally, a $\DETECT$ action is defined that declares all not-yet-found objects within the viewing frustum as found. Thus, the total number of primitive actions is $13$.
$\MOVE$ and $\LOOK$ actions have a step cost of -1.
A successful $\DETECT$ receives +1000 while a failed attempt receives -1000. A $\DETECT$ action is successful if part of a new object lies within the viewing frustum.
If multiple new objects are present within one viewing frustum when the $\DETECT$ is taken, only the maximum reward of $+1000$ is received. The task terminates either when the total planning time limit is reached or $n$ $\DETECT$ actions are taken.

\subsubsection{Baselines} We compare our approach (\emph{MR-POUCT}) with the following baselines:
  \emph{POUCT} uses the octree belief but solves the ground POMDP directly using the original POUCT algorithm. 
\emph{Options+POUCT} uses the octree belief and a resolution hierarchy, but  only the motion action abstraction (i.e. $\MOVEOP$ options) is used, meaning that the agent can move for longer distances per planning step but do not make use of state and observation abstractions.
\emph{POMCP} uses a particle belief representation which is subject to particle deprivation. Uniform random rollout policy is used for all POMDP-based methods. \emph{Exhaustive} uses a hand-coded exhaustive policy, where the agent traverses every location in the search environment. At every location, the agent takes a sequence of $\LOOK$ actions, one in each direction. Finally, \emph{Random} executes actions at uniformly at random.

Each algorithm begins with uniform prior and is allowed a maximum of 3.0s for planning each step. The total amount of allowed planning time plus time spent on belief update is 120s, 240s, 360s, and 480s for environment sizes ($m$) of 4, 8, 16, or 32, respectively. Belief update is not necessary for \emph{Exhaustive} and \emph{Random}. The maximum number of planning steps is 500. The discount factor $\gamma$ is set to $0.99$. For each $(m,n,d)$ setting, 40 trials (with random world generation) are conducted.

\subsubsection{Results}
We evaluate the scalability of our approach with 4 different settings of search space size $m\in\{4,8,16,32\}$ and 3 settings of number of objects $n\in\{2,4,6\}$, resulting in 12 combinations. The FOV range $d$ is chosen such that the percentage of the grids covered by one projection of the viewing frustum decreases as the world size $m$ increases.\footnote{The maximum FOV coverage for $m=4,8,16,$ and $32$ is $17.2\% (d=4), 8.8\% (d=6), 4.7\% (d=10),$ and $2.6\% (d=16)$, respectively.}
The sensor is assumed to be near-perfect, with $\alpha=10^5$ and $\beta=0$.
We measure the discounted cumulative reward, which reflects both the search efficiency and effectiveness, as well as the number of objects found per trial.

Results are shown in Figure~\ref{fig:scalability_overall}. Particle deprivation
happens quickly due to large observation space, and the behavior degenerates to
a random agent, 
causing {POMCP} to perform poorly. In small-scale domains, the \emph{Exhaustive} approach works well, outperforming the POMDP-based methods. We find that in those environments, the FOV can capture a significant portion of the environment, making exhaustive search desirable.
The POMDP-based approaches are competitive or better in the two largest search environments ($m=16$ and $m=32$). In particular, MR-POUCT outperforms \emph{Exhaustive} in all test cases in the larger environments, with greater margin in discounted cumulative reward; \emph{Exhaustive} takes more search steps but is less efficient.
When the search space contains fewer objects, {MR-POUCT} and {POUCT} show more resilience than {Options+POUCT}, with {MR-POUCT} performing consistently better.
This demonstrates the benefit of planning with the resolution hierarchy in octree belief especially in large search environments.


\begin{figure}[H]
  \centering
\makebox[\textwidth][c]{\includegraphics[width=1.18\textwidth]{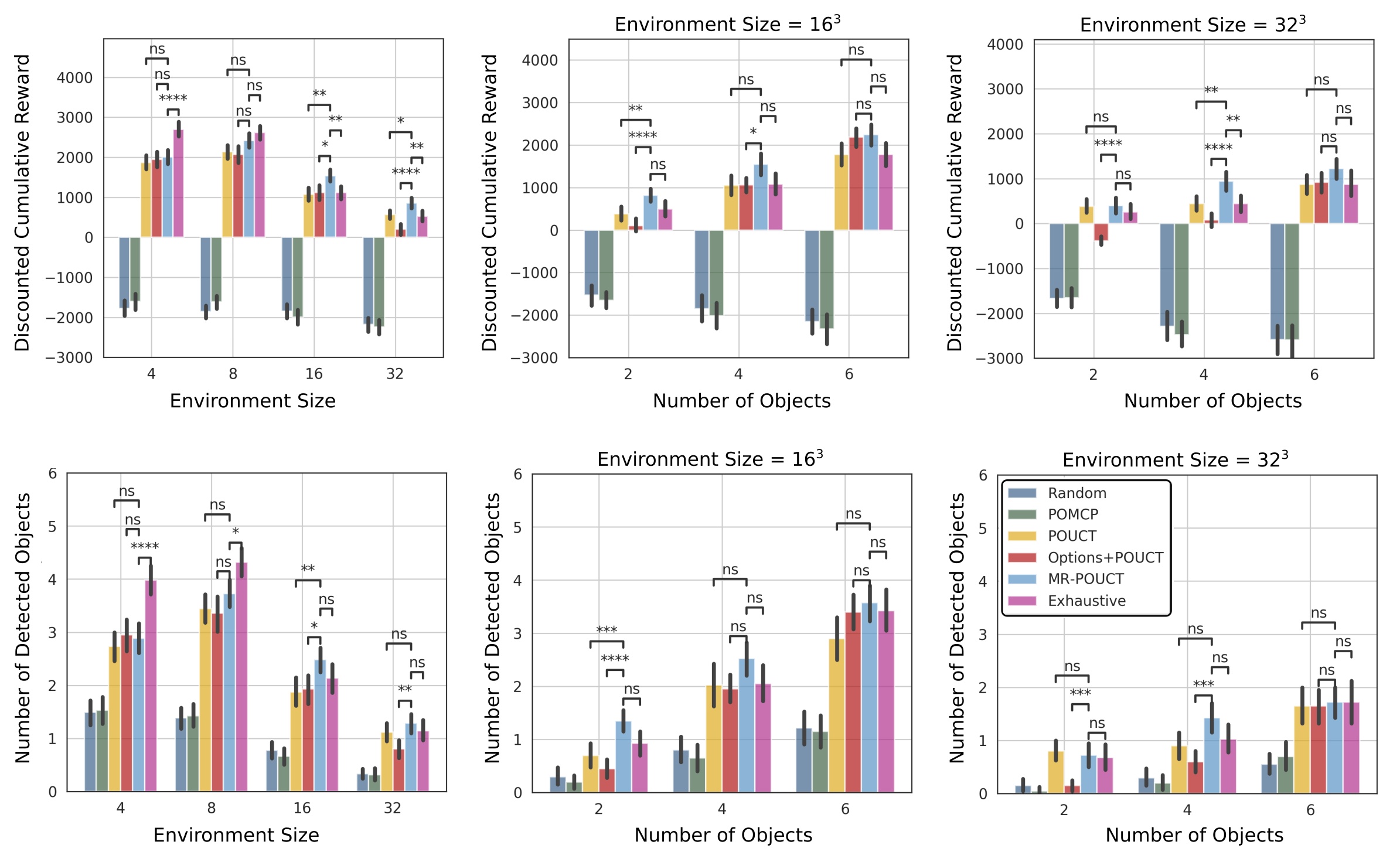}}
\caption{Discounted cumulative reward and number of detected objects as the environment size ($m$) increases and as the of number of objects ($n$) increases.  Exhaustive search performs well in small-scale environments (4 and 8) where exploration strategy is not taken advantage of. In large environments, our method MR-POUCT performs better than the baselines in most cases. The error bars are 95\% confidence intervals. The level of statistical significance is shown, comparing MR-POUCT against POUCT, Options+POUCT, and Exhaustive, respectively, indicated by \texttt{ns} ($p>0.05$), * ($p\leq 0.05$), ** ($p\leq 0.01$), *** ($p\leq 0.001$), **** ($p\leq 0.0001$).}
\label{fig:scalability_overall}
\end{figure}

\newpage
We then investigate the performance of our method with respect to changes in sensing uncertainty, controlled by the parameters $\alpha$ and $\beta$ of the observation model. According to the belief update algorithm in Section~\ref{sec:octree:belief_update}, a noisy but functional sensor should increase the belief $\Val_t^i(g)$ for object $i$ if an observed voxel at $g$ is labeled $i$, while decrease the belief if labeled \textsc{Free}. This implies that a properly working sensor should satisfy $\alpha > 1$ and $\beta < 1$. We investigate on 5 settings of $\alpha\in\{10,100,500,10^3,10^4,10^5\}$ and 2 settings of $\beta\in\{0.3, 0.8\}$. A fixed problem difficulty of $(16,2,10)$ is used to conduct this experiment.
Results in Figure~\ref{fig:quality_results} show that {MR-POUCT} is consistently better in all parameter settings. We observe that $\beta$ has almost no impact to any algorithm's performance as long as $\beta < 1$, whereas decreasing $\alpha$ changes the agent behavior such that it must decide to $\LOOK$ multiple times before being certain.

\begin{figure}[t]
\centering
\includegraphics[width=\linewidth]{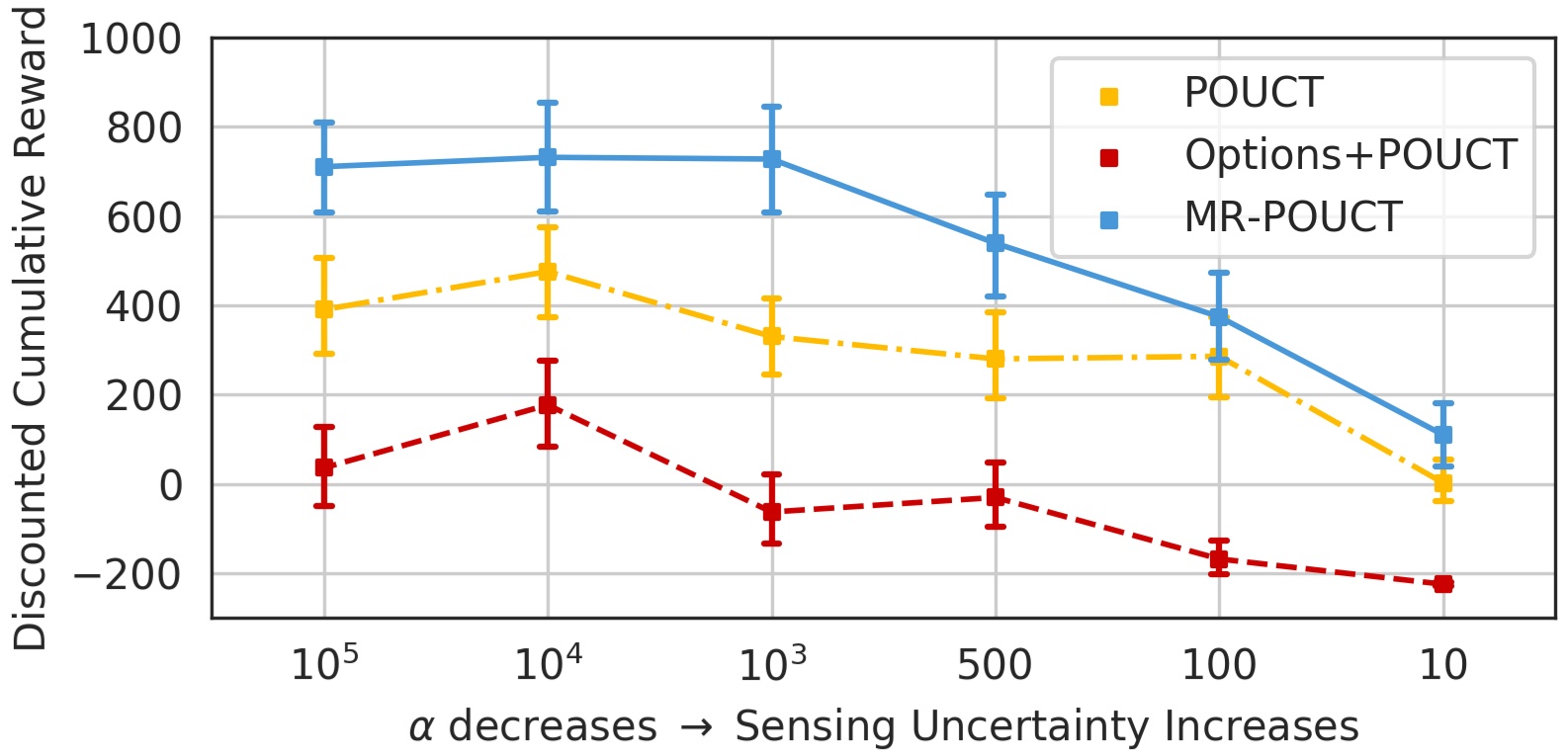}
\caption{Discounted cumulative reward with 95\% confidence interval as the sensing uncertainty increases, aggregating over the $\beta$ settings.}
\label{fig:quality_results}
\end{figure}


\subsection{Demonstration on Real-Robot}
\label{sec:3dmos:robot}

We demonstrate that our approach is scalable to real world settings by implementing the 3D-MOS problem as well as MR-POUCT for a mobile robot setting. We use the Kinova MOVO Mobile Manipulator robot, which has an
actuated torso with an extension range between around 0.05m and 0.5m, which facilitates a 3D action space. The robot operates in a lab environment, which is decomposed into two \emph{search regions} $G_1$ and $G_2$ of size roughly 10m$^2\times$ 2m (Figure.~\ref{fig:movo_seq}), each with a semantic label (``shelf-area'' for $G_1$ and ``whiteboard-area'' for $G_2$). The robot is tasked to look for $n_{G_1}$ and $n_{G_2}$ objects in each search region sequentially, where objects are represented by paper AR tags that could be in clutter or not detectable at an angle. The robot instantiates an instance of the 3D-MOS problem once it navigates to a search region. In this 3D-MOS implementation, the $\MOVE$ actions are implemented based on a topological graph on top of a metric occupancy grid map.
\begin{figure}[h]
\centering
\makebox[\textwidth][c]{\includegraphics[width=\linewidth]{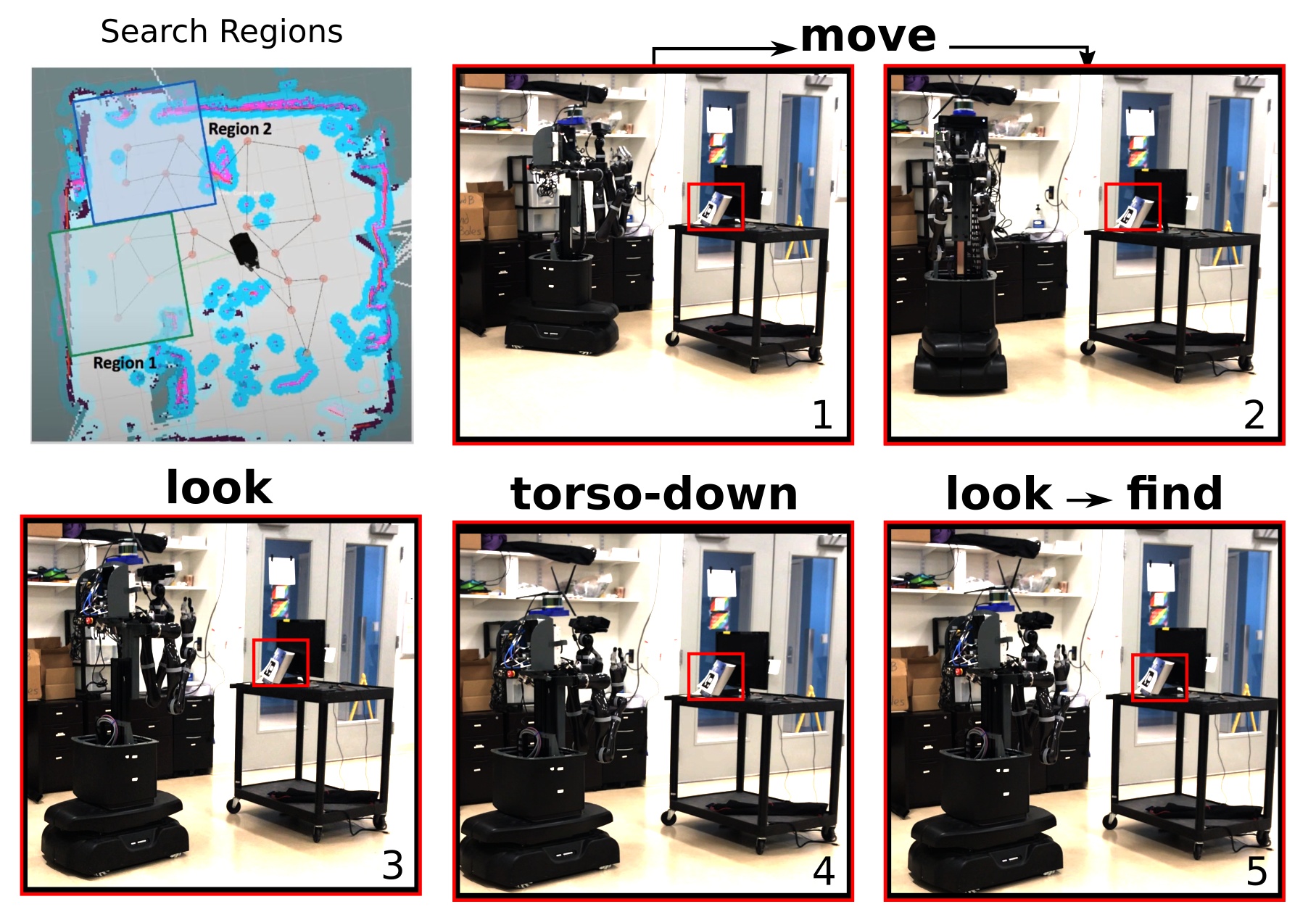}}%
\caption{Example action sequence produced by the proposed approach that enables
  a Kinova MOVO robot to perform 3D object search in two search regions separately (top left). The mobile robot
  first navigates in front of a portable table (1-2). It then takes a $\LOOK$ action to observe the space in front (3), and no target is observed since the torso is too high. The robot then decides to lower its torso (4), takes another $\LOOK$ action in the same direction, and then $\DETECT$ to mark the object as found (5). This sequence of actions demonstrate that our algorithm can produce efficient search strategies in real world scenarios. }
\label{fig:movo_seq}
\end{figure}
The neighbors of a graph node form the motion action space when the robot is at that node.
The robot can take $\LOOK$ action in 4 cardinal directions in place and receive volumetric observations; A volumetric observation is a result of downsampling and thresholding points in the corresponding point cloud. The robot was able to find 3 out of 6 total objects in the two search regions in around 15 minutes. One sequence of actions (Figure~\ref{fig:movo_seq}) shows that the robot decides to lower its torso in order to $\LOOK$ and $\DETECT$ an object.\footnote{Video footage with visualization of volumetric observations and octree belief update is available at \href{https://zkytony.github.io/3D-MOS/}{https://zkytony.github.io/3D-MOS/}.} A failure mode is that the object may not be covered by any viewpoint and thus not detected; this can be improved with a denser topological map, or by considering destinations of $\MOVE$ actions sampled from the continuous search region.

In the next chapter, we present a system for generalized 3D multi-object search, the first of its kind, and discuss its integration with different robots performing object search in different environments.

\vspace{3.0in}
\begin{center}
  THIS IS THE END OF THIS CHAPTER.
\end{center}

\chapter{GenMOS: A System for Generalized 3D Multi-Object Search}
\label{ch:genmos}
\vspace{-0.6in}
\begin{figure}[H]
  \centering
\includegraphics[width=\linewidth,draft=false]{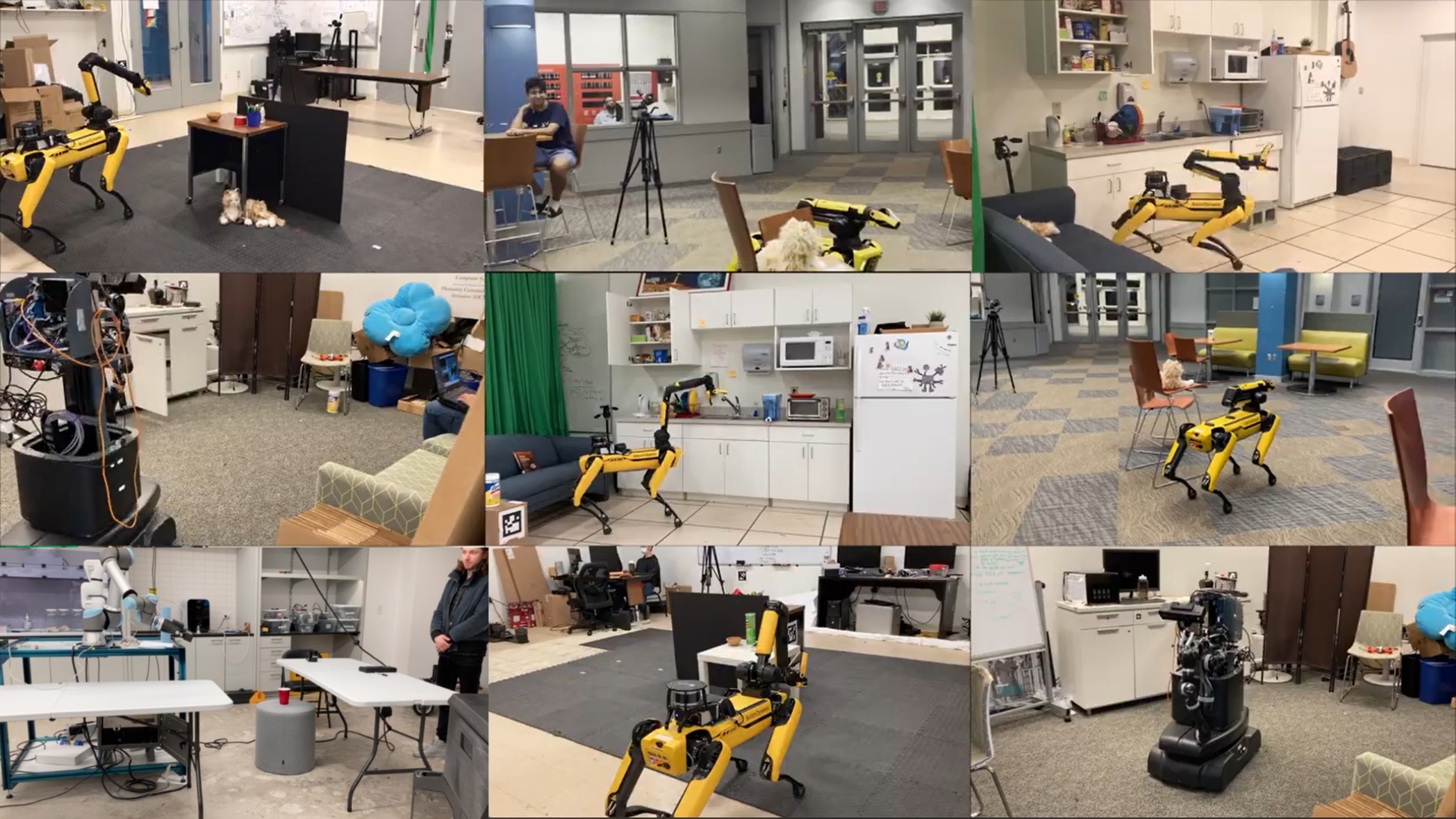}
  \caption{GenMOS enables different robots to search for objects in various 3D regions.}
  \label{fig:genmos:demogrid}
\end{figure}

\section{The GenMOS System}
\label{sec:genmos:sysov}
\lettrine{T}{owards} the goal of making object search an off-the-shelf capability for any robot,~we present GenMOS (Generalized Multi-Object Search), the first general-purpose object search system that is robot-independent and environment-agnostic (Figure~\ref{fig:genmos:demogrid}). GenMOS builds upon the methodology for 3D multi-object search described in the previous chapter, while significantly improving its practicality in the real world. Our system enables a Boston Dynamics Spot to find, for example, a cat underneath the couch, as shown in Figure~\ref{fig:teaser}.

\begin{figure}[t]
  \centering
\includegraphics[width=\linewidth,draft=false]{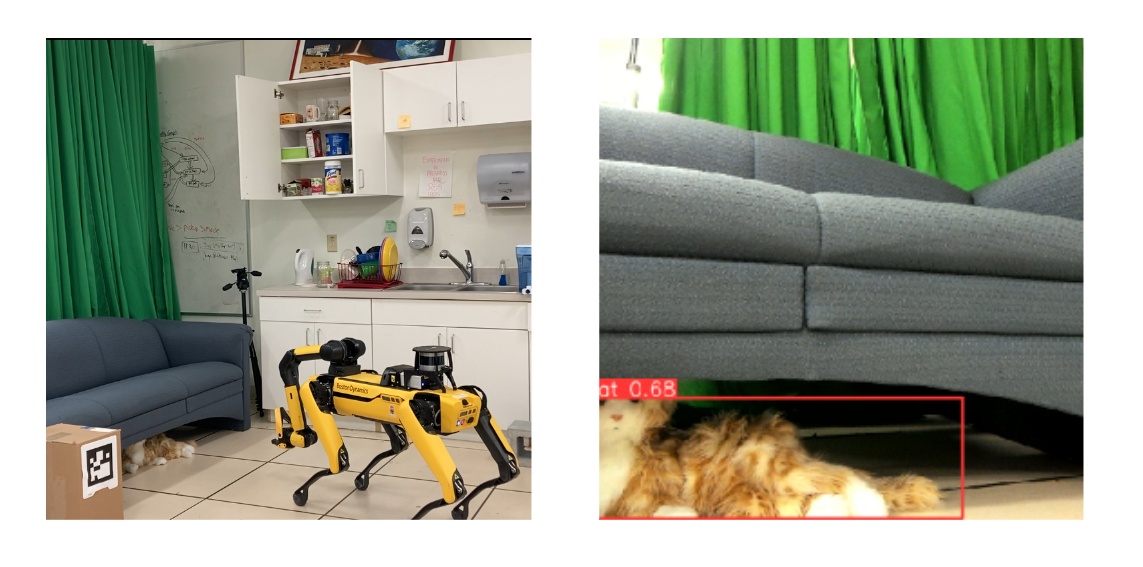}
  \caption{Our system, GenMOS, enables a Boston Dynamics Spot robot to successfully find a toy cat underneath the couch. The left image shows a third-person view of the scene. The right image shows the RGB image from the gripper camera, along with the object detection bounding box for the cat labeled.}
  \label{fig:teaser}
  \vspace{-0.2in}
\end{figure}

This chapter begins with an overview of the system's design, illustrated in Figure~\ref{fig:genmos:system}. In particular, we describe three ways that point cloud observations are used in GenMOS:
\begin{enumerate}[topsep=0pt,noitemsep,label=(\arabic*)]
\item to simulate occlusion (Figure~\ref{fig:genmos:fov});
\item to inform occupancy and initialize octree belief (Figure~\ref{fig:sim_env});
\item to sample a belief-based graph of view positions (Figure~\ref{fig:spot_cat_seq}, right column).
\end{enumerate}
Then, we describe novel algorithmic contributions regarding (2) and (3): For (2), we propose an algorithm for initializing octree beliefs given arbitrary prior distributions over object locations;
For (3), we propose an algorithm which samples a belief-dependent graph of view positions, allowing the output space of GenMOS to be the continuous space of reachable viewpoints. Subsequently, we describe the gRPC protocol of our implementation of GenMOS as well as a few useful parameters one can configure to adapt GenMOS to a specific scenario. Finally, we discuss our evaluation of GenMOS, first in a simulation domain, then integrated on three robot platforms: Boston Dynamics Spot, Kinova MOVO, and Universal Robotics UR5e.

\textbf{Contributions.} The contributions of this chapter were described in Section~\ref{sec:3dmos:contribs}. We emphasize here that the algorithms and evaluation in this chapter serve to improve and demonstrate the practicality of the octree-based 3D multi-object search approach introduced in the previous chapter, which was only evaluated in an idealistic simulation with cardinal action space and on a MOVO with a proof-of-concept system.




\begin{figure}[t]
  \centering
  \includegraphics[width=\linewidth]{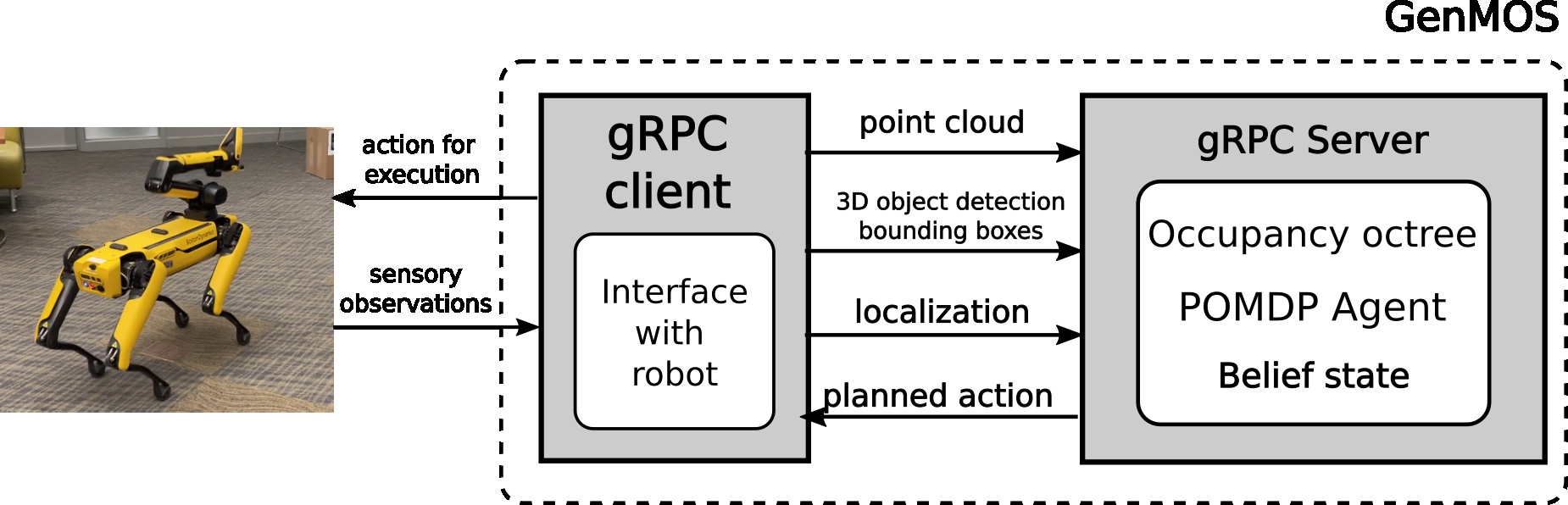}
  \caption{Overview of the GenMOS system. See Section~\ref{sec:genmos:system} for description.}
  \label{fig:genmos:system}
\end{figure}

\subsection{System Overview}
\label{sec:genmos:system}

\textbf{A gRPC-based system.} GenMOS is a client-server construct, designed and implemented based on gRPC~\cite{grpc}, a high-performance, cross-platform, and open source framework for remote procedural call (RPC). As a gRPC-based system (Figure~\ref{fig:genmos:system}), GenMOS is independent of, thus integrable to any particular robot middleware such as ROS \cite{quigley2009ros} or ROS 2 \cite{ros2}.

\textbf{Inputs and outputs.} GenMOS considers perceptual inputs including point cloud observations of the local region, 3D object detection bounding boxes (if detection occurs), and localization of robot camera pose, and it outputs a viewpoint to move to as the result of sequential online planning.


\subsubsection{Server}
\label{sec:genmos:server}

Here, I describe several important aspects of the GenMOS server.

\textbf{3D-MOS.} The server internally maintains a POMDP model of the search task, which is a 3D-MOS with an instantiation of the action space based on a graph of view positions (Section~\ref{sec:vgsample}). The definition and implementation of this model, based on \texttt{pomdp\_py} \cite{pomdp-py-2020}, is general and does not depend on any particular environment. Importantly, the server handles coordinate conversion: the client only needs to send data in the metric world frame and the server properly converts them into the POMDP's frame.

\begin{figure}[t]
  \centering
  \includegraphics[width=\linewidth]{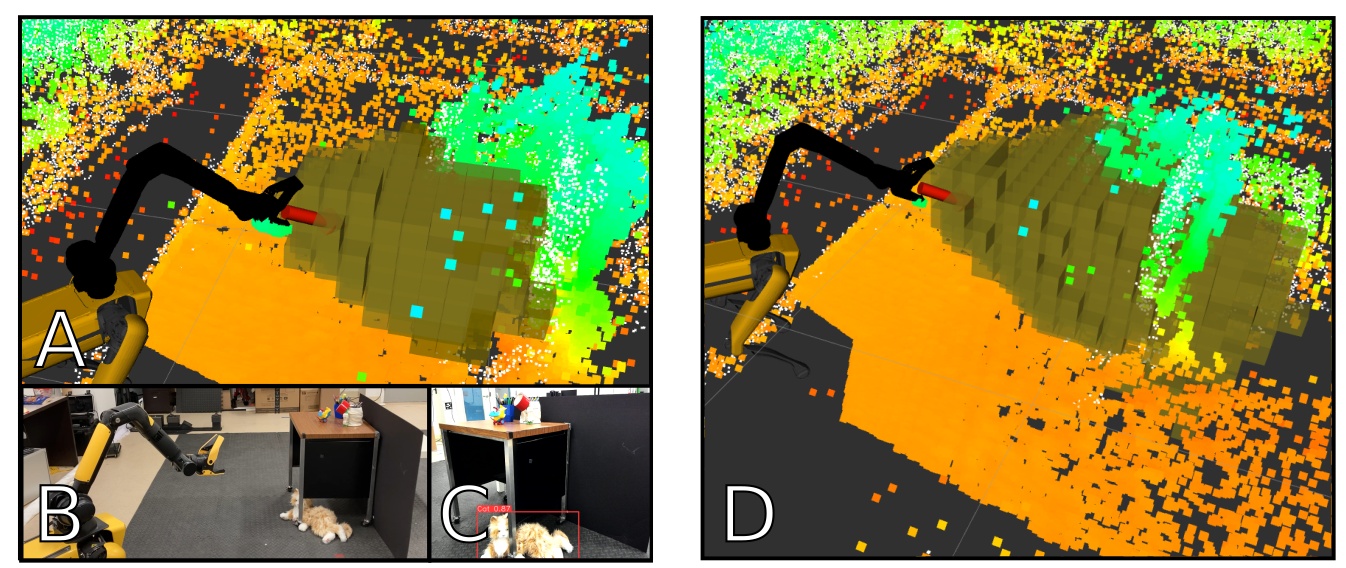}
  \caption{For belief update, GenMOS samples a volumetric observation (a set of labeled voxels within the viewing frustum) that considers occlusion based on the occupancy octree dynamically built from point cloud (A). Not enabling occlusion (D) leads to mistaken invisible locations as free. The robot is looking at a table corner (B) with its view blocked by the table and the board (C).}
  \label{fig:genmos:fov}
\end{figure}

\textbf{Occupancy octree from point cloud.} Internally, the server maintains an octree representation of the search region's occupancy, used to simulate occlusion-enabled observations for belief update.
Point cloud observations can be sent to the server to update the server's model of the search region. Specifically, the server converts the point cloud into an \emph{occupancy octree} (similar to OctoMap \cite{hornung2013octomap}), where a leaf node in the tree has an associated value of occupancy (0 for free, and 1 for occupied). The occupancy octree is used by the server for both sampling the view positions graph to avoid collision, as well as for constructing volumetric observations during belief update, where occupied nodes block the FOV and cause occlusion (Figure.~\ref{fig:genmos:fov}).

\textbf{Additional uses of point cloud.} A key aspect of GenMOS is how point cloud observations are used in three ways: (1) to simulate occlusion; (2) to inform occupancy and initialize octree belief; and (3) to sample a belief-based graph of view positions. We have illustrated (1) above. For (2) and (3), we discuss in more detail in Section~\ref{ref:prioroctree} and Section~\ref{sec:vgsample}, respectively.



\textbf{Object detection.} The server can also take in 3D object detection bounding boxes, which represent the output of a generic object detector or perception pipeline capable of estimating the 3D locations of detected objects. The bounding box's size plays a role in the octree belief update, as it influences the volumetric observation, where voxels overlapping with the bounding box are labeled by the detected object and leads to an increase in the octree belief at the corresponding locations.

\label{sec:genmos:server:2d_detection}
When 3D object detection is not available on the robot, the system can also consume label-only detections based on just images. Such label-only detections essentially correspond to a volumetric observation within the FOV where all voxels are labeled by the object, which usually covers a sizable volume. This is still useful for search, as subsequent search steps can reduce uncertainty by looking from different viewpoints.


\textbf{Server requests.} The server may also actively request information (such as additional observation about the search region's occupancy), which enables our implementation of hierarchical planning in Section~\ref{sec:hier}; there, 3D local search is triggered by a high-level action to ``search locally'' and the server would request point cloud data on the fly in order to instantiate 3D-MOS.

\vspace{-0.2em}
\subsubsection{Client}
Here, I describe several important aspects of the GenMOS client.

\textbf{Client's role.} The client sends to the server configurations of the POMDP agent, perception data, and planning requests, and executes the action returned by the server (Figure~\ref{fig:genmos:system}).  All data transmitted between the client and the server are represented as Protocol Buffer (protobuf) messages \cite{protobuf} of generic, robot-independent message types (\eg, point cloud, 3D bounding box, 6D pose, etc.).  The client is responsible for integrating with the robot hardware,  obtaining sensor data and converting them into the protobuf message types, and physically executing the actions to reach the planned viewpoints. This makes the server code independent of any specific robot.

\textbf{Planning requsts.} When planning requests are sent from the client, the server performs online planning using an asymptotically optimal, Monte Carlo Tree Search-based online POMDP planning algorithm called POUCT \cite{silver2010monte}.\footnote{See Section~\ref{sec:bg:pomdp:pouct}, page~\pageref{sec:bg:pomdp:pouct} for an introduction of POUCT.} The server converts the planned camera viewpoint $q'\in\mathcal{R}$ to metric coordinates in the frame of the search region. The client then handles moving the robot to that viewpoint. If the server plans a $\textsc{Find}$ action, the client should send back the detected target objects (if any)\footnote{The client may choose to control the robot to physically signal when $\textsc{Find}$ is taken. For example, with Spot, I let the robot close and reopen its gripper, indicating the robot's commitment to the found location.} The client is also responsible for sending new observations upon action completion.\\

Next, I follow through with explaining the algorithmic contributions of this chapter that enable the two additional uses of point cloud in GenMOS: belief initialization and view position graph sampling.

\begin{figure}[t]
  \centering
\includegraphics[width=\linewidth]{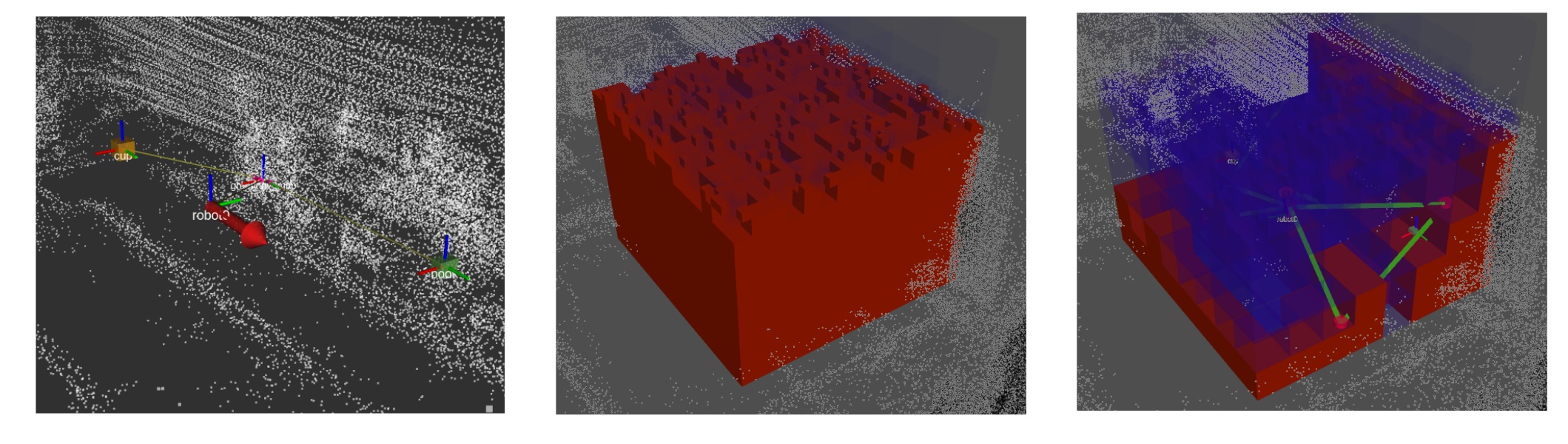}
  \caption{Left: A simulation environment where the pose of the robot's viewpoint is represented by the red arrow, and the two target objects are represented by orange and green cubes. Middle: initialized octree belief given uniform prior within a 10.2m$^2\times$2.4m region; Right: initialized octree belief within the same region, given occupancy-based prior constructed from point cloud. Colors indicate strength of belief, from red (high) to blue (low).}
  \label{fig:sim_env}
  \vspace{-0.2in}
\end{figure}

\subsection{Prior Initialization of Octree Belief}
\label{ref:prioroctree}
Octree belief covers, by definition, a cubic volume; However, the actual feasible search region is likely not cubic, and often irregular. This causes
the robot to believe constantly that the target objects are outside of the actual search region, at places imperceivable by the robot, which can impact search behavior.\footnote{This is an issued I observed with the proof-of-concept system in Chapter~\ref{ch:3dmos}.}

To address this problem, I propose an efficient algorithm for initializing an octree belief over an arbitrary search region, presented in Algorithm~\ref{alg:belief_init}. Recall that $G$ denotes the entire 3D grid map at ground resolution level underlying an octree belief. Suppose $G_*\subseteq G$ is the subset of grids in $G$ that make up the search region.\footnote{$G_*$ could be an arbitrary subset, not necessarily forming, for example, a cuboid.} Recall the definition of \emph{default value} and \emph{initial value} in Definition~\ref{def:3dmos:default_val} and Definition~\ref{def:3dmos:initial_val}, respectively (Section~\ref{sec:octree_belief}, page~\pageref{def:3dmos:default_val}).
The high-level idea of the proposed algorithm is as follows:

\begin{enumerate}[itemsep=0.5pt,topsep=0pt]
\item First, set the default value of all ground-level nodes in the octree belief to 0.

\item Then, through a sample-based procedure (with $N$ samples), ground-level nodes whose 3D positions lie within the given search region $G_*$ have their default values changed to 1.
\end{enumerate}

\noindent This effectively reduces the sample space of the octree belief to be within the search region $G_*$. Besides reducing the sample space, if we are given a prior distribution $\PriorVal^i:G_*^l\rightarrow\mathbb{R}$, we can initialize the octree belief accordingly as follows, during step 2 above:

\begin{enumerate}[itemsep=0.5pt,topsep=0pt]
\item[3.] If a prior probability $\PriorVal^i(g^l)$ is defined at octree belief node $g^l\in G_*^l$, and $g^l$ is the parent (or self) of some ground level node $g\in G$ sampled during step 2, then $\Val_1^i(g^l)$, the initial value at $g^l$ is set as $\Val_1^i(g^l)\gets\PriorVal^i(g^l)$.
\end{enumerate}

\noindent The algorithm is given below in Algorithm~\ref{alg:belief_init}. The assignment of initial value at a node is done in lines 8-10 and the tree can be pruned as in lines 12-14. This proposed algorithm has a complexity of $O(N(\log(|G|))^2)$.

\begin{algorithm}[H]
\caption{Initialize Octree Belief $(m, G_*, \PriorVal^i)\rightarrow b_1^i$}
\label{alg:belief_init}
\LinesNumbered
\SetKwInOut{Input}{input}
\SetKwInOut{Output}{output}
\SetKwInOut{Parameter}{param}
\Input{$m$: octree dimension (power of 2, such that $|G|=m^3$); \hspace{3cm} $G_*$: the actual search region, satisfying $G_*\subseteq G$; \hspace{4cm}  $\PriorVal^i$: Prior distribution of octree node values.}
\Parameter{$N$: number of samples; $B$: a 3D box, satisfying $G_*\subseteq B \subseteq G$.}
\Output{$b_1^i$: the initialized octree belief.}
Initialize octree $\Psi(b_0^i)$; Set $\Val_0^i(g)=0$ (instead of 1)\;
\For{$i\in \{1,\cdots,N\}$}{
  Set $l=0$; Sample $g^l\sim B$ \tcp*{ground resolution location}
  \While{$l\leq \log_2m$}{
    \If{$g^l\in G_*^l$}{
      Add $g^l$ to $\Psi(b_1^i)$\tcp*{insert $g^l$ to octree underlying $b_1^i$}
      Set default value $\Val_0^i(g^l)=|\textsc{Ch}^0(g^l)|$\;
      \If{$g^l\in \PriorVal^i$ }{
        Set initial value $\Val_1^i(g^l)\gets\PriorVal^i(g^l)$\;
        \tcp{otherwise $\Val_1^i(g^l)\gets \Val_0^i(g^l)$}
      }  
      \tcp{ensure parent value is sum of children}
      Update values of all parent nodes at $g^{l+1}\cdots g^{m}$\;
      \If{$\Val_1^i(g^l)=\Val_0^i(g^l)$}{remove children of $g^l$ \tcp*{Pruning}}
    }
    $l \gets l + 1$\;
  }
}
$\Norm_1\gets\Val_1^i(g^m)$ \tcp*{normalizer set to root node's value}
\end{algorithm}

In practice, the server can optionally determine the search region $G^*$ based on the occupancy octree constructed from point cloud observations. This avoids believing that the objects lie in midair. In our experiments, we assign a prior value of $100\times ((2^k)^3)$ to occupied nodes in the octree at the resolution level $k=2$, and we set the number of samples $N=3000$. Figure~\ref{fig:sim_env} visualizes an octree belief with occupancy-based prior.

\subsection{Belief-based Sampling for View Position Graph}
\label{sec:vgsample}
The evaluation in Sections~\ref{sec:3dmos:sim} and Sections~\ref{sec:3dmos:robot} of the previous chapter only considers moving the camera in cardinal directions, or over a fixed topological map. To enable planning over the continuous space of viewpoints $\mathcal{R}\subseteq \mathcal{P}\times\SO(3)$, GenMOS samples a view position graph $\mathcal{G}_t=(\mathcal{P}_V, \mathcal{E}_M)$ based on the current octree belief. At a high level, given an occupancy octree, we first sample a set of non-occupied positions $\mathcal{P}_V$ from $\mathcal{P}$ with a minimum separation threshold, (\eg, 0.75m)
and associate with each position the \emph{belief} by querying the octree belief at that position at a higher resolution level to cover more space. Then, we select top-$K$
(\eg, $K=10$)
nodes ranked by their beliefs and insert edges such that each node has a limited degree. A $\textsc{Move}(s_r,p_v)$ action then moves the robot to a viewpoint position $p_v\in\mathcal{P}_v$ on the graph. We implicitly enforce a $\textsc{Look}(\phi)$ action after a $\textsc{Move}$ action through the transition model where $\phi$ is the orientation facing the an unfound object (contained in $s$, input to the transition model). At time $t+1$, the graph is resampled \emph{if} the sum of the probability covered by positions in $\mathcal{G}_t$ is below a threshold.
(\eg, 0.4).

\subsection{The gRPC Protocol in GenMOS}
\label{sec:grpc}
In the gRPC framework, remote procedural calls (RPCs) are defined as Protocol Buffer messages \cite{protobuf}. In particular, the key RPCs in GenMOS are as follows:
\begin{itemize}[leftmargin=*,noitemsep]
\item \texttt{CreateAgent}: Upon receiving the POMDP agent configurations from the client, the server prepares for agent creation pending the first \texttt{UpdateSearchRegion} call.

\item \texttt{UpdateSearchRegion}: The client sends over a point cloud of the local search region, and the server creates or updates the occupancy octree about the search region.

\item \texttt{ProcessObservation}: The client requests belief update by sending observations such as object detection and robot pose estimation.

  \item \texttt{CreatePlanner}: The client provides hyperparameters of the planner, and the server creates a planner instance accordingly (\eg, POUCT planner in \texttt{pomdp\_py} \cite{pomdp-py-2020}).

\item \texttt{PlanAction}: The client requests the server to plan an action for an agent. An action is planned only if the last planned action has been executed successfully.

\item \texttt{ListenServer}: This is a bidirection streaming RPC that establishes a channel of communication of messages or status between the client and the server.
  \end{itemize}

\subsection{Example Configuration Parameters}
The table below lists some parameters that the GenMOS server is able to handle.
\begin{longtable}{>{\raggedright}p{1.0in}p{4.2in}}
  \toprule
  \texttt{octree\_size} & Dimension of the octree representing the search region (\eg, 32 means the octree occupies a $32^3$ grid)\\

  \midrule
  \texttt{res}  &  Resolution of a grid, \ie, length of the grid's side in meters (\eg, 0.1)\\

  \midrule
  \texttt{region\_size}  &  Defines the dimensions of a box $(w,\ell,h)$ in meters (\eg (4.0, 3.0, 1.5))\\

  \midrule
  \texttt{center}  &  Defines the XYZ location (metric)of the search region's center (\eg (-0.5, -1.65, 0.25)))\\

  \midrule
  \multicolumn{2}{l}{\texttt{prior\_from\_occupancy}$\ $ ``True'' to use occupancy-based prior}\\

  \midrule
  \multicolumn{2}{l}{\texttt{occupancy\_fill\_height}$\ $ ``True'' to consider the space below obstacles  }\\
  &$\qquad\qquad$ into search space\\

  \midrule
  \texttt{num\_nodes} & maximum number of view positions on graph (\eg, 10)\\

  \midrule
  \texttt{sep} & minimum separation between nodes (in meters) (\eg, 0.4)\\

  \midrule
  \texttt{inflation} & radius to blow up obstacles for view position sampling \\

  \midrule
  \texttt{num\_sims} & number of samples for MCTS-based online POMDP planning.\\

  \bottomrule


  \caption{Example configuration parameters in GenMOS}
  \label{tab:genmos:params}
\end{longtable}

\section{Evaluation of GenMOS}

We implemented the gRPC protocols of the GenMOS system described in Section~\ref{sec:grpc}. The result is a single package for multi-object search in 3D regions that can provide the object search functionality as long as the perception inputs are given, which are generic point cloud and object detection results that a robot typically should be able to provide.

There are two hypotheses that we test through our evaluation: (1) The octree belief-based planning algorithm that the package implements is effective for 3D object search; (2) The package does enable real robots to search for and find objects in 3D regions in different environments within a reasonable time budget.

To test the first hypothesis, we conduct an experiment in simulation (Section~\ref{sec:genmos:sim}).
To test the second hypothesis, we deploy our system for object search with a Boston Dynamics Spot robot in two different local regions: a region of arranged tables and a kitchen region (Section~\ref{sec:spot}), and we also implement a preliminary hierarchical planning algorithm for a demonstration over a larger lobby area (Section~\ref{sec:hier}). We further integrate GenMOS with Kinova MOVO and Universal Robotics UR5e robotic arm and test search behavior enabled by GenMOS.

\subsection{Evaluation in Simulation}
\label{sec:genmos:sim}
We tasked a simulated robot (represented as an arrow for its viewpoint)
to search for two virtual objects (cubes) with volume 0.002m$^3$ each uniformly randomly placed in a region of size 10.2m$^2\times$ 2.4m. The robot's frustum camera model had a FOV angle of 60 degrees, minimum range of 0.2m and maximum range of 2.0m.

We experimented with three types of priors, groundtruth, uniform, and occupancy-based prior, at two different resolution levels, 0.001$m^3$ (octree size 32$\times$32$\times$32) and 0.008$m^3$ (octree size 16$\times$16$\times$16) representing search granularity.  For the best-performing setting (non-groundtruth), we also compared the use of the POUCT planner against two baselines: Random moves to a uniformly sampled view position graph node (Section~\ref{sec:vgsample}), and Greedy is a next-best view planner that moves to the view position graph node that is closest to the highest belief location for some target. Both baseline planners take $\textsc{Find}$ upon target detection.

\begin{table}[t]
  \centering
    \centering
  \begin{tabular}{lllll}
        \toprule
        Prior type (resolution)     & Length     & Planning  & Total      & success \\
        \multicolumn{1}{r}{with POUCT}  & (m)        & time (s)  & time (s)   & rate \\
       \midrule
        Uniform (0.008m$^3$)   &  22.13    & 24.28          & 166.18        & 50\%\\
        Occupancy (0.008m$^3$) &  23.89    & 22.66       & 159.10        & 60\%\\
        Uniform (0.001m$^3$)   &  6.42     & 10.47          & 99.66         & 90\%\\
    Occupancy (0.001m$^3$)$^*$ &  3.22    & 7.42            & 64.12         & 100\%\\
    Groundtruth            &  0.44     & 1.97           & 17.82         & 100\%\\
    \midrule
    \multicolumn{1}{r}{ $^*$with Random}  & 12.18  & 0.19 & 167.20 & 55\%\\
    \multicolumn{1}{r}{ $^*$with Greedy}  & 3.48  & 0.12 & 81.80 & 85\%\\
        \bottomrule
  \end{tabular}
  \caption{Simulation results. We compare the search performance between different prior
    belief and resolution settings. The results for the first three colums are averaged over 20 trials. }
  \label{tab:genmos:results}
\end{table}


We evaluate the search performance by four metrics: total path length traversed during search (Length), \emph{total} time used for POMDP planning (Planning time), total system time (Total time), and success rate. Total system time included time for planning, executing navigation actions, receiving observations, belief update and visualization; the simulated robot has a translational velocity of 1.0m/s, and a rotational velocity of 0.87rad/s.

We perform 20 search trials per method and report the average of each metric in Table~\ref{tab:genmos:results}.\footnote{Simulation experiments were run on a computer with i7-8700 CPU.} Each trial was allowed 180s total system time (excluding the time for visualization). Results indicate that the system achieved high success rate especially at high resolution under occupancy-based prior. We observed that searching with a resolution level more coarse than the target size hurts performance, while having occupancy-based prior improves. Additionally, Greedy was much faster than POUCT in planning time yet lead to lower success rate within the time budget and longer total time than using POUCT.  Our intuition is that, while Greedy prioritizes looking at a location with the highest belief, POUCT considers the search of multiple objects in a sequence.



\subsection{Deployment on the Boston Dynamics Spot}
\label{sec:spot}
We deploy our system to the Boston Dynamics Spot \cite{bdspot} by writing a client for GenMOS
that interfaces with the Spot SDK.\footnote{We integrated Spot SDK with
  ROS \cite{quigley2009ros} to use RViZ \cite{kam2015rviz};
  Our computer that ran GenMOS for Spot has an i7-9750H CPU with an RTX 2060 GPU.} Spot is a mobile robot that is robust at
navigation while avoiding obstacles. Our Spot robot is equipped with an arm that has a
gripper with an RGB-D camera, which has a depth range of around 1.5m. However, motion planning of the arm does not
have collision checking.
Nevertheless, our package is able to output viewpoints
that are of safe distances from obstacles to enable collision-free search,
leveraging the point cloud received from the Spot's on-board cameras. We use
Spot's off-the-shelf GraphNav service to map the search region (without the
presence of the target objects) and then localize the robot within it.

$\qquad$\vspace{0.3in}
\begin{figure}[H]
  \centering
  \includegraphics[width=0.8\linewidth]{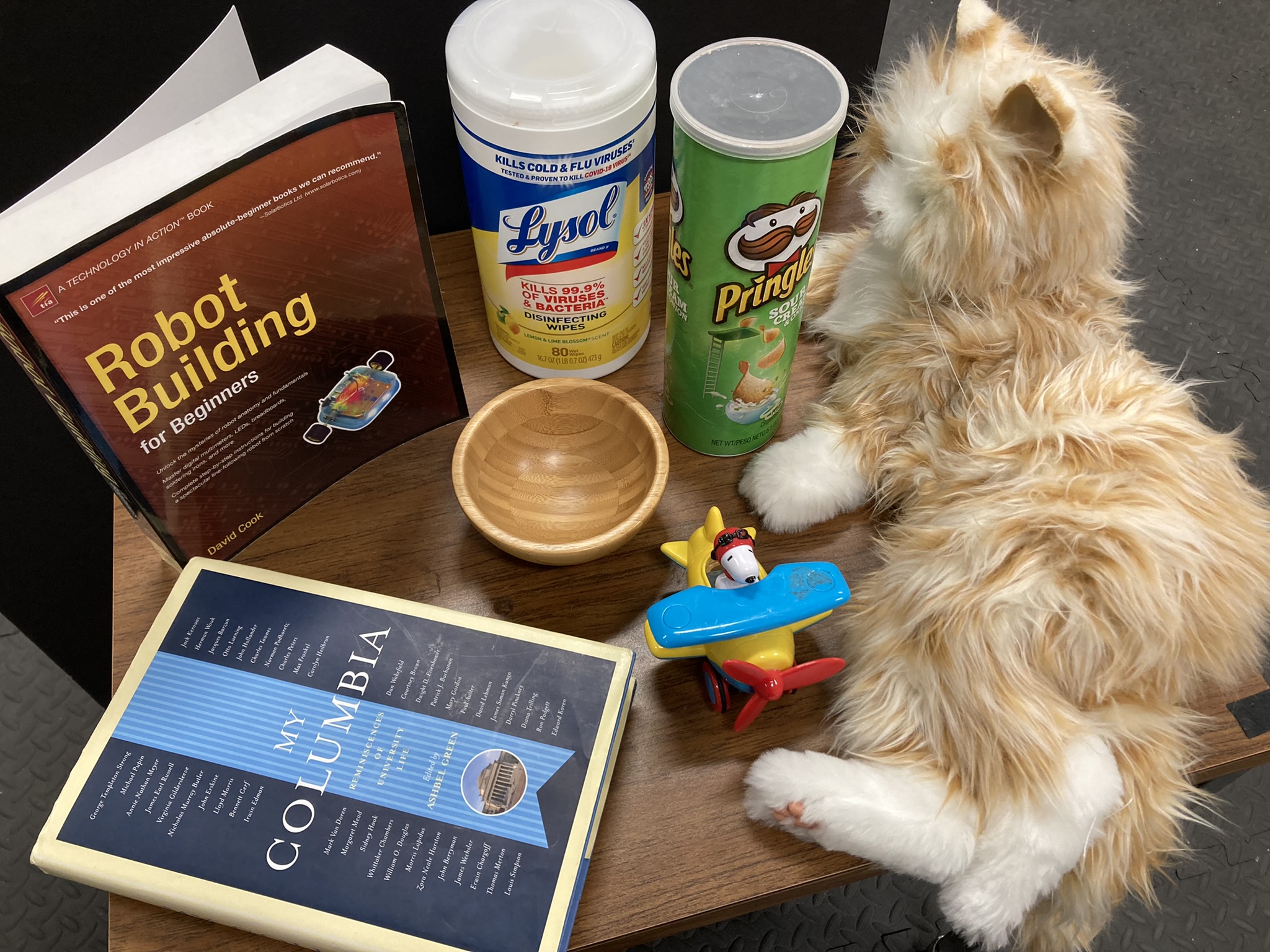}
  \caption{Candidate target objects in our evaluation. From left to right, the object labels are: Columbia Book, Robot Book, Bowl, Lysol, ToyPlane, Pringles, and Cat.}
  \label{fig:objects}
\end{figure}

We task the robot to search in 2 different local regions in different rooms of our lab (Figure~\ref{fig:regions}). The first region (of size 9m$^2\times 1.5$m) consists of two tables and a separation board which creates occlusion; The target objects can be on the floor, or on or under tables; Note that our system is given only point cloud observations to infer potential target locations. The second region (of size 7.5m$^2\times 2.2$m)  is a kitchen area,
where target objects can be on the countertop, on or underneath the couch, on the shelf, or in the sink. In both environments, the resolution of the octree belief is set to 0.001$m^3$ with a size of 32$\times$32$\times$32. The robot is given at most 10 minutes to search. We collected a dataset of 230 images and trained a YOLOv5 detector \cite{glenn_jocher_2020_4154370} with 1.9 million parameters for the objects of interest (Figure~\ref{fig:objects}). We project the 2D bounding box to 3D using the depth image from the gripper camera.

\newpage
$\qquad$\vspace{0.3in}
\begin{figure}[H]
  \centering
  \includegraphics[width=\linewidth]{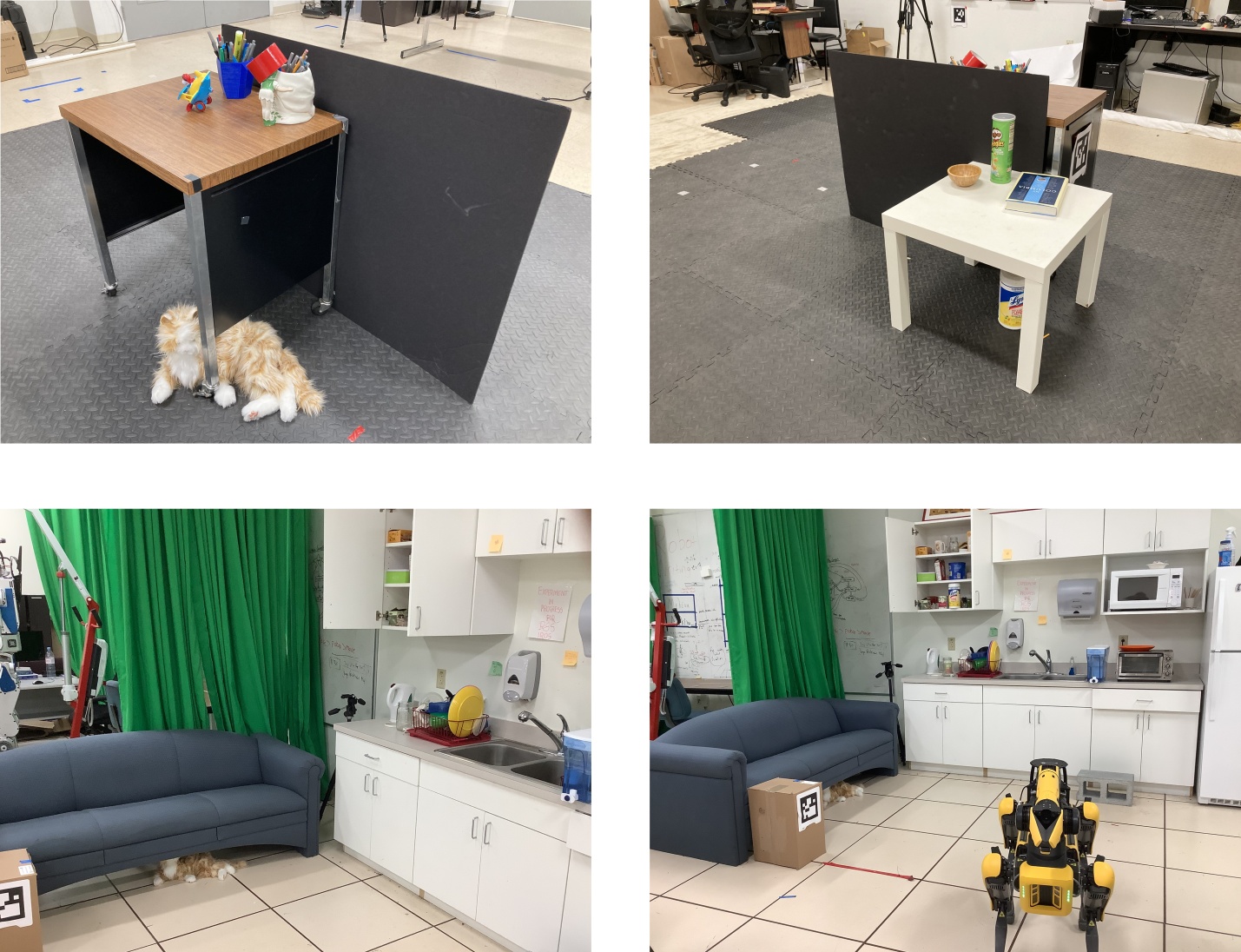}
  \caption{Local regions in our evaluation with Spot. Upper two: two views of
    the arranged tables region. A black board separates the two tables to block
    the view from one side to the other. Bottom two: two views of the kitchen region, with a couch, a countertop, and a shelf.
  }
  \label{fig:regions}
\end{figure}

Figure~\ref{fig:examples} contains illustrations of several key frames during the search trials in both regions. Video footages of the search together with belief state visualization are available in the supplementary video. In the arranged tables region, our system enables Spot to simultaneously search for four objects (Cat, Pringles, Lysol, and ToyPlane), and successfully find three objects in 6.5 minutes. In the kitchen region, our system enables Spot to find a Cat placed underneath the couch within one minute.  However, we do observe that search success deteriorates due to false negatives from the object detector, as well as conservative viewpoint sampling for obstacle avoidance, which prevents the robot to plan top-down views from above the countertop, for example. Overall, our system enables the robot to search for objects in different environments within a moderate time budget.

\begin{figure}[H]
  \centering
\makebox[\textwidth][c]{\includegraphics[width=1.18\textwidth]{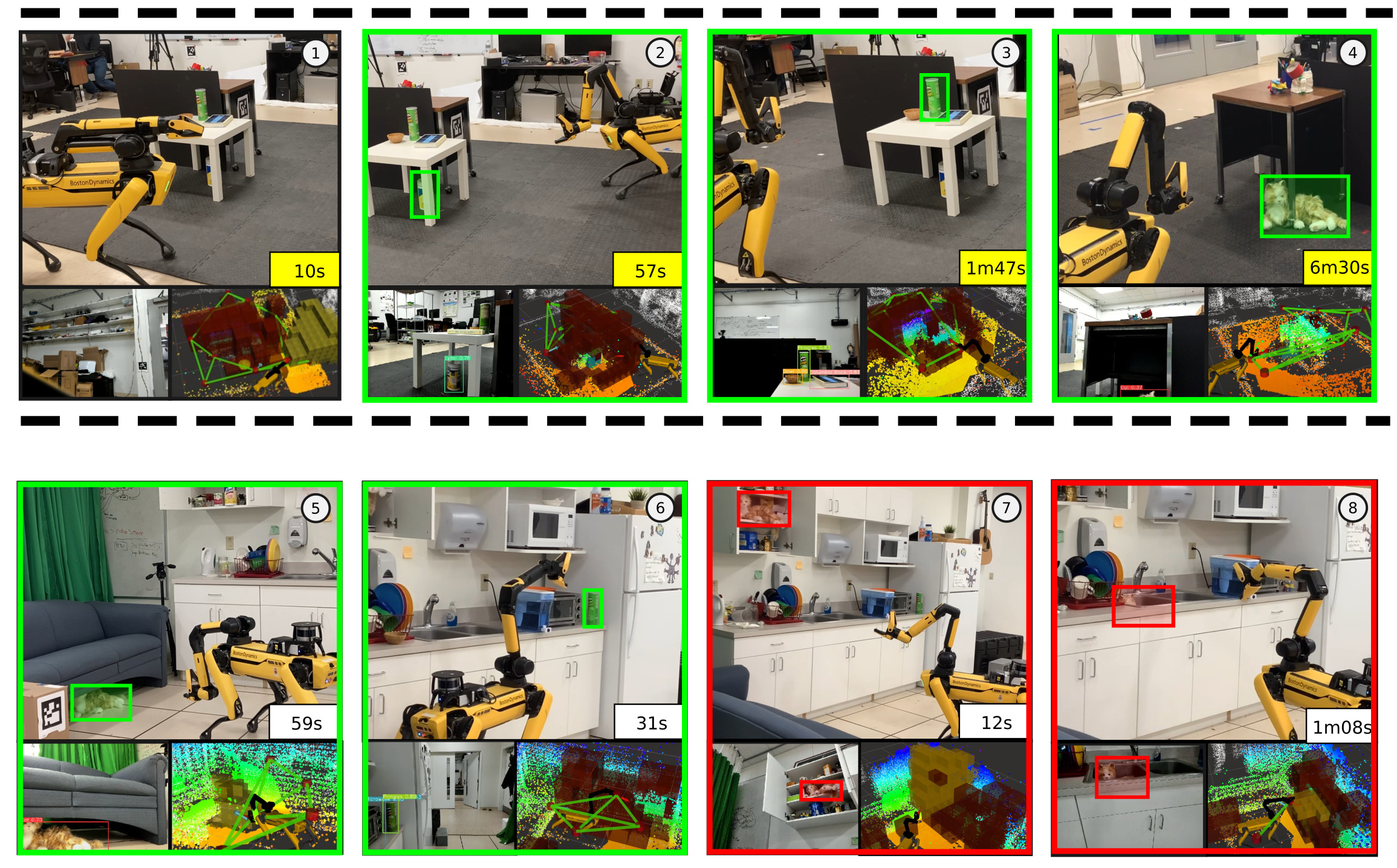}}
  \caption{Key frames from local region search trials. Each frame consists of three images: a third-person view (top), an image from Spot's gripper camera with object detection (bottom left), and a combined visualization of the octree belief, viewpoint graph, and local point cloud observations. Green boxes indicate successfully finding the marked object. Red boxes indicate failure of finding the object due to false negatives in object detection. The yellow or white box on the right of each frame indicates the amount of time passed since the start of the search. Frames at the top row belong to a single trial in the table region, while frames at the bottom row belong to distinct trials in the kitchen region. The top row (1-4) shows that GenMOS enables Spot to successfully find multiple objects in the table region: Lysol under the white (2), Pringles on the white table (3), and the Cat on the fllor under the wooden table (4). The bottom row shows that GenMOS enables Spot to find a Cat underneath the couch (5), and the Pringles at the countertop corner (6). (7-8) shows a failure mode, where the GenMOS plans a reasonable viewpoint, while the object detector fails to detect the object (Cat) on the shelf or in the sink. Video: \url{https://youtu.be/TfCe2ZVwypU}}
  \label{fig:examples}
\end{figure}

\begin{figure}[H]
  \centering
  \includegraphics[width=\linewidth]{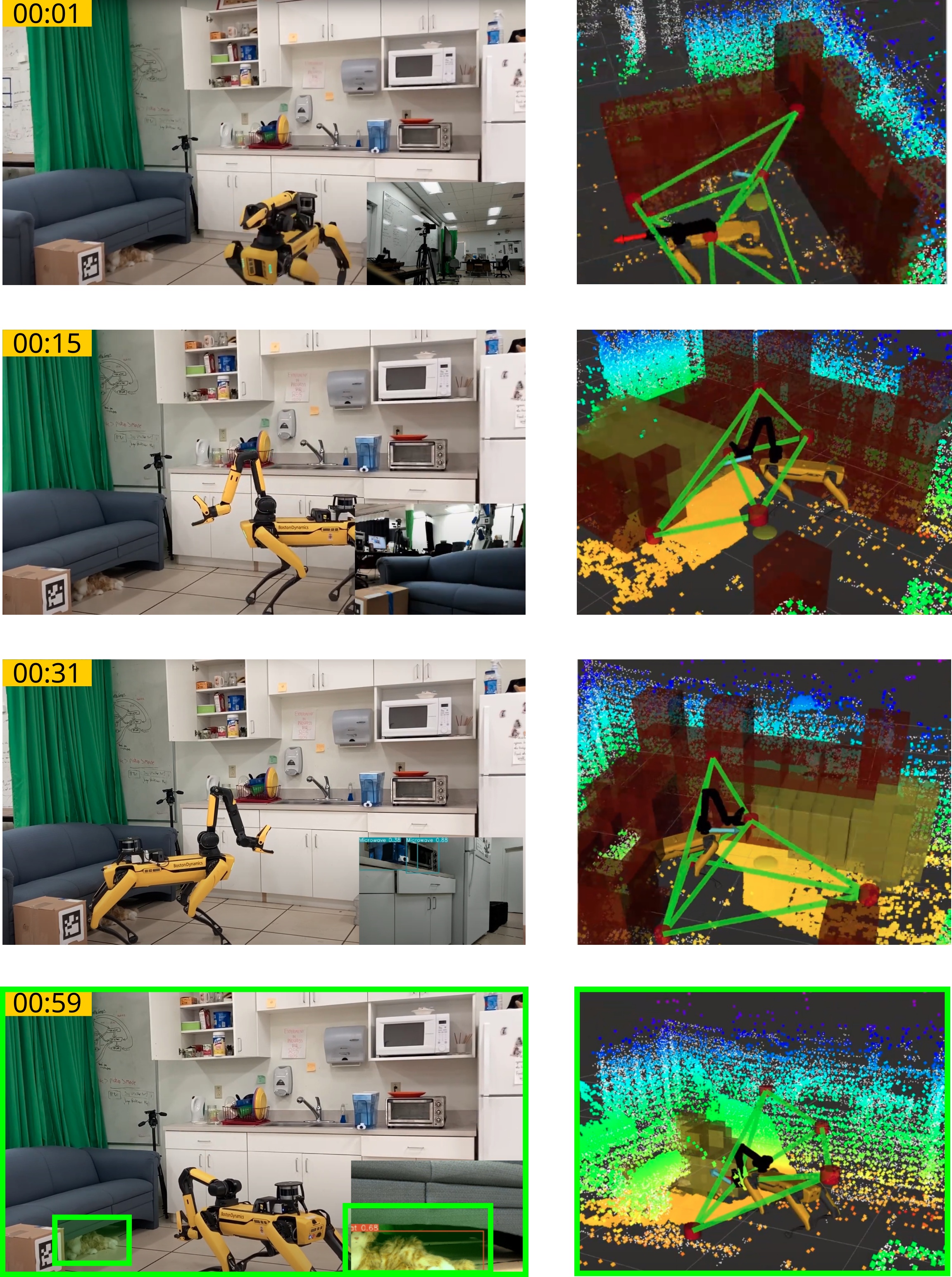}
  \caption{Sequence of frames from the search trial where Spot is tasked to find the toy cat under the couch. GenMOS enables Spot to find the hidden cat under one minute. Red boxes represent octree belief, initialized based on occupancy.
  }
  \label{fig:spot_cat_seq}
\end{figure}

\begin{figure}[H]
  \centering
\makebox[\textwidth][c]{\includegraphics[width=1.18\textwidth,draft=false]{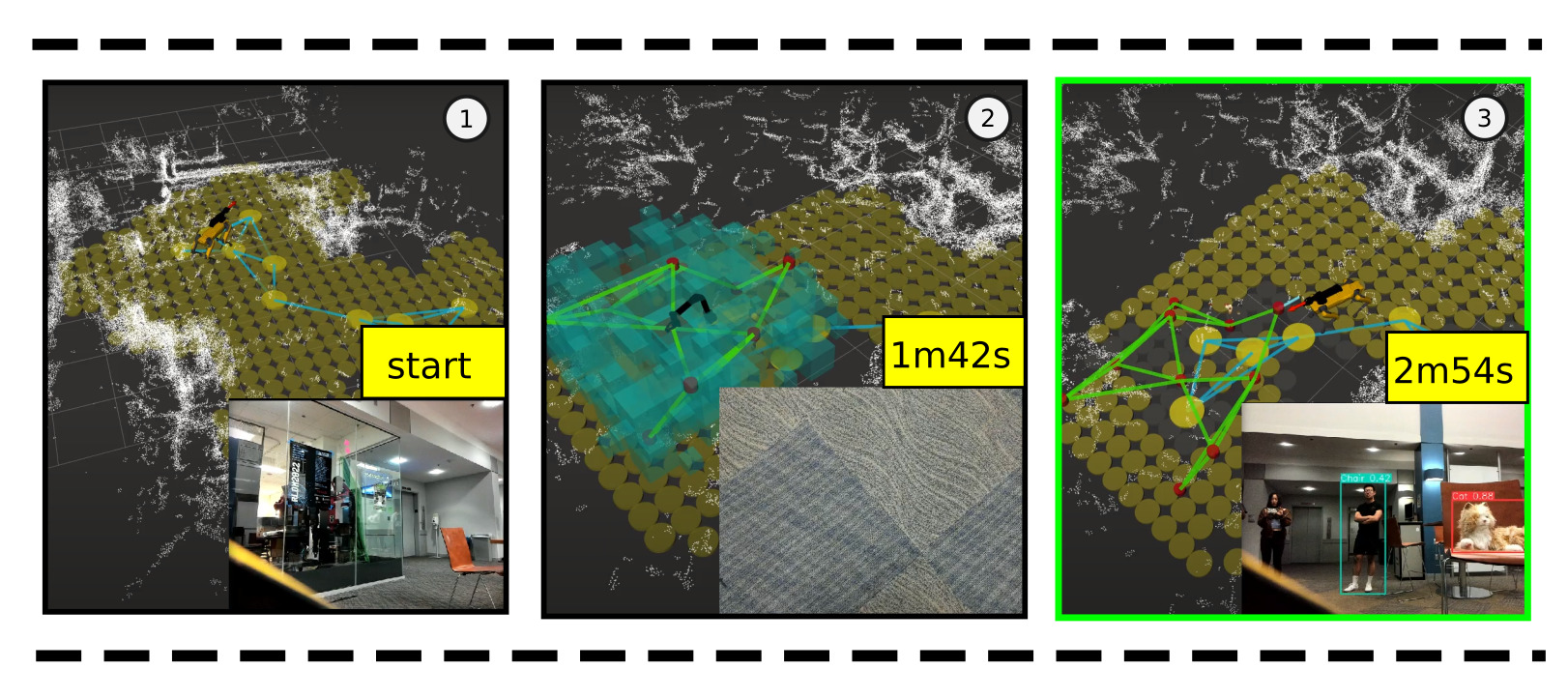}}
  \caption{Demonstration of hierarchical planning where a 2D global search is integrated
    with 3D local search through the \emph{stay} action \cite{zheng2022towards}. This system enables the Spot robot to find a Cat in a lobby area within 3~minutes. (1)~Initial state; (2)~searching in a 3D local region; (3)~the robot detects the Cat and the search finishes.}
  \label{fig:lobby}
\end{figure}

\subsubsection{Extension to Hierarchical Planning}
\label{sec:hier}
We envision the integration of our 3D local search algorithm with a global search algorithm so that a larger search space can be handled. To this end, we implemented a hierarchical planning algorithm that contains a 2D global planner (with the same multi-object search POMDP model as in Chapter \ref{ch:3dmos} but in 2D), where the global planner has a \emph{stay} action (no viewpoint change) which triggers
the initialization of a 3D local search agent. In particular, our implementation uses the \texttt{ListenServer} streaming RPC; when the planner \emph{decides} to search locally, we let the server send a message that triggers the client to send over an \texttt{UpdateSearchRegion} request to initialize the local 3D search agent.

The starting belief of the 3D local agent is initialized based on the 2D global belief; the 2D global belief is in turn updated by projecting the 3D field of view down to 2D. We set the resolution of 2D search to be 0.09$m^2$, and the resolution of 3D search to be 0.001$m^3$. We test this system in a lobby area of size 25m$^2\times 1.5$m, where the robot is tasked to find the toy cat on a tall chair (Figure~\ref{fig:lobby}). The search succeeded within three minutes, covering roughly 15m$^2$.

\newpage
$\qquad$\vspace{0.3in}
\begin{figure}[H]
  \centering
  \includegraphics[width=\linewidth]{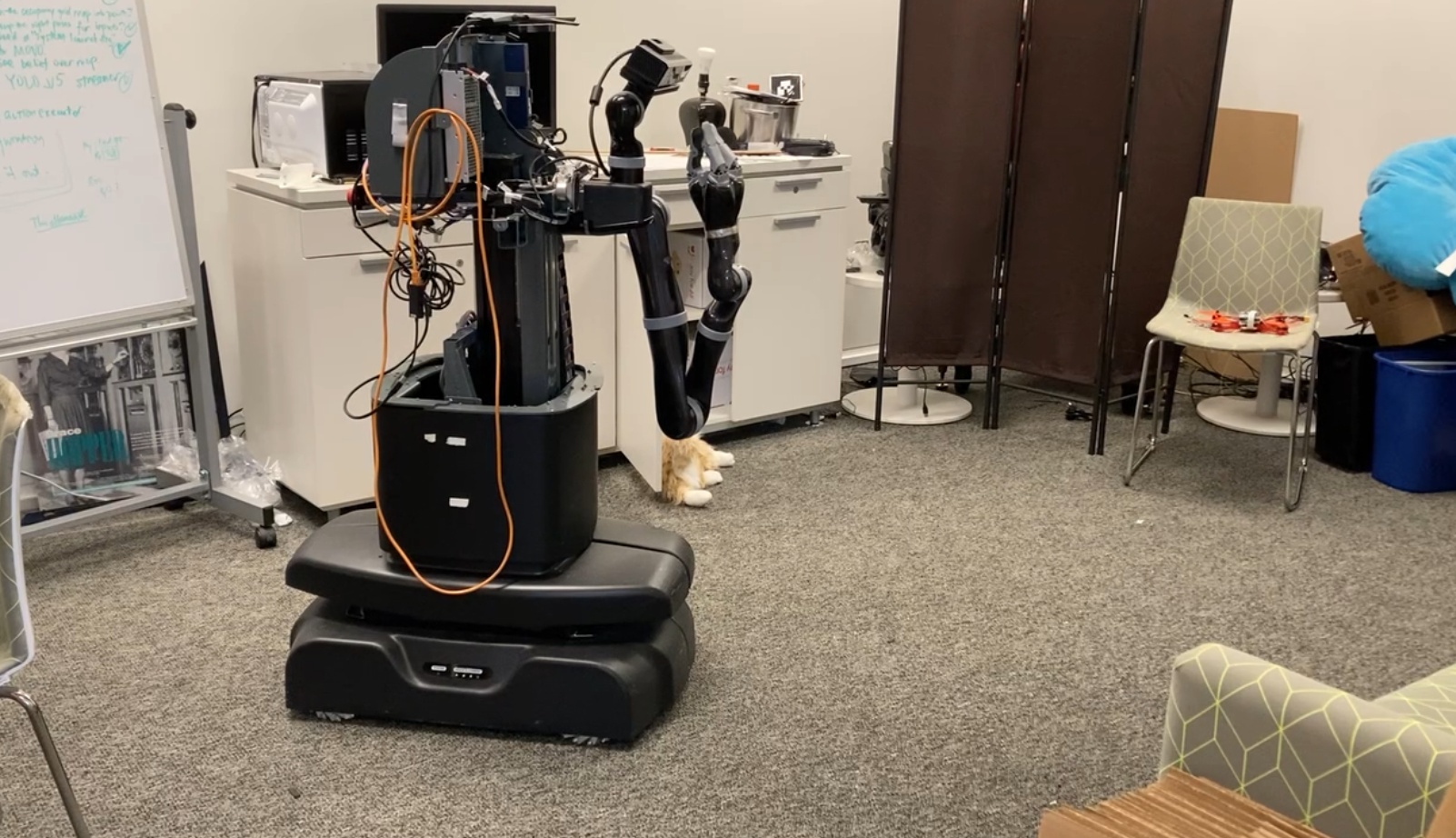}
  \caption{Test environment for object search with MOVO using GenMOS. The target object is, again, the toy cat. In this case, it is lying on the floor next to the opened door.
  }
  \label{fig:movo_setup}
\end{figure}

\subsection{Deployment on the Kinova MOVO Robot}
We additionally deployed GenMOS to the Kinova MOVO mobile manipulator, a robot
with a mobile base, an extensible torso, and a head that can pan and tilt, and
it is equipped with a Kinect V2 RGBD camera.  Similar to Spot, we deployed
GenMOS to MOVO by integrating the GenMOS gRPC client with the perception,
navigation and control stacks of MOVO, which is based on ROS Kinetic. Since the maintenance of MOVO by Kinova has terminated since 2019, deploying GenMOS on MOVO poses a greater challenge compared to Spot. Nevertheless, through Docker \cite{merkel2014docker}, GenMOS was successfully integrated through implementing a client for MOVO, and it enabled MOVO to do object search.

We evaluated the resulting object search
system in a small living room environment (Figures~\ref{fig:movo_example1} and \ref{fig:movo_example2}). The robot is able to perform search
and successfully finds a toy cat on the floor in around 2 minutes. Compared to Spot,
however, MOVO is less agile and prone to collision with obstacles while
navigating between viewpoints during the search.

\begin{figure}[H]
  \centering
  \includegraphics[width=0.81\linewidth,draft=false]{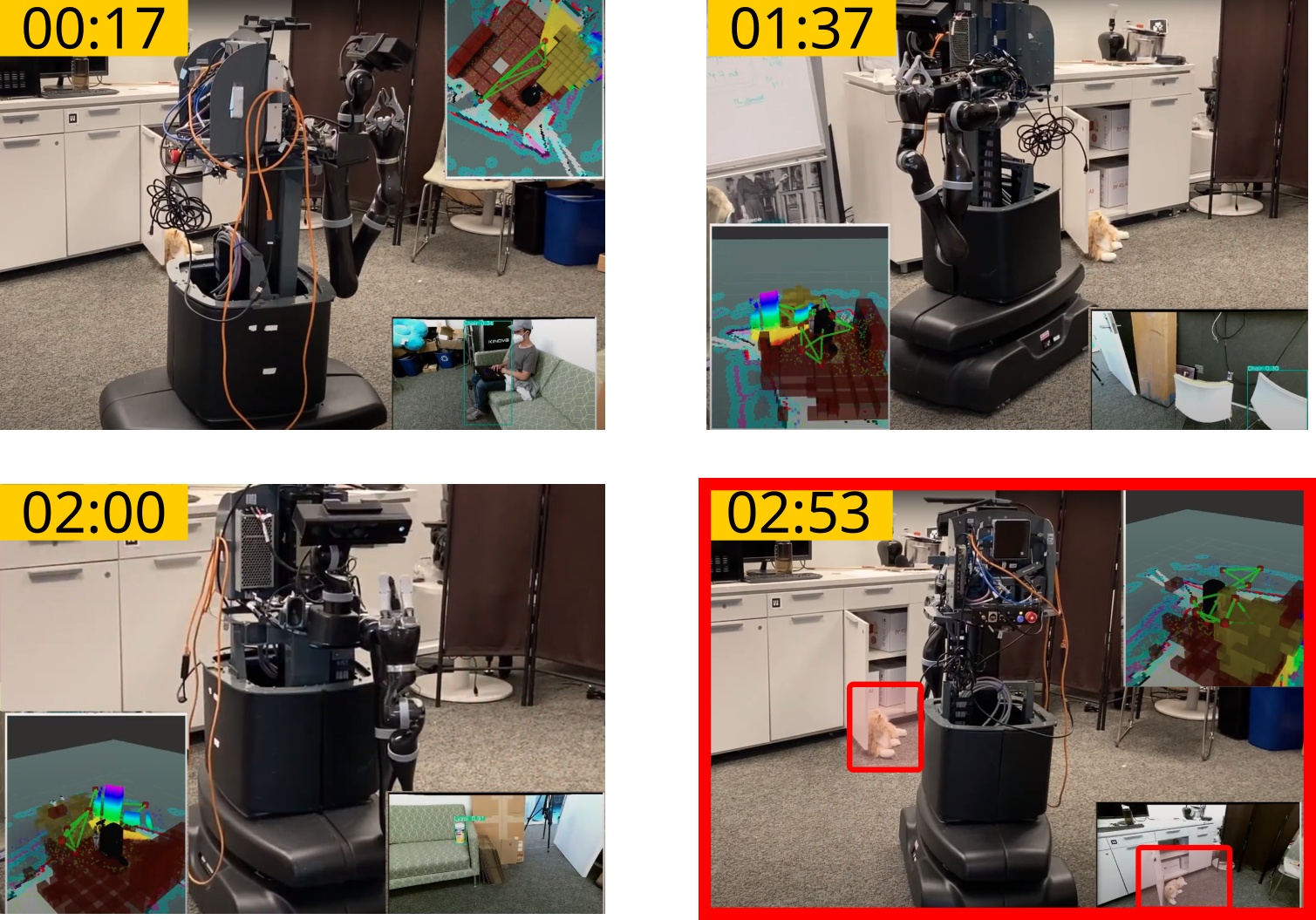}
  \caption{Here, the toy cat lies on the floor next to the opened door. MOVO eventually looked in the right direction, but the object detector failed to recognize it.
  }
  \label{fig:movo_example1}
\end{figure}

\begin{figure}[H]
  \centering
  \includegraphics[width=0.81\linewidth,draft=false]{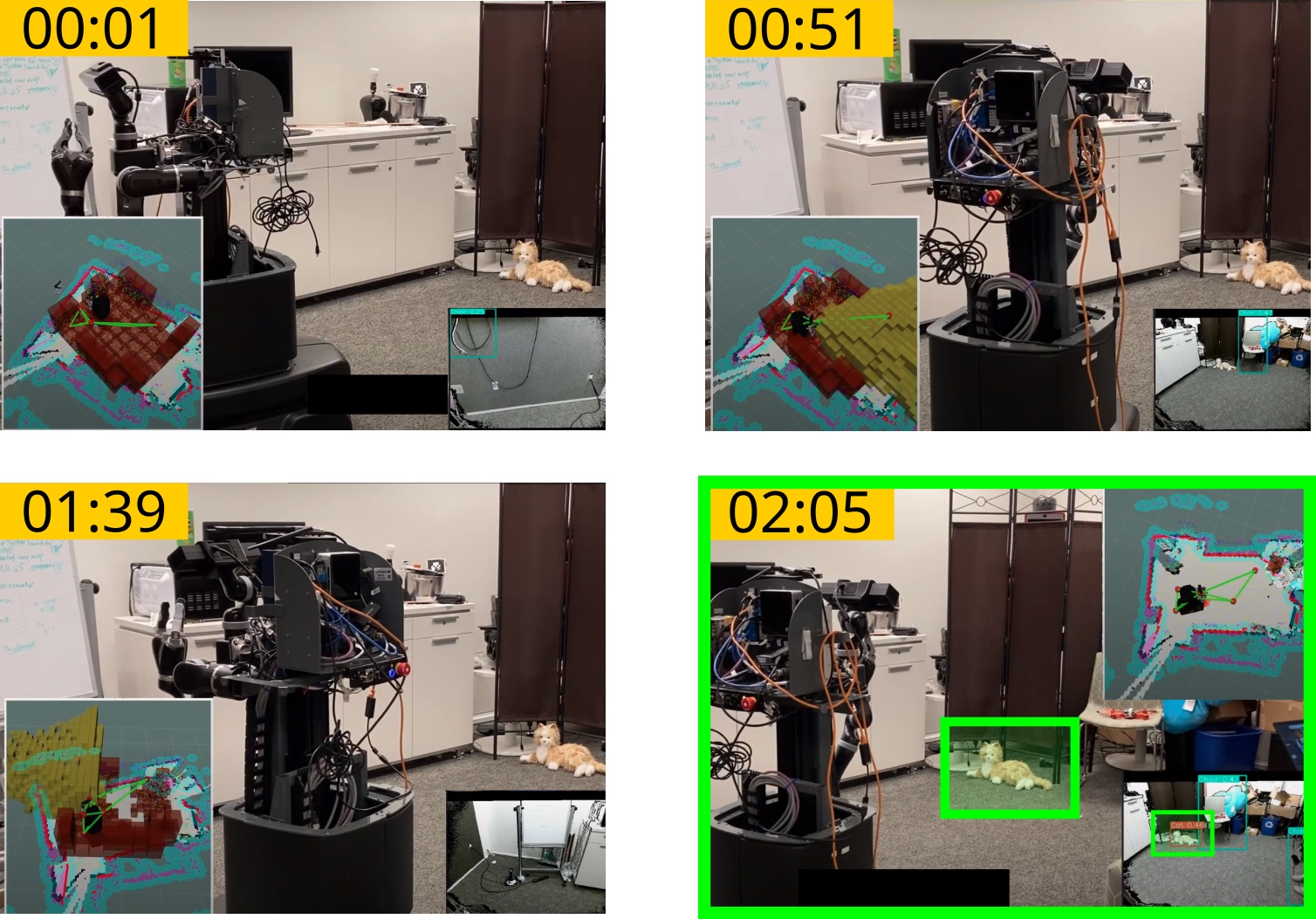}
  \caption{Here, the toy cat is in front of the room divider. Although the detector failed at first (2), MOVO recovered and found the target on the second try (4).
  }
  \label{fig:movo_example2}
\end{figure}

\newpage
\begin{figure}[H]
  \centering
  \includegraphics[width=0.5\linewidth]{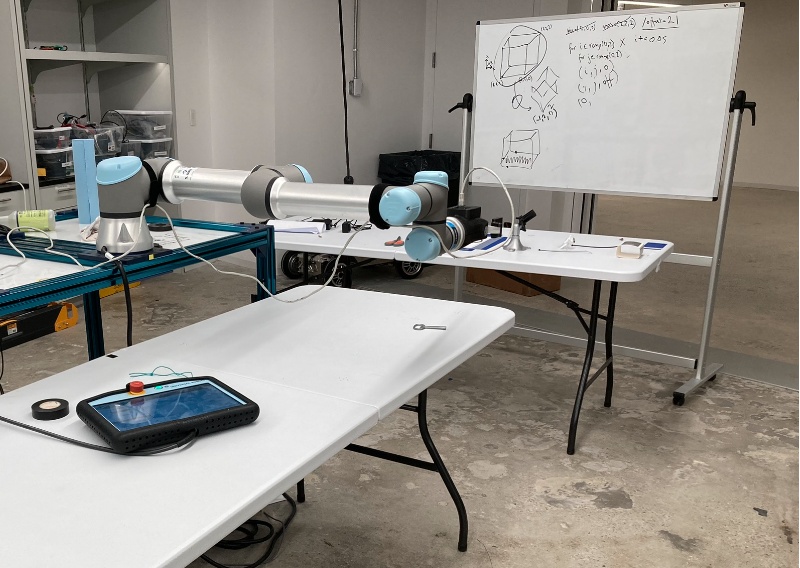}
  \caption{Test environment for UR5e. The robot's gripper had a camera. The target object is a cup placed either on or slightly under the farther table, initially out of sight for the gripper camera.
  }
  \label{fig:ur5e_setup}
  \vspace{-0.3in}
\end{figure}

\subsection{Deployment on the Universal Robotics UR5e Robotic Arm}
Finally, we integrate GenMOS with the UR5e robotic arm on yet a different middleware, Viam.\footnote{As an alternative to ROS, Viam (\url{https://www.viam.com/}) aims to provide fast configuration of robot hardware and distributed robot systems through an extensive web interface that accelerates collaboration as well as standardized services for vision and motion planning as building blocks.}. The UR5e arm is mounted on a table and it is equipped with a camera on its gripper. We applied an off-the-shelf RGB object detection model trained on MS-COCO \cite{lin2014microsoft} and tasked the arm to find a red cup. The cup is initially out of sight, either on or below a different table. Since object detection lacks depth, we considered label-only detection (\ie, discarding the 2D bounding box and only keeping the event that a certain object is detected). As discussed in Section~\ref{sec:genmos:server} (page \pageref{sec:genmos:server:2d_detection}), GenMOS can accommodate to such a scenario as it is expected to look from different viewpoints to reduce the region of uncertainty, once a detection is made. Indeed, we observed this type of behavior on UR5e (see Figure~\ref{fig:ur5e_example1}).

A few caveats should be noted about this particular system. I expected the arm to be able to reliably motion plan to viewpoints produced by GenMOS. However, in practice, motion planning frequently fails.\footnote{This is in part due to the fact that motion services with Viam were in early phases of development when I visited, but also that GenMOS does not internally consider kinematic constraints; it works if motion planning fails a few times as GenMOS would simply replan, but the failure was too frequent at the time to keep trying motion planning to any viewpoint produced by GenMOS .} As a work-around that still tests GenMOS's object search planning ability, I predefined a set of viewpoints for which motion planning works, and when a viewpoint is planned by GenMOS, the arm is moved to the closest viewpoint in this set. Another issue that I worked around is that the detector works poorly if the image is rotated due to gripper rotation. So, I made the gripper automatically level every time it reaches a destination view pose, which was done by commanding the last joint to offset the end effector's rotation around the wrist.

\begin{figure}[ph!]
  \centering
  \includegraphics[width=\linewidth]{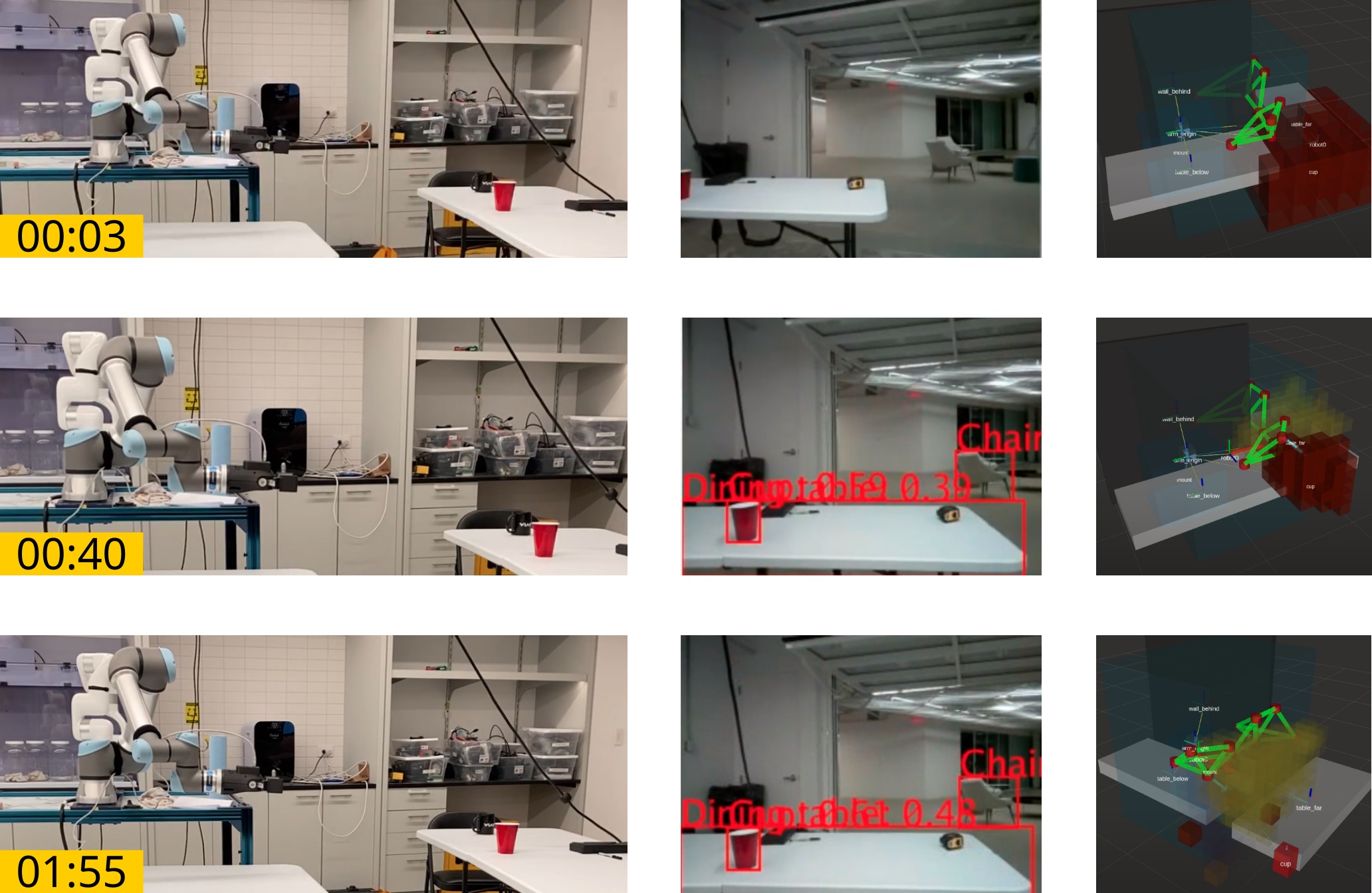}
  \caption{The UR5e arm moves back and forth, reducing uncertainty to a few grids.
  }
  \label{fig:ur5e_example1}
\end{figure}

\begin{figure}[ph!]
  \centering
  \includegraphics[width=\linewidth]{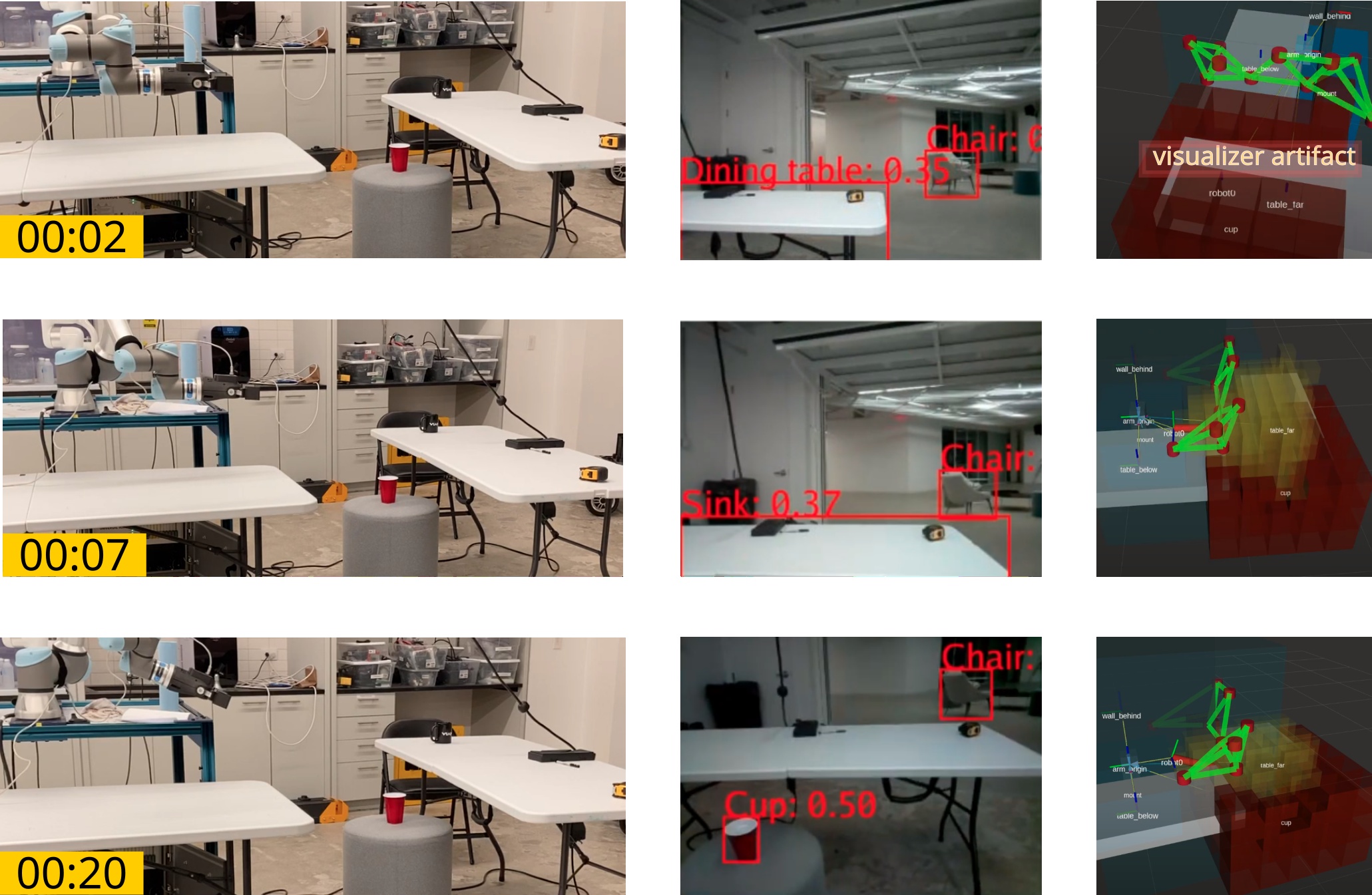}
  \caption{The UR5e arm looks down; yet belief update missed the detection.
  }
  \label{fig:ur5e_example2}
\end{figure}


\chapter{Correlational Object Search}
\label{ch:cospomdp}
\section{Motivation - Finding Hard-to-Detect Objects}
\lettrine{O}{bject} search can make a difference in many applications
including domestic services \cite{sprute2017ambient,zeng2020semantic}, search
and rescue \cite{eismann2009automated,sun2016camera}, and elderly care
\cite{idrees2020robomem,loghmani2018towards}. In realistic settings, however, the object being searched for will often be small, outside the current field of view, and hard to detect. For example, a household robot must be very close to a fork in order to be able to detect it; likewise, a warehouse robot may have to locate a particular machine within a very large factory. To be effective, the robot must generate efficient search strategies that require as few timesteps as possible.

\vspace{-0.2in}
\subsection{Why Correlations?}

In such settings, when the target object is hard to detect, \emph{correlational information} can be extremely useful. Specifically, suppose the robot is equipped with a prior about the relative spatial locations of object types (\eg, refrigerators tend to be near forks). Then, it can leverage this information as a powerful heuristic to narrow down or ``focus'' the search space, by first focusing its efforts on locating easier-to-detect objects that are highly correlated with the target object, and only then focusing on locating the target object itself. Doing so has the potential to greatly improve search efficiency, as the robot no longer needs to waste time considering strategies that, \eg, search for a fork in a bathroom. Unfortunately, previous approaches to object search with correlational information tend to resort to ad-hoc or greedy search strategies~\cite{aydemir2013active,kollar2009utilizing,zeng2020semantic}, which may not scale well to complex environments.

\begin{figure}[t]
  \centering
  \noindent
    \includegraphics[width=\columnwidth, draft=false]{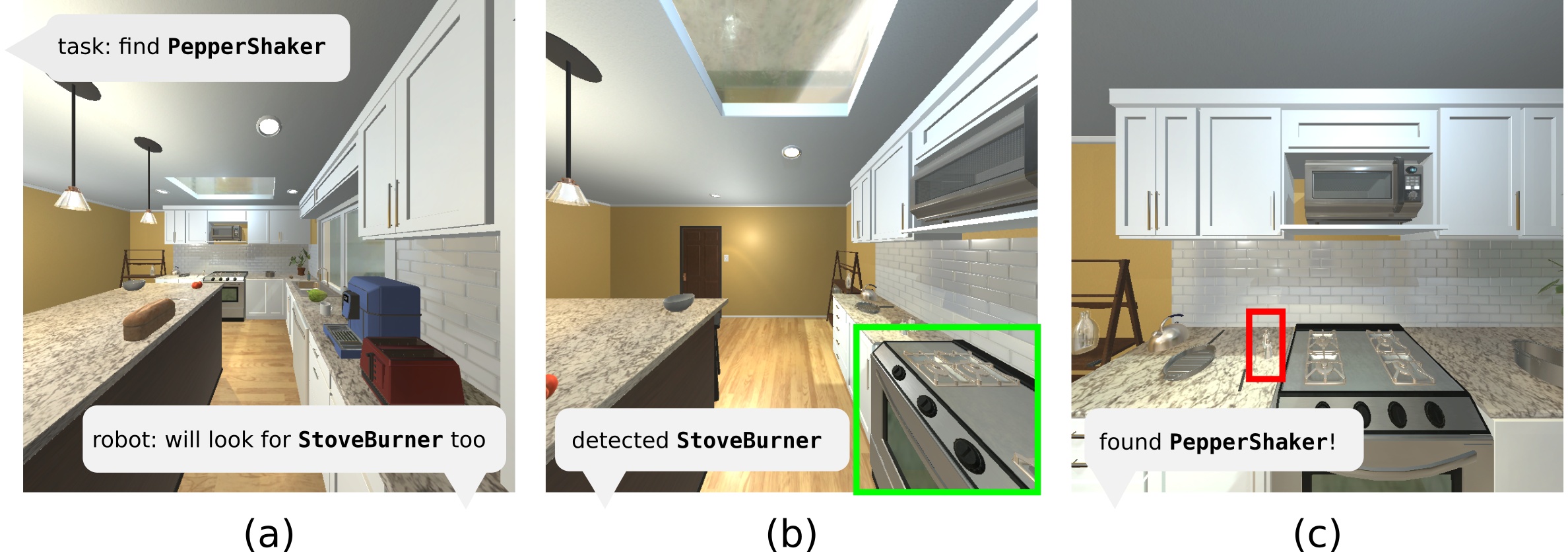}
    \caption{We study the problem of object search using correlational information about spatial relations between objects. This example illustrates a desirable search behavior in an AI2-THOR scene, where the robot leverages the detection of a \texttt{StoveBurner} to more efficiently find a hard-to-detect \texttt{PepperShaker}.}
  \label{fig:teaser}
\end{figure}

\subsection{Remark on Previous Work}

We follow a long line of work that models the object search problem as a partially observable Markov decision process (POMDP) \cite{aydemir2013active,li2016act,xiao_icra_2019,wandzel2019multi,zheng2021multi}. This formalization is useful because object search over long horizons is naturally a sequential, partially observed decision-making problem: (1) the robot must search for the target object by visiting multiple viewpoints in the environment sequentially, and (2) the robot must maintain and update a measure of uncertainty over the location of the target object, via its belief state.

\citet{garvey1976} and \citet{wixson1994using} pioneered the paradigm of \emph{indirect search}, where an intermediate object (such as a desk) that is typically easier to detect is located first, before the target object (such as a keyboard). More recently, probabilistic graphical models have been used to model object-room or object-object spatial correlations \cite{aydemir2013active, zeng2020semantic,kollar2009utilizing,lorbach2014prior}.
In particular, \citet{zeng2020semantic}~proposed a factor graph representation for different types of object spatial relations. Their approach produces search strategies in a greedy fashion by selecting the next-best view to navigate towards, based on a hybrid utility of navigation cost and the likelihood of detecting objects. In our evaluation, we compare our sequential decision-making approach with a greedy, next-best view baseline based on that work \cite{zeng2020semantic}.

Recently, the problem of semantic visual navigation \cite{zhu2017target,batra2020objectnav,wortsman2019learning,qiu2020learning,mayo2021visual} received a surge of interest in the deep learning community. In this problem, an embodied agent is placed in an unknown environment and tasked to navigate towards a given semantic target (such as ``kitchen'' or ``chair''). The agent typically has access to behavioral datasets for training on the order of millions of frames and the challenge is typically in generalization. Our work considers the standard evaluation metric (SPL \cite{anderson2018evaluation}) and task success criteria (object visibility and distance threshold \cite{batra2020objectnav}) from this body of work. However, our setting differs fundamentally in that the search strategy is not a result of training but a result of solving an optimization problem.

\section{Contributions}
In this work, we make the following contributions:
\begin{itemize}[itemsep=0.5pt,topsep=0pt]

\item We propose COS-POMDP, which contains a correlation-based observation model that captures spatial relations between objects

\item We prove that COS-POMDPs produce equivalent solutions to a naive formulation where the state space is a joint over all object locations, despite COS-POMDPs having a much smaller state space;

\item We address scalability by proposing a hierarchical
  planning algorithm, where a high-level COS-POMDP plans subgoals, each fulfilled by a low-level planner that plans with low-level actions (\ie, given primitive actions); both levels plan online based on a shared and updated COS-POMDP belief state, enabling closed-loop planning;

\item We investigate the influence of correlational information when searching for hard-to-detect targets, and the benefit of optimizing for a sequence of actions as opposed to selecting the next-best view.

\item We conduct experiments in AI2-THOR \cite{kolve2017ai2}, a realistic simulator of household environments, and use YOLOv5 \cite{redmon2016you,glenn_jocher_2020_4154370} as the object detector.

  \begin{itemize}[itemsep=0.5pt,topsep=0pt, label=$\circ$]
  \item Our results show that, when the given correlational information is accurate, COS-POMDP leads to more robust search perfomance when the target object is hard to detect. In particular, for target objects with a true positive detection rate below 40\%,
  COS-POMDP improves the POMDP baseline that ignores correlational information by 70\% and a greedy, next-best view baseline by 170\%, in terms of the SPL \cite{anderson2018evaluation} metric, commonly used for evaluating navigation agents in simulated environments \cite{wortsman2019learning,qiu2020learning,batra2020objectnav}.
  \end{itemize}

\end{itemize}


\section{Problem Formulation}
\label{sec:prob}
We formulate correlational object search as a planning problem, where a robot must search for a target object given correlational information with other objects in the environment. We begin by describing the underlying search environment and the capabilities of the robot. Then we define the inputs to the robot and the solution expected to be produced by the robot to solve this problem.

\subsection{Search Environment and Robot Capabilities}
\label{sec:cos_envrob}
The search environment contains a target object and $n$ additional static objects. The set of possible object locations is discrete, denoted as $\mc{X}$. The locations of the target object $\xtarget\in\mc{X}$ and other objects $x_1,\ldots, x_n\in \mc{X}$ are unknown to the robot, and follow a latent joint distribution $\Pr(x_1,\ldots,x_n,\xtarget)$. The robot is given as input a factored form of this distribution, defined later in Sec.~\ref{sec:cos}.

The robot can observe the environment from a discrete set of viewpoints, where each viewpoint is specified by the position and orientation of the robot's camera. These viewpoints form the necessary state space of the robot, denoted as $\Srobot$. The initial viewpoint is denoted as $\srobotinit$. By taking a primitive move action $a$ from the set $\mc{A}_m$, the robot changes its viewpoint subject to transition uncertainty $T_m(\srobot',\srobot,a)=\Pr(\srobot'|\srobot,a)$. Also, the robot can decide to finish a task at any timestep by choosing a special action \DECLARE, which deterministically terminates the process.

At each timestep, the robot receives an observation $z$ factored into two independent components $z=(\zrobot,\zobjects)$. The first component $\zrobot\in\Srobot$ is an estimation of the robot's current viewpoint following the observation model $\Orobot(\zrobot,\srobot)=\Pr(\zrobot|\srobot)$. The second component $\zobjects=(z_1,\ldots, z_n,\ztarget)$ is the result of performing object detection. Each element, $z_i\in\mc{X}\cup\{\texttt{null}\}$, $i\in\{1,\ldots,n,\text{target}\}$, is the detected location of object $i$ within the field of view, or \texttt{null} if not detected. The observation $z_i$ about object $i$ is subject to limited field of view and sensing uncertainty captured by a \emph{detection model} $D_i(z_i,x_i,\srobot)=\Pr(z_i|x_i,\srobot)$; Here, a common conditional independence assumption in object search is made \cite{zeng2020semantic,wandzel2019multi}, where $z_i$ is conditionally independent of the observations and locations of all other objects given its location and the robot state $\srobot$. The set of detection models for all objects is $\mc{D}=\{D_1,\ldots,D_n,D_{\text{target}}\}$. In our experiments, we obtain parameters for the detection models based on the performance of the vision-based object detector (Sec.~\ref{sec:exp:detectors}). 

\subsection{The Correlational Object Search Problem}\label{sec:cos}
Although the joint distribution of object locations is latent, the robot is assumed to have access to a factored form of that distribution, that is, $n$ conditional distributions, $\mc{C}=\{C_1,\ldots,C_n\}$ where $C_i(x_i,\xtarget)=\Pr(x_i|\xtarget)$ specifies the spatial correlation between the target and object $i$. We call each $C_i$ a \emph{correlation model}. This model can be learned from data or specified based on environment-specific knowledge.

The robot performs search by taking a sequence of move actions to observe different parts of the environment, and terminates the search by taking \DECLARE. We are now ready to define the \emph{correlational object search} problem:

\textbf{Problem 1 (Correlational Object Search).} Given as input a tuple $$(\mc{X},\mc{C},\mc{D},\srobotinit, \Srobot, \Orobot, \mc{A}_m,T_m),$$ the robot must perform a sequence of actions, $a_{1:T}=(a_1,\ldots,a_T)$ of length $T\geq 1$, where $a_1,\ldots,a_{T-1}\in\mc{A}_m$ and $a_T$ is \DECLARE. The action sequence $a_{1:T}$ is called a \emph{solution}. A solution is \emph{successful} if the robot state sequence and the target location satisfy certain criteria upon taking the \DECLARE{} action. In our evaluation in AI2-THOR, we use the success criteria recommended by \citet{batra2020objectnav} and the commonly-used SPL metric proposed by \citet{anderson2018evaluation} to measure the efficiency of successful searches.
The objective is to produce a successful solution that reaches the target object while minimizing the total distance traveled by the robot.


\section{Correlational Object Search as a POMDP}

The POMDP is an extensively studied framework for optimizing
sequential decisions under partial observability and uncertainty in motion and
sensing.  Both challenges (partial observability and uncertainty) considered in POMDP arise naturally in object search.
In addition, the objective of minimizing the navigation distance while
successfully finding the target can be represented by the POMDP objective of
maximizing the discounted cumulative rewards.
Therefore, we model the correlational object search task as a POMDP.

We first provide a condensed review of POMDPs; for more information, we refer
the reader to
\cite{kaelbling1998planning,shani2013survey,somani2013despot,silver2010monte}. Then,
we present the COS-POMDP, a POMDP formulation that addresses the correlational
object search problem, followed by a discussion on its optimality. COS-POMDP expands the observation space of the overarching object search POMDP (Chapter~\ref{ch:overarching}) by considering observations about objects other than the target object, and formulates the corresponding observation model using spatial correlation.


\subsection{COS-POMDP}
\label{sec:cospomdps}

Given an instance of the correlational object search problem defined in Sec.~\ref{sec:cos}, we define the \textbf{C}orrelational \textbf{O}bject \textbf{S}earch POMDP (COS-POMDP) as follows:
\begin{itemize}
\item\textbf{State space.} The state space $\mc{S}$ is factored to include the robot state $\srobot\in\Srobot$ and the target state $\xtarget\in\mc{X}$. A state $s\in\mc{S}$ can be written as $s=(\srobot,\xtarget)$. Importantly, no other object state is included in $\mc{S}$.

\item\textbf{Action space.}
 The action space is $\mc{A}=\mc{A}_m\cup\{\DECLARE\}$. 

\item\textbf{Observation space.}
The observation space $\mc{Z}$ is factored over the objects, and each $z\in \mc{Z}$ is written as $z=(\zrobot,\zobjects)$, where $\zobjects=(z_1,\ldots,z_n,\ztarget)$.

\item\textbf{Transition model.} The objects are assumed to be static.  Actions
$a_m\in\mc{A}_m$ change the robot state from $\srobot$ to $\srobot'$ according to $T_m$, and taking the
$\DECLARE$ action terminates the task deterministically.

\item\textbf{Observation model.} By definition of $z$, we have
\begin{align}
\Pr(z|s)&=\Pr(\zrobot|\srobot)\Pr(\zobjects|s)\\
&=\Orobot(\zrobot,\srobot)\Pr(\zobjects|s)
\end{align}
Under the conditional independence
assumption in Sec.~\ref{sec:prob}, $\Pr(\zobjects|s)$ can be compactly factored
as:
\begin{align}
&\Pr(\zobjects|s)=\Pr(z_1,\ldots,z_n,\ztarget|\xtarget,\srobot)\\
&\ \ =\Pr(\ztarget|\xtarget,\srobot)\prod_{i=1}^n\Pr(z_i|\xtarget,\srobot)\label{eq:factoring}
\end{align}
The first term in Eq~(\ref{eq:factoring}) is defined by $\Dtarget$, and each $\Pr(z_i|\xtarget,\srobot)$ is called a \emph{correlational observation model},
written as:
\begin{align}
  &\Pr(z_i|\xtarget,\srobot)=\sum_{x_i\in\mc{X}}\Pr(x_i,z_i|\xtarget,\srobot)\\
  &\qquad=\sum_{x_i\in\mc{X}}\Pr(z_i|x_i,\srobot)\Pr(x_i|\xtarget)\label{eq:corr_obz_model}
\end{align}
where the two terms in Eq~(\ref{eq:corr_obz_model}) are the detection model $D_i\in\mc{D}$ and correlation model $C_i\in\mc{C}$, respectively.

\item\textbf{Reward function.} The reward function, $R(s,a)=R(\srobot,\xtarget,a)$, is defined as follows. Upon taking \DECLARE, the task outcome is determined based on $\srobot, \xtarget$,
which is successful if the robot orientation is facing the target and its position is within a distance threshold to the target.
If successful, then the robot receives $\rmax\gg 0$, and $\rmin\ll 0$ otherwise.
  Taking a move action from $\mc{A}_m$ receives a negative reward which corresponds to the action's cost. In our experiments, we set $\rmax=100$ and $\rmin=-100$. Each primitive move action (\eg, \texttt{MoveAhead})  receives a step cost of -1.
\end{itemize}

\subsection{Optimality of COS-POMDPs}
The state space of a COS-POMDP involves only the robot and target object states. A natural question arises: have we lost any necessary information? In this section, we show that
COS-POMDPs are optimal, in the following sense. If we imagine solving a ``full'' POMDP corresponding  to the COS-POMDP, whose state space contains all object states, then the solutions to the COS-POMDP are equivalent. Note that a belief state in this ``full'' POMDP scales exponentially in the number of objects.

We begin by precisely defining the ``full'' POMDP, henceforth called
the F-POMDP, corresponding to a COS-POMDP. The F-POMDP has identical action space, observation space, and transition model as the COS-POMDP. The reward function is also identical since it only depends on the target object state, robot state, and the action taken. F-POMDP differs in the state space and observation model:
\begin{itemize}
\item \textbf{State space:} The state is $s=(\srobot,\xtarget,x_1,\ldots,x_n)$.
\item \textbf{Observation model:} Under the conditional independence assumption stated in Sec.~\ref{sec:prob}, the model for observation $z_i$ of object $x_i$ involves just the detection model: $\Pr(z_i| s)=\Pr(z_i| x_i, \srobot)$.
  \end{itemize}

  Since the COS-POMDP and the F-POMDP share the same action and observation spaces, they have the same history space as well. We first show that given the same policy, the two models have the same distribution over histories.

  \textbf{Theorem 1.}
    Given any policy $\pi:h_t\rightarrow a$, the distribution of histories is identical between the COS-POMDP and the F-POMDP.

  \begin{proof} See Appendix \ref{appdx:proof1} (page~\pageref{appdx:proof1}).
  \end{proof}

    Using Theorem 1, we are equipped to make a statement
    about the value of following a given policy in either the COS-POMDP or the
    F-POMDP.

  \textbf{Corollary 1.} Given any policy $\pi:h_t\rightarrow a$ and $h_t$, the value $V_{\pi}(h_t)$ is identical between the COS-POMDP and the F-POMDP.

    \begin{proof}
      By definition, the value of a POMDP at a history is the expected discounted cumulative reward with respect to the distribution of future action-observation pairs. Theorem~1 states that the COS-POMDP and F-POMDP have the same distribution of histories given $\pi$. Furthermore, the reward function depends only on the states of the robot and the target object. Thus, this expectation is equal for the two POMDPs at any $h$.
    \end{proof}

  Finally, we can show that COS-POMDPs are optimal in the sense as discussed before.

 \textbf{Corollary 2.}
An optimal policy $\pi^*$ for either the COS-POMDP or the F-POMDP is also optimal for the other.

\begin{proof}
  Suppose, without loss of generality, that $\pi^*$ is optimal for the COS-POMDP but not the F-POMDP. Let $\pi'$ be the optimal policy for the F-POMDP. By the definition of optimality, for at least some history $h$ we must have $V_{\pi'}(h) > V_{\pi^*}(h)$. By Corollary 1, for any such $h$ the COS-POMDP also has value $V_{\pi'}(h)$, meaning $\pi^*$ is not actually optimal for the COS-POMDP; this is a contradiction.
\end{proof}

\section{Hierarchical Planning}
Despite the optimality-preserving reduction of state space in a COS-POMDP, directly planning over the primitive move actions is not scalable to practical domains even for state-of-the-art online POMDP solvers \cite{silver2010monte}. This is especially the case when in-place rotation actions are considered, since identical viewpoints may be repeatedly visited at different depth levels in the search tree, limiting the size of the search region considered during planning. At the same time, however, planning POMDP actions at the low level has the benefit of controlling fine-grained movements, allowing goal-directed behavior to emerge automatically at this level. Therefore, we seek an algorithm that can reason about both searching over a large region as well as careful search in the area around the robot. This is practical because typical mobile robots can
be controlled both at the low level of motor velocities and the high level of navigation goals \cite{zheng-ros-navguide, macenski2020marathon}.

\begin{figure*}[t]
    \centering
    \makebox[\textwidth][c]{\includegraphics[width=1.2\textwidth]{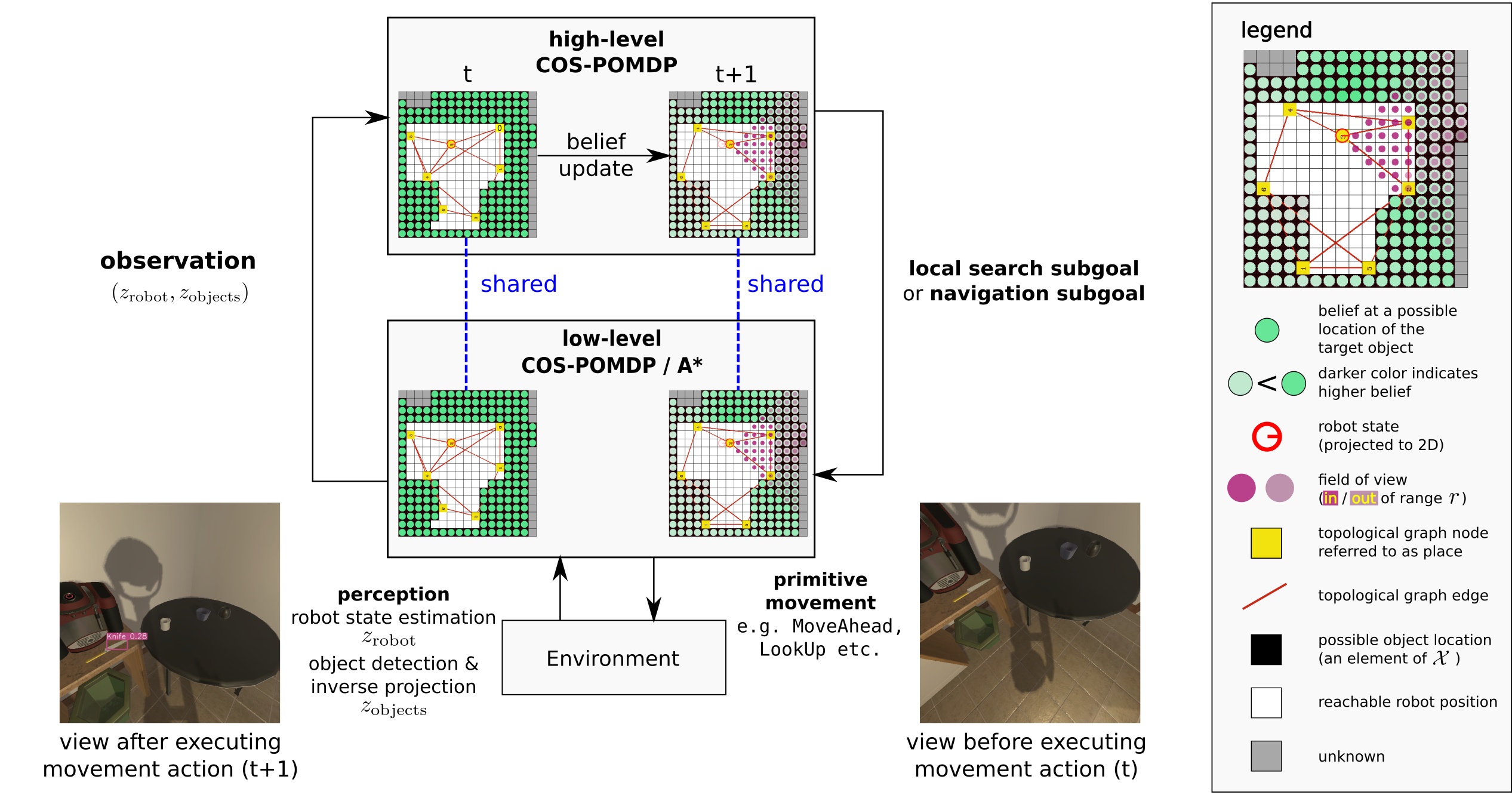}}
    \caption{\textbf{Illustration of the Hierarchical Planning Algorithm.} A high-level COS-POMDP plans subgoals that are fed to a low-level planner to produce low-level actions. The belief state is shared across the levels. Both levels plan with updated beliefs at every timestep.
    }
    \label{fig:method}
\end{figure*}

\begin{algorithm}[t]
\caption{OnlineHierarchicalPlanning}
\label{alg:hierplanning}
\SetArgSty{textup}
\textbf{Input:} {$P=(\mc{X},\mc{C},\mc{D},\srobotinit, \Srobot, \Orobot, \mc{A}_m,T_m)$}.\\
\textbf{Parameter:} {maximum number of steps $T_{\text{max}}$.}\\
\textbf{Output:} {A solution $a_{1:T}$ (Problem 1)}.\\
$\btarget^1(\xtarget)\gets$ Uniform($\mc{X}$)\;
$\brobot^1(\srobot)\gets\indicator(\srobot=\srobotinit)$\;
$b^1\gets(\btarget^1,\brobot^1)$\;
$t\gets 1$\;
\While{$t \leq T_{\text{max}}$ and $a_{t-1}\neq\texttt{Done}$}{
    $(\mc{V},\mc{E})\gets$ SampleTopoGraph($\mc{X},\Srobot,\btarget^t$)\;
    $P_{\text{H}}\gets$ HighLevelCOSPOMDP($P,\mc{V},\mc{E},b^t$)\;
    subgoal $\gets$ plan POMDP online for $P_{\text{H}}$\;
    \uIf{subgoal is \emph{navigate to a node in $\mc{V}$}}{
        $\srobot\gets \argmax_{\srobot}\brobot(\srobot)$\;
        $a_t \gets$ A$^*$(subgoal, $\srobot$, $\mc{A}_m,T_m$)\;
    }
    \uElseIf{subgoal is \emph{search locally}}{
        $P_{\text{L}}\gets$ LowLevelCOSPOMDP($P$, $b^t$)\;
        $a_t \gets$ plan POMDP online for $P_{\text{L}}$\;
    }
    \ElseIf{subgoal is \emph{Done}}{
        $a_t\gets\texttt{Done}$
    }
    $z_t\gets$ execute $a_t$ and receive observation\;
    $b^{t+1}\gets \text{BeliefUpdate}(b^t,a_t,z_t)$\;
    $t\gets t+1$\;
}
\end{algorithm}

Hence, we propose a hierarchical planning algorithm to apply COS-POMDPs in realistic domains. The algorithm is presented in Algorithm~\ref{alg:hierplanning} and illustrated in Fig~\ref{fig:method}. To enable the planning of searching over a large region, we first generate a topological graph, where nodes are places accessible by the robot, and edges indicate navigability between places \cite{toponets-iros-2019}. This is done by the SampleTopoGraph procedure (Appendix \ref{appdx:algs}). In this procedure, the nodes are sampled based on the robot's current belief in the target location $\btarget^t$, and edges are added such that the graph is connected and every node has an out-degree within a given range, which affects the branching factor for planning. An example output is illustrated in Fig~\ref{fig:method}.




Then, a high-level COS-POMDP $P_{\text{H}}$ is instantiated. The state and observation spaces, the observation model, and the reward model, are as defined in Sec \ref{sec:cospomdps}. The move action set and the corresponding transition model are defined according to the generated topological graph. Each move action represents a subgoal of \emph{navigating} to another place, or the subgoal of \emph{searching locally} at the current place.  Both types of subgoals can still be understood as viewpoint-changing actions, except the latter keeps the viewpoint at the same location. For the transition model $T(s',g,s)$ where $g$ represents the subgoal, the resulting viewpoint (\ie, $\srobot'\in s'$) after completing a subgoal is located at the destination of the subgoal with orientation facing the target object location ($\xtarget\in s$). The \texttt{Done} action is also included as a dummy subgoal to match the definition of the COS-POMDP action space (Sec~\ref{sec:cospomdps}).

At each timestep, a subgoal
is planned using an online POMDP planner, and a low-level planner is instantiated
corresponding to the subgoal. This low-level planner then plans to output an
action $a_t$ from the action set $\mc{A}=\mc{A}_m\cup\{\DECLARE\}$, which is used for
execution.  In our implementation, for \emph{navigation} subgoals, an A$^*$ planner is used, and for \emph{searching locally}, a low-level COS-POMDP $P_{\text{L}}$ is instantiated with the primitive movements $\mc{A}_m$ in its action space. (We use PO-UCT \cite{silver2010monte} as the online POMDP solver in our experiments.)

Upon executing the low-level action $a_t$, the robot receives an observation $z_t\in\mc{Z}$ from its on-board perception modules for robot state estimation and object detection. This observation is used to update the belief of the high-level COS-POMDP, which is shared with the low-level COS-POMDP.

Finally, the process starts over from the first step of sampling a topological graph. If the high-level COS-POMDP plans a new subgoal different the current one, then the low-level planner is re-instantiated.

This algorithm plans actions for execution in an online, closed-loop fashion, allowing reasoning about viewpoint changes both at the level of places in a topological graph as well as fine-grained movements.

\section{Evaluation}

\subsection{Experimental Setup}
We test the following hypotheses through our experiments: (1) Leveraging correlational information with easier-to-detect objects can benefit the search for hard-to-detect objects; (2)~Optimizing over an action sequence improves performance compared to greedily choosing the next-best view.

\subsubsection{AI2-THOR}
We conduct experiments in AI2-THOR \cite{kolve2017ai2}, a realistic simulator of in-household rooms. It has a total of 120 scenes divided evenly into four room types: \emph{Bathroom}, \emph{Bedroom},
\emph{Kitchen}, and \emph{Living room}.
For each room type, we use the first 20 scenes for training a vision-based object detector and learning object correlation models (used in some experiments), and the last 10 for evaluating performance.

The robot can take primitive move actions from the set:

$\{\texttt{MoveAhead},
\texttt{RotateLeft}, \texttt{RotateRight}, \texttt{LookUp}, \texttt{LookDown}\}$.

\noindent \texttt{MoveAhead} moves
the robot forward by 0.25m. \texttt{RotateLeft}, \texttt{RotateRight} rotate the robot
in place by 45$^\circ$. \texttt{LookUp}, \texttt{LookDown} tilt the camera up or down by 30$^\circ$. The transition function of the robot's viewpoint when taking primitive move actions is deterministic. All methods (Sec~\ref{sec:baselines}) receive observations of the robot's viewpoint without noise, that is $\Orobot(\zrobot,\srobot)=\indicator(\zrobot=\srobot)$. To be successful, when the robot takes \texttt{Done}, the robot must be within a Euclidean distance of 1.0m from the target object while the target object is visible in the camera frame. The maximum number of steps $T_{\text{max}}$ is 100.



\subsubsection{Object Detector}\label{sec:exp:detectors}
Unlike previous work in object search evaluated using a ground truth object
detector \cite{qiu2020learning} or detectors with synthetic noise and detection
ranges \cite{zeng2020semantic}, we use a vision-based object detector, YOLOv5
\cite{glenn_jocher_2020_4154370}, since it is
more realistic and suitable for our motivation. We collect training data by randomly placing the agent in
the training scenes. Table~\ref{tab:objects} and Table~\ref{tab:detcorr} contain detection statistics of the target objects and correlated objects in validation scenes, respectively. The pixel coordinates within the bounding boxes returned by YOLOv5 are downsampled and inverse projected to positions in the 3D world frame, using the provided depth image.

\textbf{Detection Model.} Vision detectors can sometimes detect small objects from far away. Therefore, we consider a line-of-sight detection model with a limited field of view angle to enable POMDP planning:
\begin{align*}
  &D(z_i,x_i,\srobot)=\Pr(z_i|x_i,\srobot)\\
  &\qquad=\begin{cases}
    1.0 - \text{TP} & s_i\in \mc{V}(\srobot) \land z_i = \texttt{null}\\
    \delta \text{FP} / |\mc{V}_{E}(r)| & s_i\in \mc{V}(\srobot) \land \norm{z_i- x_i} > 3\sigma\\
    \delta \mc{N}(z_i;x_i,\sigma^2) & s_i\in \mc{V}(\srobot) \land \norm{z_i- x_i}\leq 3\sigma\\
    1.0 - \text{FP} & s_i\not\in \mc{V}(\srobot) \land z_i = \texttt{null}\\
    \delta \text{FP} / |\mc{V}_{E}(r)|  & s_i\not\in \mc{V}(\srobot) \land z_i \neq \texttt{null}\\
  \end{cases}
\end{align*}
This detection model is parameterized by: TP, the true positive rate; FP, the false positive
rate; $r$, the average distance between the robot and the object for true positive detections;
$\sigma$, the width of a small region around the true object location where a detection
made within that region, though not exactly accurate, is still accepted as a true positive detection. We set $\sigma=0.5$m. The notation $\mc{N}(\cdot;x_i,\sigma^2)$ denotes a Gaussian distribution with mean $x_i$ and covariance $\sigma^2\mathbf{I}$. The $\mc{V}(\srobot)$ denotes the line-of-sight field of view with a 90$^\circ$ angle. The $\mc{V}_E(r)$ denotes the region inside the field of view that is within distance $r$ from the robot.
The weight $\delta=1$ if the detection is within $\mc{V}_E(r)$, and otherwise $\delta=\exp(-\norm{z_i-\srobot}-r)^2$.

\subsubsection{Target Objects}
The list of target object classes and other correlated classes for each room type is listed below (with no particular order). For detection statistics, please refer to Table~\ref{tab:main} and Table~\ref{tab:detcorr} (Appendix~\ref{appdx:detcorr}).

\begin{itemize}
\item \emph{Bathroom} - Targets: \texttt{Fauct}, \texttt{Candle}, \texttt{ScrubBrush};
  Correlated objects: \texttt{ToiletPaperHanger}, \texttt{Towel}, \texttt{Mirror}, \texttt{Toilet}, \texttt{SoapBar}.
\item \emph{Bedroom} - Targets: \texttt{AlarmClock}, \texttt{Pillow}, \texttt{CD};
  Correlated objects: \texttt{Laptop}, \texttt{DeskLamp}, \texttt{Mirror}, \texttt{LightSwitch}, \texttt{Bed}.
\item \emph{Kitchen} - Targets are \texttt{Bowl}, \texttt{Knife}, \texttt{PepperShaker};
  Correlated classes are: \texttt{Lettuce}, \texttt{LightSwitch}, \texttt{Microwave}, \texttt{Plate}, \texttt{StoveKnob}
\item \emph{Living room} - Targets are \texttt{CreditCard}, \texttt{RemoteControl}, \texttt{Television};
  Correlated classes are: \texttt{LightSwitch}, \texttt{Pillow}, \texttt{HousePlant}, \texttt{Laptop}, \texttt{FloorLamp}, \texttt{Painting}.
\end{itemize}

\subsubsection{Correlation Model}
We consider a binary correlation model that takes into account whether the correlated object and the target are close or far. Specifically, we define:
\begin{align}
  &C(\xtarget,x_i)=\Pr(x_i|\xtarget)\\
  &=\begin{cases}
    1 & \text{Close}(i,\text{target}) \land \norm{x_i - \xtarget} < d(i,\target)\\
    0 & \text{Close}(i,\text{target}) \land \norm{x_i - \xtarget} \geq d(i,\target)\\
    1 & \text{Far}(i,\text{target}) \land \norm{x_i - \xtarget} > d(i,\target)\\
    0 & \text{Far}(i,\text{target}) \land \norm{x_i - \xtarget} \leq d(i,\target)\\
    \end{cases}
\end{align}
where Close$(\cdot,\cdot)$ and Far$(\cdot,\cdot)$ are class-level predicates, $\norm{\cdot}$ denotes the Euclidean distance, and $d(\cdot,\cdot)$ is the expected distance between the two objects. This model is applicable between arbitrary object classes and can be estimated based on instances of object classes. In Sec.~\ref{sec:ablation}, we conduct an ablation study where $d(\cdot,$target$)$ is estimated under different scenarios: \textbf{accurate}: based on object ground truth locations in the deployed scene; \textbf{learned} (lrn): based on instances in training scenes; \textbf{wrong} (wrg): same as accurate except we flip the close/far relationship between the objects so that they do not match the scene.

\subsubsection{Implementation Detail of COS-POMDP}
Objects exist in 3D space in AI2-THOR scenes, and the robot can rotate its
camera both horizontally and vertically. Our implementation of COS-POMDP allows
for search in such a setting by estimating, in the belief over target locations,
both the 2D position of the target as well as the height of the target. Since the
robot can tilt only its camera within a small range of angles, we consider a
discrete set of possible height values, \texttt{Above}, \texttt{Below}, and \texttt{Same},
which indicates the object is above, below, or at the same level with respect
to the camera's current tilt angle. Our implementation is based on the \texttt{pomdp\_py} \cite{pomdp-py-2020} library.


\subsubsection{Evaluation Metric}
We use three metrics: (1) success weighted by inverse path length (SPL) \cite{anderson2018evaluation}; (2) success rate (SR) and (3) dicounted cumulative rewards (DR).
The SPL of each search trial is defined as $SPL=S\cdot \ell / \max(p, \ell)$
where $S$ is the binary success outcome of the search, $\ell$ is the shortest
path between the robot and the target, and $p$ is the actual search path. The SPL measures the search performance by taking into account both the success and efficiency of the search. It
is a difficult metric because $\ell$ uses information about the true object location.
However, it does not penalize excessive rotations \cite{batra2020objectnav}.
Therefore, we also include discounted cumulative rewards ($\gamma=0.95$) which takes such actions
into account.

\begin{figure*}[t]
    \centering
    \makebox[\textwidth][c]{\includegraphics[width=1.2\textwidth]{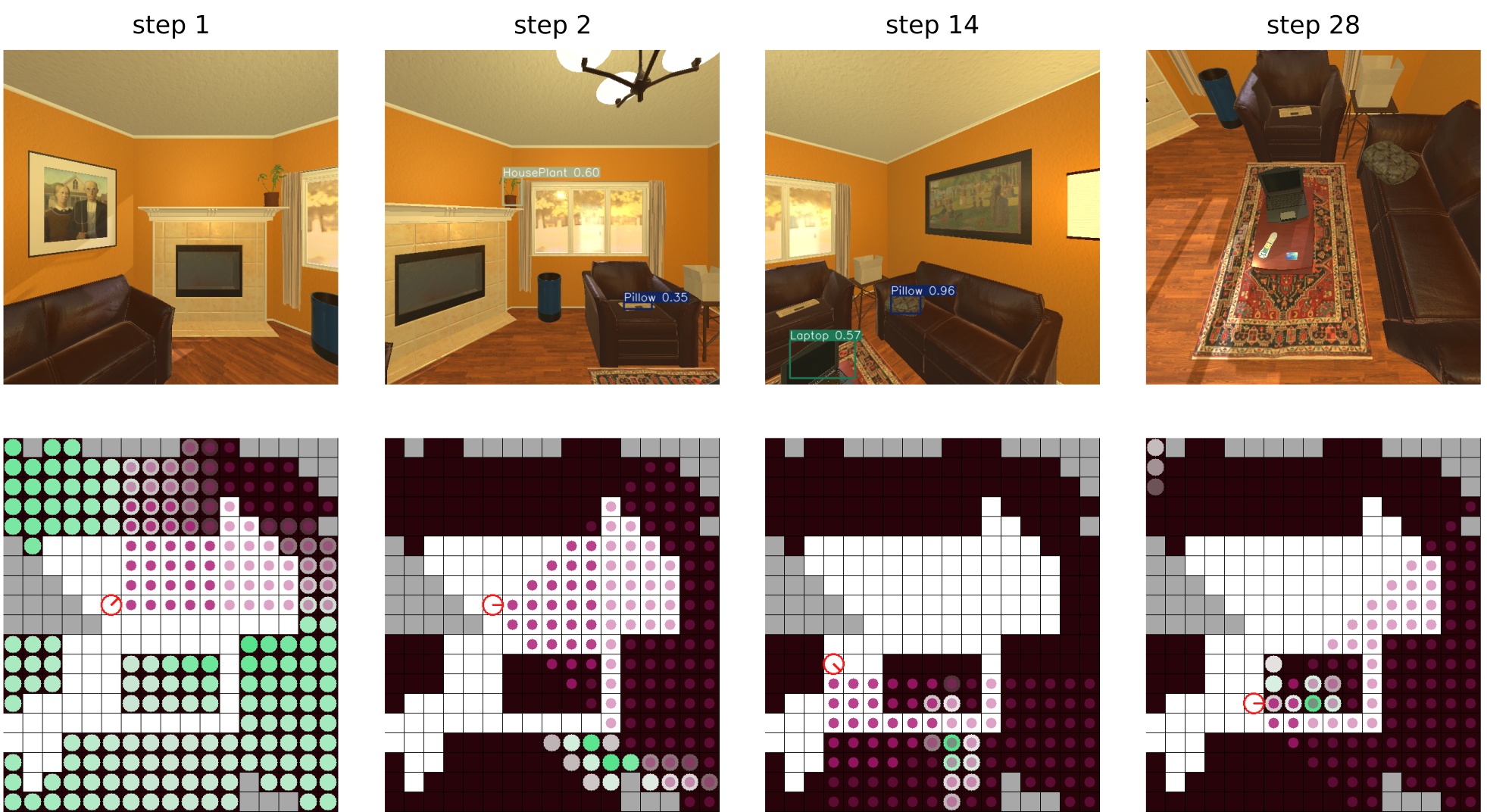}}
    \caption{\textbf{Example Sequence.} Top: first-person view with object detection bounding boxes. Bottom: Visualization of belief state corresponding to each view. See Fig~\ref{fig:method} for the legend of the belief state visualization. Our method (COS-POMDP) successfully finds a \texttt{CreditCard} in a living room scene, leveraging the detection of other objects such as \texttt{FloorLamp} and \texttt{Laptop}. For more examples, please refer to the video at \url{https://youtu.be/wd1tmD0mckY}.}
    \label{fig:example}
    \vspace{-0.2cm}
\end{figure*}
\label{sec:experiments}
\begin{figure}[t]
    \centering
    \includegraphics[width=\linewidth]{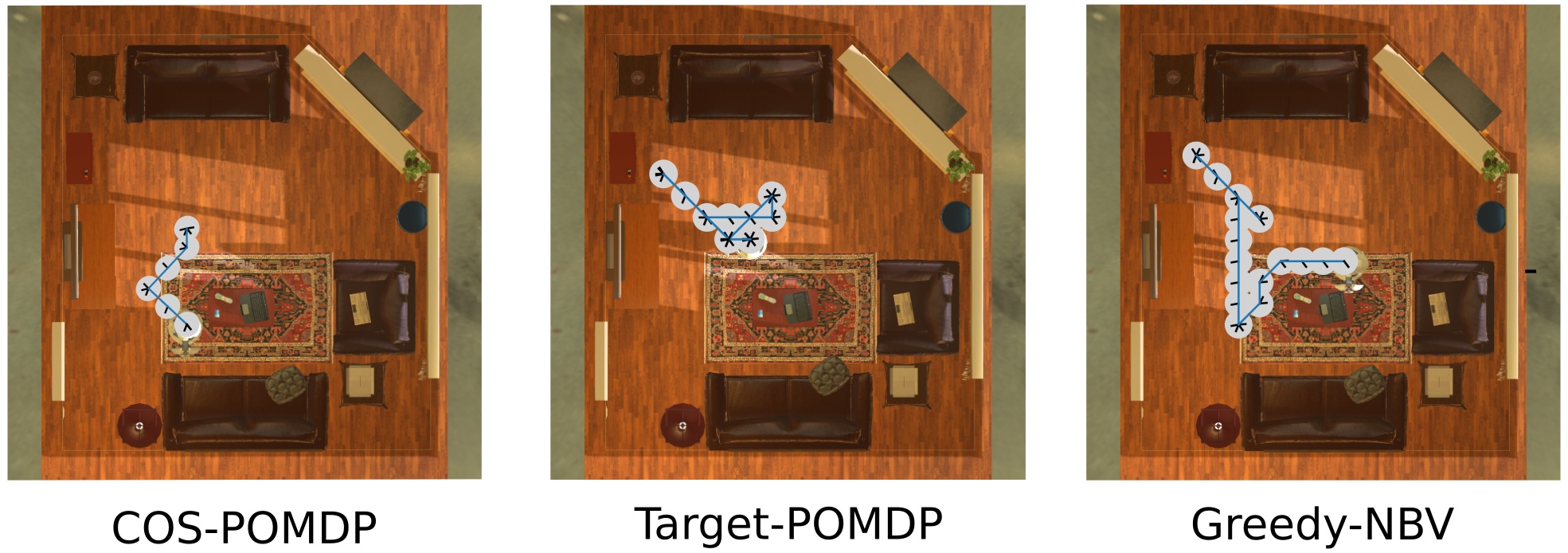}
    \caption{
    Visualization of robot trajectory produced by different methods for the example shown in Fig.~\ref{fig:example}. Each gray circle represents the position of a viewpoint, and each black line segment indicates the orientation of the robot's camera at a viewpoint. The path traversed during the search is shown in blue.
    }
    \label{fig:paths}
    \vspace{-0.3cm}
\end{figure}

\subsubsection{Baselines}
\label{sec:baselines}
Baselines are defined in the caption of Table~\ref{tab:main}.
Note that for \textbf{Greedy-NBV}, based on \cite{zeng2020semantic}, a weighted particle belief is used to maintain the belief over the joint state over all object locations. During planning,
the agent selects the next best viewpoint to navigate towards based on a cost function that considers
both navigation distance and the probability of detecting any object.
This provides a baseline that is in contrast to the sequential decision-making paradigm considered by COS-POMDPs and the modeling of only robot and target states.
\footnote{I attempted a comparison with deep reinforcement learning methods \cite{wortsman2019learning,qiu2020learning}, yet I was blocked by the fact that their codebases were developed for earlier versions of AI2-THOR (v1.0) and use different configurations (\eg~rotating at 90 degrees instead of 45 degrees). The trained models performed poorly on the newer version I was using (v3.3.4) due to backwards incompatibility.}

\subsection{Results and Discussions}
Our main results are shown in Table~\ref{tab:main}. The performance of \textbf{COS-POMDP} is the most consistent compared to other baselines, with \textbf{COS-POMDP} performing either the best or the second best in the four room types.  The performance is broken down by target classes in Table~\ref{tab:objects}. \textbf{Greedy-NBV} performs well in \emph{Bedroom}; it appears to experience less inaccuracy in the particle-based belief over all objects as a result of particle reinvigoration in bedroom compared to the other room types.
\textbf{COS-POMDP} appears to be the most robust when the target object has significant noise of being correctly detected, including  \texttt{ScrubBrush}, \texttt{CreditCard}, \texttt{Candle} \texttt{RemoteControl}, \texttt{Knife}, and \texttt{CD}. An example search trial of \textbf{COS-POMDP} for \texttt{CreditCard} is shown in Fig~\ref{fig:example} and the search paths of the methods under comparison are visualized in Fig~\ref{fig:paths}. For target objects with a true positive detection rate below 40\%, \textbf{COS-POMDP} improves the POMDP baseline that ignores correlational information by 70\% in terms of the SPL metric, and is more than 1.7 times better than the greedy baseline. Indeed, when the target object is reliably detectable, such as \texttt{Television}, the ability to detect multiple other objects may actually hurt performance, compared to \textbf{Target-POMDP}, due to the noise from detecting those other objects and the influence on search behavior. These results demonstrate that COS-POMDPs can be applied to search for hard-to-detect objects leveraging the more reliable detection of correlated objects.

\subsubsection{Ablation Studies}\label{sec:ablation}
We also conduct two ablation studies. First, we equip \textbf{COS-POMDP} with a groundtruth object detector, as done in \cite{qiu2020learning}, henceforth called \textbf{COS-POMDP (gt)}. This shows the performance when the detections of both the target and correlated objects involve no noise at all. We observe better or competitive performance from using groundtruth detectors across all metrics in all room types.

Additionally, we use correlations obtained by learning from data (\textbf{COS-POMDP (lrn)}) as well as incorrect correlation information that is the reverse of the correct one (\textbf{COS-POMDP (wrg)}).  Indeed, using accurate correlations provides the most benefit, while correlations learned through this simple method could offer benefit compared to using incorrect correlations in some cases (\emph{Bathroom} and \emph{Bedroom}), but can also backfire and hurt performance in other others. Therefore, properly learning correlation is important, while leveraging a reliable source of information, for example, from a human at the scene, may offer the most benefit.

\begin{sidewaystable}[ph!]
    \setlength{\tabcolsep}{2pt}
    \renewcommand{\arraystretch}{1.0}{
      \resizebox{\textwidth}{!}{%
      \centering
\begin{tabular}{l|ccc|ccc|ccc|ccc}
\specialrule{.8pt}{2pt}{0pt}
                      & \multicolumn{3}{c|}{Bathroom} & \multicolumn{3}{c|}{Bedroom} & \multicolumn{3}{c|}{Kitchen} & \multicolumn{3}{c}{Living room}                                                                                                                                                                         \\
Method                & SPL (\%)                      & DR                           & SR (\%)                      & SPL (\%)               & DR                     & SR (\%)        & SPL (\%)               & DR                     & SR (\%)        & SPL (\%)               & DR                      & SR (\%)        \\
\specialrule{.4pt}{0pt}{0pt}
Random                & 0.00 (0.00)                   & -82.75 (3.43)                & 0.00                         & 0.00 (0.00)            & -85.27 (3.82)          & 0.00           & 6.90 (9.81)            & -68.51 (15.61)         & 6.90           & 0.00 (0.00)            & -82.37 (3.62)           & 0.00           \\
Greedy-NBV    & 14.85 (9.40)                  & -18.86 (12.14)               & 35.71                        & \textbf{31.10 (17.86)} & \textbf{-6.97 (14.20)} & \textbf{40.91} & 12.03 (9.01)           & -17.16 (12.85)         & 32.14          & 7.13 (7.11)            & -21.41 (8.21)           & 20.00          \\
Target-POMDP       & 24.17 (12.03)                 & \textbf{-2.60 (17.02)}       & \textbf{66.67}               & 14.70 (12.86)          & -26.74 (13.27)         & 31.58          & 14.82 (9.22)           & -20.06 (13.85)         & 37.04          & \textbf{29.23 (15.34)} & -30.65 (13.60)          & 48.00          \\
COS-POMDP     & \textbf{30.03 (13.59)}        & -14.92 (12.76)               & 56.00                        & 28.54 (17.63)          & -16.02 (14.03)         & 40.00          & \textbf{20.95 (13.10)} & \textbf{-4.67 (14.71)} & \textbf{44.00} & 27.76 (15.21)          & \textbf{-12.32 (15.71)} & \textbf{48.15} \\
\specialrule{.4pt}{0pt}{0pt}
  COS-POMDP (gt) & \textbf{33.38 (13.92)}        & -11.69 (13.24)               & 62.96                        & \textbf{39.22 (19.56)} & -13.50 (17.28)         & \textbf{56.25} & \textbf{36.92 (14.33)} & \textbf{-2.92 (16.46)} & \textbf{64.00} & \textbf{35.71 (16.05)} & \textbf{-9.31 (13.09)}  & \textbf{62.50} \\
\specialrule{.4pt}{0pt}{0pt}
COS-POMDP (lrn)    & 19.77 (12.07)                 & -21.53 (13.13)               & 45.83                        & 16.43 (14.53)          & -33.73 (11.05)         & 23.81          & 6.29 (6.93)            & -32.72 (13.62)         & 17.86          & 14.76 (11.41)          & -43.09 (13.57)          & 25.00          \\
COS-POMDP (wrg)    & 10.83 (7.79)                  & -19.00 (10.21)               & 28.00                        & 14.54 (14.29)          & -32.54 (14.87)         & 27.78          & 8.80 (7.38)            & -20.49 (10.63)         & 25.93          & \textbf{29.34 (16.10)} & -16.15 (11.96)          & \textbf{54.17} \\
\specialrule{.8pt}{0pt}{0pt}
\end{tabular}
}
}
  \caption{\textbf{Main and Ablation Study Results.} Unless otherwise specified, all methods use the YOLOv5 \cite{glenn_jocher_2020_4154370} vision detector and are given accurate correlational information. \textbf{Target-POMDP} uses the hierarchical planning except only the target object is detectable. \textbf{Greedy-NBV} is a next-best view approach based on \cite{zeng2020semantic}. \textbf{Random} chooses actions uniformly at random. The highest value of each metric per room type is bolded. Parentheses contain 95\% confidence interval. Ablation study results are bolded if it outperforms the best result from the main evaluation. Metrics are success weighted by inverse path length (SPL) \cite{anderson2018evaluation}, discounted cumulative reward (DR), and success rate (SR). \textbf{COS-POMDP} is more consistent, performing either the best or the second best across all room types and metrics.
  }
  \label{tab:main}
\end{sidewaystable}

\begin{sidewaystable}[ph!]
\setlength{\tabcolsep}{4pt}
\renewcommand{\arraystretch}{1.2}{
  \resizebox{\textwidth}{!}{%
  \centering
  \begin{tabular}{ll|ccc|ccc|ccc|ccc}
\specialrule{.8pt}{2pt}{0pt}
                              &               &      &      &         & \multicolumn{3}{c|}{Greedy-NBV }                                          & \multicolumn{3}{c|}{Target-POMDP }                                        & \multicolumn{3}{c}{COS-POMDP }                            \\
        Room Type             & Target Class  & TP   & FP   & $r$ (m) & SPL (\%)                                 & DR                      & SR (\%)        & SPL (\%)                              & DR                     & SR (\%)        & SPL (\%)               & DR                      & SR (\%)        \\
\specialrule{.4pt}{0pt}{0pt}
\multirow{3}{*}{Bathroom}     & Faucet        & 56.1 & 8.0  & 2.16    & 31.45 (20.79)                            & 6.11 (21.02)            & \textbf{77.78} & \textbf{40.21 (27.65)}                & \textbf{13.97 (30.40)} & 75.00          & 24.59 (29.84)          & -28.23 (26.54)          & 50.00          \\
                              & Candle        & 29.4 & 2.4  & 1.81    & 12.52 (20.12)                            & -22.81 (20.80)          & 22.22          & 18.63 (14.51)                         & -6.61 (33.50)          & \textbf{75.00} & \textbf{33.89 (21.83)} & \textbf{-2.94 (19.08)}  & 66.67          \\
                              & ScrubBrush    & 64.3 & 9.9  & 1.71    & 2.00 (4.52)                              & -37.79 (17.36)          & 10.00          & 7.38 (13.80)                          & -22.68 (34.63)         & 40.00          & \textbf{31.12 (32.00)} & \textbf{-15.08 (28.68)} & \textbf{50.00} \\
\cline{1-14}
\multirow{3}{*}{Bedroom}      & AlarmClock    & 79.6 & 7.4  & 2.77    & \textbf{48.10 (43.23)}                   & \textbf{-0.80 (26.52)}  & \textbf{57.14} & 14.64 (46.61)                         & -17.37 (44.62)         & 25.00          & 35.07 (33.97)          & -15.52 (23.95)          & 44.44          \\
                              & Pillow        & 88.3 & 5.2  & 2.43    & \textbf{30.04 (49.28)}                   & \textbf{-13.59 (38.13)} & 33.33          & 5.16 (14.32)                          & -29.49 (30.07)         & \textbf{40.00} & 21.02 (90.44)          & -14.85 (98.69)          & 33.33          \\
                              & CD            & 48.6 & 4.5  & 1.70    & 18.59 (21.89)                            & \textbf{-7.36 (26.46)}  & 33.33          & 19.49 (22.66)                         & -29.12 (21.97)         & 30.00          & \textbf{24.01 (28.21)} & -17.03 (24.75)          & \textbf{37.50} \\
\cline{1-14}
\multirow{3}{*}{Kitchen}      & Bowl          & 60.6 & 11.5 & 1.75    & \textbf{19.88 (26.57)}                   & -15.76 (32.76)          & 33.33          & 16.33 (16.00)                         & -10.06 (27.39)         & \textbf{55.56} & 16.27 (20.22)          & \textbf{-0.19 (37.01)}  & 42.86          \\
                              & Knife         & 37.7 & 8.7  & 1.68    & 8.22 (12.85)                             & -17.74 (26.85)          & 33.33          & 5.13 (11.84)                          & -37.99 (17.17)         & 11.11          & \textbf{23.97 (25.58)} & \textbf{-2.59 (25.33)}  & \textbf{50.00} \\
                              & PepperShaker  & 38.1 & 9.4  & 1.43    & 8.39 (10.53)                             & -17.90 (17.39)          & 30.00          & \textbf{22.99 (23.22)}                & -12.14 (31.40)         & \textbf{44.44} & 21.26 (30.90)          & \textbf{-11.17 (30.04)} & 37.50          \\
\cline{1-14}
\multirow{3}{*}{Living room} & Television    & 85.3 & 5.2  & 2.59    & 8.98 (18.36)                             & -22.86 (13.31)          & 20.00          & \textbf{59.56 (25.42)}                & \textbf{-6.50 (19.73)} & \textbf{88.89} & 44.53 (34.78)          & -10.79 (31.79)          & 55.56          \\
                              & RemoteControl & 69.6 & 4.5  & 1.93    & 9.24 (13.99)                             & -13.21 (20.44)          & 30.00          & 26.67 (35.38)                         & -29.18 (20.47)         & 42.86          & \textbf{33.49 (31.89)} & \textbf{8.94 (27.67)}   & \textbf{66.67} \\
                              & CreditCard    & 42.9 & 4.3  & 1.48    & 3.18 (7.19)                              & \textbf{-28.15 (11.70)} & 10.00          & 0.91 (2.09)                           & -55.95 (22.44)         & 11.11          & \textbf{5.26 (8.08)}   & -35.12 (24.57)          & \textbf{22.22} \\

\specialrule{.8pt}{0pt}{0pt}
\end{tabular}
  }
  }
  \caption{
\textbf{Detection Statistics and Object Search Results Grouped by Target Classes.} TP: true positive rate (\%); FP: false positive rate (\%); $r$: average distance to the true positive detections (m). We estimated these values by running the vision detector at 30 random camera poses per validation scene. Target objects are sorted by average detection range.  Parentheses contain 95\% confidence interval. Metrics are success weighted by inverse path length (SPL) \cite{anderson2018evaluation}, discounted cumulative reward (DR), and success rate (SR).  \textbf{COS-POMDP} performs more robustly for hard-to-detect objects, such as \texttt{ScrubBrush}, \texttt{CD}, \texttt{Candle}, \texttt{Knife}, and \texttt{CreditCard}.
  }
    \label{tab:objects}
\end{sidewaystable}

\section{Summary}
In this chapter, we formulated the problem of correlational object search and proposed COS-POMDP, a POMDP-based approach to model this problem. Our quantitative evaluation, conducted in AI2-THOR \cite{kolve2017ai2} using the YOLOv5 \cite{redmon2016you,glenn_jocher_2020_4154370} detector, demonstrates the benefit of our approach in exploiting the correlational information with easier-to-detect objects to find hard-to-detect objects. Future work directions include studying the search behavior given different kinds of learned correlation models as well as in more complex settings that involve \eg, container opening and dynamic objects.


\newpage
\section{Appendix}
\subsection{Proof of Theorem 1}
\label{appdx:proof1}
\textbf{Theorem 1.}
    Given any policy $\pi:h_t\rightarrow a$, the distribution of histories is identical between the COS-POMDP and the F-POMDP.

\begin{proof}
  We prove this by induction. When $t=1$, the statement is true because both histories are empty. The inductive hypothesis assumes that the distributions $\Pr(h_t)$ is the same for the two POMDPs at $t\geq 1$. Then, by definition, $\Pr(h_{t+1})=\Pr(h_t,a_t,z_t)=\Pr(z_t|h_t,a_t)\Pr(a_t|h_t)\Pr(h_t)$. Since $\Pr(a_t|h_t)$ is the same under the given $\pi$, we can conclude $\Pr(h_{t+1})$ is identical if the two POMDPs have the same $\Pr(z_t|h_t,a_t)$. We show that this is true as follows.

  First, we sum out the state $s_t$ at time $t$:
  \begin{align}
    &\Pr(z_t|h_t,a_t)= \sum_{s_t}\Pr(s_t,z_t|h_t,a_t)
      \intertext{By definition of conditional probability,}
    &\ \ =\sum_{s_t}\Pr(z_t|s_t,h_t,a_t)\Pr(s_t|h_t,a_t)
      \intertext{Since $s_t$ is does not depend on $a_t$ (which affects $s_{t+1}$),}
    &\ \ =\sum_{s_t}\Pr(z_t|s_t,h_t,a_t)\Pr(s_t|h_t)\label{eq:shared}
  \end{align}
  Suppose we are deriving this distribution for COS-POMDP, denoted as $\PrCOS(z_t|h_t,a_t)$. Then, by definition, the state $s_t=(\xtarget,\srobot)$. Therefore, we can write:
  \begin{align}
      \begin{split}
        &\PrCOS(z_t|h_t,a_t)\\
        &=\sum_{\cosstate}\Pr(z_t|\cosstate,h_t,a_t)\\
        &\quad\quad\quad\quad\quad\quad\times\Pr(\cosstate|h_t)
      \end{split}
      \intertext{Summing out $x_1,\ldots, x_n$,}
      \begin{split}
        &=\sum_{\cosstate}\sum_{x_1,\cdots, x_n}\Pr(x_1,\ldots,x_n,z_t|\cosstate,h_t,a_t)\\
        &\quad\quad\quad\quad\quad\quad\times\Pr(\cosstate|h_t)
      \end{split}
          \intertext{Merging sum,}
      \begin{split}
        &=\sum_{\fullstate}\Pr(x_1,\ldots,x_n,z_t|\cosstate,h_t,a_t)\\
        &\quad\quad\quad\quad\times\Pr(\cosstate|h_t)
      \end{split}
          \intertext{By the definition of conditional probability,}
      \begin{split}
        &=\sum_{\fullstate}\Pr(z_t|x_1,\ldots,x_n,\cosstate,h_t,a_t)\\
        &\quad\quad\quad\quad\times\Pr(x_1,\ldots,x_n|\cosstate,h_t,a_t)\\
        &\quad\quad\quad\quad\times\Pr(\cosstate|h_t)
      \end{split}
          \intertext{Again, because the object locations are independent of $a_t$,}
      \begin{split}
        &=\sum_{\fullstate}\Pr(z_t|x_1,\ldots,x_n,\cosstate,h_t,a_t)\\
        &\quad\quad\quad\quad\times\Pr(x_1,\ldots,x_n|\cosstate,h_t)\\
        &\quad\quad\quad\quad\times\Pr(\cosstate|h_t)
      \end{split}
      \intertext{By the definition of conditional probability again,}
      \begin{split}
        &=\sum_{\fullstate}\Pr(z_t|x_1,\ldots,x_n,\cosstate,h_t,a_t)\\
        &\quad\quad\quad\quad\times\Pr(x_1,\ldots,x_n,\cosstate|h_t)\label{eq:connection}
      \end{split}
    \intertext{Note that $(\xtarget,\srobot,x_1,\ldots,x_n)$ is a state in F-POMDP. Denote the state space of F-POMDP as $\SFPOMDP$. According to Eq (\ref{eq:shared}), we can write the above Eq (\ref{eq:connection}) as}
        &=\sum_{s_t\in\SFPOMDP}\Pr(z_t|s_t,h_t,a_t)\Pr(s_t|h_t)\\
        &=\PrFPOMDP(z_t|h_t,a_t)
  \end{align}
\end{proof}

\newpage
\subsection{Auxiliary Procedures}
\label{appdx:algs}

Algorithm~\ref{alg:sampletopo} is the pseudocode of the SampleTopoGraph algorithm, implemented for our experiments in AI2-THOR. We set $M=10$, $\dsep=1.0$m, $\mindeg=3$, $\maxdeg=5$. In our implementation, the topological graph is resampled only if the cumulative belief captured by the nodes in the current topological graph, $\sum_{\srobot\in\mc{V}}p(\srobot)$, is below 50$\%$. Otherwise, the same topological graph will be returned. Note that depending on the application domain, a different algorithm for forming the topological graph may be used in the place of SampleTopoGraph in the OnlineHierarchicalPlanning algorithm (Algorithm~\ref{alg:hierplanning}).

\begin{algorithm}[hbt]
\caption{SampleTopoGraph}
\label{alg:sampletopo}
\SetArgSty{textup}
\SetKwInOut{Input}{Input}
\SetKwInOut{Output}{Output}
\SetKwInOut{Parameter}{Parameter}
\textbf{Input:} {$\mc{X},\Srobot,\btarget$}\\
\textbf{Parameter:} {maximum number of nodes $M$, minimum separation between nodes $\dsep$, minimum and maximum out-degrees $\mindeg,\maxdeg$}\\
\textbf{Output:} {A topological graph $(\mc{V},\mc{E})$}\\
\tcp{Obtain mapping from $\srobot$ to a set of closest locations}
\ForEach{$x\in\mc{X}$}{
    $\srobot^{\text{closest}} \gets \argmin_{s_r\in\Srobot}\norm{s_r\text{.pos}-x}$\;
    $U(\srobot^{\text{closest}}) \gets U(\srobot^{\text{closest}}) \cup\{x\}$\;
}
\tcp{Construct probability distribution over $\Srobot$ using $\btarget$}
\ForEach{$\srobot\in\Srobot$}{
    $p(\srobot)\gets\sum_{x\in U(\srobot)}\btarget(x)$
}
\tcp{Construct nodes and edges}
$\mc{V}\gets$ sample $\leq M$ nodes from $\Srobot$ according to $p$ such that any pair of nodes has a minimum distance of $\dsep$\;
$\mc{E}\gets$ add edges between nodes in $\mc{V}$ so that the graph is connected and each node has an out-degree between $\mindeg$ and $\maxdeg$\;
\Return{($\mc{V},\mc{E})$}
\end{algorithm}

\newpage
\subsection{Detection Statistics for Correlated Object Classes}
Table~\ref{tab:detcorr} shows the detection statistics for correlated object classes. The detection statistics of target object classes can be found in Table~\ref{tab:objects}.
\label{appdx:detcorr}
\begin{table}[H]
  \setlength{\tabcolsep}{4pt}
  \centering
  \begin{tabular}{ll|ccc}
\specialrule{.8pt}{2pt}{0pt}
        Room Type             & Correlated Object Class  & TP   & FP   & $r$ (m) \\
\specialrule{.4pt}{0pt}{0pt}
    \multirow{5}{*}{Bathroom} & Mirror                   & 76.9 & 3.7  & 2.10\\
                              & ToiletPaperHanger        & 84.4 & 1.5  & 1.96\\
                              & Towel                    & 79.4 & 2.7  & 1.88\\
                              & Toilet                   & 86.3 & 3.5  & 1.81\\
                              & SoapBar                  & 73.2 & 1.8  & 1.53\\
\cline{1-5}
    \multirow{5}{*}{Bedroom}  & DeskLamp                 & 89.5 & 2.6  & 2.41\\
                              & Bed                      & 63.5 & 0.6  & 2.39\\
                              & Mirror                   & 86.0 & 0.6  & 2.27\\
                              & LightSwitch              & 76.3 & 2.8  & 2.26\\
                              & Laptop                   & 75.9 & 1.2  & 2.19\\
\cline{1-5}
    \multirow{5}{*}{Kitchen}  & LightSwitch              & 90.0 & 3.9  & 2.57\\
                              & Microwave                & 75.3 & 5.6  & 2.31\\
                              & StoveKnob                & 82.8 & 5.6  & 2.00\\
                              & Lettuce                  & 98.6 & 0.3  & 1.98\\
                              & Plate                    & 60.6 & 3.2  & 1.90\\

\cline{1-5}
    \multirow{6}{*}{Living room}  & FloorLamp            & 71.7 & 5.1 & 3.44\\
                              & Painting                 & 85.2 & 4.0 & 3.18\\
                              & LightSwitch              & 80.6 & 1.5 & 3.10\\
                              & HousePlant               & 82.9 & 3.9 & 3.00\\
                              & Pillow                   & 67.4 & 2.8 & 2.84\\
                              & Laptop                   & 66.3 & 2.6 & 2.24\\
\specialrule{.8pt}{0pt}{0pt}
\end{tabular}
  \caption{ \textbf{Detection Statistics for Correlated Object Classes.} TP: true positive rate (\%); FP: false positive rate (\%); $r$: average distance to the true positive detections (m). We estimated these values by running the vision detector at 30 random camera poses per validation scene. The correlated object classes for each room type are sorted by average detection range.
  }
  \label{tab:detcorr}
\end{table}

\newpage
\subsection{Evaluation on a Toy Domain: HallwaySearch}

\subsubsection{HallwaySearch}
This domain is designed to be minimal and allows the use of an offline POMDP solver until convergence. We evaluate COS-POMDPs on this domain to investigate the influence that detecting spatially correlated objects has on the expected return. In HallwaySearch, there are two objects in a hallway: a target object and a spatially correlated object, both at unknown locations sampled from a joint distribution. The robot has two detectors, one for each object, that return a binary observation indicating successful or \texttt{null} detection of the object. The detector for the target object has a small range that only returns a successful detection if the robot is directly on top of the object. The detector for the spatially correlated object has a larger range that can also return a successful detection from the two adjacent locations. Both detectors are noisy, and false negatives and false positives may occur.

\subsubsection{Baselines}
\begin{itemize}
\item \emph{Corr}. Solve the COS-POMDP.

\item \emph{Target}. Rather than solving the COS-POMDP, we solve a minimal POMDP for the object search task that ignores the correlational information and assumes the object locations are independent. As a result, the robot uses the target object detector, but never needs its other detectors.
\end{itemize}

\begin{figure}[t]
  \centering
  \begin{subfigure}[b]{\textwidth}
    \centering
\includegraphics[width=\textwidth]{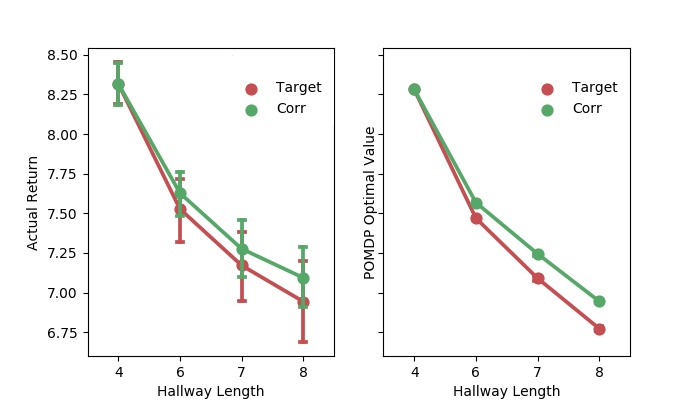}
  \end{subfigure}
  \hfill
  \begin{subfigure}[b]{\textwidth}
    \centering
    \includegraphics[width=\textwidth]{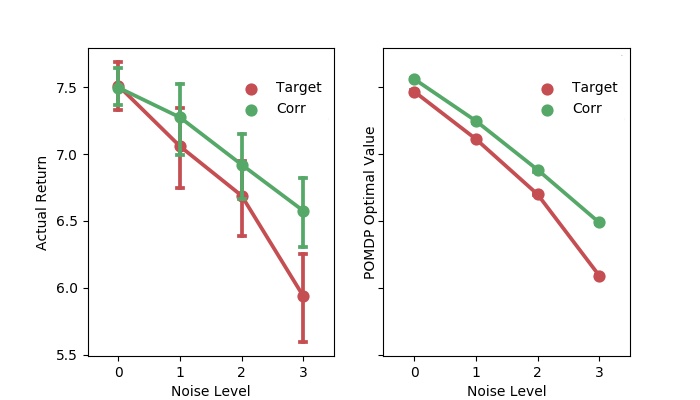}
  \end{subfigure}
  \caption{Results in the HallwaySearch domain. Top row shows estimated returns (left) and POMDP optimal value (right) as a function of hallway length; bottom row shows estimated returns (left) and POMDP optimal value (right) as a function of detector noise levels. The \emph{Target} baseline (red) does not consider correlational information, and we can see that it always performs worse than \emph{Corr} (green), which does. Thus, this experiment supports our first hypothesis.}
  \label{fig:hallway}
\end{figure}

\subsubsection{Experimental Procedure and Results}
For HallwaySearch, we conduct two sets of experiments. In the first experiment, the hallway length varies from 4 to 8 while the robot has a perfect detector for both objects. In the second experiment, the hallway length is fixed at 4 and the noise of the target detector varies, specified by pairs of (false positive, false negative) rates that range from zero to 10\%. We report both the optimal POMDP value as given by SARSOP and the approximate discounted cumulative reward calculated over 100 trials.

The results for the HallwaySearch domain are shown in Figure~\ref{fig:hallway}. We observe that considering the correlational information (green curves) leads to greater or equal optimal POMDP value than not considering it (red curves) for all experimental settings. This suggests that the optimal COS-POMDP policy makes use of the detector for the correlated object, improving the expected returns. The actual return follow a similar pattern. In this domain, the impact due to detector noise is significant, and using the correlational information leads to more robust performance. The variance in the estimates of the actual returns (left plots) is due to the stochasticity of the observation model and object locations. Overall, this experiment supports our first hypothesis: that using correlational information can improve the performance of object search, both in expectation and in empirical returns.

\chapter{Spatial Language Understanding for Object Search}
\label{ch:sloop}

\section{Motivation - Why Spatial Language?}
\label{sec:intro}
\vspace{-2em}
\blfootnote{Project website: \url{https://h2r.github.io/sloop/}}
\lettrine{C}{onsider} the scenario in which a tourist is looking for an ice cream truck in an amusement park. She asks a passer-by and gets the reply \emph{the ice cream truck is behind the ticket booth}. The tourist looks at the amusement park map and locates the ticket booth. Then, she is able to infer a region corresponding to that statement and find the ice cream truck, even though the spatial preposition \emph{behind} is inherently ambiguous and subjective to the passer-by. Robots capable of understanding spatial language can leverage prior knowledge possessed by humans to search for objects more efficiently, and interface with humans more naturally. Such capabilities can be useful for applications such as autonomous delivery and search-and-rescue, where the customer or people at the scene communicate with the robot via natural language.

\subsection{Challenges}
This problem is challenging because humans produce diverse spatial language phrases based on their observation of the environment and knowledge of target locations, yet none of these factors are available to the robot. In addition, the robot may operate in a different area than where it was trained. The robot must generalize its ability to understand spatial language across environments.

\begin{figure}[ph!]
    \centering
    \includegraphics[width=\linewidth]{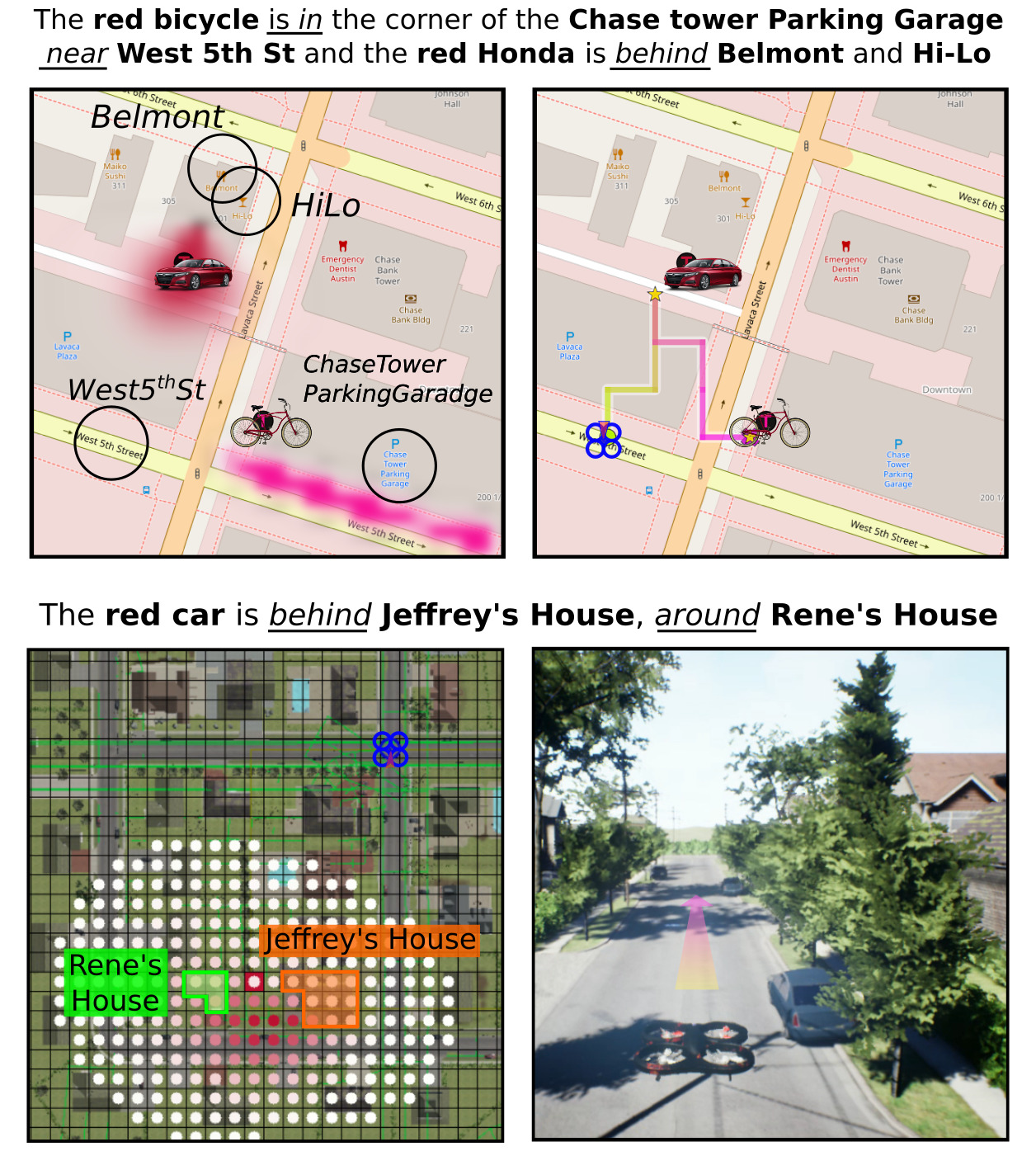}
    \caption{
      Given a spatial language description, a drone with limited field of view must find target objects in a city-scale environment. Top row: example trial from OpenStreetMap \cite{OpenStreetMap}. Bottom row: example trial from AirSim \cite{airsim2017fsr}. Left side: belief over target location after incorporating spatial language using the proposed approach. Right side: Screenshot of simulation with search path.
    }
    \vspace{-0.5cm}
    \label{fig:firstfig}
  \end{figure}

\subsection{Remark on Previous Work}

Prior works on spatial language understanding assume referenced objects already exist in the robot's world model~\cite{tellex2011forklift, Fasola2013UsingSS,janner2018representation} or within the robot's field of view~\cite{Blukis:18drone}. Works that consider partial observability do not handle ambiguous spatial prepositions \cite{hemachandra2015learning,wandzel2019oopomdp} or assume a robot-centric frame of reference \cite{bisk2018learning,patki2020language}, limiting the ability to understand diverse spatial relations that provide critical disambiguating information, such as \emph{behind the ticket booth}.
For downstream tasks, existing works primarily consider spatial language as goal or trajectory specification \cite{vogel-jurafsky-2010-learning,Kollar2010TowardUNInstructionFollowing}. They require large datasets of spatial language paired with reference paths to learn a policy by end-to-end reward-based learning \cite{sotp2019acl,wang2019reinforced} or imitation learning \cite{blukis2018mapping}, which is expensive to acquire for realistic environments (such as urban areas), and generalization to such environments is an ongoing challenge \cite{Blukis:18drone,bisk2016natural,blukis2019learning}.

Partially Observable Markov Decision Process (POMDP) \cite{kaelbling1998planning} is a principled decision making framework widely used in the object search literature \cite{aydemir2013active,xiao_icra_2019,zheng2021multi}, due to its ability to capture uncertainty in object locations and the robot's perception. \citet{wandzel2019oopomdp} proposed Object-Oriented POMDP (OO-POMDP), which factors the state and observation spaces by objects and is designed to model tasks in human environments.

Spatial language is a way of communicating spatial information of objects and their relations using spatial prepositions (e.g. \emph{on}, \emph{between}, \emph{front}) \cite{hayward1995spatial}. Understanding spatial language for decision making requires the robot to map symbols of the given language description to concepts and structures in the world, a problem referred to as language grounding. Most prior works on spatial language grounding assume a fully observable domain \cite{tellex2011forklift, Fasola2013UsingSS, Blukis:18drone, vogel-jurafsky-2010-learning, Kollar2010TowardUNInstructionFollowing}, where the referenced objects have known locations or are within the field of view, and the concern is synthesizing navigation behavior faithful to a given instruction (e.g. \emph{Go to the right  side of the rock} \cite{Blukis:18drone}). Recent works aim to map such instructions directly to low-level controls or navigation trajectories leveraging deep reinforcement learning \cite{vogel-jurafsky-2010-learning,sotp2019acl,Blukis2020:fewshot-drone} or imitation learning \cite{Blukis:18drone, wang2019reinforced,blukis2018mapping}, requiring large datasets of instructions paired with demonstrations. In this work, the referenced target objects have unknown locations, and the robot has a limited field of view. We regard the spatial language description as observation and obtain policy through online planning.

Spatial language understanding in partially observable environments is an emerging area of study \cite{hemachandra2015learning, wandzel2019oopomdp, patki2020language, thomason2019visiondialogue}. \citet{thomason2019visiondialogue} propose a domain where the robot, tasked to reach a goal room, has access to a dialogue with an oracle discussing the location of the goal during execution. \citet{hemachandra2015learning} and \citet{patki2020language} infer a distribution over semantic maps for instruction following then plan actions through behavior inference. These instructions are typically FoR-independent or involve only the robot's own FoR. In contrast, we consider language descriptions with FoRs relative to referenced landmarks. \citet{wandzel2019oopomdp} propose the Object-Oriented POMDP framework for object search and a proof-of-concept keyword-based model for language understanding in indoor space. Our work handles diverse spatial language using a novel spatial language observation model and focuses on search in cityscale environments. We evaluate our system against a keyword-based baseline similar to the one in \cite{wandzel2019oopomdp}.

Cognitive scientists have grouped FoRs into three categories: absolute, intrinsic, and relative \cite{majid2004can,shusterman2016frames}. Absolute FoRs (e.g. for \emph{north}) are fixed and depend on the agreement between speakers of the same language. Intrinsic FoRs (e.g. for \emph{at the door of the house}) depend only on properties of the referenced object.  Relative FoRs (e.g. for \emph{behind the ticket booth}) depend on both properties of the referenced object and the perspective of the observer. In this chapter, spatial descriptions are provided by human observers who may impose relative FoRs or absolute FoRs.

\section{Contributions}
In this work, we make the following contributions:
\begin{itemize}[itemsep=0.5pt,topsep=0pt]
\item We introduce SLOOP (Spatial Language Object-Oriented POMDP), which extends OO-POMDP by considering spatial language as an additional perceptual modality. We derive a probabilistic model to capture the uncertainty of the language through referenced objects and landmarks. This enables the robot to incorporate into its belief state spatial information about the referenced object via belief update.

\item We apply SLOOP to object search in city-scale environments given a spatial language description of target locations. Search begins after the initial belief update upon receiving the spatial language. Note that in general, because SLOOP regards spatial language as an observation, the language can be received during task execution.

\item We collected a dataset of five city maps from OpenStreetMap \cite{OpenStreetMap} as well as spatial language descriptions through Amazon Mechanical Turk (AMT).

\item To understand ambiguous, context-dependent prepositions (e.g.~\emph{behind}), we develop a simple convolutional neural network that infers the latent frame of reference (FoR) given an egocentric synthetic image of the referenced landmark and surrounding context. This FoR prediction model is integrated into the spatial language observation model in SLOOP.

\item We evaluate both the FoR prediction model and the object search
  performance under SLOOP using the collected dataset.  Results show that our
  method leads to search strategies that find objects faster with higher success
  rate by exploiting spatial information from language compared to a
  keyword-based baseline used in prior work \cite{wandzel2019oopomdp}. We also
  report results for varying language complexity and spatial prepositions to
  discuss advantages and limitations of our approach.

\item We demonstrate SLOOP for object search in AirSim
  \cite{airsim2017fsr}, a realistic drone simulator shown in
  Fig.~\ref{fig:firstfig}, where the drone is tasked to find cars in a
  neighborhood
  environment. 
  We also demonstrate SLOOP on a real-robot by integrating it with the Boston Dynamics Spot robot searching for objects in a lab environment.

\end{itemize}

\section{Problem Formulation}
\label{sec:setting}
We are interested in the problem setting similar to the opening scenario in the \nameref{sec:intro}. A robot is tasked to search for $N$ targets in an urban or suburban area represented as a discrete 2D map of landmarks $\mathcal{M}$. The robot is equipped with the map and can detect the targets within the field of view.  However, the robot has no knowledge of object locations a priori, and the map size is substantially larger than the sensor's field of view, making brute-force search infeasible. A human with access to the same map and prior knowledge of object locations provides the robot with a natural language description. For example, given the map in Fig.~\ref{fig:firstfig}, one could say \emph{the red car is behind Jeffrey's House, around Rene's House.} The language is assumed to mention some of the target objects and their spatial relationships to some of the known landmarks, yet there is no assumption about how such information is described.
To be successful, the robot must incorporate information from spatial language to efficiently search for the target object.

This problem can be formulated as the multi-object search (MOS) task \cite{wandzel2019oopomdp}, modeled as an OO-POMDP. The state $s_i=(x_i,y_i)$ is the location for target $i$, $1\leq i\leq N$. The robot state $s_r=(x,y,\theta,\mathcal{F})$ consists of the robot's pose $(x,y,\theta)$ and a set of found targets $\mathcal{F}\subseteq\{1,\cdots,N\}$. There are three types of actions: \textsc{Move} changes the robot pose (possibly stochastically); \textsc{Look} processes sensory information within the current field-of-view; \textsc{Find}($i$) marks object $i$ as found. In our implementation of MOS, a \textsc{Look} action is automatically taken following every \textsc{Move}. Upon taking \textsc{Find}, the robot receives reward $R_{\text{max}}\gg 0$ if an unfound object is within the field of view, and $R_{\text{min}}\ll 0$ otherwise. Other actions receive $R_{\text{step}}<0$. The desired policy accounts for the belief over target locations while efficiently exploring the map.

Note that in our evaluation, we use a synthetic detector that returns observations conditioned on the ground truth object locations for belief update during execution. Our POMDP-based framework can easily make use of a realistic observation model instead, for example, based on processing visual data \cite{xiao_icra_2019, monso2012pomdp}. Training vision-based object detectors is outside the scope of this chapter. Our focus is on spatial language understanding for planning object search strategies.


In the next section, we introduce SLOOP, with a particular focus on the observation space and observation model. Then, to apply SLOOP to our problem setting, we describe our implementation of the spatial language observation model on 2D city maps, which includes a convolutional network model for FoR prediction.

\begin{figure}[t]
    \centering
    \makebox[\textwidth][c]{\includegraphics[width=\textwidth]{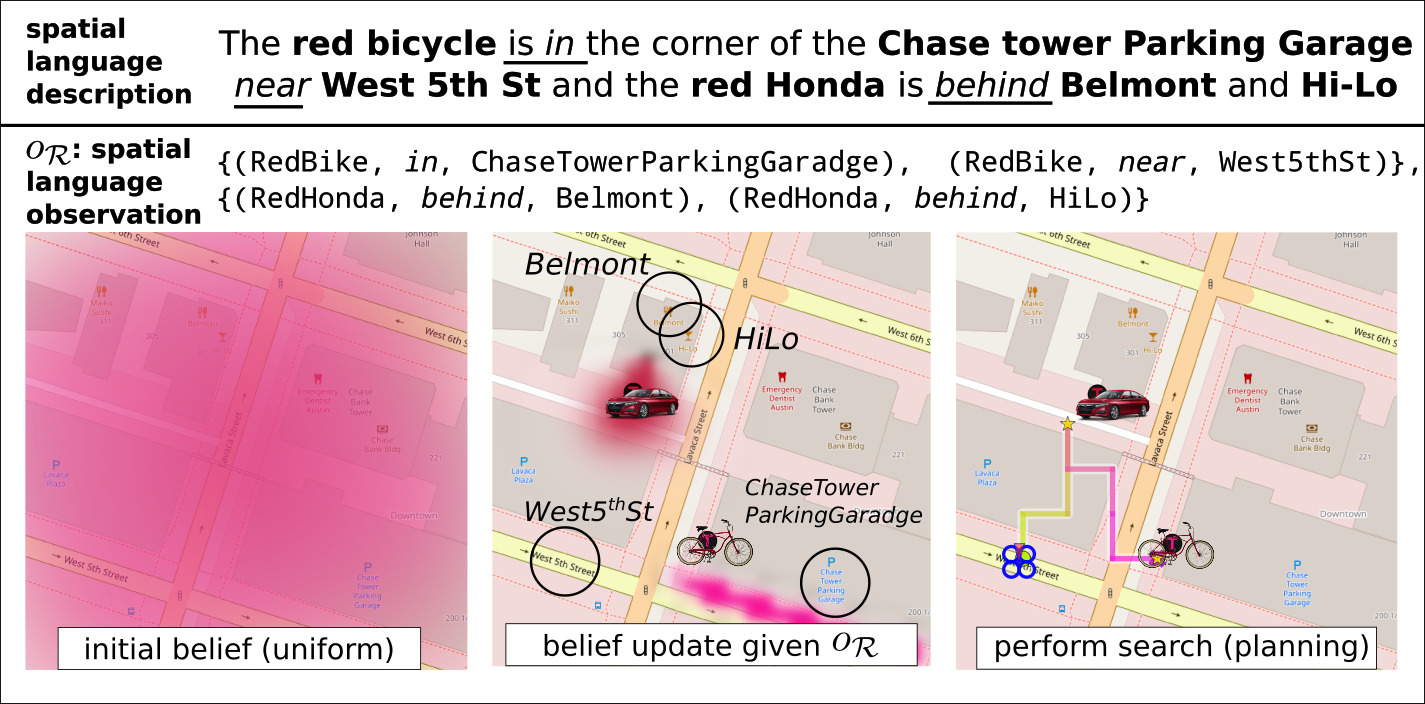}}
    \caption{
      We consider a spatial language description as an observation $o_{\sprl}$, which is a set of $(f, r, \gamma)$ tuples, obtained through parsing the input language. We propose an observation model that incorporates the spatial information in $o_{\sprl}$ into the robot's belief about target locations, which benefits subsequent object search performance.
    }
    \vspace{-0.5cm}
    \label{fig:method}
  \end{figure}

\vspace{-0.5em}
\section{Spatial Language Object-Oriented POMDP (SLOOP)}

SLOOP augments an OO-POMDP defined over a given map $\mathcal{M}$ with a spatial language observation space and a spatial language observation model. The map $\mathcal{M}$ consists of a discrete set of locations and contains landmark information (e.g. landmark's name, location and geometry), such that $N$ objects exist on the map at possibly unknown locations. Thus, the state space can be factored into a set of $N$ objects plus the given map $\mathcal{M}$ and robot state $s=(s_1,\cdots,s_N,s_r,\mathcal{M})$. SLOOP does not augment the action space, thus the action space of the underlying OO-POMDP is left unchanged. Because the transition and reward functions are, by definition, independent of observations, they are also kept unchanged in SLOOP. Next, we introduce spatial language observations.

\label{sec:integrate_pomdp}

\subsection{Spatial Language Observation}

According to \citet{landau1993whatwhere}, the standard linguistic representation of an object's place requires three elements: the object to be located (\emph{figure}), the reference object (\emph{ground}), and their relationship (\emph{spatial relation}). We follow this convention and represent spatial information from a given natural language description in terms of atomic propositions, each represented as a tuple of the form $(f,r,\gamma)$, where $f$ is the \emph{figure}, $r$ is the \emph{spatial relation} and $\gamma$ is the \emph{ground}. 
In our case, $f$ refers to a target object, $\gamma$ refers to a landmark on the map, and $r$ is a predicate that is true if the locations of $f$ and $\gamma$ satisfy the semantics of the spatial relation.
As an example, given spatial language \emph{the red Honda is behind Belmont, near Hi-Lo}, two tuples of this form can be extracted (parentheses indicate role): \{(\texttt{RedCar}($f$), \texttt{behind}($r$), \texttt{Belmont}($\gamma$)), (\texttt{RedCar}($f$), \texttt{near}($r$), \texttt{HiLo}($\gamma$))\}.

We define a \emph{spatial language observation} as a set of $(f,r,\gamma)$ tuples extracted from a given spatial language. We define the spatial language observation space to be the space of all possible tuples of such form for a given map, objects, and a set of spatial relations.

\subsection{Spatial Language Observation Model}

We denote a spatial language observation as $o_{\sprl}$. Our goal now is to derive $\Pr(o_{\sprl}|s',a)$, the observation model for spatial language. We can split $o_{\sprl}$ into subsets, $o_{\sprl}=\cup_{i=1}^No_{\sprl_i}$, where each $o_{\sprl_i}=\cup_{k=1}^{L}(f_i,r_k,\gamma_k)$, $L=|o_{\sprl_i}|$ is the set of tuples where the figure is target object $i$. Since the human describes the target objects with respect to landmarks on the map, the spatial language is conditionally independent from the robot state and action given map $\mathcal{M}$ and the target locations $s_1,\cdots,s_N$. Therefore,
\begin{align}
    \Pr(o_{\sprl}|s',a)&=\Pr(\cup_{i=1}^N o_{\sprl_i}|s_1',\cdots,s_N',\mathcal{M})
\end{align}
We make a simplifying assumption that $o_{\sprl_i}$ is conditionally independent of all other $o_{\sprl_j}$ and $s_{j}'$ ($j\neq i$) given  $s_i'$ and map $\mathcal{M}$. We make this assumption because the human observer is capable of producing a language $o_{\sprl_i}$ to describe target $i$ given just the target location $s_i$ and the map $\mathcal{M}$. Thus,
\begin{align}
  \Pr(\cup_{i=1}^N o_{\sprl_i}|s_1',\cdots,s_N',\mathcal{M})=\prod_{i=1}^N\Pr(o_{\sprl_i} | s_i', \mathcal{M})
\end{align}
where $\Pr(o_{\sprl_i} | s_i', \mathcal{M})$ models the spatial information in the language description for object $i$. For each spatial relation $r_j$ in $o_{\sprl_i}$ whose interpretation depends on the FoR imposed by the human observer (e.g. \emph{behind}), we introduce a corresponding random variable $\Psi_j$ denoting the FoR vector that distributes according to the indicator function $\Pr(\Psi_j=\psi_j) = \mathds{1}(\psi_{j}=\psi_{j}^*)$, where $\psi_j^*$ is the one imposed by the human, unknown to the robot. Then, our model for $\Pr(o_{\sprl_i} | s_i', \mathcal{M})$ becomes:
\begin{align}
  \Pr(o_{\sprl_i} | s_i', \mathcal{M}) =\prod_{j=1}^L \Pr(r_j | \gamma_j, \psi_j^*, f_i, s_i', \mathcal{M}) \label{eq:sprls}
\end{align}
The step-by-step derivation can be found in the supplementary material. It is straightforward to extend this model as a mixture model $\Pr(o_{\sprl_i} | s_i', \mathcal{M})=\sum_{k=1}^mw_k\Pr_k(o_{\sprl_i} | s_i', \mathcal{M})$, $\sum_{k=1}^mw_k=1$, where multiple interpretations of the spatial language are used to form separate distributions then combined into a weighted-sum. This effectively smooths the distribution under individual interpretations, which improves object search performance in our evaluation (Figure~\ref{fig:detections}),

However, to proceed modeling $\Pr(r_j | \gamma_j, \psi_j^*, f_i, s_i', \mathcal{M})$ in Eq.~(\ref{eq:sprls}), we notice that it depends on the unknown FoR $\psi_j^*$. Therefore, we consider two subproblems instead: The approximation of $\psi_j^*$ by a predicted value $\hat{\psi}$ and the modeling of $\Pr(r_j | \gamma_j, \hat{\psi_j}, f_i, s_i', \mathcal{M})$. Next, we describe our approach to these two subproblems for object search in city-scale environments.

\subsection{Learning to Predict Latent Frame of Reference}
\label{sec:for_prediction}

\begin{figure*}[t]
  \centering
  \makebox[\textwidth][c]{\includegraphics[width=1.15\textwidth]{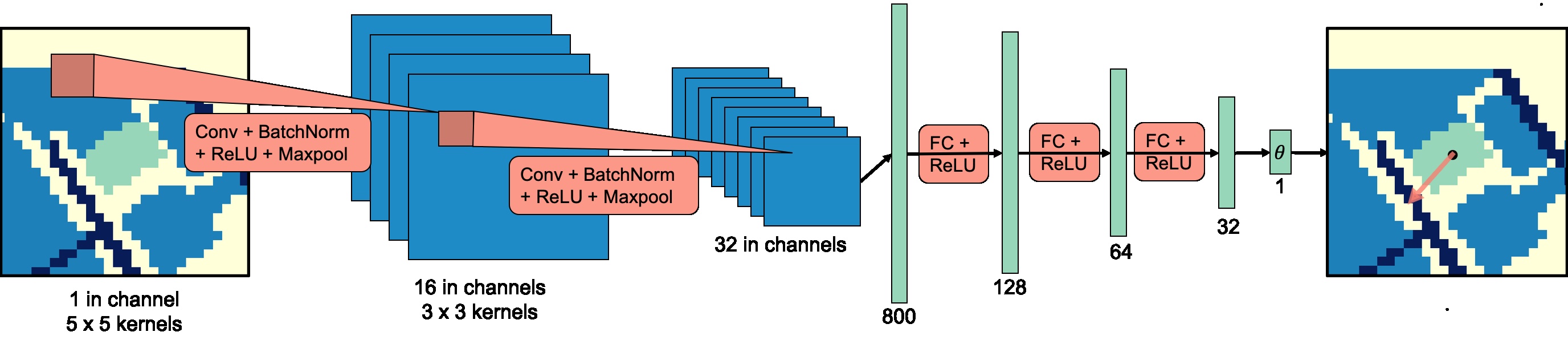}}
  \caption{\textbf{Frame of Reference Prediction Model Design.} In this example taken from our dataset, the model is predicting the frame of reference for the preposition \emph{front}. The grayscale image is rendered with color ranging from blue (black), green (gray) to yellow (white). Green highlights the referenced landmark, dark blue the streets, blue the surrounding buildings, and yellow the background.}
  \label{fig:model_design}
\end{figure*}
Here we describe our approach to predict $\psi_j^*$ corresponding to a given $(f_i,r_j,\gamma_j)$ tuple, which is critical for correct resolution of spatial relations. Taking inspiration from the ice cream truck example where the tourist can infer a potential FoR by looking at the 2D map of the park, we train a model that predicts the human observer's imposed FoR based on the environment context embedded in the map.

We define an FoR in a 2D map as a single vector $\psi_{j}=(x,y,\theta)$ located at $(x,y)$ at an angle $\theta$ with respect to the $+x$ axis of the map.
We use the center-of-mass of the ground as the origin $(x,y)$. We make this approximation since our data collection shows that when looking at a 2D map, human observers tend to consider the landmark as a whole without decomposing it into parts. Therefore, the FoR prediction problem becomes a regression problem of predicting the angle $\theta$ given a representation of the environment context.

We design a convolutional neural network, which
takes as input a grayscale image representation of the environment context where the ground in the spatial relation is highlighted. Surrounding landmarks are also highlighted with different brightness for streets and buildings (Figure~\ref{fig:model_design}). The image is shifted to be egocentric with respect to the referenced landmark, and cropped into a 28$\times$28 pixel image. The intuition is to have the model focus on immediate surroundings, as landmarks that are far away tend not to contribute to inferring the referenced landmark's properties. The model consists of two standard convolution modules followed by three fully connected layers.
These convolution modules extract an 800-dimension feature vector, feeding into the fully connected layers, which eventually output a single value for the FoR angle. We name this model \textbf{EGO-CTX} for egocentric shift of the contextual representation.

Regarding the loss function,
a direct comparison with the labeled FoR angle is not desirable. For example, suppose the labeled angle is $0$ (in radians). Then, a predicted angle of $0.5$ is qualitatively the same as another predicted angle of $-0.5$ or $2\pi - 0.5$.
For this reason, we apply the following treatment to the difference between predicted angle $\theta$ and the labeled angle $\theta^*$. Here, both $\theta$ and $\theta^*$ have been reduced to be between $0$ to $2\pi$:
\begin{align}
\label{eq:lossf}
\ell(\theta,\theta^*)=
    \begin{cases}
      2\pi - \abs{\theta-\theta^*}, & \text{if}\  \abs{\theta-\theta^*} > \pi,\\
      \abs{\theta-\theta^*}, & \text{otherwise}\\
    \end{cases}
\end{align}
This ensures that the angular deviation used to compute the loss ranges between $0$ to $\pi$.
The learning objective is to reduce such deviation to zero. To this end, we minimize the mean-squared error loss $L(\bm{\theta},\bm{\theta}^*) = \frac{1}{N}\sum_{i=1}^N\left( \ell (\theta_i,\theta^*_i)\right)^2$,
where $\bm{\theta},\bm{\theta}^*$ are predicted and annotated angles in the training set of size $N$. This objective gives greater penalty to angular deviations farther away from zero.

In our experiments, we combine the data by antonyms and train two models for each baseline: a \textbf{front} model used to predict FoRs for \emph{front} and \emph{behind}, and a \textbf{left} model used for \emph{left} and \emph{right}\footnote{We do not train a single model for all four prepositions since \emph{left} and \emph{right} often also suggest an \emph{absolute} FoR used by the language provider when looking at 2D maps, while \emph{front} and \emph{behind} typically suggest a \emph{relative} FoR.}.
We augment the training data by random rotations for \textbf{front} but not  for \textbf{left}.\footnote{Again, because \emph{left} and \emph{right} may imply either absolute or relative FoR.} We use the Adam optimizer~\cite{kingma2014adam} to update the network weights with a fixed learning rate of $1\times 10^{-5}$. Early stopping based on validation set loss is used with a patience of 20 epochs~\cite{prechelt1998early}.


\subsection{Modeling Spatial Relations}
We model $\Pr(r_j | \gamma_j, \hat{\psi_j}, f_i, s_i', \mathcal{M})$ as a Gaussian following prior work and evidence from animal behavior \cite{fasola2013iros, o1996geometric}:
\begin{align}
\begin{split}
&\Pr(r_j | \gamma_j,\hat{\psi}_j,f_i, s_i', \mathcal{M})\\
&\qquad=\abs{u(s_i',\gamma_j,\mathcal{M})\bigcdot v(f_i,r_j,\gamma_{j},\hat{\psi}_{j})}\\
&\qquad\qquad\times\exp\left(-dist(s_i',\gamma_j,\mathcal{M})^2/2\sigma^2\right)
\end{split}
\end{align}
where $\sigma$ controls the steepness of the distribution based on the spatial relation's semantics and landmark size, and $dist(s_i',\gamma_j,\mathcal{M})$ is the distance between $s_i'$ to the closest position within the ground $\gamma_j$ in map $\mathcal{M}$, and $u(s_i',\gamma_j,\mathcal{M})\bigcdot v(f_i,r_j,\gamma_{j},\hat{\psi}_{j})$ is the dot product between $u(s_i',\gamma_j,\mathcal{M})$, the unit vector from $s_i'$ to the closest position within the ground $\gamma_j$ in map $\mathcal{M}$, and $v(f_i,r_j,\gamma_{j},\hat{\psi}_{j})$, a unit vector in in the direction that satisfies the semantics of the proposition $(f_i,r_j,\gamma_j)$ by rotating $\hat{\psi}_{j}$.
The dot product is skipped for prepositions that do not require FoRs (e.g. \emph{near}). We refer to \citet{landau1993whatwhere} for a list of prepositions meaningful in 2D that require FoRs: \emph{above, below, down, top, under, north, east, south, west, northwest, northeast, southwest, southeast, front, behind, left, right}.

\section{Data Collection}
\label{sec:data_collect}
In this section, we describe our data collection process as well as a pipeline for spatial information extraction from natural language.
We use maps from OpenStreetMap (OSM), a free and open-source database of the world map with voluntary landmark contributions \cite{OpenStreetMap}.
We scrape landmarks in 40,000m$^2$ grid-regions with a resolution of 5m by 5m grid cells in five different cities leading to a dimension of 41$\times$41 per grid map\footnote{Because of the curvature of the earth, the grid cells and overall region is not perfectly square, which is why the grid is not perfectly 40x40}: Austin, TX; Cleveland, OH; Denver, CO; Honolulu, HI, and Washington, DC. Geographic coordinates of OSM landmarks are translated into grid map coordinates.


To collect a variety of spatial language descriptions from each city, we randomly generate 30 environment configurations for each city, each with two target objects.
We prompt Amazon Mechanical Turk (AMT) workers to describe the location of the target objects and specify that the robot knows the map but does not know target locations.
Each configuration is used to obtain language descriptions from up to eleven different workers.
The descriptions are parsed using our pipeline described next in Sec.~\ref{sec:spatial_info_extraction}. Examples are shown in Fig.~\ref{fig:data_info}. Screenshots of the survey and statistics of the dataset are provided in the supplementary material.

The authors annotated FoRs for \emph{front, behind, left} and \emph{right} through a custom annotation tool which displays the AMT worker's language alongside the map without targets.
We manually infer the FoR used by the AMT worker, similar to what the robot is tasked to do.
This set of annotations are used as data to train and evaluate our FoR prediction model. Prepositions such as \emph{north, northeast} have absolute FoRs with known direction. Others are either difficult to annotate (e.g. \emph{across}) or have too little samples (e.g. \emph{above}, \emph{below}).

\subsection{Spatial Information Extraction from Natural Language}
\label{sec:spatial_info_extraction}


We designed a pipeline to extract spatial relation triplets from the natural language using the spaCy library \cite{spaCy2} for noun phrase (NP) identification and dependency parsing, as it achieves good performance on these tasks. Extracted NPs are matched against synonyms of target and landmark symbols using cosine similarity. All paths from targets to landmarks in the dependency parse tree are extracted to form the $(f,r,\gamma)$ tuples used as spatial language observations (Sec.~\ref{sec:integrate_pomdp}).

Our spatial language understanding models assume as input language that has been
parsed into $(f,r,\gamma)$ tuples, but is not dependent on this exact pipeline for
doing so. Future work could explore alternative methods for parsing and entity
linking, including approaches optimized for the task of spatial language
resolution. In our end-to-end experiments, we report the task performance both when using this parsing pipeline and when using manually annotated $(f,r,\gamma)$ tuples to indicate the influence of parsing on search performance.


\begin{figure}[t]
\centering
\includegraphics[width=\linewidth]{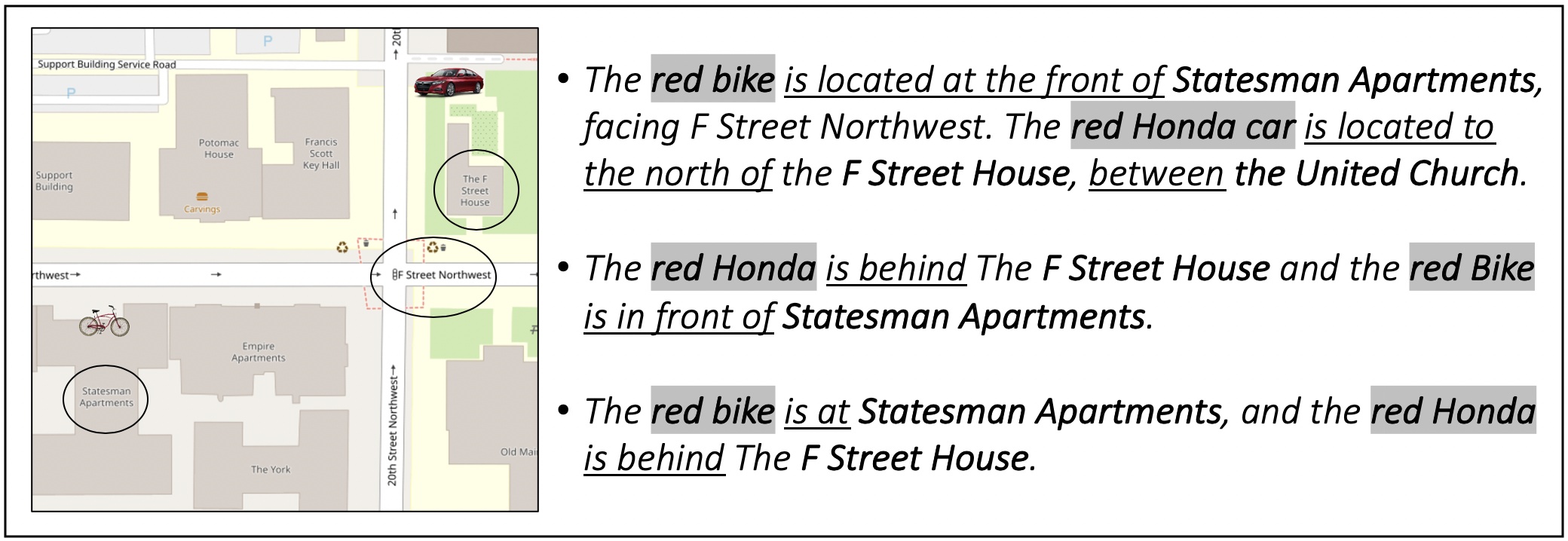}
\caption{Map screenshot shown to AMT workers paired with collected spatial language descriptions.}
\label{fig:data_info}
\end{figure}


\section{Evaluation}
\label{sec:evaluation}


\subsection{Evaluation of Frame of Reference Prediction Model}
\label{sec:foref_eval}
We test the generalizability of our FoR prediction model (\textbf{EGO-CTX}) by cross-validation. The model is trained on maps from four cities and tested on the remaining held-out city for all possible splits. We evaluate the model by the angular deviation between predicted and annotated FoR angles, in comparison with three baselines and human performance:
The first is a variation (\textbf{CTX}) that uses a synthetic image with the same kind of contextual representation
yet without egocentric shift. The second is another variation (\textbf{EGO}) that performs egocentric shift and also crops a $28\times 28$ window, but only highlights the referenced landmark at the center without contextual information. The random baseline (\textbf{Random}) predicts the angle at random uniformly between $[0,2\pi]$. The \textbf{Human} performance is obtained by first computing the differences between pairs of annotated FoR angles for the same landmarks (Eq.~\ref{eq:lossf}), then averaging over all such differences for landmarks per city.
Each pair of FoRs may be annotated by the same or different annotators.
Taking the average gives a sense of the disagreement among the annotators' interpretation of spatial relations.

\begin{figure}[ph!]
\centering
\vspace{-0.3cm}
  \includegraphics[width=\linewidth]{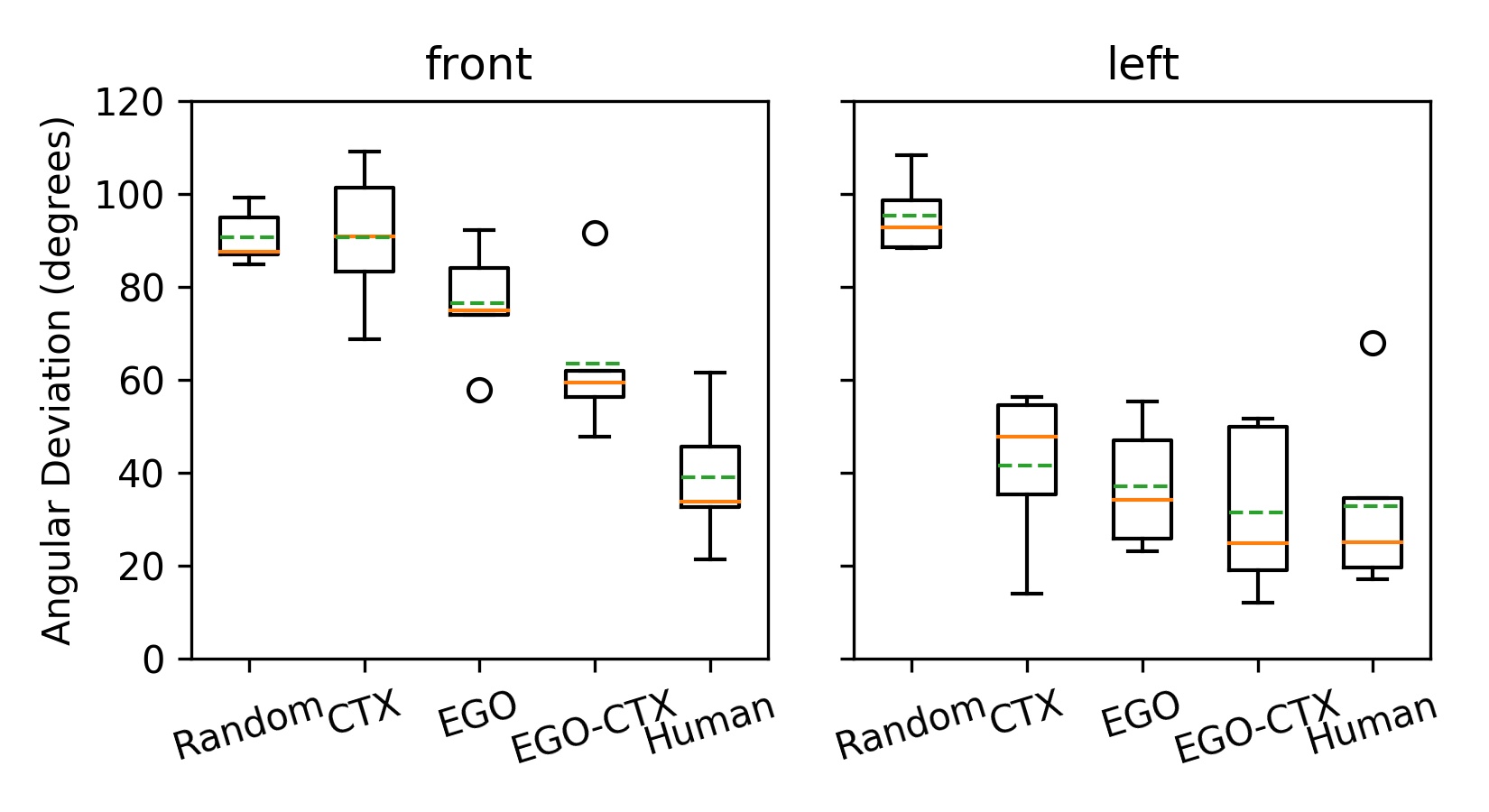}
  \caption{FoR prediction results. The solid orange line shows the median, and the dotted green line shows the mean. The circles are outliers. Lower is better.}
  \label{fig:foref_boxplot}
\end{figure}

\begin{figure}[ph!]
  \includegraphics[width=\linewidth]{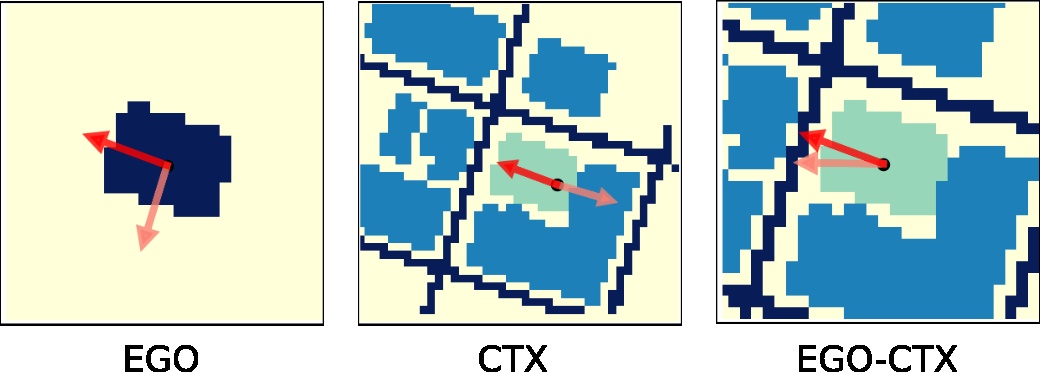}
  \caption{Visualization of FoR predictions for \emph{front}. Darker arrows indicate labeled FoR, while brighter arrows are predicted FoR.}
  \label{fig:foref_examples}
\vspace{-0.3cm}
\end{figure}

The results are shown in Figure~\ref{fig:foref_boxplot}. Each boxplot summarizes the means of baseline performance in the five cross-validation splits.
The results demonstrate that \textbf{EGO-CTX} shows generalizable behavior close to the human annotators, especially for \textbf{front}. We observe that our model is able to predict \emph{front} FoRs roughly perpendicular to streets against other buildings, while the baselines often fail to do so (Figure~\ref{fig:foref_examples}). The competitive performance of the neural network baselines in \textbf{left} indicates that for \emph{left} and \emph{right}, the FoR annotations are often absolute, i.e. independent of the context. Our model as well as baselines are limited in determining, for example, whether the speaker refers to the left side of the map (absolute FoR), or the left side of the street relative to a perceived forward direction (relative FoR).


\subsection{End-to-End Evaluation}
\label{sec:end-to-end-eval}
We randomly select 20 spatial descriptions per city.  We task the robot to search for each target object mentioned in every description separately, resulting in a total of 40 search trials per city, 200 in total. Cross-validation is employed such that for each city, the robot uses the FoR prediction model trained on the other four cities. For each step, the robot can either move or mark an object as detected. The robot can move by rotating clockwise or counterclockwise for 45 degrees, or move forward by 3 grid cells (15m). The robot receives observation through an on-board fan-shaped sensor after every move. The sensor has a field of view with an angle of 45 degrees and a varying depth of 3, 4, 5 (15m, 20m, 25m). As the field of view becomes smaller, the search task is expected to be more difficult. 
The robot receives $R_{\text{step}}=-10$ step cost for moving and $R_{\max}=+1000$ for correctly detecting the target, and $R_{\min}=-1000$ if the detection is incorrect.  The rest of the domain setup follows~\cite{wandzel2019oopomdp}.

\emph{Baselines.} \textbf{SLOOP} uses the spatial language observation model without mixture, that is, for each object, it computes the observation distribution in Eq.~(\ref{eq:sprls}) by multiplying the distributions for each spatial relation; With the same observation distribution, \textbf{SLOOP (m=2)} mixes in one distribution computed by treating all prepositions as \emph{near} with weight 0.2; Also with the same observation distribution, \textbf{SLOOP (m=4)} mixes in three additional distributions: one ignores FoR-dependent prepositions, one treating all prepositions as \emph{near}, one treating all prepositions as \emph{at}, with weights 0.25, 0.1, 0.05, respectively. The baseline \textbf{MOS (keyword)} uses a keyword-based model based on \cite{wandzel2019oopomdp} that assigns a uniform probability over referenced landmarks in a spatial language but does not incorporate information from spatial prepositions. Finally, \textbf{informed} and \textbf{uniform} are upper and lower bounds: for the \textbf{informed}, the agent has an initial belief that has a small Gaussian noise over the groundtruth location\footnote{The noise is necessary for object search, otherwise the task is trivial.}; \textbf{uniform} uses a uniform prior. We also report the performance with annotated spatial relations and landmarks to show search performance if the languages are parsed correctly.


For all baselines, we use an online POMDP solver, POMCP \cite{silver2010monte} but with a histogram belief representation to avoid particle depletion.
The number of simulations per planning step is 1000 for all baselines.  The discount factor is set to 0.95.  The robot is allowed to search for 200 steps per search task, since search beyond this point will earn very little discounted reward and is not efficient.

\begin{figure}[ph!]
  \centering
  \includegraphics[width=\linewidth]{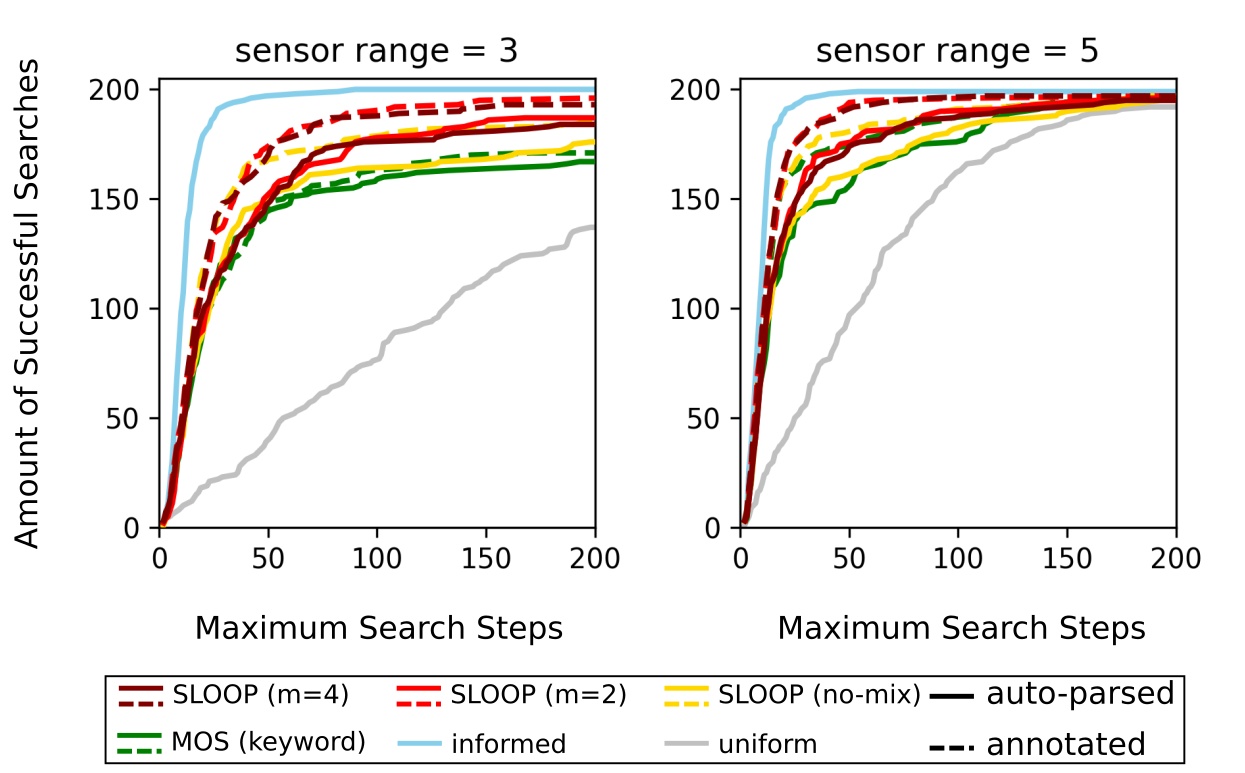}
  \caption{Number of completed search tasks as the maximum search step increases. Steeper slope indicates greater efficiency and success rate of search.}
  \label{fig:detections}
  \vspace{-1.0cm}
\end{figure}

\begin{figure}[ph!]
  \centering
  \includegraphics[width=\linewidth]{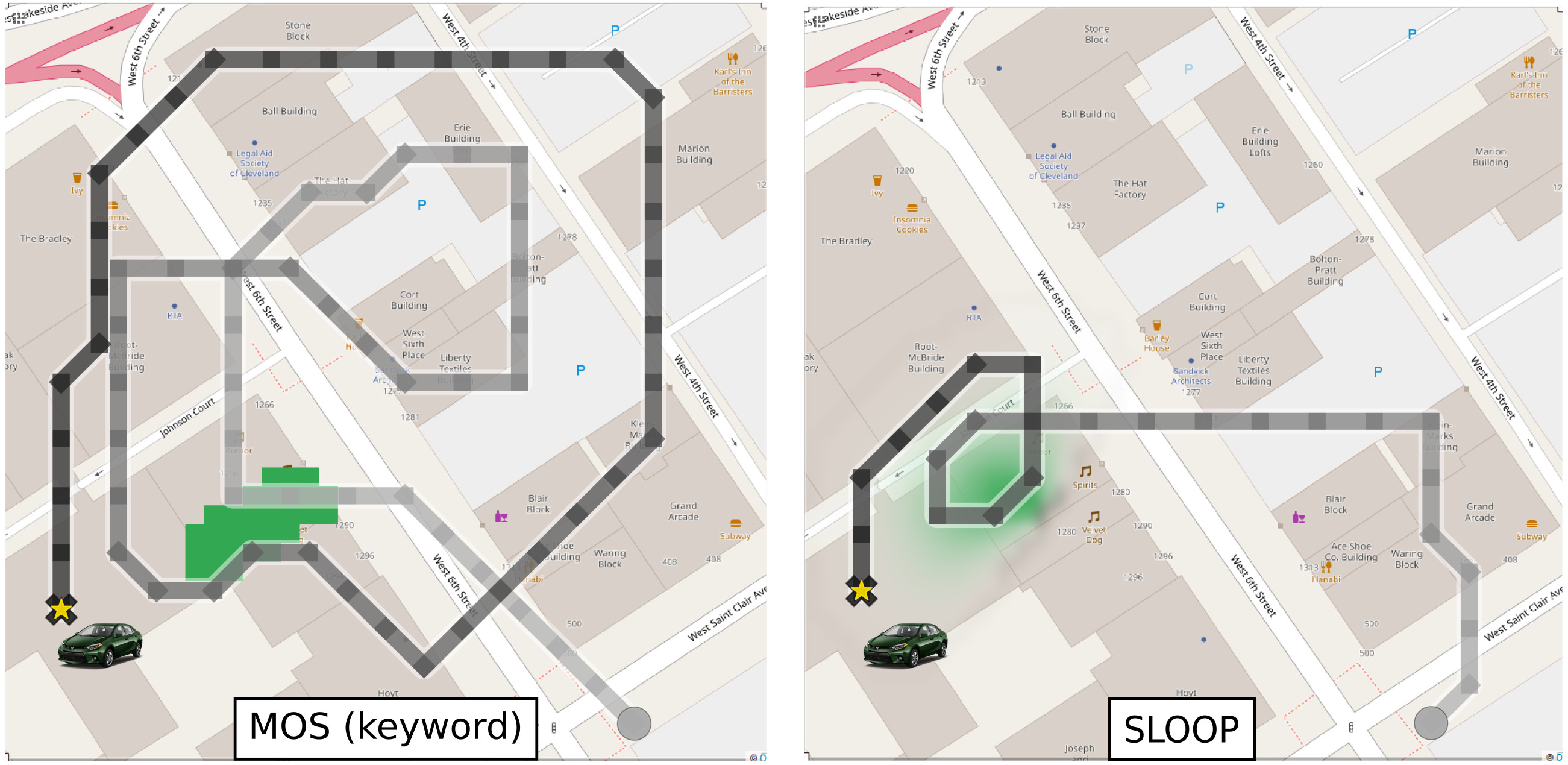}
  \caption{Example object search trial for description ``the green toyota is
    \emph{behind} velvet dog'' from AMT.  The green region shows the
    distribution over the object location after interpreting the description.
    Our method enables probabilistic interpretation of the spatial language
    leading to more efficient search strategy.}
  \label{fig:osmexample}
  \vspace{-0.2cm}
\end{figure}

\emph{Results.} We evaluate the effectiveness and efficiency of the search by the amount of search tasks the robot completed (i.e. successfully found the target) under a given limit of search steps (ranging from 1 to 200).
Results are shown in Figure~\ref{fig:detections}. The results show that using spatial language with SLOOP outperforms the keyword-based approach in MOS. The gain in the discounted reward is statistically significant over all sensor ranges comparing \textbf{SLOOP} with \textbf{MOS (keyword)}, and for sensor range 3 comparing the annotated versions. We observe that using mixture models for spatial language improves search efficiency and success rate over \textbf{SLOOP}. We observe improvement when the system is given annotated spatial relations. This suggests the important role of language parsing for spatial language understanding.
Figure~\ref{fig:osmexample} shows a trial comparing \textbf{SLOOP} and \textbf{MOS (keyword)}.


\begin{table}[ph!]
  \centering
    \begin{tabular}{rr|rrr}
      \toprule
      spatial     &  No.            & MOS (keyword)                          & SLOOP                         & SLOOP (m=4)\\              
      preposition &  trials    & annodated                                              & annotated                     & annotated\\                
      \midrule                                                                                                                                       %
      on	& 59	       & 200.65 (78.06)	                                        & 267.00 (76.45)	        &  \textbf{290.05 (70.51)}\\                          
      at	& 58	       & 179.42 (81.46)	                                        & 237.36 (81.03)	        &  \textbf{238.24 (80.50)} \\                         
      near	& 35	       & 97.64 (135.39)	                                        & \textbf{280.69 (109.00)}	&  249.35 (113.60) \\                         
      between	& 25	       & 21.48 (116.93)	                                        & 172.93 (143.77)	&  \textbf{175.59 (136.71)}  \\                         
      in	& 22	       & 302.05 (151.42)	                                & \textbf{398.88 (119.32)}	&  307.45 (141.67) \\                         
      \midrule
      north	& 9	       & 222.28 (291.13)	                                & 201.88 (296.49)	        & \textbf{365.14 (246.05)}\\
      southeast	& 7	       & 306.77 (341.43)	                                & 553.82 (174.04)	        & \textbf{549.43 (165.83)}\\
      southwest	& 7	       & -75.67 (205.98)	                                & \textbf{1.37 (271.46)}	                & -27.63 (281.89)\\
      east	& 6	       & 56.57 (337.58)	                                        & 290.68 (303.53)	        & \textbf{439.99 (276.85)}\\
      northwest	& 6	       & \textbf{385.41 (320.82)}	                                & 43.57 (282.88)	        & -1.82 (256.71) \\
      south	& 6	       & 79.12 (289.54)	                                        & 310.29 (410.04)	        & \textbf{494.26 (161.60)}\\

      west	& 4	       & -160.91 (188.02)	                                & 234.93 (587.57)	        & \textbf{327.13 (245.28)}\\

      northeast	& 2	       & -167.99 (660.76)	                                & 206.42 (138.93)	        & \textbf{213.17 (977.62)}\\
      \midrule
      front	& 25	       & \textbf{246.96 (142.45)}	                        & 168.91 (150.41)	        & 160.55 (136.88)  \\                         
      behind    & 8            & 128.47 (356.25)	                                & 101.20 (333.38)	        & \textbf{140.92 (333.61)}\\
      right	& 4	       & 19.75 (697.88)	                                        & 160.14 (601.54)	        & \textbf{336.00 (725.84)}\\
      left	& 3	       & \textbf{247.35 (363.32)}	                        & 192.93 (393.75)	        & 231.75 (330.33)\\
      \midrule
      front (good) & 15	       & 255.85 (210.46)		                        & \textbf{421.83 (143.65)}	& 222.67 (264.50) \\                          
      behind (good) & 6	       & 145.26 (489.55)		                        & 207.58 (430.52)		& \textbf{359.80 (753.15)}     \\        
      front (bad)  & 10        & \textbf{281.04 (226.47)}				& -208.92 (11.39)	        & 93.80 (176.11)             \\              
      behind (bad) & 2	       & \textbf{78.11 (3771.84)}	                        & -217.95 (23.53)	        & -77.97 (223.71)                \\           
      \bottomrule
\end{tabular}
\caption{Mean (95\% CI) of discounted cumulative reward for different prepositions
  evaluated on language descriptions with annotated spatial relations. The value with highest mean per row is bolded.
  }
\label{tab:prepositions}

\end{table}

We analyze the performance with respect to different spatial prepositions. We report results for annotated languages as they reflect the performance obtained if the prepositions are correctly identified. Results for the smallest sensor range of 3 is shown in Table~\ref{tab:prepositions}.
\textbf{SLOOP} outperforms the baseline for the majority of prepositions.
For prepositions \emph{front}, \emph{behind}, \emph{left}, and \emph{right}, our investigation shows the performance of \textbf{SLOOP} polarizes where trials with ``good'' FoR (i.e. ones in the correct direction towards the true target location) leads to a much greater performance than the counterpart (``bad'' FoR). Yet, \textbf{MOS (keyword)} is not subject to such polarization and the target often appears close to the landmark for these prepositions. We observe that \textbf{SLOOP (m=4)} using mixture is able to consistently improve the reward for most of the prepositions, indicating the benefit of modeling multiple interpretations of the spatial language.

\begin{table}[ph!]
  \centering
\begin{tabular}{rr|rrr}
\toprule
  No. spatial~      &   No.~   & MOS (keyword)    & SLOOP                         & SLOOP (m=4)\\              
  prepositions      & trials   & annotated                                  & annotated                     & annotated\\
\midrule
1	          &    100           & 234.32 (72.64)	                        & \textbf{320.91 (64.23)}		& 289.34 (66.42)\\
2                 &    83            & 179.18 (68.60)	                        & 264.08 (62.39)		& \textbf{286.19 (60.83)}\\
3	          &    14            & 26.96 (165.98)	                        & 115.44 (202.42)		& \textbf{215.30 (200.99)}\\
\bottomrule
\end{tabular}
\caption{Mean (95\% CI) of discounted cumulative reward for completed search tasks as language complexity (number of spatial relations) increases.}
\label{tab:numrels}
\vspace{-1em}
\end{table}

Finally, we analyze the relationship between the performance and varying complexity of the language description, indicated by the number of spatial relations used to describe the target location.  Again, we used annotated languages for this experiment for the smallest sensor range of 3. Results in Table~\ref{tab:numrels} indicate that understanding spatial language can benefit search performance, with a wider gain as the number of spatial relations increases.
Again, using mixture model in \textbf{SLOOP (m=4)} improves the performance even more.

\subsection{Demonstration on AirSim}
We implemented \textbf{SLOOP (m=4)} on AirSim \cite{airsim2017fsr}, a realistic drone simulator built on top of Unreal Engine \cite{unrealengine}. Similar to our evaluation in OpenStreetMap, we discretize the map into 41$\times$41 grid cells. We use the same fan-shaped model of on-board sensor as in OpenStreetMap. As mentioned in Sec.~\ref{sec:setting}, sensor observations are synthetic, based on the ground-truth state. Additionally, although the underlying localization and control is continuous, the drone plans discrete navigation actions (move forward, rotate left 90$^{\circ}$, rotate right 90$^{\circ}$). We annotated landmarks (houses and streets) in the scene on the 2D grid map. Houses with heights greater than flight height are subject to collision and results in a large penalty reward (-1000). Checking for collision in the POMDP model for this domain helped prevent such behavior during planning. We found that the FoR prediction model trained on OpenStreetMap generalizes to this domain, consistently producing reasonable FoR predictions for \emph{front} and \emph{behind}. This shows the benefit of using synthetic images of top-down street maps. The drone is able to plan actions to search back and forth to find the object, despite given inexact spatial language description. Please refer to the video demo on our project for the examples shown in Fig.~\ref{fig:firstfig} and \ref{fig:airsimexample}.

\begin{figure}[H]
  \centering
  \includegraphics[width=\linewidth]{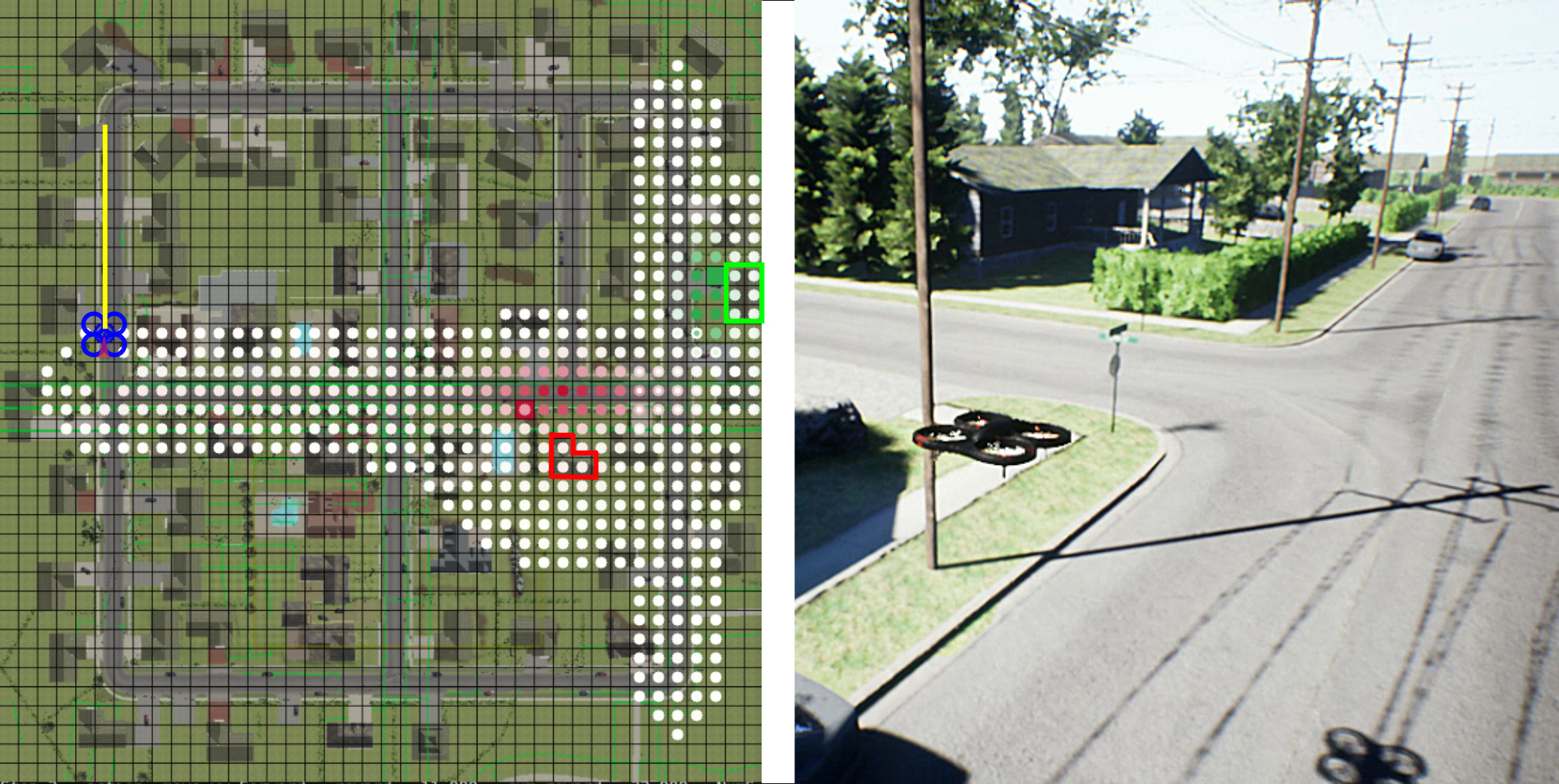}
  \caption{Example trial from AirSim demonstration. Given spatial language description: \emph{The red car is by the side of Main street in front of Zoey's House (red), while the green car is within Annie's House on the right side of East street (green).} Left: belief over two target objects (red and green car). Right: screenshot from AirSim. Video: \url{https://youtu.be/V54RY8v8VmA}}

  \label{fig:airsimexample}
\vspace{-0.1cm}
\end{figure}

\subsection{Demonstration on Boston Dynamics Spot }
We take one step further and demonstrate SLOOP on a physical robot. In particular, we task a Boston Dynamics Spot robot to find a book in a lab environment. We developed a 2D MOS-based object search system integrated with ROS that can interact with Spot through the Spot SDK. This system uses histogram-based 2D belief, and the object search POMDP model is the same 2D MOS model used for simulation experiments.\footnote{In fact, this system is the npredecessor to GenMOS for 3D-MOS in Chapter~\ref{ch:genmos}, GenMOS is robot-independent (does not rely on ROS).} Object detection was done through projecting segmentation masks obtained with Mask-RCNN \cite{he2017mask} to 2D grid cells using depth from Spot's gripper camera. To obtain the same kind of 2D grid map representation as in simulation experiments, we first use Spot's built-in capability to create a 3D point cloud map of the lab and then project it down to 2D, where points are filtered to separate free space from obstacles. Then,  in order to apply SLOOP for 2D MOS, we build a map of landmarks similar to OpenStreetMap by driving the robot around the lab, and recording object detection results projected onto a 2D grid, and then assigning a unique name to each detected object. Since our FoR prediction model works over synthetic images, they can be directly applied to this landmark map. The end result is an interface that allows a person to type in a free-form spatial language with respect to the landmarks, and the belief gets updated after the system interprets the spatial language, and the robot goes off to search under the updated belief;  Planning is based on the hierarchical planning algorithm  introduced for COS-POMDP in Chapter~\ref{ch:cospomdp} that combines local, fine-grained search actions with navigation subgoals over a topological graph.  In Figure~\ref{fig:spot_sloop}, we contrast the system's behavior with and without spatial language. SLOOP enables the robot to quickly narrow down the search region given \emph{``the book is in front of the Monitor,''} whereas the baseline without language results in searching all over the lab.

\begin{figure}[H]
  \centering
  \includegraphics[width=\linewidth]{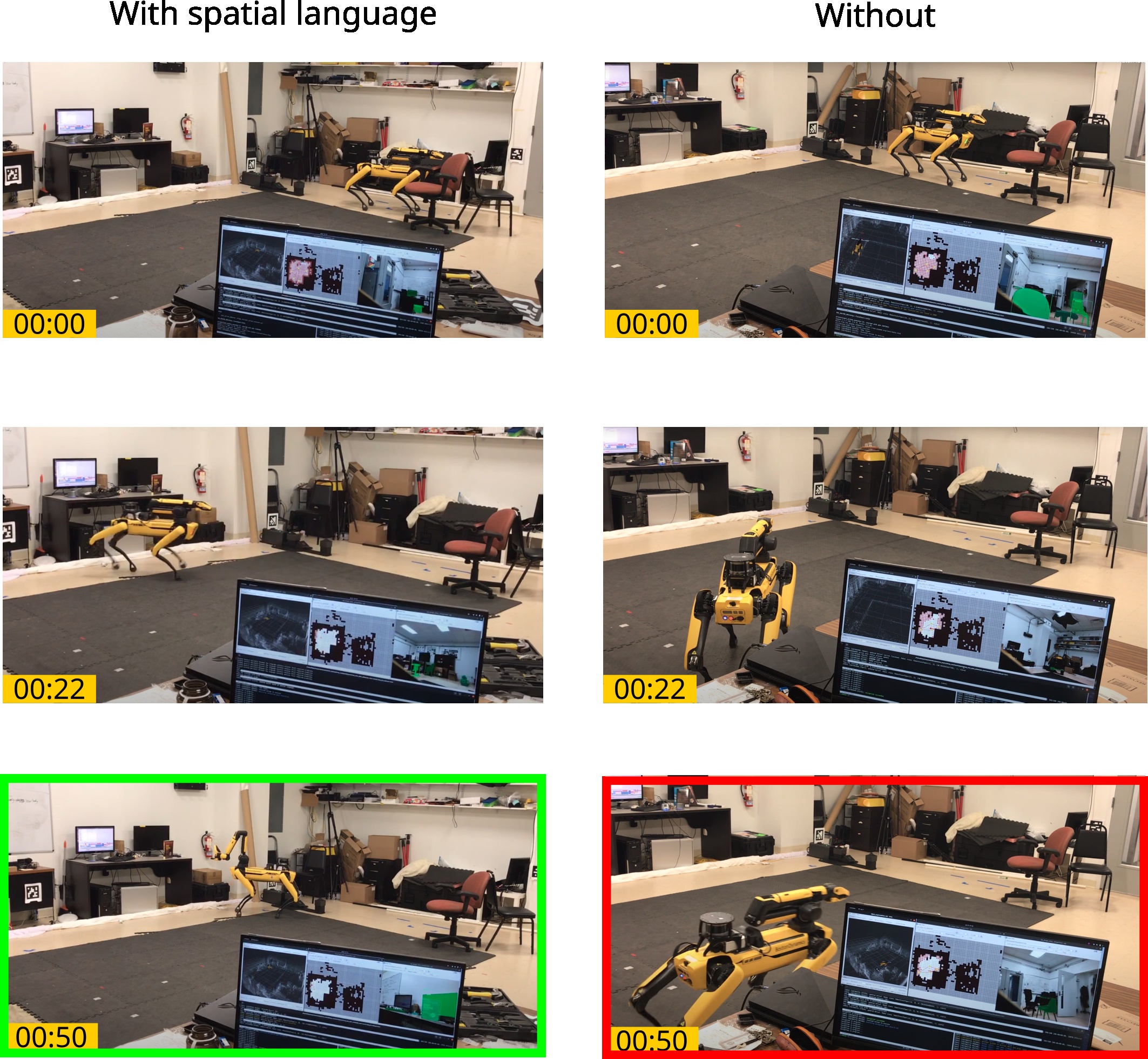}
  \caption{Demonstration of SLOOP on Spot, given the spatial language \emph{``the book is in front of the Monitor.''} The total time used for ``with spatial language'' includes the time to initially type in the language, while the baseline without language proceeds to search right off the bat. Nevertheless, SLOOP quickly narrows down the search space and the robot successfully finds the book.}
  \label{fig:spot_sloop}
  \vspace{-0.2cm}
\end{figure}

\newpage
\section{Summary}
\label{sec:conclusion}
This chapter first presents a formalism for integrating spatial language into the POMDP belief state as an observation, then a convolutional neural network for FoR prediction shown to generalize well to new cities. Simulation experiments show that our system significantly improves object search efficiency and effectiveness in city-scale domains
through understanding spatial language. For future work, we plan to investigate compositionality in spatial language for partially observable domains.

\section{Appendix}
\subsection{Derivation of Spatial Language Observation Model}

Here we provide the derivation for Eq.~(\ref{eq:sprls}).
Using the definition of $o_{\sprl_i}$, 
\begin{align}
\Pr(o_{\sprl_i} | s_i', \mathcal{M})&=\Pr(f_i, r_1,\gamma_1,\cdots,r_L,\gamma_L | s_i', \mathcal{M})\\
&=\Pr(r_1,\cdots,r_L | \gamma_1,\cdots,\gamma_L,f_i, s_i', \mathcal{M})
\times\Pr(\gamma_1,\cdots,\gamma_L, f_i| s_i', \mathcal{M})\label{eq:por1}\\
&=\frac{1}{Z}\prod_{j=1}^L\Pr(r_j | \gamma_j,f_i, s_i', \mathcal{M})\label{eq:por2}
\end{align}
The first term in (\ref{eq:por1}) is factored by individual spatial relations, because each $r_j$ is a predicate that, by definition, involves only the figure $f_i$ and the ground $\gamma_j$, therefore it is conditionally independent of all other relations and grounds given $f_i$, its location $s_i'$, and the landmark $\gamma_j$ and its features contained in $\mathcal{M}$. Because the robot has no prior knowledge regarding the human observer's language use,\footnote{In general, the human observer may produce spatial language that mentions arbitrary landmarks and figures whether they make sense or not.}  the second term in (\ref{eq:por1}) is uniform with probability $1/Z$ where $Z$ is the constant size of the support for $\gamma_1,\cdots,\gamma_L,f_i$. This constant can be canceled out during POMDP belief update upon receiving the spatial language observation, using the belief update formula in  Section~\ref{sec:pomdp:formal}. We omit this constant in Eq.~(\ref{eq:sprls}).

For predicates such as \emph{behind}, its truth value depends on the relative FoR imposed by the human observer who knows the target location. Denote the FoR vector corresponding to $r_{j}$ as a random variable $\Psi_{j}$ that distributes according to the indicator function $\Pr(\Psi_{j}=\psi_{j})=\mathbbm{1}(\psi_{j}=\psi_{j}^*)$, where $\psi_{j}^*$ is the one imposed by the human. Then regarding $\Pr(r_{j}|\gamma_j,f_i,s_i',\mathcal{M})$, we can sum out $\Psi_j$:
\begin{align}
  \Pr(r_{j}|\gamma_j,f_i,s_i',\mathcal{M})&=\frac{\sum_{\psi_j}\Pr(r_{j},\gamma_j,f_i,s_i',\mathcal{M}|\psi_j)\Pr(\psi_j)}{\Pr(\gamma_j,f_i,s_i',\mathcal{M})}\\
  \intertext{Since the distribution for $\Psi_j$ is an indicator function,}
                                          &=\frac{\Pr(r_{j},\gamma_j,f_i,s_i',\mathcal{M}|\psi_j^*)}{\Pr(\gamma_j,f_i,s_i',\mathcal{M})}\label{eq:por3}\\
  \intertext{By the law of total probability,}
                                          &=\frac{\Pr(r_j|\gamma_j,\psi_j^*,f_i,s_i',\mathcal{M})\Pr(\gamma_j,f_i|s_i',\mathcal{M},\psi_j^*)\Pr(s_i',\mathcal{M}|\psi_j^*)}{\Pr(\gamma_j,f_i|s_i',\mathcal{M})\Pr(s_i',\mathcal{M})}\\
  \intertext{Using the fact that $s_i',\mathcal{M}$ is independent of $\psi_j^*$,}
                                          &=\frac{\Pr(r_j|\gamma_j,\psi_j^*,f_i,s_i',\mathcal{M})\Pr(\gamma_j,f_i|s_i',\mathcal{M},\psi_j^*)\Pr(s_i',\mathcal{M})}{\Pr(\gamma_j,f_i|s_i',\mathcal{M})\Pr(s_i',\mathcal{M})}\\
  \intertext{Canceling out $\Pr(s_i',\mathcal{M})$,}
                                          &=\frac{\Pr(r_j|\gamma_j,\psi_j^*,f_i,s_i',\mathcal{M})\Pr(\gamma_j,f_i|s_i',\mathcal{M},\psi_j^*)}{\Pr(\gamma_j,f_i|s_i',\mathcal{M})}\\
  \intertext{Similar to (\ref{eq:por1})-(\ref{eq:por2}), $\Pr(\gamma_j,f_i|s_i',\mathcal{M},\psi_j^*)$ and $\Pr(\gamma_j,f_i|s_i',\mathcal{M})$ are uniform with the same support. Canceling them out,}
&=\Pr(r_j | \gamma_j,\psi_j^*,f_i, s_i', \mathcal{M})\label{eq:por4}
\end{align}

\subsection{Data Collection Details}


\subsubsection{Amazon Mechanical Turk Questionnaire}
As described in Section~\ref{sec:data_collect}, we collect a variety of spatial language descriptions from five cities. We randomly generate 10 unique configurations of two object locations for every pair of object symbols from \{\texttt{RedBike}, \texttt{RedCar}, \texttt{RedCar}\}. Each configuration is used to obtain language descriptions from up to eleven different workers.
By showing a picture of the objects placed on the map screenshot, we prompt AMT workers to describe the location of the target objects. We first show an example task as shown in Fig~\ref{fig:amt-questionnaire} (top). Then we prompt them with the actual task and a text box to submit their response, as shown in Fig~\ref{fig:amt-questionnaire} (bottom). Note that we specify that the robot does not know where the target objects are, but that it knows the buildings, streets and other landmarks available on the map. We encourage them to use the information available on the map in their description.

\begin{figure}[H]
\begin{subfigure}{\textwidth}
  \centering
  \includegraphics[width=0.75\linewidth]{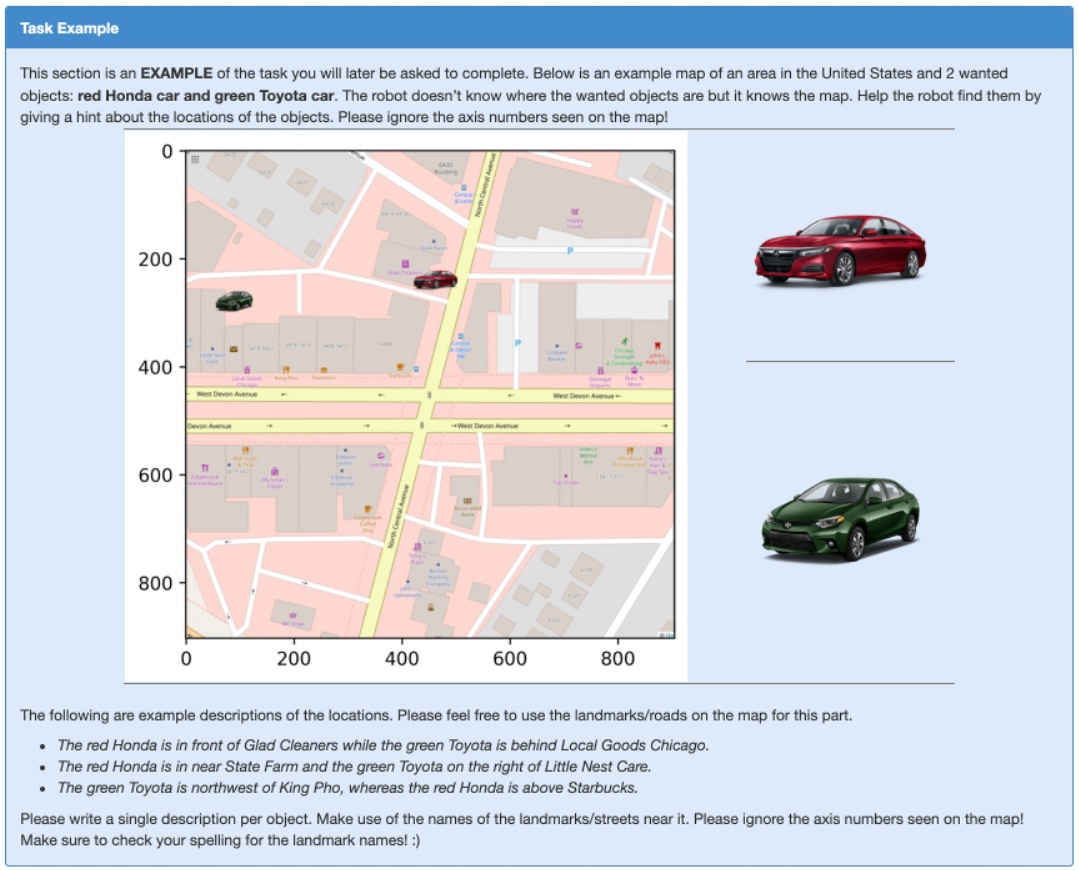}
\end{subfigure}
\begin{subfigure}{\textwidth}
  \centering
  \includegraphics[width=0.75\linewidth]{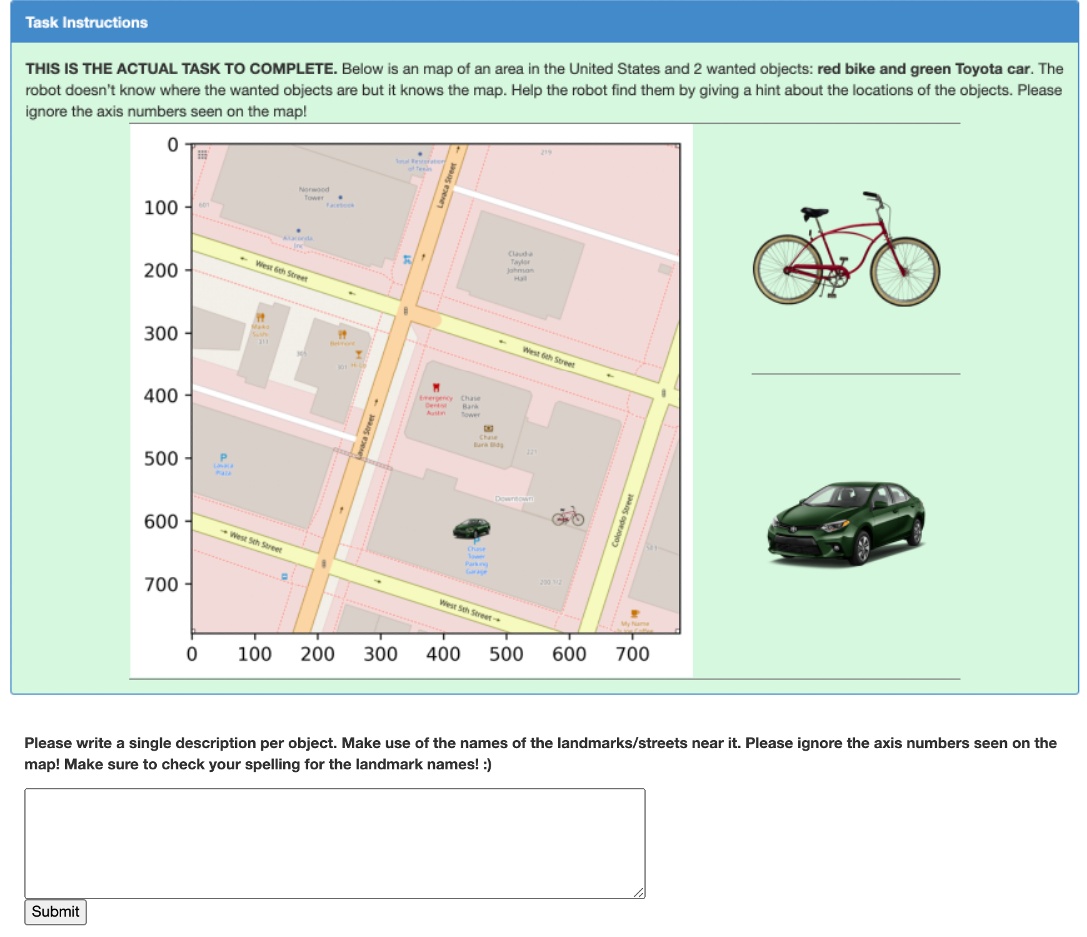}
\end{subfigure}
\caption{AMT Questionnaire Screenshot. Top: an example task shown to the AMT workers prior to the actual task that doesn't vary from prompt to prompt. Bottom: the actual task shown to the AMT workers. The objects and the locations change in every prompt and are all unique.}
\label{fig:amt-questionnaire}
\end{figure}

\subsection{Distribution of Collected Predicates}
Each description is parsed using our pipeline described in Section~\ref{sec:spatial_info_extraction}. 1,521 out of 1,650 gathered descriptions were successfully parsed; meaning at least one spatial relation was extracted from the sentence. The distribution of all spatial predicates are shown in Fig~\ref{fig:pred-dist}. Note that we include the word ``is'' on the list since it often appears in  ``is in'' or ``is at'', yet the parser sometimes skip the word after it due to an artifact. We excluded language descriptions parsed with such artifact from the ones used in the end-to-end object search evaluation.


\begin{figure}[H]
  \centering
\makebox[\textwidth][c]{\includegraphics[width=1.15\textwidth]{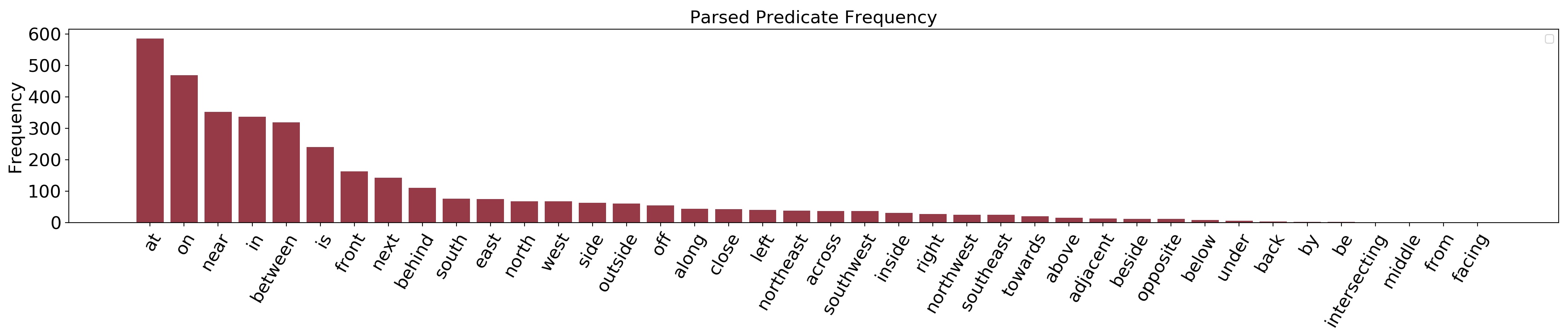}}
\caption{Distribution of collected parsable predicates sorted from most frequent to least.}
\label{fig:pred-dist}
\end{figure}

\begin{table}[H]
\centering
\begin{tabular}{rr}
\toprule
Spatial preposition & \# of FoR annotations \\
\midrule
front            & 121              \\
behind           & 54               \\
left             & 51               \\
right            & 47               \\
\bottomrule
\end{tabular}
\caption{Number of FoR annotations per spatial relation.}
\label{tab:number-for-annot}
\end{table}

\newpage
\subsection{FoR Annotation}
To collect the FoR annotation, we create our custom annotation tool. The FoR consists of a front (purple) and right (green) vector that show how the speaker considers the direction of the ground according to their given spatial description. The interface (Fig~\ref{fig:annotator-gui}) works as follows: (1) The annotator first clicks the ``Annotate'' button. (2) The interface prompts the annotator the language phrase corresponding to a spatial relation to be annotated (e.g. ``\texttt{RedBike} \texttt{is in front of} \texttt{EmpireApartments}''), which is composed using the parsed $(f,r,\gamma)$ tuple. (3) to annotate an FoR, the annotator clicks on the map as the origin of the FoR, and then clicks on another point on the map as the end point of the \emph{front} vector. The vector for \emph{right} is automatically computed to be 90 degrees clockwise with respect to the \emph{front} vector. (4) After annotating one FoR, the annotator clicks ``Next'' to move on, and the process starts over again from step (2). In total we have 273 FoR annotations. Table~\ref{tab:number-for-annot} shows the amount of annotation per spatial relation. You can download the dataset by visiting the website linked in the footnote on the first page.


\begin{figure}[t]
  \centering
  \includegraphics[width=0.67\linewidth]{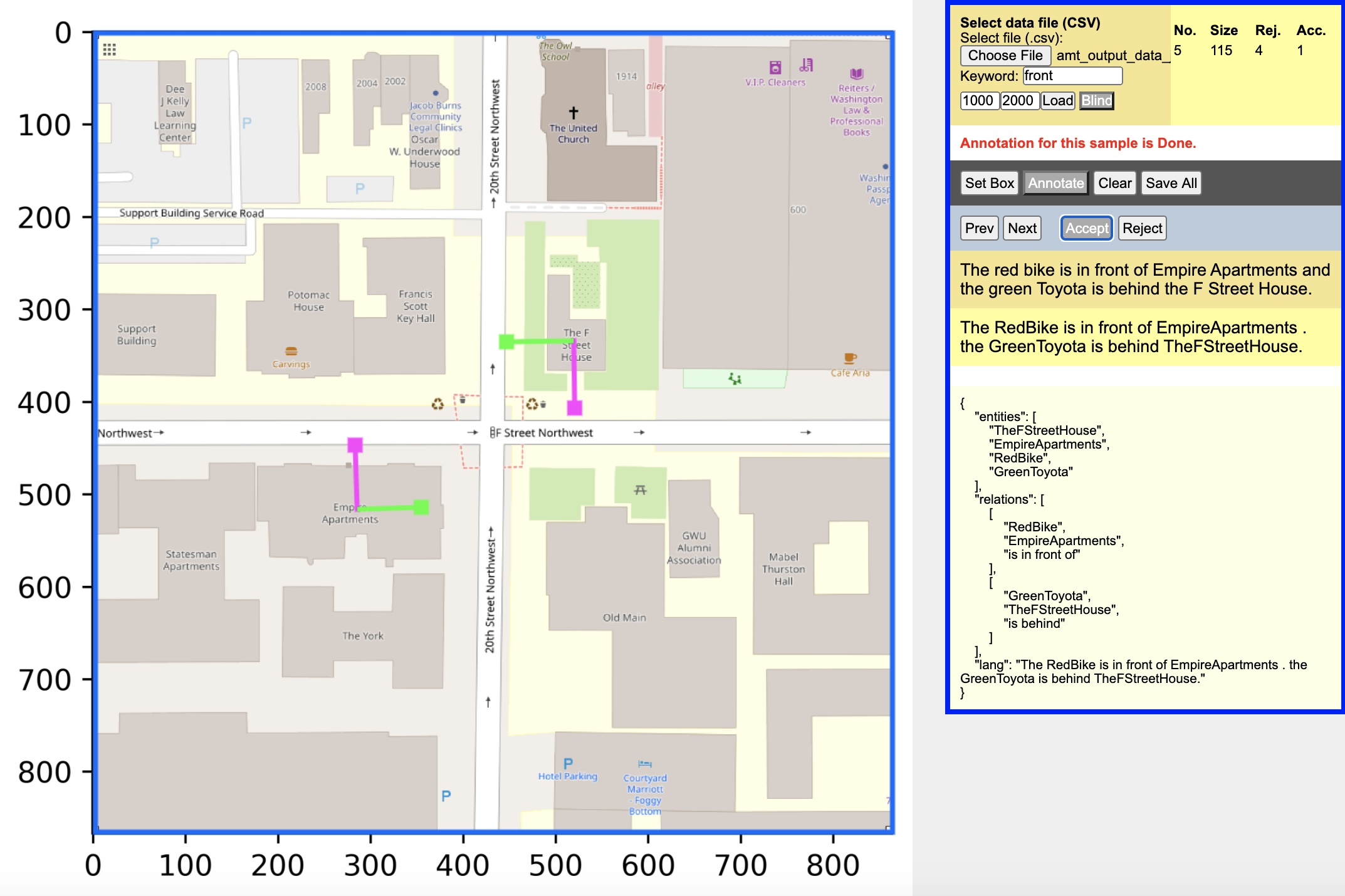}
\caption{The FoR annotator GUI interface. The FoR consists of a front (purple) and right (green) vector. The target objects are not shown, and the annotator only has access to the spatial language description, and the map image. This mimics the situation faced by the robot in our task.}
\label{fig:annotator-gui}
\end{figure}

\chapter{Dialogue Object Search: Preliminary Work}
\label{ch:dialog}
$\qquad$\vspace{-1.1in}

  \begin{figure}[H]
\centering
\includegraphics[width=\linewidth]{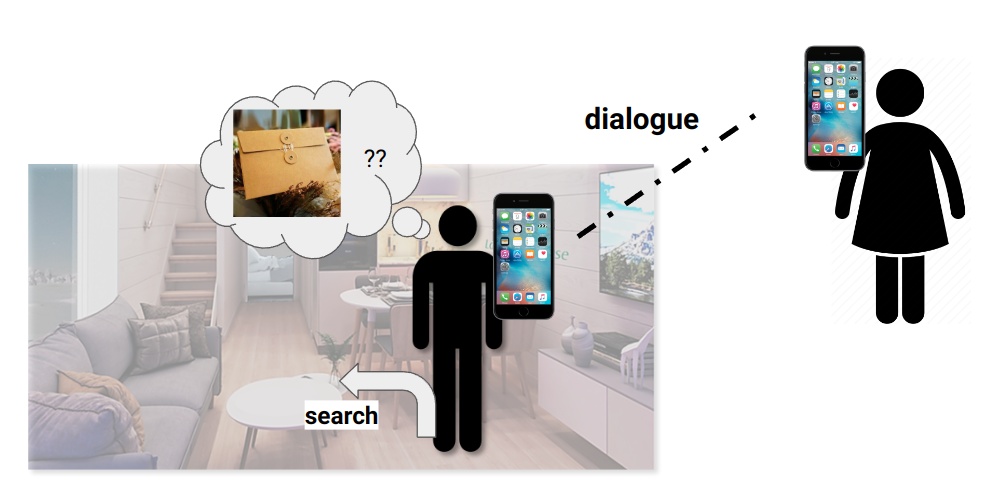}
\caption{The motivating scenario for dialogue object search. Imagine one person is searching for a file at home. Not knowing where it could be, he calls a family member who knows more. Then, he is able to verbally engage in the dialogue over the phone while performing search physically. We view dialogue object search as one route towards realizing such ability to decide what to say and how to act simultaneously, fundamental for future collaborative robots.}
\label{fig:mos3d}
\end{figure}

\section{Motivation}
  \lettrine{W}{e} envision robots that can collaborate and communicate seamlessly with
  humans.  It is necessary for such robots to decide both what to say and how to
  act, while interacting with humans.  To this end, we introduce a new task,
  \emph{dialogue object search}: A robot is tasked to search for a target object
  (e.g. fork) in a human environment (e.g., kitchen), while engaging in a ``video
  call'' with a remote human who has additional but inexact knowledge about the
  target's location. That is, the robot conducts speech-based dialogue with the human,
  while sharing the image from its mounted camera. This task is challenging at
  multiple levels, from data collection, algorithm and system development, to
  evaluation. Despite these challenges, we believe such a task blocks the path
  towards more intelligent and collaborative robots.  In this extended abstract,
  we motivate and introduce the dialogue object search task and analyze examples
  collected from a pilot study. We then discuss our next steps and conclude with
  several challenges on which we hope to receive feedback.

  Humans can act in the physical world (such as walking, looking, or opening a cabinet) while having a conversation with others. As robots enter homes and care centers, we envision them to have such capability as well when collaborating and communicating with humans. To achieve this, robots must decide both what to say and how to act towards a goal. This involves combining task-oriented dialogue systems with decision making under uncertainty for embodied agents. Traditionally, dialogue systems have involved users interacting with a virtual agent for tasks such as technical support \cite{mouromtsev2015spoken}, personal assistance (e.g., Siri) and booking reservations~\cite{wen2016network,wei2018airdialogue}. While recent works have proposed datasets that combine dialogue and dynamic, embodied decision making \cite{de2018talk,thomason2020vision}, the investigated problems over these datasets are limited to prediction tasks that bypass the challenges of evaluating a conversational embodied agent. For example, the Navigation from Dialog History Task \cite{thomason2020vision} asks the agent to predict the next navigation action, given a history of dialogue and past navigation actions. The tourist localization task \cite{de2018talk} asks the system to predict a location given a language description.

Our goal is to enable robots to naturally engage in a dialogue with a human
while completing a task autonomously. We believe a task that captures the
sequential nature of both the dialogue and physical decision making is necessary
for in-depth study towards this goal. We choose to focus on object search, a
useful and widely-studied problem~\cite{aydemir2013active,kollar2009utilizing,sloop-roman-2021,zheng2021multi}

\section{Contributions}
The main contribution is we introduce a new task called \emph{dialogue object search}.
From the pilot study (Sec.~\ref{sec:pilot}), we observed that participants produced language and behavior that are more natural using speech, because text-based dialogue requires users to decide whether to type or act at every step. Additionally, we summarize the set of intent types that are observed to succinctly capture the intents behind the unstructured utterances in the dialogue object search task.

\section{Dialogue Object Search}
A robot is tasked to search for a target object in a human
environment (e.g., kitchen) while engaging in an audio
dialogue with a remote human assistant, who possesses inexact prior knowledge
about the target object's location. In our pilot study, this is given in the form of a 2D scatter plot (Fig.~\ref{fig:timeline}). The robot
has a mounted RGB-D camera, and shares its view with the human assistant. We
assume the robot and the human assistant have access to two different sequences
of RGB-D images of the scene, which represent their prior experiences of living in
that environment. Target objects are excluded from these images.
The robot must decide what to say and how to act, in
order to efficiently find the target object while naturally interacting and
collaborating with the human assistant.


Our inspiration for the above setting comes from the following scenario
between two people living together (family or friends).  One person is searching
for something, such as a document or a key, but not sure where it is. They
decide to video call the other home member who is currently out of the house but may have a better idea.  They then engage in a dialogue while the first person conducts the search for the target object.
We envision that in the future, this could happen between a home assistant robot and a human user.

\afterpage{%
    \clearpage
    \thispagestyle{empty}
    \begin{sidewaysfigure}
    \centering
    \makebox[\textwidth][c]{\includegraphics[width=\linewidth]{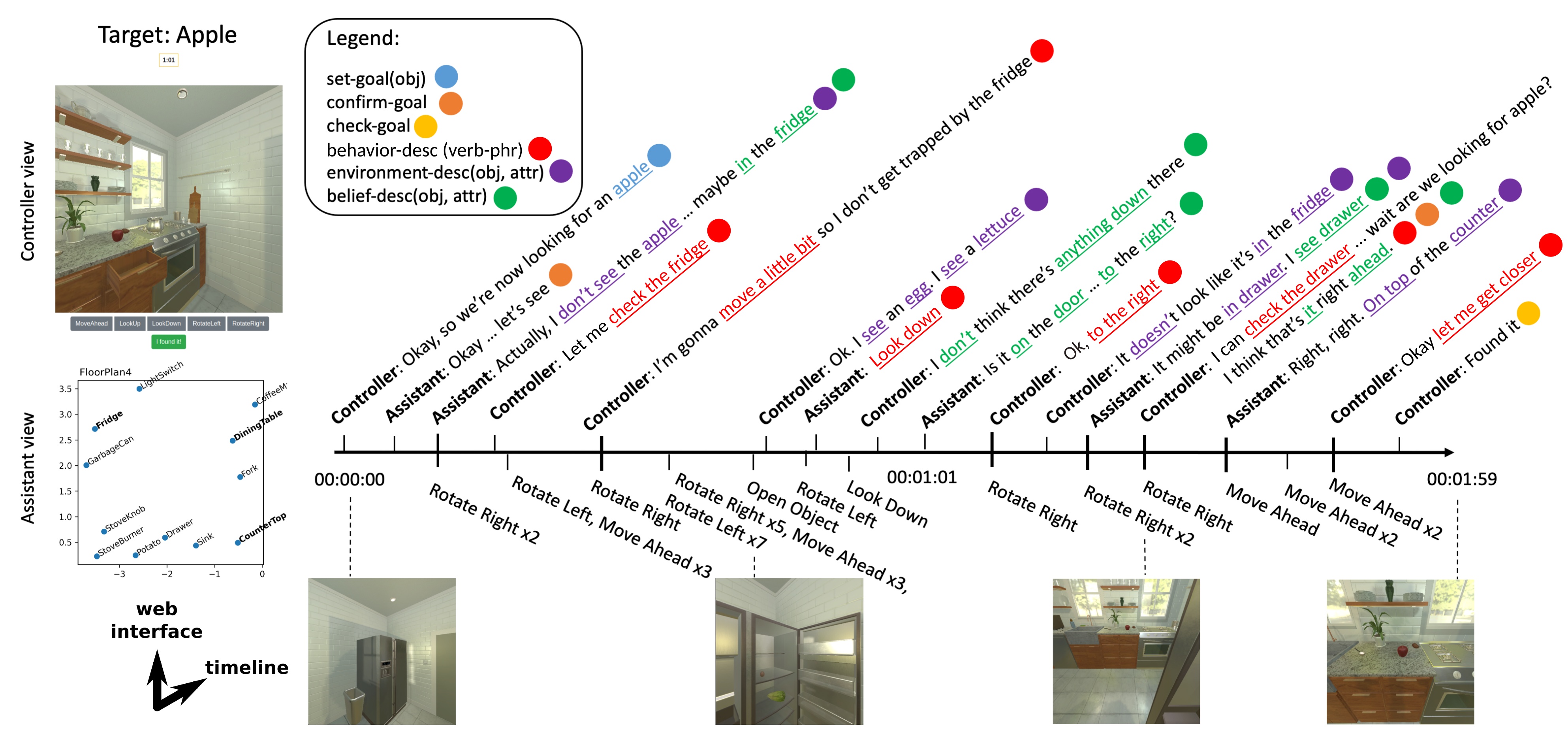}}
    \caption{We conducted a pilot study to understand desirable behavior for the dialogue object search task. Shown here is a screenshot of the web interface (left) and the dialogue and actions organized onto a timeline (right), for an object search trial where the target object is \texttt{Apple}. We classified the dialogue utterances into a preliminary set of parameterized intents, indicated by the colors.}
    \label{fig:timeline}
    \end{sidewaysfigure}
    \clearpage
}
\section{Pilot Study}
\label{sec:pilot}

To investigate the above task, we first attempted to understand how a human would behave if they are in the robot's position. We designed and conducted a pilot study among three pairs of people (authors' lab members) using AI2-THOR~\cite{kolve2017ai2} as the simulated home environments. In this study, we designate two roles according to the above problem setting. The \emph{Assistant} is the person assisting in the process as the robot searches for a given target object. The \emph{Controller} is the person who is taking on the role of the robot. Due to the pandemic, we used Zoom to record the audio and create transcripts of the dialogue. We implemented a web-based data collection tool where the \emph{Controller} controls the agent in AI2-THOR through the web interface, and the \emph{Assistant} has access to a 2D scatter plot of a subset of objects in the scene (Fig.~\ref{fig:timeline}). Each pair of participants are assigned three object search trials in one environment. They have 90 seconds to explore the environment (with target objects removed) and 180 seconds to complete each trial. In addition to dialogue audio and transcripts, we collected data about the scene per view, the action executed, and the agent's groundtruth pose as provided by the AI2-THOR framework. We considered a discrete action space of \{\emph{MoveAhead}(0.25m), \emph{RotateLeft}(45$\degree$), \emph{RotateRight}(45$\degree$), \emph{LookUp}(30$\degree$), \emph{LookDown}(30$\degree$), \emph{Open}, \emph{Close}\}.\footnote{We first experimented with a rotation angle of 90$\degree$ following \cite{gordon2018iqa,ye2021hierarchical}, but experienced sudden jumps that are unnatural as felt by the participants. Therefore, we switch to 45$\degree$, also used by some existing works \cite{wortsman2019learning,qiu2020learning}}.

Despite the small scale of our pilot dataset, we observed some interesting behaviors shared between trials. For example, at the beginning of the object search trials the \emph{Assistant} would specify the target object and the \emph{Controller} would confirm. Additionally, as the task progresses, both roles would describe behaviors, beliefs about the environment and location of objects, and visual observations. We codified these into a set of preliminary intent types; some examples are given in the figure above. Using this pilot dataset, we have started to explore the development of an autonomous agent (\emph{Controller}), both modular and end-to-end that can plan actions for this task.

As mentioned in the introduction, we experimented with both speech-based dialogue and text-based dialogue, using the recording and chat features of Zoom. With speech, participants typically engage in frequent back-and-forth, as the \emph{Controller} controls the agent. Such exchanges involve discussing, for example, the scene and possible target locations. Participants report that when using text, the \emph{Controller} must decide between controlling the agent in AI2-THOR versus typing in the chat. Consequently, they would try to search for the object themselves without interacting with the \emph{Assistant}, who, as a result, finds it difficult to tell if their input is being considered by the \emph{Controller}. This suggests collecting dialogue data through text is unnatural and misaligned with our goal.

\section{Discussion \& Next Steps}
Though truthful to the task, our pilot data collection procedure is currently not scalable. We plan to implement a system that can be deployed on the crowdsourcing platform Amazon Mechanical Turk (AMT), to pair up Turkers to participate in the task entirely through their web browsers for accessibility. AMT a powerful platform, yet not designed for multi-user tasks. Due to audio communication and running AI2-THOR servers, we face a more difficult situation than ~\citet{das2017visual} who had to implemented a live chatbot on AMT. We also need solutions to scalable and accurate transcription of the collected audio as well as intent labling. We seek suggestions for strategies to collect such data at scale. In terms of evaluation, we believe both experiment with simulated assistants and real human assistants are necessary. For the simulated assistant, we are considering an oracle agent that communicates using template-based language. The goal of this simulated agent is to facilitate efficient and repeatable evaluations during algorithm development for the embodied dialogue agent, which could be a long-term effort. Ultimately, the agent should be deployed to perform the task with real human subjects. We plan to consider objective metrics for both object search performance (e.g., success rate and discounted total return and dialogue quality~\cite{venkatesh2018evaluating}, and, eventually, subjective metrics such as naturalness~\cite{hung2009towards}. We believe finding solutions to scalable speech-based dialogue data collection for embodied tasks and plausible evaluation protocol are daunting, yet unavoidable challenges towards future collaborative robots.

\vspace{3.0in}
\begin{center}
  THIS IS THE END OF THIS CHAPTER.
\end{center}


\chapter{Conclusion}
\label{ch:closing}

\lettrine{T}his thesis was motivated by the practical challenges for object search and the goal of making object search an off-the-shelf ability for robots. It argues for using POMDP to model object search and exploiting structure in the human world and human-robot interaction to achieve practical and effective object search systems.

We have provided a few examples in supporting this argument: octrees for 3D multi-object search, spatial correlations for searching for hard-to-detect objects, and structure in spatial language for searching with a hint. We were able to demonstrate our proposed algorithms in realistic simulators and on real robots. Notably, we built GenMOS, the first robot-independent, environment agnostic system for object search in 3D, and integrated it with three different robots in different environments. As a preliminary work, we also identified patterns in the intents behind utterances for dialogue object search.

We have also proposed taxonomies for three main aspects of object search studies: problem settings, solution methods, and the systems
and applied them in the literature review of more than 125 papers related to object search. 

\paragraph{Future Work}
\noindent
Chapter~\ref{ch:overarching} has provided a generic POMDP model for object search, which has been the parent model for all the POMDP models we developed for specific problems. However, the problems studied in this thesis are still basic in the sense that they have involved static target objects and no environment interaction. Achieving generalized object search beyond the basic setting is the big open problem and ultimate pursuit in this field. Additionally, it is beneficial to study ways to learn models of the environment which can be used for object search planning. This is likely more useful for search involving environment interactions.

If the question is whether a generic POMDP for object search model in Chapter~\ref{ch:overarching} is extensible to handle the additional challenges, there is both potential to explore and foreseeable hurdles to worry. In terms of potential, the transition function in the POMDP can be easily extended to consider motion of target objects. Besides, the high-level $\LOOK$ action can be extended to be an abstraction over not only navigation-based search actions, but also manipulation-based, which all serve the purpose of perceiving some part of the environment.  However, trouble arises as the consequences of manipulation to the environment may become intractable to model. For example, suppose that a book might be underneath a bed. In this case, a good way to search underneath the bed is not to bend over and look, but to find a tool, such as a stick, to sweep underneath the bed. A robot that can reason at this level faces the challenge of considering consequences of sweeping in the environment. How to make a robot that can perform search while being aware of the consequence of its own actions to the environment? This is a valuable direction of future work.

\vspace{3.0in}
\begin{center}
  THIS IS THE END OF THIS CHAPTER.
\end{center}

\setcounter{chapter}{0}
\renewcommand{\thechapter}{\Alph{chapter}}
\renewcommand{\chaptertitlename}{\textsc{Appendix}}
\renewcommand{\thesection}{\arabic{section}}
\addtocontents{toc}{\vspace{1\baselineskip}}  
\addcontentsline{toc}{chapter}{Appendix}

\chapter{The \texttt{pomdp\_py} library}
\label{ch:pomdp-py}

 In this appendix, we present \texttt{pomdp\_py}, a general purpose Partially Observable Markov Decision Process (POMDP) library written in Python and Cython. Existing POMDP libraries often hinder accessibility and efficient prototyping due to the underlying programming language or interfaces, and require extra complexity in software toolchain to integrate with robotics systems. \texttt{pomdp\_py} features simple and comprehensive interfaces capable of describing large discrete or continuous (PO)MDP problems. Here, we summarize the design principles and describe in detail the programming model and interfaces in \texttt{pomdp\_py}. We also describe intuitive integration of this library with ROS (Robot Operating System), which enabled our torso-actuated robot to perform object search in 3D.  Finally, we note directions to improve and extend this library for POMDP planning and beyond.

\section{Introduction}

Partially Observable Markov Decision Processes (POMDP) are a sequential decision-making framework suitable to model many robotics problems, from localization and mapping \cite{ocana2005indoor} to human-robot interaction \cite{whitney2017reducing}. Early efforts in developing tools for POMDPs attempt to separate solvers from domain description by creating specialized file formats to specify POMDPs \cite{pomdpfileformat,pomdpxformat}, which are not designed for large and complex problems. Among libraries under active development, Approximate POMDP Planning Toolkit (APPL) \cite{somani2013despot} and AI-Toolbox \cite{aitoolbox} are implemented in C++ and contain numerous solvers. However, the learning curve for these libraries is steep as C++ is less accessible to current researchers in general compared to Python \cite{virtanen2020scipy}. POMDPs.jl \cite{egorov2017pomdps} is a POMDP library with a suite of solvers and domains, written in Julia. Though promising, Julia has yet to achieve a wide recognition and creates language barrier for many researchers. POMDPy \cite{emami2015pomdpy} is implemented purely in Python. Yet with an original focus on POMCP implementation, it assumes a blackbox world model in its POMDP interface, limiting its extensibility. Finally, a promising toolchain is to use Relational Dynamic Influence Diagram Language (RDDL) \cite{sanner2010relational} to describe factored POMDPs and solve them via ROSPlan \cite{cashmore2015rosplan}, recently demonstrated for object fetching \cite{canal2019probabilistic}. Nevertheless, using this set of tools adds overhead of using a classical fluent-based planning paradigm, which is not required to describe and solve POMDPs in general.

This leads to our belief that there lacks a POMDP library with simple interfaces that brings together both accessibility and performance. We address this demand by presenting \texttt{pomdp\_py}, a framework to build and solve POMDP problems written in Python and Cython \cite{behnel2011cython}. It features simple and comprehensive interfaces to describe POMDP or MDP problems, and can be integrated with ROS \cite{quigley2009ros} intuitively through \texttt{rospy}. In the rest of this appendix, we first review POMDPs, then illustrate the design principles and key features of \texttt{pomdp\_py}, including integration with ROS. Finally, we note directions to improve and extend this library, in hope of cultivating an open-source community for POMDP-related research and development. The documentation of \texttt{pomdp\_py} is available at: \url{https://h2r.github.io/pomdp-py/html/}. Tutorials on example domains can be found in the documentation.
This library is currently actively developed as we continue our POMDP-related research.

\section{POMDPs}
POMDPs \cite{kaelbling1998planning} model sequential decision making problems where the agent must act under partial observability of the environment state. Refer to Background (Section~\ref{sec:pomdp:formal}, page~\pageref{sec:pomdp:formal}) for a formal introduction of POMDPs.

\textbf{Solvers.} Most recent POMDP solvers are \emph{anytime algorithms}~\cite{zilberstein1996using,ross2008online}, due to the intractable computation required to solve POMDPs exactly \cite{madani1999undecidability}.
There are currently two major camps of anytime solvers, point-based methods \cite{kurniawati2008sarsop,shani2013survey} which approximates the belief space by a set of reachable $\alpha$-vectors, and Monte-Carlo tree search-based methods \cite{silver2010monte,somani2013despot} that explores a subset of future action-observation sequences.

Currently, \texttt{pomdp\_py} contains an implementation of POMCP and PO-UCT \cite{silver2010monte}, as well as a naive exact value iteration algorithm without pruning \cite{kaelbling1998planning}. The interfaces of the library support implementation of other algorithms; We hope to cultivate a community to implement more solvers or create bridges between \texttt{pomdp\_py} and other libraries.

\textbf{Belief representation} The partial observability of environment state implies that the agent has to maintain a posterior distribution over possible states \cite{thrun2005probabilistic}. The agent should update this belief distribution through new actions and observations. A tabular belief representation requires nested iterations over the state space to update the belief, which is computationally intractable in large domains. Particle belief representation is a simple and scalable belief representation which is updated through matching simulated and real observations exactly \cite{silver2010monte}. Different schemes of weighted particles have been proposed to handle large or continuous observation spaces where exact matching results in particle depletion \cite{sunberg2018online,garg2019despot}.

\texttt{pomdp\_py} does not commit to any specific belief representation. It provides implementations for basic belief representations and update algorithms, including tabular, particles, and multi-variate Gaussians, but more importantly  allows the user to create their own new or problem-specific representation, according to the interface of a generative probability distribution.

\section{Design Philosphy}

\label{sec:design}
 Our goal is to design a framework that allows simple and intuitive ways of defining POMDPs at scale for both discrete and continuous domains, as well as solving them either through planning or through reinforcement learning. In addition, we implement this framework in Python and Cython to improve accessibility and prototyping efficiency without losing orders of magnitude in performance \cite{behnel2011cython,smith2015cython}. 
We summarize the design principles behind \texttt{pomdp\_py} below:
 \begin{itemize}
     \item Fundamentally, we view the POMDP scenario as the interaction between an \emph{agent} and the \emph{environment}, through a few important generative probability distributions ($\pi$, $T,O,R$ or blackbox model $G$).
     \item The agent and the environment may carry different models to support learning, since for real-world problems especially in robotics, the agent generally does not know the true transition or reward models underlying the environment, and only acts based on a simplified or estimated model.
     \item The POMDP domain could be very large or continuous, thus explicit enumeration of elements in the spaces should be optional.
     \item The representation of belief distribution is decided by the user and can be customized, as long as it follows the interface of a generative distribution.
     \item Models can be reused across different POMDP problems. Extensions of the POMDP framework to, for example, decentralized POMDPs, should also be possible by building upon existing interfaces.
 \end{itemize}

 \section{Programming Model and Features}
 The basis of \texttt{pomdp\_py} is a set of simple interfaces that collectively form a framework for building and solving POMDPs.

 When defining a POMDP, one first defines the \emph{domain} by implementing the \texttt{State}, \texttt{Action}, \texttt{Observation} interfaces. The only required functions for each interface are \texttt{\_\_eq\_\_} and \texttt{\_\_hash\_\_}. For example, the interface for \texttt{State} is simply\footnote{Note that the code snippets here are modified or shortened slightly for display purposes. Please refer to the code on github: \url{https://github.com/h2r/pomdp-py/}}:
 \begin{verbatim}
class State:
   def __eq__(self, other):
       raise NotImplementedError
   def __hash__(self):
       raise NotImplementedError
 \end{verbatim}
 Next, one defines the \emph{models} by implementing the interfaces \texttt{TransitionModel}, \texttt{ObservationModel}, etc. Note that one may define a different transition and reward model for the agent than the environment (e.g. for learning). One also defines a \texttt{PolicyModel} which (1) determines the action space at a given history or state, and (2) samples an action from this space according to some probability distribution.  Implementing these models involves implementing the \texttt{probability}, \texttt{sample} and \texttt{argmax} functions. For example, the interface for \texttt{ObservationModel}, modeling $O(s',a,o)=\Pr(o|s',a)$, is:
 \begin{verbatim}
class ObservationModel:
   def probability(self, observation, next_state, action):
       """Returns the probability Pr(o|s',a)."""
       raise NotImplementedError
   def sample(self, next_state, action):
       """Returns a sample o ~ Pr(o|s',a)."""
       raise NotImplementedError
   def argmax(self, next_state, action):
       """Returns o* = argmax_o Pr(o|s',a)."""
       raise NotImplementedError
   def get_all_observations(self, *args, **kwargs):
       """Returns a set of all possible
       observations, if feasible."""
       raise NotImplementedError
 \end{verbatim}
 It is up to the user to choose which subset of these functions to implement, depending on the domain. These interfaces aim to remind users the essence of models in POMDPs.

\begin{figure}[t]
  \centering
  \makebox[\textwidth][c]{\includegraphics[width=1.18\textwidth,draft=false]{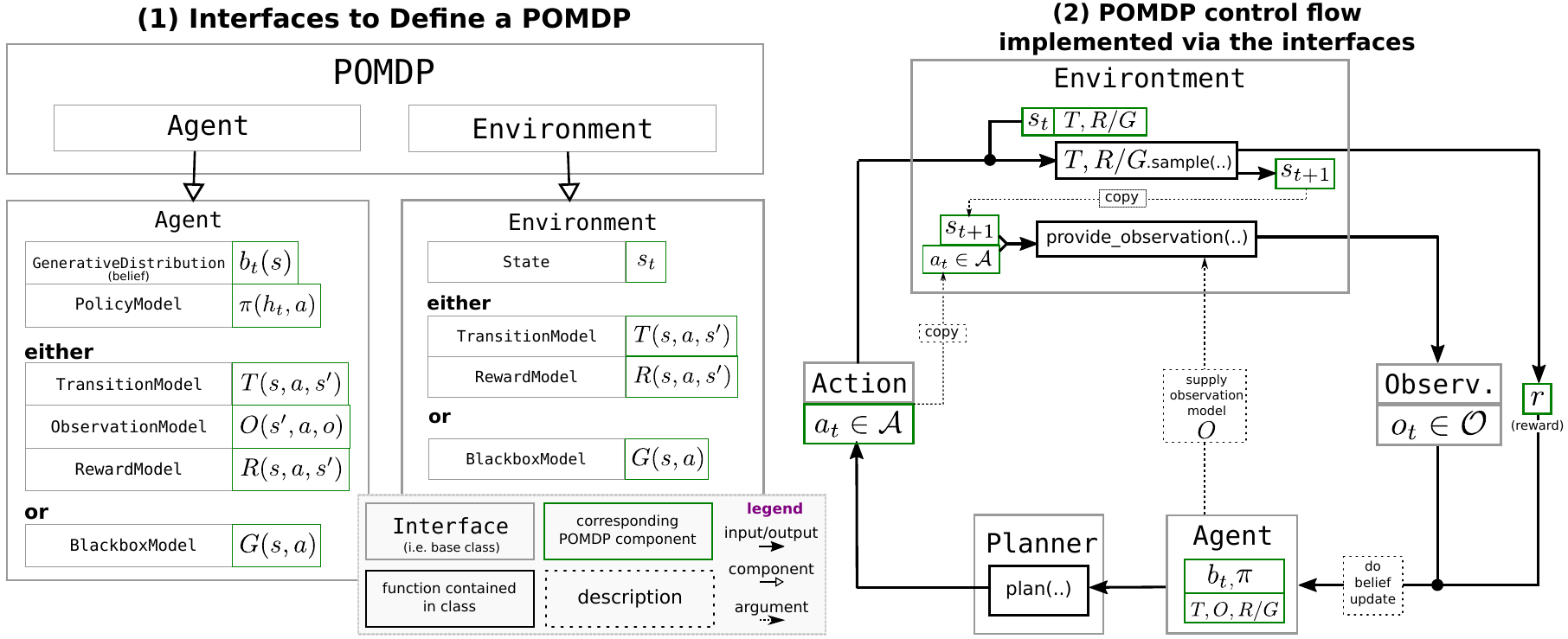}}
    \caption{(1) Core Interfaces in the \texttt{pomdp\_py} framework; (2) POMDP control flow implemented through interaction between the core interfaces.}
    \label{fig:framework}
\end{figure}

 To instantiate a POMDP, one provides parameters for the models, the initial state of the environment, and the initial belief of the agent. For the Tiger problem\footnote{\url{https://h2r.github.io/pomdp-py/html/examples.tiger.html}} \cite{kaelbling1998planning}, for example,
 \begin{verbatim}
 s0 = random.choice(list(TigerProblem.STATES))
 b0 = pomdp_py.Histogram({State("tiger-left"): 0.5,
         State("tiger-right"): 0.5})
 tiger_problem = TigerProblem(..., s0, b0)
 \end{verbatim}
 Here, \texttt{TigerProblem} is a \texttt{POMDP} whose constructor takes care of initializing the \texttt{Agent} and \texttt{Environment} objects, and is instantiated by parameters (omitted), initial state and belief. Note that it is entirely optional to explicitly define a problem class such as \texttt{TigerProblem} in order to program the POMDP control flow, discussed below.

 To solve a POMDP with \texttt{pomdp\_py}, here is the control flow one should implement that contains the basic steps:
\begin{enumerate}
    \item Create a planner (\texttt{Planner}), i.e. a POMDP solver.
    \item Agent plans an action $a\in\Aspace$ through the planner.
    \item Environment state transitions $s_t\rightarrow s_{t+1}$ according to its transition model.
    \item Agent receives an observation $o_t$ and reward $r_t$ from the environment.
    \item Agent updates history and belief. $h_t,b_t\rightarrow h_{t+1},b_{t+1}$, where $h_{t+1}=h_{t}(a_t,o_t)$.
    \item Unless \emph{termination condition} is true, repeat steps 2-5.
\end{enumerate}
The \texttt{Planner} interface is as follows. The planner may be updated given a real action and a real observation, which is necessary for MCTS-based solvers.
\begin{verbatim}
class Planner:
    def plan(self, agent):
        """The agent carries the information:
        Bt, ht, O,T,R/G, pi, needed for planning"""
        raise NotImplementedError
    def update(self, agent, action, observation):
        """Updates the planner based on real action
        and observation. Updates the agent belief
        accordingly if necessary. """
        pass
\end{verbatim}

\noindent\textbf{Code Organization.}  In a more complicated problem such as the Light-Dark domain \cite{platt2010belief} or Multi-Object Search with fan-shaped sensors \cite{wandzel2019multi}, it may be tricky to organize the code base and be consistent across different problems. Below we provide a recommendation of the package structure to use \texttt{pomdp\_py} to guide the development and facilitate code sharing:
\lstset{
  frame=tb,
  language=C,
  basicstyle={\ttm},
  keywordstyle={\ttb\color{deepblue}},
  commentstyle=\color{deepgreen}
}
\begin{verbatim}
 - domain/
    - state.py        // State
    - action.py       // Action
    - observation.py  // Observation
    - ...
 - models/
    - transition_model.py   // TransitionModel
    - observation_model.py  // ObservationModel
    - reward_model.py       // RewardModel
    - policy_model.py       // PolicyModel
    - ...
- agent/
    - agent.py  // Agent
    - ...
- env/
    - env.py  // Environment
    - ...
 - problem.py  // POMDP
\end{verbatim}
The recommendation is to separate code for domain, models, agent and environment, and have simple generic filenames. As in the above tree, files such as \texttt{state.py} or \texttt{transition\_model.py} are self-evident in their role. The \texttt{problem.py} file is where the specific implementation of the \texttt{POMDP} class is defined, and where the logic of control flow is implemented. Refer to the Multi-Object Search example in the documentation for more detail\footnote{\url{https://h2r.github.io/pomdp-py/html/examples.mos.html}}.\\

\noindent\textbf{Object-Oriented POMDPs.}
OO-POMDP \cite{wandzel2019multi} is a particular kind of factored POMDP that factors the state and observation spaces into a set of $n$ objects. For instance, $\Pr(s'|s,a)=\prod_i\Pr(s'_i|s,a)$, $i\in\{1,\cdots,n\}$. The belief space is also factored, which allows the belief space to grow linearly instead of exponentially as the number of objects increases. Each object is of a certain class and has a set of attributes. The values of these attributes constitute the state of an object. In \texttt{pomdp\_py}, we provide interfaces to implement OO-POMDPs, which serves as an example of extending the basic POMDP framework to create another class of model. These interfaces include \texttt{OOState}, \texttt{OOBelief}, \texttt{OOTransitionModel}, etc.\\

\noindent\textbf{Integration with ROS.}
ROS \cite{quigley2009ros} is an open-source system that builds a network connecting computing stations and robots, where \emph{nodes} interact with one another through publishing messages or making service requests. It is typical to separate nodes that manage resources and controls the robot from nodes that runs sophisticated algorithms. This is the case of \texttt{pomdp\_py} as well. The POMDP-related computations can be done on a node that implements the POMDP control flow (see the six steps above). Inside this node, when an action is selected by the \texttt{Planner} (step 1), the node can publish a message to the nodes for robot control so that the robot can execute the action (step 2). The environment state automatically updates in the real world as a result of that action (step 3), and the node receives the sensor measurements or other forms of observations through subscribed topics (step 4), and performs belief update (step 5). This process is repeated until termination condition is met (step 6). ROS provides a package \texttt{rospy} which eases the integration of the POMDP control flow with the robot system.\\


\section{Summary}
We present a POMDP library, named \texttt{pomdp\_py}, that brings together accessibility to programmers through Python as well as performance through Cython, with an intuitive design and straightforward integration with ROS. The programming model is designed to encourage organized development and code sharing within a community, and it has potential to support research besides POMDP planning, including reinforcement learning, transfer learning, and multi-agent systems.



\vskip 0.2in

\thispagestyle{plain}
\addtocontents{toc}{\vspace{1\baselineskip}}  
\chapter*{{\normalfont \huge REFERENCES \hfill}}
\renewcommand{\bibsection}{}
\phantomsection
\addcontentsline{toc}{chapter}{References}
\fancyhead{} 
\fancyhead[CO]{\hfill\nouppercase{\textsc{References}}} 
\fancyhead[CE]{\nouppercase{\textsc{References}\hfill}} 
\bibliography{dissertation}
\bibliographystyle{apacite}

\newpage
\clearpage
\thispagestyle{plain}
\chapter*{{\normalfont \huge REVISION NOTICE \hfill}}
\phantomsection
\addcontentsline{toc}{chapter}{Revision Notice}
\fancyhead{} 
\fancyhead[CO]{\hfill\nouppercase{\textsc{Revision Notice}}} 
\fancyhead[CE]{\nouppercase{\textsc{Revision Notice}\hfill}} 
The present document reflects the following revisions:
\begin{center}
  \begin{itemize}[itemsep=0.5pt,topsep=0pt]
  \item January 13th, 2023: submitted in full to the Brown University Graduate School.
  \item January 20th, 2023: revised the chapter on GenMOS (Chapter~\ref{ch:genmos}) following paper acceptance to ICRA 2023. In particular, added Figure~\ref{fig:genmos:demogrid}, and clarified contributions in Section~\ref{sec:genmos:sysov}. Added definitions of default value and initial value of an octree belief node in Definition~\ref{def:3dmos:default_val} and \ref{def:3dmos:initial_val} in Chapter~\ref{ch:3dmos}, then referenced in Chapter~\ref{ch:genmos} for clarity.
  \item January 22nd, 2023: completed the acknowledgements. Fixed writing bugs in the \texttt{pomdp\_py} appendix and added a figure based on \cite{pomdp-py-2020}.
  \item March 4th, 2023: updated Chapter \ref{ch:genmos} on GenMOS with additional simulation results (Greedy and Random baselines) and discussion from the camera-ready.
  \item May 3rd, 2023: fixed a typo of index $k$ (should be $0$, but was $1$) in Equation~\ref{eq:pomdp:pull_gamma} in Appendix \ref{appdx:pomdp:derivation} on the derivation of POMDP Bellman equation.
  \end{itemize}
\end{center}
\restoregeometry

\end{document}